\documentclass[pdflatex,sn-mathphys-num]{sn-jnl}

\usepackage{graphicx}
\usepackage{multirow}
\usepackage{amsmath,amssymb,amsfonts}
\usepackage{amsthm}%
\usepackage{mathrsfs}%
\usepackage[title]{appendix}%
\usepackage{xcolor}%
\usepackage{textcomp}%
\usepackage{manyfoot}%
\usepackage{booktabs}%
\usepackage{listings}%

\usepackage{algorithm, algorithmic, subfigure}
\usepackage{colortbl}
\usepackage{color}
\usepackage{framed}

\usepackage{pifont,bbding,fontawesome}

\definecolor{lightgray}{gray}{0.8}

\allowdisplaybreaks[4]

\theoremstyle{thmstyleone}%
\newtheorem{theorem}{Theorem}
\newtheorem{proposition}[theorem]{Proposition}%
\newtheorem{assumption}[theorem]{Assumption}%
\newtheorem{lemma}[theorem]{Lemma}%

\theoremstyle{thmstyletwo}%

\theoremstyle{thmstylethree}%
\newtheorem{definition}{Definition}%

\begin{document}

\title[Article Title]{Understanding the Generalization of Bilevel Programming in Hyperparameter Optimization: \\A Tale of Bias-Variance Decomposition}

\author[1]{\fnm{Yubo} \sur{Zhou}}\email{ybzhouni@gmail.com}

\author*[1,2]{\fnm{Jun} \sur{Shu}}\email{junshu@xjtu.edu.cn}

\author[1]{\fnm{Junmin} \sur{Liu}}\email{junminliu@mail.xjtu.edu.cn}

\author[1,2]{\fnm{Deyu} \sur{Meng}}\email{dymeng@mail.xjtu.edu.cn}

\affil[1]{\orgdiv{School of Mathematics and Statistics and Ministry of  Education Key Lab of  Intelligent Networks and Network Security}, \orgname{Xi'an Jiaotong University}, \orgaddress{\city{Xi'an}, \postcode{100190}, \state{Shaan'xi Province}, \country{P. R. China}}}

\affil[2]{\orgname{Pazhou Lab (Huangpu)}, \orgaddress{\street{Street}, \city{Guangzhou}, \state{Guangdong Province}, \country{P. R. China}}}

\abstract{Gradient-based hyperparameter optimization (HPO) have emerged recently, leveraging bilevel programming techniques to optimize hyperparameter by estimating hypergradient w.r.t. validation loss. Nevertheless, previous theoretical works mainly focus on reducing the gap between the estimation and ground-truth (i.e., the bias), while ignoring the error due to data distribution (i.e., the variance), which degrades performance. To address this issue, we conduct a bias-variance decomposition for hypergradient estimation error and provide a supplemental detailed analysis of the variance term ignored by previous works. We also present a comprehensive analysis of the error bounds for hypergradient estimation. This facilitates an easy explanation of some phenomena commonly observed in practice, like overfitting to the validation set. Inspired by the derived theories, we propose an ensemble hypergradient strategy to reduce the variance in HPO algorithms effectively. Experimental results on tasks including regularization hyperparameter learning, data hyper-cleaning, and few-shot learning demonstrate that our variance reduction strategy improves hypergradient estimation. To explain the improved performance, we establish a connection between excess error and hypergradient estimation, offering some understanding of empirical observations.}


\maketitle

\addtocontents{toc}{\protect\setcounter{tocdepth}{-1}}
\section{Introduction}\label{section1}
\vspace{2mm}
Machine learning has shown effectiveness in fields like image classification \citep{he2016deep}, NLP \citep{devlin2018bert}, and speech recognition \citep{oord2016wavenet}. Deep neural networks, with their complex architectures, have many hyperparameter, making them prone to overfitting, where hyperparameter choice greatly impacts model performance. Thus, finding optimal hyperparameter is crucial for achieving good results.

Traditional trial-and-error methods for hyperparameter tuning are time-consuming, biased, and error-prone. To improve this, numerous hyperparameter optimization (HPO) techniques have been developed \citep{king1995statlog, kohavi1995automatic}, enhancing efficiency and reproducibility. Early methods like grid search and random search \citep{bergstra2012random} were followed by advanced ones like Bayesian optimization \citep{snoek2012practical}. More recently, gradient-based HPO techniques have been introduced, capable of optimizing high-dimensional hyperparameter using automatic differentiation \citep{lorraine2020optimizing}.

Formally, gradient-based HPO could be formulated as the following bilevel optimization \citep{colson2007overview} framework:
\begin{align}
\label{equation1-1}
\min_{\boldsymbol{\lambda}\in\mathbb{R}^p}f(\boldsymbol{\lambda}):=F(\boldsymbol{\lambda},\boldsymbol{\theta}^*(\boldsymbol{\lambda})),\quad\text{s.t.}\quad \boldsymbol{\theta}^*(\boldsymbol{\lambda})=\operatorname*{argmin}_{\boldsymbol{\theta}\in\mathbb{R}^r} G(\boldsymbol{\lambda},\boldsymbol{\theta}),
\end{align}
where $F$ and $G$ be the outer and inner objective functions, respectively. The hyperparameter $\boldsymbol{\lambda}$ is obtained by minimizing $F(\boldsymbol{\lambda}, \boldsymbol{\theta}^*(\boldsymbol{\lambda}))$, where $\boldsymbol{\theta}^*(\boldsymbol{\lambda})$ is the optimal solution of $G(\boldsymbol{\lambda}, \boldsymbol{\theta})$. Bilevel optimization methods \citep{franceschi2017forward,franceschi2018bilevel,grazzi2020iteration} typically involve inner and outer loop optimizations. In the inner loop, gradient descent is used to approximate the minimum of $G(\boldsymbol{\lambda}, \boldsymbol{\theta})$ for a given $\boldsymbol{\lambda}$, while the outer loop optimizes $\boldsymbol{\lambda}$ by estimating the hypergradient $\nabla f(\boldsymbol{\lambda})$. Solving for the hypergradient often requires an exact solution to the inner problem, which can be computationally expensive, especially in large-scale settings. Therefore, the inner problem is typically solved with $K$ gradient descent steps, using the intermediate $\boldsymbol{\theta}_{K}(\boldsymbol{\lambda})$ to approximate $\boldsymbol{\theta}^*(\boldsymbol{\lambda})$.

To estimate the hypergradient $\nabla f(\boldsymbol{\lambda})$, two main approaches exist: approximate implicit differentiation (AID) and iterative differentiation (ITD) \citep{franceschi2017forward,grazzi2020iteration}. AID applies the implicit function theorem, using methods like conjugate gradient \citep{pedregosa2016hyperparameter} and Neumann series \citep{lorraine2020optimizing}. ITD stores the iterative trajectory of the inner problem and computes it via automatic differentiation. While the methods differ, both ensure convergence to the exact hypergradient. \cite{grazzi2020iteration} analyze the convergence of both approaches under contraction conditions. \cite{liu2020generic} extend the analysis to non-singleton inner problems, and \cite{liu2021towards} provide convergence results for non-convex cases.

Gradient-based HPO algorithms have achieved promising results for a series of applications \citep{liu2021investigating}, however,  the related existing convergence guarantees still need to be further ameliorated. E.g., some commonly encountered practical phenomena, like overfitting to validation set \citep{franceschi2018bilevel, bao2021stability} still cannot be soundly explained by these theories. This is mainly since current theoretical results of hypergradient estimation are rooted in analyzing the difference between estimated and ground-truth hypergradient, while they have not emphasized estimation error closely related to the data distribution, which yet should be significant for exploring the insightful rationality of the intrinsic HPO mechanism. Gradient-based HPO in Eq. (\ref{equation1-1}) can be rewritten as the following objective by considering data distribution:
\begin{align}\label{equation1-2}
\min_{\boldsymbol{\lambda}\in\mathbb{R}^p}f(\boldsymbol{\lambda})=\min_{\boldsymbol{\lambda}\in\mathbb{R}^p}\mathbb{E}_{\mathcal{D}^{tr},\mathcal{D}^{val}\sim \mathscr{P}}[f(\boldsymbol{\lambda};(\mathcal{D}^{tr},\mathcal{D}^{val}))],
\end{align}
where
\begin{align*}
		f(\boldsymbol{\lambda};(\mathcal{D}^{tr},\mathcal{D}^{val}))=F(\boldsymbol{\lambda},\boldsymbol{\theta}^*(\boldsymbol{\lambda};\mathcal{D}^{tr});\mathcal{D}^{val}),\quad\text{s.t.}\quad \boldsymbol{\theta}^*(\boldsymbol{\lambda};\mathcal{D}^{tr})=\operatorname*{argmin}_{\boldsymbol{\theta}\in\mathbb{R}^r} G(\boldsymbol{\lambda},\boldsymbol{\theta};\mathcal{D}^{tr}),
\end{align*}
and $\mathcal{D}^{tr}$, $\mathcal{D}^{val}$ denote the training and validation sets for inner-level and outer-level optimization, respectively. 
We can see that current approaches formulate the outer-level objective as the minimization of validation loss (i.e., loss computed on validation data $\mathcal{D}^{val}$), and the inner-level objective as the minimization of training loss (i.e., loss computed on training data $\mathcal{D}^{tr}$), where training data $\mathcal{D}^{tr}$ and validation data $\mathcal{D}^{val}$ are both sampled from the data distribution $\mathscr{P}$. 
Therefore, the outer-level objective $f(\boldsymbol{\lambda})$ is functional to minimize the expected validation loss calculated over training and validation data. Most existing gradient-based HPO algorithms only set single fixed training and validation data protocol to solve 
Eq. (\ref{equation1-2}), which can hardly guarantee an accurate approximation of the expected hypergradient across different sampling data protocols. Such inaccuracy in hypergradient estimation inclines to accumulate throughout the iterative process and ultimately tends to result in various impacts on generalization performance.
This process that examines how data distribution influences hypergradient estimation could help provide complementary analyses for previous 
theoretical results of gradient-based HPO algorithms.

To illustrate this, we analyze the error of hypergradient estimation in terms of the HPO objective in Eq. (\ref{equation1-2}) via bias-variance decomposition techniques as follows:
\begin{align}\label{equation1-4}
&\underbrace{\mathbb{E}_{\mathcal{D}^{tr},\mathcal{D}^{val} \sim \mathscr{P}}\{\Vert \widehat{\nabla}{ f}(\boldsymbol{\lambda};(\mathcal{D}^{tr},\mathcal{D}^{val}))-\overline{\nabla} {f}(\boldsymbol{\lambda})\Vert^2\}}_{\text{\textbf{Error}}}=\notag\\
& \underbrace{\mathbb{E}_{\mathcal{D}^{tr},\mathcal{D}^{val}}\{\Vert\hat{\nabla} f(\boldsymbol{\lambda};(\mathcal{D}^{tr},\mathcal{D}^{val}))-\widetilde{\nabla} { f}(\boldsymbol{\lambda})\Vert^2\}}_{\text{\textcolor{red}{\textbf{Variance}}}}		
+\underbrace{\Vert\widetilde{\nabla} { f}(\boldsymbol{\lambda}) - \overline{\nabla} f(\boldsymbol{\lambda})\Vert^2}_{\text{\textcolor{blue}{\textbf{Bias$^2$}}}},
\end{align}
where $\widehat{\nabla} f(\boldsymbol{\lambda};(\mathcal{D}^{tr},\mathcal{D}^{val}))$ denotes the computed hypergradient by existing HPO algorithms on given training and validation sets, e.g., $\widehat{\nabla}{ f}(\boldsymbol{\lambda};(\mathcal{D}^{tr},\mathcal{D}^{val}))=\nabla_{\boldsymbol{\lambda}} F(\boldsymbol{\lambda},\boldsymbol{\theta}_{K}(\boldsymbol{\lambda};\mathcal{D}^{tr});\mathcal{D}^{val})$, $\widetilde{\nabla} { f}(\boldsymbol{\lambda})=\mathbb{E}_{\mathcal{D}^{tr},\mathcal{D}^{val} \sim \mathscr{P}}\widehat{\nabla}{ f}(\boldsymbol{\lambda};(\mathcal{D}^{tr},\mathcal{D}^{val}))$ is the expected hypergradient estimation, and $\overline{\nabla} f(\lambda)$ denotes the underlying ground-truth hypergradient estimation. The above decomposition provides a fine depiction of the error in terms of hypergradient estimation. Specifically, the latter is the bias square term, which could be bounded by $\Vert\widetilde{\nabla} { f}(\boldsymbol{\lambda}) - \overline{\nabla} f(\boldsymbol{\lambda})\Vert^2$, representing the difference between the hypergradient estimated by gradient-based HPO algorithms and the underlying ground-truth one. Current theoretical convergence results \citep{grazzi2020iteration, ji2021bilevel, liu2020generic, liu2021towards} are closely related to this bias square term estimation. The former is the variance term, revealing the deviation between the empirical hypergradient estimation of a single fixed training validation protocol and the expected hypergradient estimation w.r.t. underlying the data distribution. 

\begin{table*}[t]
\caption{Comparisons of hypergradient estimation and excess error.}\label{table1-1}
\centering
\resizebox{0.85\textwidth}{!}{
\begin{tabular}{@{\extracolsep\fill}lcccc}
\toprule%
& \multicolumn{2}{@{}c@{}}{Hypergradient Estimation} & \multicolumn{2}{@{}c@{}}{Excess Error Analysis} \\\cmidrule{2-5}
& Bias & Variance & Generalization Error & Training Error\\
\midrule
Grazzi et al. \citep{grazzi2020iteration} & \ding{51} & \ding{55} & \ding{55} & \ding{55}\\
\midrule
Bao et al. \citep{bao2021stability} & \ding{55} & \ding{55} & \ding{51} & \ding{55} \\
\midrule
Ours & \ding{51} & \ding{51} & \ding{51} & \ding{51} \\
\botrule
\end{tabular}}
\end{table*}

It actually can be seen that the variance term is also crucial to analyze the error of the estimated hypergradient in Eq. (\ref{equation1-4}), which yet still has limited research before. In this paper, we will specifically focus on the theoretical analysis of the variance term related to data distribution, and thus provide a supplemental analysis of hypergradient estimated by existing HPO algorithms, e.g., \cite{grazzi2020iteration}. In comparison to the limited work on generalization analysis proposed by \cite{bao2021stability}, we further conduct a more comprehensive analysis on excess error by establishing connections with existing training error analysis based on hypergradient estimation error, as illustrated in Table \ref{table1-1}.

In summary, the main contributions of this work can be presented as follows:

(1) We conduct a bias-variance decomposition of hypergradient estimation error for gradient-based HPO algorithms (i.e., Eq. (\ref{equation1-4})). Such decomposition analysis shows that while most existing theoretical results focus on the error analysis of the bias square term, we can provide a supplemental detailed analysis of the variance term ignored before. The novel theoretical result provides a sound rationality explanation for more commonly observed phenomena in HPO practice, such as the widely encountered overfitting issue to the validation set \citep{franceschi2018bilevel} demonstrated in Fig. \ref{Figure721-2}, which can yet be hardly well explained by previous theoretical analysis of the bias term.

(2) We provide comprehensive error bounds of hypergradient estimation for AID and ITD strategies, revealing some factors that influence the variance term. We highlight the utility of our analysis framework for obtaining a bias-variance decomposition of hypergradient estimation error on a one-dimensional ridge regression problem. Besides, the simulating results in Section \ref{sec5} are also well-aligned with the revealed theoretical insights. 

(3) Inspired by the conducted theoretical bounds, we propose an ensemble average strategy borrowed from a typical cross-validation process aiming to more effctively reduce the variance of existing HPO algorithms. 
To reduce the computation cost, an online ensemble hypergradient estimation strategy is developed to improve the hypergradient estimation for HPO problem. We experimentally demonstrate that the proposed variance reduction strategy evidently helps improve hypergradient estimation across multiple HPO problems, including regularization parameter learning, data hyper-cleaning and few-shot learning.

(4) We establish a connection between the excess error analysis of HPO algorithms and our proposed error bounds of hypergradient estimation. Besides, we decompose excess error into generalization error and training error, providing error bounds for them by employing hypergradient estimation and uniform stability, respectively. Based on the derived bounds, we make a fine analysis on the effects of various influencing factors on the excess error and offer some insight to ease understanding of existing HPO algorithms. Experimental results presented in Section \ref{section7} validate our theory findings.

The remainder of the paper is organized as follows. Section \ref{section2} reviews preliminaries of gradient-based HPO. In Section \ref{section3}, we perform a bias–variance decomposition of hypergradient estimation error and derive comprehensive error bounds. Section \ref{section4} introduces a variance-reduction strategy motivated by these bounds. We instantiate our framework on one-dimensional ridge regression in Section \ref{sec5}. Section \ref{section7} presents experiments on regularization parameter learning, data hyper-cleaning and few-shot learning to demonstrate the effectiveness of our approach. Section \ref{section8} discusses the connection between excess error and our derived hypergradient estimation error. In Section \ref{section6}, we survey related work. Finally, we conclude and outline future directions.

\section{Preliminary}\label{section2}
\subsection{Hyperparamter Optimization}\label{section21} 
Let $\mathscr{D}$, $\boldsymbol{\Theta}$ and $\boldsymbol{\Lambda}$ represent data, parameter and hyperparameter spaces, respectively. Let $\mathcal{A}_{\text{ml}}: \mathscr{D} \rightarrow \boldsymbol{\Theta}$ denote a machine learning algorithm with hyperparameter $\boldsymbol{\lambda}\in\boldsymbol{\Lambda}$, and we have $\boldsymbol{\theta}=\mathcal{A}_{\text{ml}}(\mathcal{D}^{tr};\boldsymbol{\lambda})$ and $ \mathcal{D}^{tr}\in\mathscr{D}$ and $\boldsymbol{\theta}\in\boldsymbol{\Theta}$ are training data and model parameter, respectively. Hyperparameter choices can greatly influence model performance, so finding configurations that ensure robust generalization has become a major focus of recent machine learning research. Hyperparameter optimization \citep{hutter2019automated,shu2021learning} is a commonly used strategy to determine the hyperparameter $\boldsymbol{\lambda}$. Specifically, given a data distribution $\mathscr{P}$ on the data space $\mathscr{D}$, we can solve the following objective:
\begin{align*}
	\boldsymbol{\lambda}^*=\arg\min_{\boldsymbol{\lambda}\in\boldsymbol{\Lambda}}\mathcal{R}(\boldsymbol{\lambda}, \mathcal{A}_{\text{ml}}) =\arg\min_{\boldsymbol{\lambda}\in\boldsymbol{\Lambda}}\mathbb{E}_{\mathcal{D}^{tr},\mathcal{D}^{val}\sim\mathscr{P}}\left[\mathcal{L}(\mathcal{A}_{\text{ml}}(\mathcal{D}^{tr};\boldsymbol{\lambda}),\mathcal{D}^{val})\right],
\end{align*}
where $\mathcal{R}$ measures the expected loss $\mathcal{L}$ of a model generated by algorithm $\mathcal{A}_{\text{ml}}$ with hyperparameter $\boldsymbol{\lambda}$ on training data $\mathcal{D}^{tr}$ and evaluated on validation data $\mathcal{D}^{val}$, and $\mathcal{D} = \mathcal{D}^{tr} \cup \mathcal{D}^{val}$.
 
For a specific machine learning problem, we only have access to a finite $N$ examples $\mathcal{D}\sim\mathscr{P}$, and we randomly split $\mathcal{D}$ into two subsets $\mathcal{D}^{tr}_{u_i}$ and $\mathcal{D}^{val}_{u_i}$, i.e., 
\begin{align}\label{equation21-4}
	\{(\mathcal{D}^{tr}_{u_i},\mathcal{D}^{val}_{u_i})\}_{i=1}^U= \mathcal{S}_{(\mathcal{D}, \{u_i\}_{i=1}^U)}, \mathcal{D} = \mathcal{D}^{tr}_{u_i} \cup \mathcal{D}^{val}_{u_i}, \mathcal{D}^{tr}_{u_i} \cap \mathcal{D}^{val}_{u_i} = \emptyset,
\end{align}
where $\mathcal{S}$ denotes a \textit{data splitting} process, $\{u_i\}_{i=1}^U\in\mathbb{N}_+$ denotes a series of random seeds constituted by employing random sampling method, and each $u_i$ can conduct a specific training and validation decomposition from the entire dataset.\footnote{For simplicity, we consider random sampling method in this work. Generally, $\mathcal{D}^{val}_{u_i}$ are firstly sampled from $\mathcal{D}$, and the rest of $\mathcal{D}$ constitutes $\mathcal{D}^{tr}_{u_i}$. In the main paper, we consider different random number seeds $\{u_i\}_{i=1}^U$ corresponding to the different splittings. For the case of different random number seeds $\{u_i\}_{i=1}^U$ correspond to same data splittings, we give the analysis in Appendix \ref{sectionA10-v3}.}
Generally, we require the splitting ratio $\gamma = |\mathcal{D}^{val}_{u_i}|/|\mathcal{D}^{tr}_{u_i}|$ to be located in $(0,1)$.
Now, the HPO process could be expressed by
\begin{align}\label{equation21-3}
\hat{\boldsymbol{\lambda}}^*=\arg\min_{\boldsymbol{\lambda}\in\boldsymbol{\Lambda}}\hat{\mathcal{R}}(\boldsymbol{\lambda}, \mathcal{A}_{\text{ml}};\mathcal{S}_{(\mathcal{D}, \{u_i\}_{i=1}^U)})= 	
\arg\min_{\boldsymbol{\lambda}\in\boldsymbol{\Lambda}}\frac{1}{U}\sum_{i=1}^{U} \left[\mathcal{L}(\mathcal{A}_{\text{ml}}(\mathcal{D}^{tr}_{u_i};\boldsymbol{\lambda}),\mathcal{D}^{val}_{u_i})\right].
\end{align}
Eq. (\ref{equation21-3}) encompasses various validation protocols, such as the commonly used $k$-fold cross-validation method. When we only produce a single data splitting (i.e., $U=1$), it generally degenerates to typical hold-out method in practice \footnote{For simplicity, we rewrite $(\mathcal{D}^{tr}_{u_1}, \mathcal{D}^{val}_{u_1})$ as $(\mathcal{D}^{tr}, \mathcal{D}^{val})$ if $U=1$.}
\begin{align}\label{equation21-5}
\hat{\boldsymbol{\lambda}}^*=\arg\min_{\boldsymbol{\lambda}\in\boldsymbol{\Lambda}}\hat{\mathcal{R}}(\boldsymbol{\lambda}, \mathcal{A}_{\text{ml}};\mathcal{S}_{(\mathcal{D}, u_1)})= 	
\arg\min_{\boldsymbol{\lambda}\in\boldsymbol{\Lambda}} \left[\mathcal{L}(\mathcal{A}_{\text{ml}}(\mathcal{D}^{tr};\boldsymbol{\lambda}),\mathcal{D}^{val})\right],
\end{align}
where $\mathcal{D} = \mathcal{D}^{tr} \cup \mathcal{D}^{val}$.

To solve the above HPO problem, extensive methods \citep{hutter2015beyond} have been proposed. Early attempts mainly pay attention to gradient-free HPO \citep{snoek2012practical,bergstra2012random}, in which hyperparameter are chosen/searched to optimize the validation loss after completing training of the model parameter. These methods achieve promising performance on some tasks, while in general they can be hardly utilized to handle more practical optimization problems with more than 20 hyperparameter in a satisfactorily efficient and accurate manner.
\subsection{Gradient-based HPO}\label{section22}
Recently, gradient-based HPO methods \citep{maclaurin2015gradient, pedregosa2016hyperparameter, franceschi2017forward,shu2019meta} use gradients to allow optimization of validation loss w.r.t. thousands of hyperparameter, achieving excellent performance on various complicated HPO problems. Specifically, they rewrite Eq. (\ref{equation21-5}) as the following bilevel optimization problem \citep{colson2007overview} to compute the hypergradient of validation loss w.r.t. hyperparameter:
\begin{align} 
		{\boldsymbol{\lambda}}^* = \arg\min_{\boldsymbol{\lambda}\in\boldsymbol{\Lambda}}\hat{\mathcal{R}}^{val}(\boldsymbol{\lambda},\boldsymbol{\theta}^*(\boldsymbol{\lambda}; \mathcal{D}^{tr});\mathcal{D}^{val}),\text{ s.t. }\boldsymbol{\theta}^*(\boldsymbol{\lambda}; \mathcal{D}^{tr})=\arg\min_{\boldsymbol{\theta}\in\boldsymbol{\Theta}}\hat{\mathcal{R}}^{tr}(\boldsymbol{\lambda},\boldsymbol{\theta};\mathcal{D}^{tr}),  \label{equation22-1-b}
\end{align}
where $\hat{\mathcal{R}}^{tr}(\boldsymbol{\cdot},\boldsymbol{\cdot};\mathcal{D}^{tr})$ and $\hat{\mathcal{R}}^{val}(\boldsymbol{\cdot},\boldsymbol{\cdot};\mathcal{D}^{val})$ denotes the empirical risks on $\mathcal{D}^{tr}$ and $\mathcal{D}^{val}$, respectively, and $\boldsymbol{\theta}^*(\boldsymbol{\lambda}; \mathcal{D}^{tr})$ is achieved by the lower-level optimization process in Eq. (\ref{equation22-1-b}). Typically, there exist two kinds of methodologies to solve the above bilevel optimization problem:
iterative differentiation (ITD) \citep{maclaurin2015gradient, franceschi2017forward} and approximate implicit differentiation (AID) \citep{pedregosa2016hyperparameter, rajeswaran2019meta, lorraine2020optimizing}. Both approaches optimize the optimal model parameter of the inner-level problem by performing a multi-step gradient descent strategy, while they optimize hyperparameter using different hypergradient computations.

\begin{algorithm}[t]
	\caption{Iterative Differentiation (ITD) For HPO}
	\renewcommand{\algorithmicrequire}{\textbf{Input:}}
	\renewcommand{\algorithmicensure}{\textbf{Output:}}
	\begin{algorithmic}[1]
		\REQUIRE Number of outer-level iteration steps $T$; number of inner-level iteration steps $K$; initialization ${\boldsymbol{\theta}}_0$ and ${\boldsymbol{\lambda}}_0$; learning rate scheme $\alpha_{in}$ and $\alpha_{out}$.
		\ENSURE Parameter ${\boldsymbol{\theta}}_{\text{ITD}}$ and hyperparameter ${\boldsymbol{\lambda}}_{\text{ITD}}$.
		\FOR{$t=0$ {\bfseries to} $T-1$}
		\STATE ${\boldsymbol{\theta}}_0^{(t)}\leftarrow{\boldsymbol{\theta}}_0$, ${\boldsymbol{\lambda}}^{(0)}\leftarrow{\boldsymbol{\lambda}}_0$.
		\FOR{$k=0$ {\bfseries to} $K-1$}
		\STATE $\boldsymbol{\theta}_{k+1}^{(t)}\leftarrow\boldsymbol{\theta}_{k}^{(t)}-\alpha_{in}\nabla_{\boldsymbol{\theta}}\hat{\mathcal{R}}^{tr}(\boldsymbol{\lambda},\boldsymbol{\theta};\mathcal{D}^{tr})\Big{|}_{\boldsymbol{\theta}=\boldsymbol{\theta}_k^{(t)}}$.
		\ENDFOR
		\STATE Set ${f}(\boldsymbol{\lambda};(\mathcal{D}^{tr},\mathcal{D}^{val}))=\hat{\mathcal{R}}^{val}(\boldsymbol{\lambda},\boldsymbol{\theta}_K(\mathcal{D}^{tr});\mathcal{D}^{val})$ and compute $\widehat{\nabla}{f}(\boldsymbol{\lambda};(\mathcal{D}^{tr},\mathcal{D}^{val}))$ using automatic differentiation.
		\STATE $\boldsymbol{\lambda}^{(t+1)}\leftarrow\boldsymbol{\lambda}^{(t)}-\alpha_{out}\widehat{\nabla}{f}(\boldsymbol{\lambda};(\mathcal{D}^{tr},\mathcal{D}^{val}))\Big{|}_{\boldsymbol{\lambda}=\boldsymbol{\lambda}^{(t)}}$.
		\ENDFOR
		\STATE \textbf{return} $\boldsymbol{\lambda}^{(T)}$ and $\boldsymbol{\theta}_K^{(T)}$
	\end{algorithmic}
	\label{algorithm21-1}
 
\end{algorithm}

Specifically, ITD computes hypergradient via backpropagation. Specifically, given a hyperparameter configuration, ITD first updates the parameter $\boldsymbol{\theta}$ by executing $K$ steps gradient descend at the inner-level optimization to approximate the solution of Eq. (\ref{equation22-1-b}). The whole computation graph contains $K$ parameter updating functions, which are differentiable w.r.t. $\boldsymbol{\lambda}$. As a result, we can compute the hypergradient of validation loss w.r.t. $\boldsymbol{\lambda}$ in Eq. (\ref{equation22-1-b}) by directly backpropagating along the computation graph. Based on such obtained hypergradient, we could further optimize hyperparameter at the outer-level. Algorithm \ref{algorithm21-1} shows the overall computation process of ITD, and we omit the dependency of $\boldsymbol{\lambda}$ and $\boldsymbol{\theta}$ on $\mathcal{D}^{tr}$ and $\mathcal{D}^{val}$ for simplicity.
The explicit form of the estimated hypergradient of validation loss w.r.t. hyperparameter $\boldsymbol{\lambda}$ is given by the following proposition. Please see more proof details in \citep{ji2021bilevel}.
\begin{proposition}\label{sec22-propos1}
$\widehat{\nabla}{f}(\boldsymbol{\lambda})$ takes the analytical form of $\widehat{\nabla}{f}(\boldsymbol{\lambda})=\nabla_{\boldsymbol{\lambda}}\hat{\mathcal{R}}^{val}(\boldsymbol{\lambda},\boldsymbol{\theta}_K)-\alpha_{in}\sum_{k=0}^{K-1}\nabla^2_{\boldsymbol{\lambda}\boldsymbol{\theta}}\hat{\mathcal{R}}^{tr}(\boldsymbol{\lambda},\boldsymbol{\theta}_k)\prod_{j=k+1}^{K-1}(I-\alpha_{in}\nabla_{\boldsymbol{\theta}}^2\hat{\mathcal{R}}^{tr}(\boldsymbol{\lambda},\boldsymbol{\theta}_j))\nabla_{\boldsymbol{\theta}}\hat{\mathcal{R}}^{val}(\boldsymbol{\lambda},\boldsymbol{\theta}_K))$.
\end{proposition}Proposition \ref{sec22-propos1} indicates that the differentiation of $\widehat{\nabla}{f}(\boldsymbol{\lambda})$ involves the computation of second-order derivatives and requires the storage of long trajectories from inner-level iterations. 

\begin{algorithm}[t]
	\caption{Approximate Implicit Differentiation (AID) For HPO}
	\renewcommand{\algorithmicrequire}{\textbf{Input:}}
	\renewcommand{\algorithmicensure}{\textbf{Output:}}
	\begin{algorithmic}[1]
		\REQUIRE Number of outer-level iteration steps $T$; number of inner-level iteration steps $K$; number of iteration steps $Z$; initialization ${\boldsymbol{\theta}}_0$ and ${\boldsymbol{\lambda}}_0$; learning rate scheme $\alpha_{in}$ and $\alpha_{out}$.
		\ENSURE Parameter ${\boldsymbol{\theta}}_{\text{AID}}$ and hyperparameter ${\boldsymbol{\lambda}}_{\text{AID}}$.
		\FOR{$t=0$ {\bfseries to} $T-1$}
		\STATE ${\boldsymbol{\theta}}_0^{(t)}\leftarrow{\boldsymbol{\theta}}_0$, ${\boldsymbol{\lambda}}^{(0)}\leftarrow{\boldsymbol{\lambda}}_0$.
		\FOR{$k=0$ {\bfseries to} $K-1$}
		\STATE $\boldsymbol{\theta}_{k+1}^{(t)}\leftarrow\boldsymbol{\theta}_{k}^{(t)}-\alpha_{in}\nabla_{\boldsymbol{\theta}}\hat{\mathcal{R}}^{tr}(\boldsymbol{\lambda},\boldsymbol{\theta};\mathcal{D}^{tr})\Big{|}_{\boldsymbol{\theta}=\boldsymbol{\theta}_k^{(t)}}$.
		\ENDFOR
		\STATE Compute $\boldsymbol{v}_{Z}^{(t)}$ after $Z$ steps of a solver for the system
		\begin{align}\label{equation22-A1}
			\nabla_{\boldsymbol{\theta}}^2\hat{\mathcal{R}}^{tr}(\boldsymbol{\lambda}, \boldsymbol{\theta};\mathcal{D}^{tr})\Big{|}_{\boldsymbol{\theta}=\boldsymbol{\theta}_{K}^{(t)}}\boldsymbol{v}=\nabla_{\boldsymbol{\theta}}\hat{\mathcal{R}}^{val}(\boldsymbol{\lambda}, \boldsymbol{\theta};\mathcal{D}^{val})\Big{|}_{\boldsymbol{\theta}=\boldsymbol{\theta}_{K}^{(t)}}.
		\end{align}
		\STATE Compute the approximate gradient as
		\begin{align*}
				\widehat{\nabla}{f}(\boldsymbol{\lambda})=\nabla_{\boldsymbol{\lambda}}\hat{\mathcal{R}}^{val}(\boldsymbol{\lambda}, \boldsymbol{\theta};\mathcal{D}^{val}))\Big{|}_{\boldsymbol{\lambda}=\boldsymbol{\lambda}^{(t)}}-\nabla^2_{\boldsymbol{\lambda}, \boldsymbol{\theta}}\hat{\mathcal{R}}^{tr}(\boldsymbol{\lambda}, \boldsymbol{\theta};\mathcal{D}^{tr})\Big{|}_{\boldsymbol{\lambda}=\boldsymbol{\lambda}^{(t)}, \boldsymbol{\theta}=\boldsymbol{\theta}_{K}^{(t)}}\boldsymbol{v}_{Z}^{(t)}.
		\end{align*}
		\STATE $\boldsymbol{\lambda}^{(t+1)}\leftarrow\boldsymbol{\lambda}^{(t)}-\alpha_{out}\widehat{\nabla} {f}(\boldsymbol{\lambda})\Big{|}_{\boldsymbol{\lambda}=\boldsymbol{\lambda}^{(t)}}$.
		\ENDFOR
		\STATE \textbf{return} $\boldsymbol{\lambda}^{(T)}$ and $\boldsymbol{\theta}_K^{(T)}$
	\end{algorithmic}
	\label{algorithm21-2}
\end{algorithm}

As a comparison, AID computes the hypergradient by solving a linear system derived from the implicit function theorem. The overall algorithm is shown in Algorithm \ref{algorithm21-2}. Specifically, we can compute the hypergradient via the chain rule as follows:
\begin{align} \label{eq212-1}
	\widehat{\nabla}{f}(\boldsymbol{\lambda})=\nabla_{\boldsymbol{\lambda}}\hat{\mathcal{R}}^{val}(\boldsymbol{\lambda},\boldsymbol{\theta}_K(\boldsymbol{\lambda};\mathcal{D}^{tr});\mathcal{D}^{val})+\nabla_{\boldsymbol{\lambda}}\boldsymbol{\theta}_K(\boldsymbol{\lambda};\mathcal{D}^{tr})\nabla_{\boldsymbol{\theta}}\hat{\mathcal{R}}^{val}(\boldsymbol{\lambda},\boldsymbol{\theta}_K(\boldsymbol{\lambda};\mathcal{D}^{tr});\mathcal{D}^{val}).
\end{align}
Then we can directly obtain the implicit gradient by implicit function theorem:
\begin{align*}
	\nabla_{\boldsymbol{\lambda}}\boldsymbol{\theta}_K(\boldsymbol{\lambda};\mathcal{D}^{tr})=-\Big{(}\nabla_{\boldsymbol{\theta}}^2\hat{\mathcal{R}}^{tr}(\boldsymbol{\lambda},\boldsymbol{\theta}_K(\boldsymbol{\lambda});\mathcal{D}^{tr})\Big{)}^{-1}\nabla^2_{\boldsymbol{\lambda},\boldsymbol{\theta}}\hat{\mathcal{R}}^{tr}(\boldsymbol{\lambda},\boldsymbol{\theta}_K(\boldsymbol{\lambda});\mathcal{D}^{tr}).
\end{align*}
Then, Eq.\eqref{eq212-1} can take the form of $\widehat{\nabla}{f}(\boldsymbol{\lambda})=\nabla_{\boldsymbol{\lambda}}\hat{\mathcal{R}}^{val}(\boldsymbol{\lambda},\boldsymbol{\theta}_K(\boldsymbol{\lambda});\mathcal{D}^{val})-{(}\nabla_{\boldsymbol{\theta}}^2\hat{\mathcal{R}}^{tr}(\boldsymbol{\lambda},\boldsymbol{\theta}_K(\boldsymbol{\lambda});\mathcal{D}^{tr}){)}^{-1}\nabla_{\boldsymbol{\theta}\boldsymbol{\lambda}}^2\hat{\mathcal{R}}^{tr}(\boldsymbol{\lambda},\boldsymbol{\theta}_K(\boldsymbol{\lambda});\mathcal{D}^{tr})\nabla_{\boldsymbol{\theta}}\hat{\mathcal{R}}^{val}(\boldsymbol{\lambda},\boldsymbol{\theta}_K(\boldsymbol{\lambda});\mathcal{D}^{val})$. For the second term, it can be gained by solving the linear system Eq. (\ref{equation22-A1}) of Algorithm \ref{algorithm21-2}. It is indicated that AID computes hypergradient without the need to store inner iteration trajectories \citep{liu2021investigating}, unlike ITD. However, it typically requires a larger number of iterations to accurately solve linear system in Eq. (\ref{equation22-A1}).

 \section{Rethinking Hypergradient Estimation From a Bias-Variance Decomposition Perspective}\label{section3}
Current HPO theories generally derive convergence bounds for hypergradient error on a fixed train–validation split, ignoring data variability. Accounting for this yields a bias–variance decomposition: existing results address the bias but omit the variance term. In this section, we analyze the variance arising from data distribution, thus supplementing prior hypergradient analyses. Main notations are listed in Table \ref{tableAZ-1}.
\begin{table*}[!t]
\caption{Summary of the main notations.}\label{tableAZ-1}
\centering
\resizebox{0.7\textwidth}{!}{
\begin{tabular}{ll}
\toprule
Notation& Definition\\
\midrule

$\mathscr{P}$ &Data distribution \\
$\boldsymbol{\lambda}$/$\boldsymbol{\theta}$& Hyperparameter/Model parameter\\
$\boldsymbol{\Lambda}$/$\boldsymbol{\Theta}$& Hyperparameter Space/Model parameter Space\\
$\xi$/$\zeta$& Single training sample/validation sample \\
$\mathcal{S}_{(\mathcal{D}, u_1)}$& Single data splitting \\
$\mathcal{S}_{(\mathcal{D}, \{u_i\}_{i=1}^U)}$& $U$ data splittings \\
$m^{tr}/m^{val}$&Data size of $\mathcal{D}^{tr}_{u_i}$/Data size of $\mathcal{D}^{val}_{u_i}$\\
$\mathcal{R}^{tr}/\hat{\mathcal{R}}^{tr}$& Expected loss/Empirical loss of the training data\\
$\mathcal{R}^{val}/\hat{\mathcal{R}}^{val}$& Expected loss/Empirical loss of the validation data\\
$K$/$T$ & The iteration step of inner/outer-level\\
$\Phi$&The updating function of inner-level\\
$\widehat{\nabla}{f}$&The hypergradient estimation by HPO algorithm\\
$\widetilde{\nabla}{f}$&The expectation of $\widehat{\nabla}{f}$\\
$\overline{\nabla}{f}$&Ground-truth hypergradient estimation\\
\bottomrule
\end{tabular}}

\end{table*}
\subsection{Error Analysis of Hypergradient Estimation}\label{section32}
Recently, \cite{grazzi2020iteration} provided non-asymptotic bounds on the hypergradient estimation error for both ITD and AID algorithms. Before introducing this result, we give some assumptions, which have been widely adopted in current works \citep{ghadimi2018approximation, ji2021bilevel}. We denote the Euclidean norm by $\Vert\mathbf{\cdot}\Vert$.
\begin{assumption}\label{assum33-1}
	The lower-level function $\hat{\mathcal{R}}^{tr}(\boldsymbol{\lambda},\boldsymbol{\theta})$ is $\mu$-strong-convex w.r.t. $\boldsymbol{\theta}$, i.e., for any $w, w'$, $\Vert\hat{\mathcal{R}}^{tr}(w)-\hat{\mathcal{R}}^{tr}(w')\Vert\geq \nabla\hat{\mathcal{R}}^{tr}(w')(w-w')+\frac{\mu}{2}\Vert w-w'\Vert^2$, and the outer-level function $\hat{\mathcal{R}}^{val}(\boldsymbol{\lambda},\boldsymbol{\theta})$ is non-convex w.r.t. $\boldsymbol{\theta}$. For the stochastic setting, the same assumptions hold for $\hat{\mathcal{R}}^{tr}(\boldsymbol{\lambda},\boldsymbol{\theta})$ and $\hat{\mathcal{R}}^{val}(\boldsymbol{\lambda},\boldsymbol{\theta})$, respectively.
	
\end{assumption} 

\begin{assumption}\label{assum33-2}
	Let $w=(\boldsymbol{\lambda},\boldsymbol{\theta})$ denote all parameter. The loss function $\hat{\mathcal{R}}^{val}(w)$ and $\hat{\mathcal{R}}^{tr}(w)$ satisfy: (1). The function $\hat{\mathcal{R}}^{val}(w)$ is M-Lipschitz, i.e., for any $w, w'$, $|\hat{\mathcal{R}}^{val}(w)-\hat{\mathcal{R}}^{val}(w')|\leq M\Vert w-w'\Vert$; (2). $\nabla\hat{\mathcal{R}}^{val}(w)$ and $\nabla\hat{\mathcal{R}}^{tr}(w)$ are L-Lipschitz, i.e., for any $w, w'$, $\Vert\nabla\hat{\mathcal{R}}^{val}(w)-\nabla\hat{\mathcal{R}}^{val}(w')\Vert\leq L\Vert w-w'\Vert,\text{ }\Vert\nabla\hat{\mathcal{R}}^{tr}(w)-\nabla\hat{\mathcal{R}}^{tr}(w')\Vert\leq L\Vert w-w'\Vert$. Considering the case of random sample $(\xi,\zeta)$ of the given data, the same assumptions hold for $\hat{\mathcal{R}}^{val}(w;\zeta)$ and $\hat{\mathcal{R}}^{tr}(w;\xi)$.
\end{assumption}
\begin{assumption}\label{assum33-3}
	For every $\boldsymbol{\lambda}\in\mathbb{R}^p$, (1). $\forall\boldsymbol{\theta}\in\mathbb{R}^r$, $\nabla^2_{\boldsymbol{\theta}}\hat{\mathcal{R}}^{tr}(\boldsymbol{\lambda}, \boldsymbol{\theta})$ is invertible; (2). $\Vert\boldsymbol{\theta}_K(\boldsymbol{\lambda})-\boldsymbol{\theta}^*(\boldsymbol{\lambda})\Vert\leq\rho_{\boldsymbol{\lambda}}(K)\Vert\boldsymbol{\theta}^*(\boldsymbol{\lambda})\Vert$, $\rho_{\boldsymbol{\lambda}}(K)\leq 1$, and $\rho_{\boldsymbol{\lambda}}(K)\to 0$ as $K\to +\infty$. (3). $\Vert \boldsymbol{v}_{K,Z}(\boldsymbol{\lambda})-\boldsymbol{v}_{K}(\boldsymbol{\lambda})\Vert\leq\sigma_{\boldsymbol{\lambda}}(Z)\Vert\boldsymbol{v}_{K}(\boldsymbol{\lambda})\Vert$ and $\sigma_{\boldsymbol{\lambda}}(Z)\to 0$ as $Z\to +\infty$, where $\rho_{\boldsymbol{\lambda}}(K)$ and $\sigma_{\boldsymbol{\lambda}}(Z)$ are the convergence rates of $\{\boldsymbol{\theta}_K(\boldsymbol{\lambda})\}_{K\in\mathbb{N}}$ and $\{\boldsymbol{v}_{K,Z}(\boldsymbol{\lambda})\}_{Z\in\mathbb{N}}$ respectively.
\end{assumption}
 In particular, $\boldsymbol{v}(\boldsymbol{\lambda})$ is the solution of the linear system $(I-\nabla_2\hat{\mathcal{R}}^{tr}(\boldsymbol{\lambda}, \boldsymbol{\theta}(\boldsymbol{\lambda}))^\top)\boldsymbol{v}=\nabla_2\hat{\mathcal{R}}^{val}(\boldsymbol{\lambda}, \boldsymbol{\theta}(\boldsymbol{\lambda}))$of AID. Then we give the main result in \cite{grazzi2020iteration} as the following.
\begin{theorem}\label{theorem32-1} \textbf{\textup{(ITD bound)}}
For the given data splitting $\mathcal{S}_{(\mathcal{D}, u_1)}=(\mathcal{D}^{tr}, \mathcal{D}^{val})$, suppose that Assumptions \ref{assum33-2}-\ref{assum33-3} hold and let $K\in\mathbb{N}$ with $K\geq 1$. For $\boldsymbol{\lambda}\in\mathbb{R}^p$, let $\boldsymbol{\theta}_K(\boldsymbol{\lambda})$ and ${f}(\boldsymbol{\lambda})$ be defined as the output of Algorithm \ref{algorithm21-1}. Then, ${f}(\boldsymbol{\lambda})$ is differentiable and
\begin{align}\label{eq31-1}
\Vert\widehat{\nabla}{f}(\boldsymbol{\lambda};\mathcal{S}_{(\mathcal{D}, u_1)})-\nabla {f}(\boldsymbol{\lambda};\mathcal{S}_{(\mathcal{D}, u_1)})\Vert
\leq\Big{(}2LC_{1,\boldsymbol{\lambda},\mathcal{D}^{tr}}(1+\frac{C_{2,\boldsymbol{\lambda},\mathcal{D}^{tr}}}{1-q})\frac{q+MK}{q}+\frac{MC_{2,\boldsymbol{\lambda},\mathcal{D}^{tr}}}{1-q}\Big{)}q^K,
\end{align}
where $C_{1,\boldsymbol{\lambda},\mathcal{D}^{tr}}$, $C_{2,\boldsymbol{\lambda},\mathcal{D}^{tr}}$ and $q\in(0, 1)$ are constants in Lemmas \ref{lemma33-1}-\ref{lemma-a2-5}.
\end{theorem}

\begin{theorem}\label{theorem32-2}\textbf{\textup{(AID bound)}}
	For the given data splitting $\mathcal{S}_{(\mathcal{D}, u_1)}=(\mathcal{D}^{tr}, \mathcal{D}^{val})$, suppose that Assumptions \ref{assum33-1}-\ref{assum33-3} hold. Let $\boldsymbol{\lambda}\in\mathbb{R}^p$, $K, Z\in\mathbb{N}$ and $\widehat{\nabla}{f}(\boldsymbol{\lambda})$ be defined as in Algorithm \ref{algorithm21-2}. Then,
		\begin{align}\label{eq31-2}
				\Vert\widehat{\nabla}{f}(&\boldsymbol{\lambda};\mathcal{S}_{(\mathcal{D}, u_1)})-\nabla {f}(\boldsymbol{\lambda};\mathcal{S}_{(\mathcal{D}, u_1)})\Vert
				\leq\notag\\
                &\rho_{\boldsymbol{\lambda}}(K)C_{1,\boldsymbol{\lambda},\mathcal{D}^{tr}}\Big{(}1+\frac{M+LC_{2,\boldsymbol{\lambda},\mathcal{D}^{tr}}}{\mu}+\frac{MLC_{2,\boldsymbol{\lambda},\mathcal{D}^{tr}}}{\mu^2}\Big{)}+\frac{M\sigma_{\boldsymbol{\lambda}}(Z)C_{2,\boldsymbol{\lambda},\mathcal{D}^{tr}}}{\mu},
		\end{align}
	where $C_{1,\boldsymbol{\lambda},\mathcal{D}^{tr}}$ and $C_{2,\boldsymbol{\lambda},\mathcal{D}^{tr}}$ are constants in Lemma \ref{lemma33-1}.
\end{theorem}
Given a fixed data splitting, Theorem \ref{theorem32-1} and Theorem \ref{theorem32-2} provide the bounds of hypergradient estimation error for ITD and AID, respectively. However, the underlying ground-truth hypergradient estimation is based on data distribution $\mathscr{P}$, rather than a specific training-validation protocol. To further illustrate this, we analyze the hypergradient estimation error via bias-variance decomposition techniques expressed as follows:
\begin{align}\label{equation32-1}
    &\underbrace{\mathbb{E}_{\mathcal{D},u_1}\{\Vert\widehat{\nabla}{f}(\boldsymbol{\lambda};\mathcal{S}_{(\mathcal{D}, u_1)})-\overline{\nabla}f(\boldsymbol{\lambda})\Vert^2\}}_{\text{\textbf{error}}}=\notag\\
    &\underbrace{\mathbb{E}_{\mathcal{D},u_1}\{\Vert\widehat{\nabla}{f}(\boldsymbol{\lambda};\mathcal{S}_{(\mathcal{D}, u_1)})-\widetilde{\nabla}{f}({\boldsymbol{\lambda}})\Vert^2\}}_{\text{\textcolor{red}{\textbf{Variance}}}}+\underbrace{\Vert \widetilde{\nabla}{f}({\boldsymbol{\lambda}})-\overline{\nabla}f({\boldsymbol{\lambda}})  \Vert^2}_{\text{\textcolor{blue}{\textbf{Bias$^2$}}}},
\end{align}where $\widetilde{\nabla}{f}({\boldsymbol{\lambda}})=\mathbb{E}_{\mathcal{D},u_1}\{\widehat{\nabla}{f}(\boldsymbol{\lambda};\mathcal{S}_{(\mathcal{D}, u_1)})\}$ represents the expected hypergradient estimation, and $\overline{\nabla}f({\boldsymbol{\lambda}})=\mathbb{E}_{\mathcal{D},u_1}\{{\nabla}{f}(\boldsymbol{\lambda};\mathcal{S}_{(\mathcal{D}, u_1)})\}$ represents the underlying ground-truth hypergradient. The first term of Eq. (\ref{equation32-1}) is called \textit{variance}, reflecting the variability of the estimated hypergradient around its expected value due to the diversity of data factors. The second term is called squared \textit{bias}, representing the difference between the empirical and expected hypergradient estimations via the HPO algorithm. By Jensen's inequality, we have:
\begin{align}\label{equation32-2}
\Vert \widetilde{\nabla}{f}({\boldsymbol{\lambda}})-\overline{\nabla}f({\boldsymbol{\lambda}})  \Vert^2\leq \mathbb{E}_{\mathcal{D},u_1}\{\Vert\widehat{\nabla}{f}(\boldsymbol{\lambda};\mathcal{S}_{(\mathcal{D}, u_1)})-\nabla {f}(\boldsymbol{\lambda};\mathcal{S}_{(\mathcal{D}, u_1)})\Vert^2\}.
\end{align}
By combining Eqs. (\ref{equation32-1}) and (\ref{equation32-2}), we can obtain the following expression
\begin{align}\label{equation32-3}
     \mathbb{E}_{\mathcal{D},u_1}\{\Vert\widehat{\nabla}{f}(\boldsymbol{\lambda};\mathcal{S}_{(\mathcal{D}, u_1)})-\overline{\nabla}f(\boldsymbol{\lambda})\Vert^2\}
    \leq &\mathbb{E}_{\mathcal{D},u_1}\{\Vert\widehat{\nabla}{f}(\boldsymbol{\lambda};\mathcal{S}_{(\mathcal{D}, u_1)})-\widetilde{\nabla}{f}({\boldsymbol{\lambda}})\Vert^2\}+\notag\\
    &\mathbb{E}_{\mathcal{D},u_1}\{\Vert\widehat{\nabla}{f}(\boldsymbol{\lambda};\mathcal{S}_{(\mathcal{D}, u_1)})-\nabla {f}(\boldsymbol{\lambda};\mathcal{S}_{(\mathcal{D}, u_1)})\Vert^2\},
\end{align} where $\Vert\widehat{\nabla}{f}(\boldsymbol{\lambda};\mathcal{S}_{(\mathcal{D}, u_1)})-\nabla {f}(\boldsymbol{\lambda};\mathcal{S}_{(\mathcal{D}, u_1)})\Vert^2$ is bounded by Eqs. (\ref{eq31-1}-\ref{eq31-2}) for ITD and AID, respectively. It suggests that current theoretical convergence results \citep{grazzi2020iteration, ji2021bilevel, liu2020generic, liu2021towards} are closely related to this bias square term estimation. While it is seen that there is still very limited research on the characteristic analysis of the variance term yet. This handles the capability of the existing error analysis results on revealing the theoretical insight of more practically observed empirical phenomena, such as overfitting to validation set \citep{franceschi2018bilevel,bao2021stability}, which is closely related to the influence of the variance term. In this work, we attempt to specifically focus on the analysis of the variance term related to data distribution, and thus provide a supplemental analysis of hypergradient estimated by existing HPO algorithms.
\subsection{A Close Look at Variance Estimation for Error Analysis}\label{section31}
We aim to demonstrate that the variance, which arises from the different data splittings, is a significant factor in hypergradient estimation error.
\begin{figure*}[!t]
        \centering
	\subfigure[Hypergradient visualization of the RHG method on different data splittings. ($\lambda_1$:$\text{Reg}_1$, $\lambda_2$:$\text{Reg}_2$)]{\includegraphics[width=0.43\linewidth]{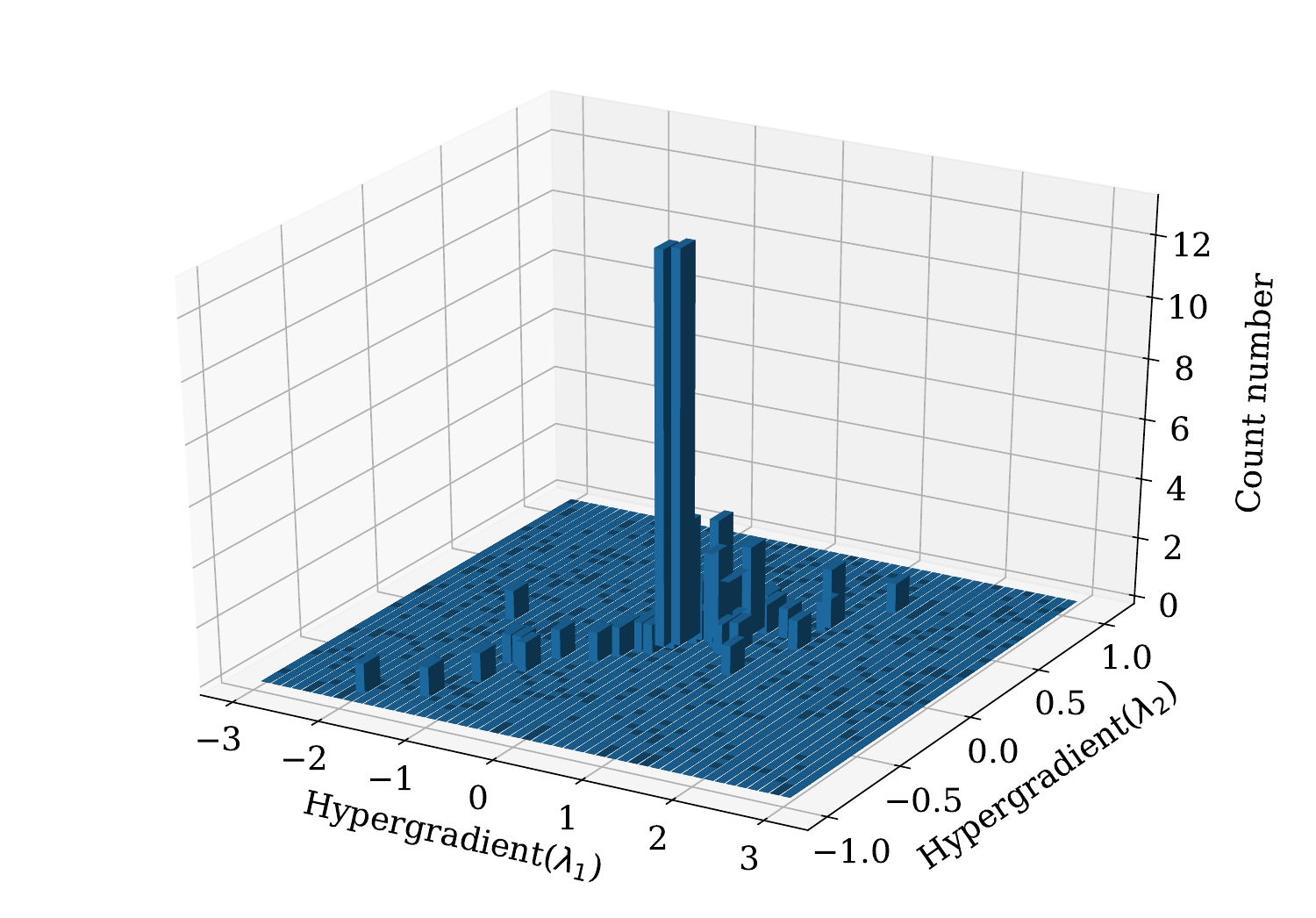}
	}
	\hspace{1cm}
	\subfigure[Hypergradient visualization of the RHG(+EHG) method on different data splittings. ($\lambda_1$:$\text{Reg}_1$, $\lambda_2$:$\text{Reg}_2$)]{\includegraphics[width=0.43\linewidth]{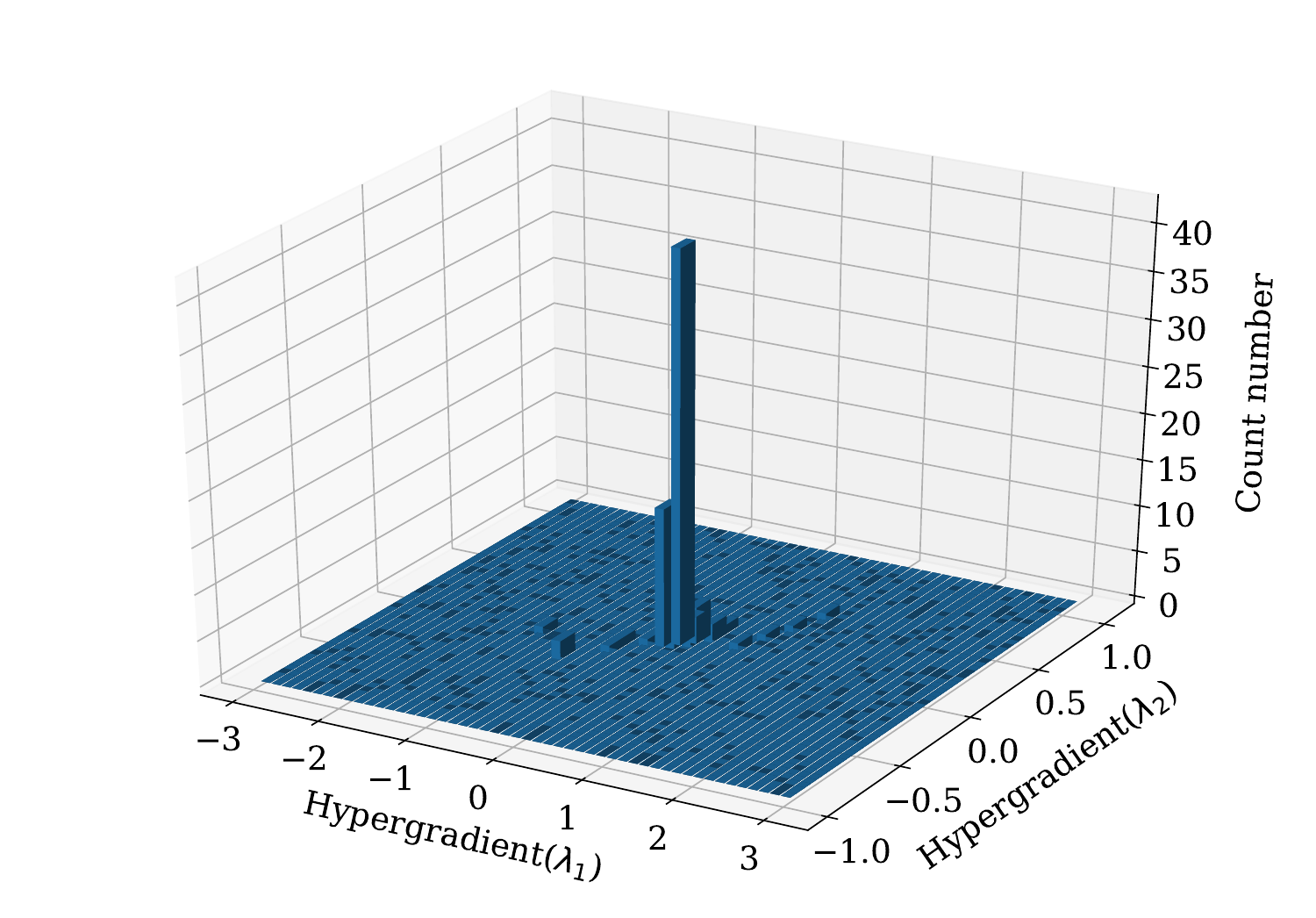}
	}
	\caption{Illustration of the impact of variance on hypergradient estimation across multiple data splittings. For the setting, we use 5-dimensional data for fitting elastic network, i.e., $\min_{\boldsymbol{\theta}}\{\sum_{i=1}^N(y_i-x_i^T\boldsymbol{\theta})^2+\lambda_1\Vert\boldsymbol{\theta}\Vert_1+\lambda_2\Vert\boldsymbol{\theta}\Vert_2^2\}$, and hyperparameter $\lambda_1$ and $\lambda_2$ are set to the regularization coefficients of L1 and L2 norms. For the RHG and RHG(+EHG, as mentioned in {Section \ref{sec4-1}}) methods, we repeat the experiments 100 times with different random seeds, where RHG \citep{franceschi2017forward} is a classic ITD. $U$ denotes the number of splittings for EHG. For details, please see Appendix \ref{sectionA16}.}
 \label{figure31-1}
\vspace{-0.3cm}
\end{figure*}

To illustrate this, we give an empirical evaluation of hypergradient estimation with various data splittings $\mathcal{S}_{(\mathcal{D}, \{u_i\}_{i=1}^U)}$, for regularization hyperparameter of the elastic network. As shown in Fig. \ref{figure31-1}, the hypergradient is scattered among various data splittings, which substantiates that the variance should play an unnegligible role for hypergradient estimation. The discrepancy can be attributed to a random sampling of training-validation data pairs, which possibly arises from the fact that observed data cannot accurately describe the distribution, as demonstrated by \citep{altman2014uncertainty}. Therefore, we will give a comprehensive analysis of hypergradient estimation including bias and variance as follows, and focus on the hypergradient variance.
\subsection{Theoretical Analysis of Hypergradient Estimation Error}\label{sec3-3}
In this section, we give the hypergradient estimation error analysis of ITD and AID, respectively.
\subsubsection{Theoretical Bound for ITD Algorithm}\label{section3-3-1}
Based on the bias-variance decomposition of Eq. (\ref{equation32-1}), we provide a supplemental analysis of Theorem \ref{theorem32-1}, which represents the current error analysis of hypergradient estimation for ITD, via introducing variance. We first provide the bound of hypergradient variance for ITD.
\begin{lemma}\label{lemma33-3}
For the given samples $\mathcal{D}\sim\mathscr{P}$ and splitting $\mathcal{S}_{(\mathcal{D}, u_1)}=(\mathcal{D}^{tr}, \mathcal{D}^{val})$, suppose that Assumptions \ref{assum33-1}-\ref{assum33-2} hold. Let $\widehat{\nabla}{f}(\boldsymbol{\lambda};\mathcal{S}_{(\mathcal{D}, u_1)})$ be defined by ITD algorithm. Then, we have 
\begin{align*}
\mathbb{E}_{\mathcal{D}, u_1}\Big{\Vert}\widehat{\nabla}{f}(\boldsymbol{\lambda};\mathcal{S}_{(\mathcal{D}, u_1)})-\widetilde{\nabla}{f}(\boldsymbol{\lambda})\Big{\Vert}^2\leq{2M^2}\Big{(}\frac{1}{m^{tr}}(1+\frac{L}{\mu})^2+\frac{1}{m^{val}}\Big{)}.
\end{align*}
\end{lemma}
\noindent\textbf{Remark.} Lemma \ref{lemma33-3} provides the variance bound for the hypergradient $\widehat{\nabla}{f}(\boldsymbol{\lambda};\mathcal{S}_{(\mathcal{D}, u_1)})$ estimated by ITD on splitting $\mathcal{S}_{(\mathcal{D}, u_1)}$, where $\widetilde{\nabla}{f}({\boldsymbol{\lambda}})=\mathbb{E}_{\mathcal{D},u_1}\{\widehat{\nabla}{f}(\boldsymbol{\lambda};\mathcal{S}_{(\mathcal{D}, u_1)})\}$. This theoretical result indicates that the bound of hypergradient variance for ITD is solely dependent on the data size (training/validation size $m^{tr}$/$m^{val}$) and is independent of variables typically considered in existing analyses \citep{grazzi2020iteration}, such as the number of iterations. Note that in Lemma \ref{lemma33-3}, $M, L, \mu$ are constants in the assumptions.

Then we get the hypergradient error bound of ITD based on Lemma \ref{lemma33-3}.
\begin{theorem}\label{theorem33-1}
	For the given samples $\mathcal{D}\sim\mathscr{P}$ and splitting $\mathcal{S}_{(\mathcal{D}, u_1)}=(\mathcal{D}^{tr}, \mathcal{D}^{val})$, suppose that Assumptions \ref{assum33-1}-\ref{assum33-2} hold. Let $\widehat{\nabla}{f}(\boldsymbol{\lambda};\mathcal{S}_{(\mathcal{D}, u_1)})$ be defined by ITD algorithm. For $\boldsymbol{\lambda}\in\boldsymbol{\Lambda}$, we have
	\begin{align*}
			&\mathbb{E}_{\mathcal{D}, u_1}\Big{\Vert}\widehat{\nabla}{f}(\boldsymbol{\lambda};\mathcal{S}_{(\mathcal{D}, u_1)})-\overline{\nabla}f(\boldsymbol{\lambda})\Big{\Vert}^2\leq B_{\text{ITD}, u_1}^2(\boldsymbol{\lambda}, K)+{2M^2}\Big{(}\frac{1}{m^{tr}}(1+\frac{L}{\mu})^2+\frac{1}{m^{val}}\Big{)},
	\end{align*}
	where 
	$B_{\text{ITD}, u_1}(\boldsymbol{\lambda}, K):= \Big{(}2LC_{1,\boldsymbol{\lambda},\mathcal{D}^{tr}}(1+\frac{C_{2,\boldsymbol{\lambda},\mathcal{D}^{tr}}}{1-q})(1+\frac{MK}{q})+\frac{MC_{2,\boldsymbol{\lambda},\mathcal{D}^{tr}}}{1-q}\Big{)}q^K$. $C_{1,\boldsymbol{\lambda},\mathcal{D}^{tr}}$, $C_{2,\boldsymbol{\lambda},\mathcal{D}^{tr}}$ and $q\in(0, 1)$ are constants in Lemmas \ref{lemma33-1}-\ref{lemma-a2-5}.
\end{theorem}
\noindent\textbf{Remark.} Theorem \ref{theorem33-1} provides the error bound of hypergradient $\widehat{\nabla}{f}(\boldsymbol{\lambda};\mathcal{S}_{(\mathcal{D}, u_1)})$ computed using ITD on splitting $\mathcal{S}_{(\mathcal{D}, u_1)}$ compared to the ground-truth hypergradient $\overline{\nabla}f(\boldsymbol{\lambda})$ in expectation. The error bound is composed of two terms, including control bias and variance, respectively. Compared to Theorem \ref{theorem32-1}, we embed data factors into our error analysis via the variance term and thus offer a more supplementary analysis.

\subsubsection{Improved Theoretical Bound for ITD Algorithm}\label{section332}
We then aim to improve the hypergradient estimation error bound of ITD in Theorem \ref{theorem33-1} by influencing the hypergradient variance. Inspired by cross-validation, we employ multiple data splittings, which help to reduce hypergradient variance, to establish error analysis of ITD.
\begin{lemma}\label{lemma332-1}
	For the given samples $\mathcal{D}\sim\mathscr{P}$ and splittings $\mathcal{S}_{(\mathcal{D}, \{u_i\}_{i=1}^U)}=(\mathcal{D}^{tr}_{u_i}, \mathcal{D}^{val}_{u_i})_{i=1}^U$, suppose that Assumptions \ref{assum33-1}-\ref{assum33-2} hold. Let $\widehat{\nabla}{f}(\boldsymbol{\lambda};\mathcal{S}_{(\mathcal{D}, u_i)})$ be defined by ITD algorithm, where $i=1, 2, \dots, U$. Then, we have 
	\begin{align*}
		\mathbb{E}\left[\Big{\Vert}\frac{\sum_{i=1}^{U}\widehat{\nabla}{f}(\boldsymbol{\lambda};\mathcal{S}_{(\mathcal{D}, u_i)})}{U}-\widetilde{\nabla}{f}(\boldsymbol{\lambda})\Big{\Vert}^2\right]\leq\frac{2M^2}{U}\Big{(}\frac{1}{m^{tr}}(1+\frac{L}{\mu})^2+\frac{1}{m^{val}}\Big{)}.
	\end{align*}
\end{lemma}
\noindent\textbf{Remark.} Compared with Lemma \ref{lemma33-3},  Lemma \ref{lemma332-1} provides the bound of variance of mean hypergradient $\frac{\sum_{i=1}^{U}\widehat{\nabla}{f}(\boldsymbol{\lambda};\mathcal{S}_{(\mathcal{D}, u_i)})}{U}$ computed over $U$ splittings. The result can be derived using Lemma \ref{lemma33-3} and the properties of variance (if $\{X_i\}_{i=1}^N$ are \textit{i.i.d.}, then $Var[\sum_{i=1}^NX_i]=\sum_{i=1}^NVar[X_i]$).

Leveraging Lemma \ref{lemma332-1}, We can then provide the hypergradient error bound for ITD.
\begin{theorem}\label{theorem332-1}
For the given samples $\mathcal{D}\sim\mathscr{P}$ and splittings $\mathcal{S}_{(\mathcal{D}, \{u_i\}_{i=1}^U)}=(\mathcal{D}^{tr}_{u_i}, \mathcal{D}^{val}_{u_i})_{i=1}^U$, suppose that Assumptions \ref{assum33-1}-\ref{assum33-2} hold. Let $\widehat{\nabla}{f}(\boldsymbol{\lambda};\mathcal{S}_{(\mathcal{D}, u_i)})$ be defined by ITD algorithm, where $i=1, 2, \dots, U$. For $\boldsymbol{\lambda}\in\boldsymbol{\Lambda}$, we have
\begin{align*}
&\mathbb{E}\left[\Big{\Vert}\frac{\sum_{i=1}^{U}\widehat{\nabla}{f}(\boldsymbol{\lambda};\mathcal{S}_{(\mathcal{D}, u_i)})}{U}-\overline{\nabla}f(\boldsymbol{\lambda})\Big{\Vert}^2\right]\leq \sup_i B_{\text{ITD},u_i}^2(\boldsymbol{\lambda}, K)+\frac{2M^2}{U}\Big{(}\frac{1}{m^{tr}}(1+\frac{L}{\mu})^2+\frac{1}{m^{val}}\Big{)},
\end{align*}
where $B_{\text{ITD}, u_i}(\boldsymbol{\lambda}, K):= \Big{(}2LC_{1,\boldsymbol{\lambda},\mathcal{D}^{tr}_{u_i}}(1+\frac{C_{2,\boldsymbol{\lambda},\mathcal{D}^{tr}_{u_i}}}{1-q})(1+\frac{MK}{q})+\frac{MC_{2,\boldsymbol{\lambda},\mathcal{D}^{tr}_{u_i}}}{1-q}\Big{)}q^K$. $C_{1,\boldsymbol{\lambda},\mathcal{D}^{tr}}$, $C_{2,\boldsymbol{\lambda},\mathcal{D}^{tr}}$ and $q\in(0, 1)$ are constants in Lemmas \ref{lemma33-1}-\ref{lemma-a2-5}.
\end{theorem}
\noindent\textbf{Remark.} Theorem \ref{theorem332-1} provides the error bound of hypergradient $\frac{\sum_{i=1}^{U}\widehat{\nabla}{f}(\boldsymbol{\lambda};\mathcal{S}_{(\mathcal{D}, u_i)})}{U}$ via $U$ splittings compared to the ground-truth hypergradient $\overline{\nabla}f(\boldsymbol{\lambda})$ in expectation. The first term is related to the bias of the hypergradient computed on data splittings and increases with the number of splittings $U$. The second term pertains to hypergradient estimation variance, which can be effectively controlled by increasing the data size and $U$. This indicates that a suitable choice of $U$ is necessary for finely controlling the hypergradient estimation error, as Fig. \ref{fig423-1} illustrates. Additionally, Theorem \ref{theorem33-1} can be regarded as a special case of Theorem \ref{theorem332-1} when considering a single splitting.
\subsubsection{Theoretical Bound for AID Algorithm} \label{section3-3-3}
Similar to ITD, we provide a supplemental analysis of Theorem \ref{theorem32-2}, which represents the current error analysis of hypergradient estimation for AID, via introducing variance. We can first give hypergradient variance bound.
\begin{lemma}\label{lemma333-1}
	For the given samples $\mathcal{D}\sim\mathscr{P}$ and splitting $\mathcal{S}_{(\mathcal{D}, u_1)}=(\mathcal{D}^{tr}, \mathcal{D}^{val})$, suppose that Assumptions \ref{assum33-1}-\ref{assum33-2} hold. Let $\widehat{\nabla}{f}(\boldsymbol{\lambda};\mathcal{S}_{(\mathcal{D}, u_1)})$ be defined by AID algorithm. Then, we have 
	\begin{align*}
		\mathbb{E}_{\mathcal{D}, u_1}\Big{\Vert}\widehat{\nabla}{f}(\boldsymbol{\lambda};\mathcal{S}_{(\mathcal{D}, u_1)})-\widetilde{\nabla}{f}(\boldsymbol{\lambda})\Big{\Vert}^2\leq\frac{12L^2M^2}{\mu^2m^{tr}}+\frac{2M^2}{m^{val}}\big{(}1+\frac{4L^2}{\mu^2}\big{)}.
	\end{align*}
\end{lemma}
Leveraging Lemma \ref{lemma333-1}, we can then give the main result of AID.
\begin{theorem}\label{theorem333-1}
	For the given samples $\mathcal{D}\sim\mathscr{P}$ and splitting $\mathcal{S}_{(\mathcal{D}, u_1)}=(\mathcal{D}^{tr}, \mathcal{D}^{val})$, suppose that Assumptions \ref{assum33-1}-\ref{assum33-3} hold. Let $\widehat{\nabla}{f}(\boldsymbol{\lambda};\mathcal{S}_{(\mathcal{D}, u_1)})$ be defined by AID algorithm. Then, we have 
	\begin{align*}
			\mathbb{E}_{\mathcal{D}, u_1}\Big{\Vert}\widehat{\nabla}{f}(\boldsymbol{\lambda};\mathcal{S}_{(\mathcal{D}, u_1)})-\overline{\nabla}f(\boldsymbol{\lambda})\Big{\Vert}^2\leq B_{\text{AID}, u_1}^2(\boldsymbol{\lambda}, K,Z)+\frac{12L^2M^2}{\mu^2m^{tr}}+\frac{2M^2}{m^{val}}\big{(}1+\frac{4L^2}{\mu^2}\big{)},
	\end{align*}
	where $B_{\text{AID}, u_1}(\boldsymbol{\lambda}, K,Z):=
				\rho_{\boldsymbol{\lambda}}(K)C_{1,\boldsymbol{\lambda},\mathcal{D}^{tr}}\Big{(}1+\frac{M+LC_{2,\boldsymbol{\lambda},\mathcal{D}^{tr}}}{\mu}+\frac{MLC_{2,\boldsymbol{\lambda},\mathcal{D}^{tr}}}{\mu^2}\Big{)}+\frac{M\sigma_{\boldsymbol{\lambda}}(Z)C_{2,\boldsymbol{\lambda},\mathcal{D}^{tr}}}{\mu}$. $C_{1,\boldsymbol{\lambda},\mathcal{D}^{tr}_{u_i}}$ and $C_{2,\boldsymbol{\lambda},\mathcal{D}^{tr}_{u_i}}$ are constants in Lemma \ref{lemma33-1}.
\end{theorem}
\noindent\textbf{Remark.} Theorem \ref{theorem333-1} provides the error bound of the hypergradient $\widehat{\nabla}{f}(\boldsymbol{\lambda};\mathcal{S}_{(\mathcal{D}, u_1)})$ computed using AID on the splitting $\mathcal{S}_{(\mathcal{D}, u_1)}$ compared to the ground-truth hypergradient $\overline{\nabla}f(\boldsymbol{\lambda})$ in the expectation, derived from the bias-variance decomposition. Compared to the error analysis of Theorem \ref{theorem32-2}, we embed data factors into our error analysis via variance term, and thus also offer a more comprehensive analysis.

\subsubsection{Improved Theoretical Bound for AID Algorithm}\label{section334}
We aim to ameliorate the hypergradient estimation error bound of AID in Theorem \ref{theorem32-2} by influencing hypergradient variance. Inspired by cross-validation, we employ multiple data splittings, which help to reduce hypergradient variance, to establish error analysis of AID.
\begin{lemma}\label{lemma334-1}
	For the given samples $\mathcal{D}\sim\mathscr{P}$ and splittings $\mathcal{S}_{(\mathcal{D}, \{u_i\}_{i=1}^U)}=(\mathcal{D}^{tr}_{u_i}, \mathcal{D}^{val}_{u_i})_{i=1}^U$, suppose that Assumptions \ref{assum33-1}-\ref{assum33-2} hold. Let $\widehat{\nabla}{f}(\boldsymbol{\lambda};\mathcal{S}_{(\mathcal{D}, u_i)})$ be defined by AID algorithm, where $i=1, 2, \dots, U$. Then, we have 
	\begin{align*}
		\mathbb{E}\left[\Big{\Vert}\frac{\sum_{i=1}^{U}\widehat{\nabla}{f}(\boldsymbol{\lambda};\mathcal{S}_{(\mathcal{D}, u_i)})}{U}-\widetilde{\nabla}{f}(\boldsymbol{\lambda})\Big{\Vert}^2\right]\leq\frac{1}{U}\Big{(}\frac{12L^2M^2}{\mu^2m^{tr}}+\frac{2M^2}{m^{val}}\big{(}1+\frac{4L^2}{\mu^2}\big{)}\Big{)}.
	\end{align*}
\end{lemma}
Leveraging Lemma \ref{lemma334-1}, we can then give the main result of AID.
\begin{theorem}\label{theorem334-1}
For the given samples $\mathcal{D}\sim\mathscr{P}$ and splittings $\mathcal{S}_{(\mathcal{D}, \{u_i\}_{i=1}^U)}=(\mathcal{D}^{tr}_{u_i}, \mathcal{D}^{val}_{u_i})_{i=1}^U$, suppose that Assumptions \ref{assum33-1}-\ref{assum33-3} hold. Let $\widehat{\nabla}{f}(\boldsymbol{\lambda};\mathcal{S}_{(\mathcal{D}, u_i)})$ be defined by AID algorithm, where $i=1, 2, \dots, U$. Then, we have
\begin{align*}
\mathbb{E}\left[\Big{\Vert}\frac{\sum_{i=1}^{U}\widehat{\nabla}{f}(\boldsymbol{\lambda};\mathcal{S}_{(\mathcal{D}, u_i)})}{U}-\overline{\nabla}f(\boldsymbol{\lambda})\Big{\Vert}^2\right]\leq\sup_i B_{\text{AID}, u_i}^2(\boldsymbol{\lambda}, K,Z)+\frac{1}{U}\Big{(}\frac{12L^2M^2}{\mu^2m^{tr}}+\frac{2M^2}{m^{val}}\big{(}1+\frac{4L^2}{\mu^2}\big{)}\Big{)},
\end{align*}
where $B_{\text{AID}, u_i}(\boldsymbol{\lambda}, K,Z):=
\rho_{\boldsymbol{\lambda}}(K)C_{1,\boldsymbol{\lambda},\mathcal{D}^{tr}_{u_i}}\Big{(}1+\frac{M+LC_{2,\boldsymbol{\lambda},\mathcal{D}^{tr}_{u_i}}}{\mu}+\frac{MLC_{2,\boldsymbol{\lambda},\mathcal{D}^{tr}_{u_i}}}{\mu^2}\Big{)}+\frac{M\sigma_{\boldsymbol{\lambda}}(Z)C_{2,\boldsymbol{\lambda},\mathcal{D}^{tr}_{u_i}}}{\mu}$. $C_{1,\boldsymbol{\lambda},\mathcal{D}^{tr}_{u_i}}$ and $C_{2,\boldsymbol{\lambda},\mathcal{D}^{tr}_{u_i}}$ are constants in Lemma \ref{lemma33-1}.
\end{theorem}
\noindent\textbf{Remark.} Theorem \ref{theorem334-1} gives the bound between the hypergradient mean $\frac{\sum_{i=1}^{U}\widehat{\nabla}{f}(\boldsymbol{\lambda};\mathcal{S}_{(\mathcal{D}, u_i)})}{U}$ via $U$ splittings and the groun-truth gradient $\overline{\nabla}f(\boldsymbol{\lambda})$. Specifically, the first term (i.e., bias) is related to optimization factors (iteration step $K$ and $Z$) and the number of splittings $U$, and the second term (i.e., variance) is only related to data factors, including data size and $U$. Moreover, Theorem \ref{theorem333-1} can be regarded as a special case of Theorem \ref{theorem334-1} when considering a single splitting.
\section{Theoretically Inspired  Hypergradient Variance Reduction Method}\label{section4}
The available data is often limited in practice. Thus, inspired by Lemmas \ref{lemma332-1} and \ref{lemma334-1}, we can reduce hypergradient variance by increasing the number of data splittings.
\subsection{The Ensemble Hypergradient Strategy of Variance Reduction}\label{sec4-1}
\begin{algorithm}[t]
	\caption{The Ensemble Hypergradient Strategy of HPO algorithm}
	\renewcommand{\algorithmicrequire}{\textbf{Input:}}
	\renewcommand{\algorithmicensure}{\textbf{Output:}}
	\begin{algorithmic}[1]
		\REQUIRE The original HPO algorithm $\mathcal{A}_{\text{hpo}}$; max iteration steps  $K$ and $T$; observed data $\mathcal{D}$; random seeds $\{u_i\}_{i=1}^{U}$; initialization ${\boldsymbol{\lambda}}_0$; learning rate scheme $\alpha_{in}$ and $\alpha_{out}$.
		\ENSURE Hyperparameter ${\boldsymbol{\lambda}}_{\{\mathcal{A}_{\text{hpo}}+\text{EHG}\}}$.
		\STATE ${\boldsymbol{\lambda}}^{(0)}\leftarrow{\boldsymbol{\lambda}}_0$, and generate splittings $(\mathcal{D}_{u_i}^{tr}, \mathcal{D}_{u_i}^{val})_{i=1}^{U}=\mathcal{S}_{(\mathcal{D}, \{u_i\}_{i=1}^U)}$.
		\FOR{$t=0$ {\bfseries to} $T-1$}
		\FOR{$i=1$ {\bfseries to} $U$}
		\STATE Calculate hypergradient $\widehat{\nabla}{f}(\boldsymbol{\lambda};(\mathcal{D}_{u_i}^{tr}, \mathcal{D}_{u_i}^{val}))\Big{|}_{\boldsymbol{\lambda}=\boldsymbol{\lambda}^{(t)}}$ from HPO algorithm $\mathcal{A}_{\text{hpo}}$.\\
		\textcolor{blue}{\# For ITD, this step corresponds to lines 2-6 of Algorithm \ref{algorithm21-1}. For AID, this step corresponds to lines 2-7 of Algorithm \ref{algorithm21-2}.}
		\ENDFOR
		\STATE $\boldsymbol{\lambda}^{(t+1)}\leftarrow\boldsymbol{\lambda}^{(t)}-\frac{\alpha_{out}}{U}\sum_{i=1}^{U}\widehat{\nabla}{f}(\boldsymbol{\lambda};(\mathcal{D}_{u_i}^{tr},\mathcal{D}_{u_i}^{val}))\Big{|}_{\boldsymbol{\lambda}=\boldsymbol{\lambda}^{(t)}}$.
		\ENDFOR
		\STATE \textbf{return} ${\boldsymbol{\lambda}}_{\{\mathcal{A}_{\text{hpo}}+\text{EHG}\}}$
	\end{algorithmic}
	\label{algorithm41-1}
\end{algorithm}
We propose the \textbf{ensemble hypergradient (EHG)} strategy to compute hypergradient. Specifically, we utilize $\widehat{\nabla}{f}(\boldsymbol{\lambda};\mathcal{S}_{(\mathcal{D},\{u_i\}_{i=1}^U)})=\frac{1}{U}\sum_{i=1}^{U}{\widehat{\nabla}{f}(\boldsymbol{\lambda},\mathcal{S}_{(\mathcal{D}, u_i)})}$ as hypergradient, where $\widehat{\nabla}{f}(\boldsymbol{\lambda};\mathcal{S}_{(\mathcal{D},\{u_i\}_{i=1}^U)})$ is an \textit{ensemble average} in statistical community \citep{suslick2001encyclopedia}. It is evident that the proposed EHG in Algorithm \ref{algorithm41-1} can be easily integrated into the current methods, including AID and ITD. Although EHG is natural and simple, we aim to emphasize the importance of reducing hypergradient estimation variance and to provide a new perspective for the design of future gradient-based algorithms. In practice, we set $U$ to 5 or 10. For a detailed discussion and analysis, {please refer to Appendix \ref{sectionA10-v3}.}

\subsection{The Proposed Online Ensemble Hypergradient Algorithm}
\begin{figure*}[!t]
	\centering
	\subfigure{
		\includegraphics[width=1.0\linewidth]{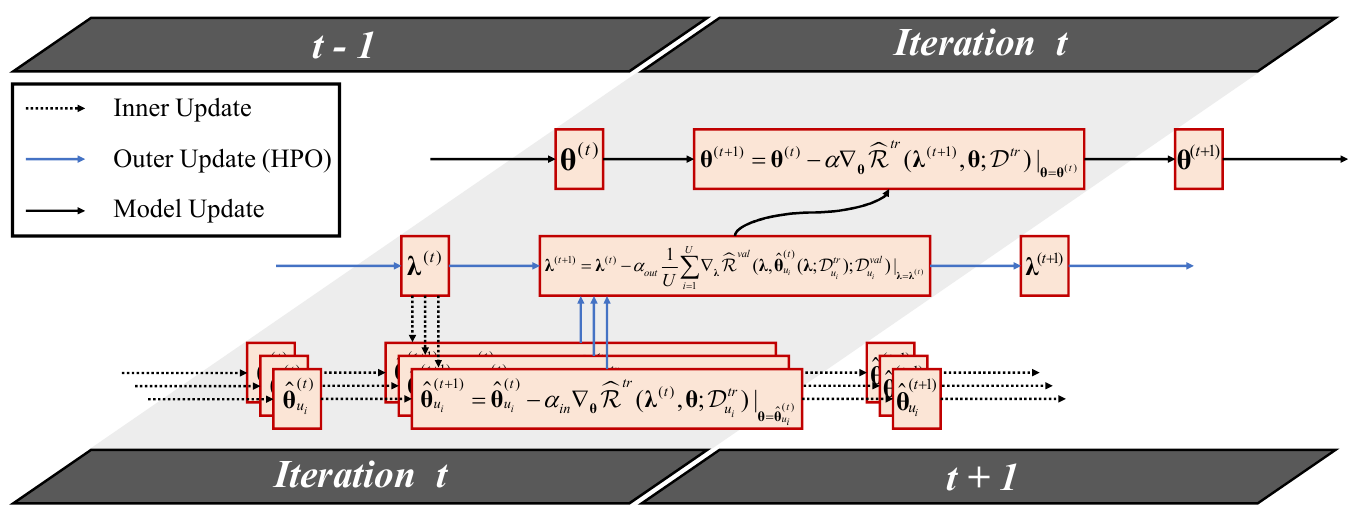}}
	\caption{OEHG Algorithm overview. At the iteration $t$, the inner-level first updates the model parameter of these $U$ data splittings by Eq. (\ref{equation53-1}). The outer-level then updates hyperparameter by Eq. (\ref{equation53-2}). The updated hyperparameter is further used to update model paratmeter by Eq. (\ref{equation53-3}).}
	\label{figure53-1}
	\vspace{-0.3cm}
\end{figure*}
Although EHG can reduce hypergradient variance, it incurs high computational costs because of the need for $K$ inner steps on each splitting. Therefore, we propose the \textbf{Online Ensemble Hypergradient (OEHG)} algorithm, wherein hyperparameter are optimized online during the model training process, thereby reducing computational costs. The OEHG is illustrated in Fig. \ref{figure53-1} and comprises the following steps:

\textbf{Constructing multiple data splittings.} Similar to cross-validation, we generate $U$ data splitting $\{\mathcal{D}^{tr}_{u_i}, \mathcal{D}^{val}_{u_i}\}_{i=1}^{U}$ from the observed data $\mathcal{D}$.

\textbf{Formulating learning manner of the inner-level.} We employ gradient descent to optimize the training loss in Eq. (\ref{equation53-1}) on each data splitting. Specifically, the updating equation of the model parameter in the inner-level can be formulated by moving the current $\hat{\boldsymbol{\theta}}^{(t)}_{u_i}$ along the descent direction of the training loss in Eq. (\ref{equation53-1}):
\begin{align}\label{equation53-1}
			\hat{\boldsymbol{\theta}}^{(t+1)}_{u_i} = \hat{\boldsymbol{\theta}}^{(t)}_{u_i} -\alpha_{in}
			\nabla_{\boldsymbol{\theta}}\hat{\mathcal{R}}^{tr}(\boldsymbol{\lambda}^{(t)},\boldsymbol{\theta};\mathcal{D}^{tr}_{u_i})\Big{|}_{\boldsymbol{\theta}=\hat{\boldsymbol{\theta}}_{u_i}^{(t)}}, 
\end{align}
where $\alpha_{in}$ is the learning rate of the model in the inner-level. The above update equation performs once in $U$ splittings, respectively.

\textbf{Updating hyperparameter in the outer-level:} Based on the model updating formulation $\hat{\boldsymbol{\theta}}^{(t+1)}_{u_i}$ in the inner-level from Eq. (\ref{equation53-1}), the hyperparameter $\boldsymbol{\lambda}$ can be readily updated guided by Eq. (\ref{equation53-2}), expressed as
\begin{align}\label{equation53-2}
	\boldsymbol{\lambda}^{(t+1)}=\boldsymbol{\lambda}^{(t)}-\alpha_{out}\frac{1}{U}\sum_{i=1}^{U}\nabla_{\boldsymbol{\lambda}}\hat{\mathcal{R}}^{val}(\boldsymbol{\lambda},\hat{\boldsymbol{\theta}}^{(t)}_{u_i}(\boldsymbol{\lambda};\mathcal{D}^{tr}_{u_i} );\mathcal{D}^{val}_{u_i})\Big{|}_{\boldsymbol{\lambda}=\boldsymbol{\lambda}^{(t)}}, 
\end{align}
where $\alpha_{out}$ is the learning rate for hyperparameter updating. Notice that $\boldsymbol{\lambda}$ in $\hat{\boldsymbol{\theta}}^{(t)}_{u_i}(\boldsymbol{\lambda}, \mathcal{D}^{tr}_{u_i})$ here is a variable instead of a quantity, which makes the gradient in Eq. (\ref{equation53-2}) able to be computed. \footnote{In some HPO cases, the formulation of $\hat{\mathcal{R}}^{val}$ may not be directly related to hyperparameter.}

\textbf{Updating model parameter in the outer-level:} Then, the updated $\boldsymbol{\lambda}^{(t+1)}$ is employed to ameliorate the parameter $\theta$ of the model in the outer-level, \textit{i.e.}, the model for test or inference: \footnote{Here, we use the same training data $\mathcal{D}^{tr}$ with the former gradient-based HPO method.}
\begin{align}\label{equation53-3}
	\boldsymbol{\theta}^{(t+1)} = \boldsymbol{\theta}^{(t)} -\alpha
	\nabla_{\boldsymbol{\theta}}\hat{\mathcal{R}}^{tr}(\boldsymbol{\lambda}^{(t+1)},\boldsymbol{\theta};\mathcal{D}^{tr})\Big{|}_{\boldsymbol{\theta}={\boldsymbol{\theta}}^{(t)}}.
\end{align}
Note that we derive with plain gradient descent here. This, however, also holds for most variants of gradient descent, like Adam \citep{kingma2014adam}. The OEHG can then be summarized in Algorithm \ref{algorithm53-1}, and Fig. \ref{figure53-1} illustrates its main implementation process (i.e., steps 6-8). All computations of gradients can be efficiently implemented by automatic differentiation techniques. The algorithm can be easily implemented using popular machine learning frameworks like PyTorch \citep{paszke2019pytorch}. It is easy to see that both the model parameter and hyperparameter gradually ameliorate their values during the learning process based on their results calculated in the last step, and the model thus tends to be updated stably.

\begin{algorithm}[t]
	\caption{The OEHG Learning Algorithm}
	\renewcommand{\algorithmicrequire}{\textbf{Input:}}
	\renewcommand{\algorithmicensure}{\textbf{Output:}}
	\begin{algorithmic}[1]
		\REQUIRE Observed data $\mathcal{D}$, size of data splittings $U$, max iteration steps $T$.
		\ENSURE Model parameter $\boldsymbol{\theta}^{(T)}$ and hyperparameter $\boldsymbol{\lambda}^{(T)}$
		\STATE Construct $U$ data splittings $\{\mathcal{D}_{u_i}^{tr}, \mathcal{D}_{u_i}^{val}\}_{i=1}^{U}$ from $\mathcal{D}$.
		\STATE Initialize model parameter $\boldsymbol{\theta}^{(0)}$, hyperparameter $\boldsymbol{\lambda}^{(0)}$, the model parameter $\{\hat{\boldsymbol{\theta}}^{(0)}_{u_i}\}_{i=1}^{U}$ of data splittings.
		\FOR{$t=1$ {\bfseries to} $T-1$}
		\FOR{$i=1$ {\bfseries to} $U$}
		\STATE Update $\hat{\boldsymbol{\theta}}^{(t+1)}_{u_i}$ and calculate $\hat{\boldsymbol{\theta}}^{(t+1)}_{u_i}(\boldsymbol{\lambda})$ by Eq. (\ref{equation53-1}).
		\ENDFOR
		\STATE Update $\boldsymbol{\lambda}^{(t+1)}$ by Eq. (\ref{equation53-2}).
		\STATE Update $\boldsymbol{\theta}^{(t+1)}$ by Eq. (\ref{equation53-3}).
		\ENDFOR
	\end{algorithmic}
	\label{algorithm53-1}
\end{algorithm}

\section{An Instance of Hypergradient Variance Reduction}  \label{sec5}
We highlight the utility of our analysis framework for obtaining a bias-variance decomposition of hypergradient estimation error in Section \ref{section3} and validate EHG for hypergradient variance reduction in Section \ref{section4} on ridge regression problem.
\subsection{Ridge Regression}
Considering the standard linear regression model $\mathcal{Y}=\mathcal{X}\boldsymbol{\theta}+\epsilon$, where $\mathcal{X}\in\mathbb{R}^{n\times r}$, and each row $x_i$ in $\mathcal{X}$ represents a $r$-dimensional sample with $r$ features. The corresponding elements $y_i$s in $\mathcal{Y}\in\mathbb{R}^n$ are its continuous responses (or outputs). We assume uncorrelated noise with zero mean, i.e., $\mathbb{E}[\epsilon]=0$, and $Cov[\epsilon]=\sigma^2 I_n$. We employ ridge regression to estimate the parameter $\boldsymbol{\theta}\in\mathbb{R}^r$
, solving the following optimization problem $\hat{\boldsymbol{\theta}}=\arg\min_{\boldsymbol{\theta}\in\mathbb{R}^r}\Vert \mathcal{X}\boldsymbol{\theta}-\mathcal{Y}\Vert^2_2+\lambda\Vert\boldsymbol{\theta}\Vert_2^2$, where $\lambda>0$ is a regularization parameter. The solution has the closed form $\hat{\boldsymbol{\theta}}=(\mathcal{X}^{\top}\mathcal{X}+\lambda I_r)^{-1}\mathcal{X}^{\top}\mathcal{Y}$.
\subsection{Bias-variance Decomposition of Hypergradient estimation error}\label{section5-2}
The HPO objective of ridge regression is $\lambda^*=\arg\min_{\lambda>0}\mathbb{E}_{\mathcal{D}^{tr},\mathcal{D}^{val}}[\Vert\mathcal{X}^{val}\hat{\boldsymbol{\theta}}_{\lambda,\mathcal{D}^{tr}}-\mathcal{Y}^{val}\Vert^2]$, where $\mathcal{D}^{tr}=\{\mathcal{X}^{tr},\mathcal{Y}^{tr}\}$, $\mathcal{D}^{val}=\{\mathcal{X}^{val},\mathcal{Y}^{val}\}$, and $\hat{\boldsymbol{\theta}}_{\lambda,\mathcal{D}^{tr}}=\big{(}(\mathcal{X}^{tr})^{\top}\mathcal{X}^{tr}+\lambda I\big{)}^{-1}(\mathcal{X}^{tr})^{\top}\mathcal{Y}^{tr}$.
\begin{figure*}[!t]
	\centering
        \subfigure[]{\includegraphics[width=0.31\textwidth]{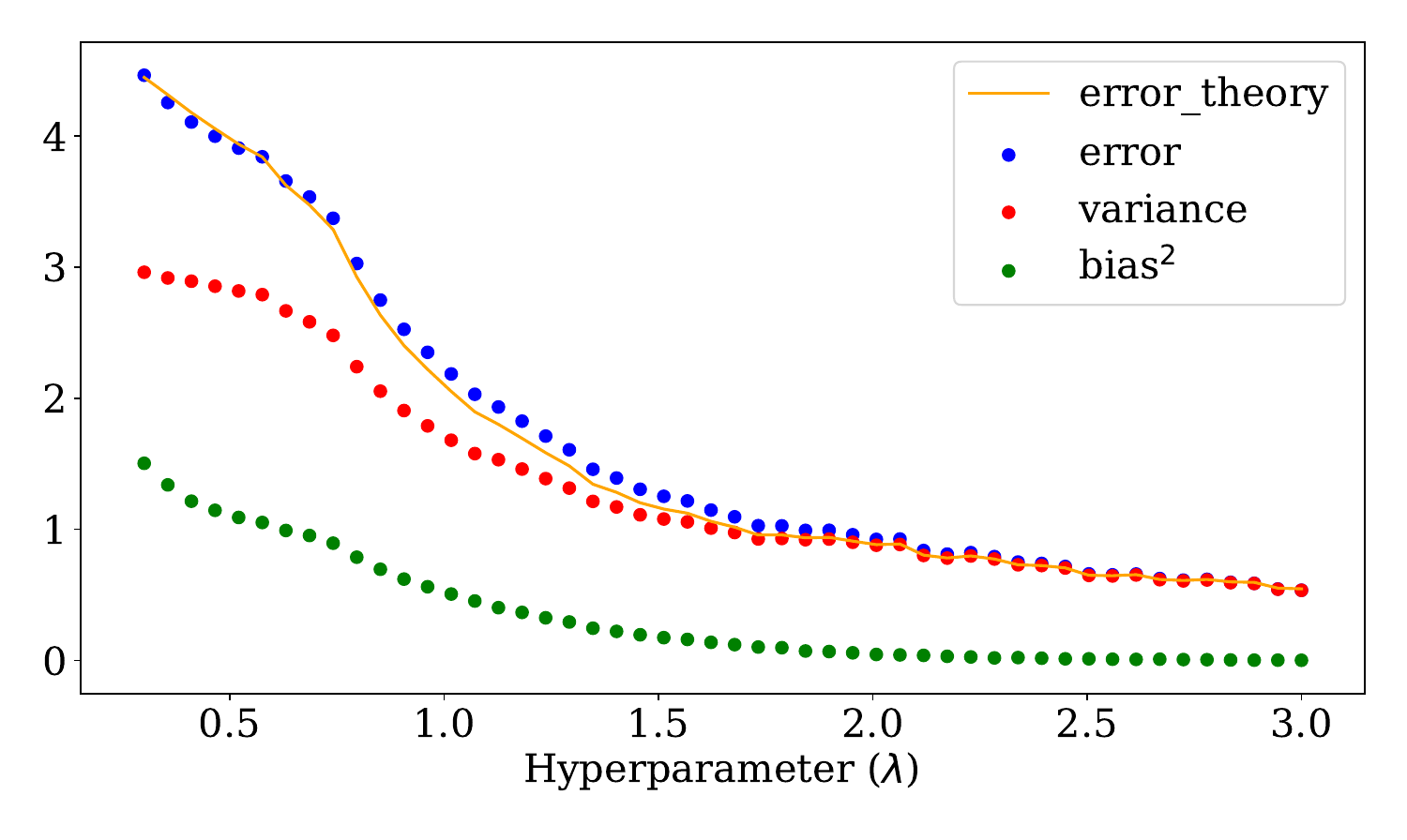}
	}
	\subfigure[]{\includegraphics[width=0.31\textwidth]{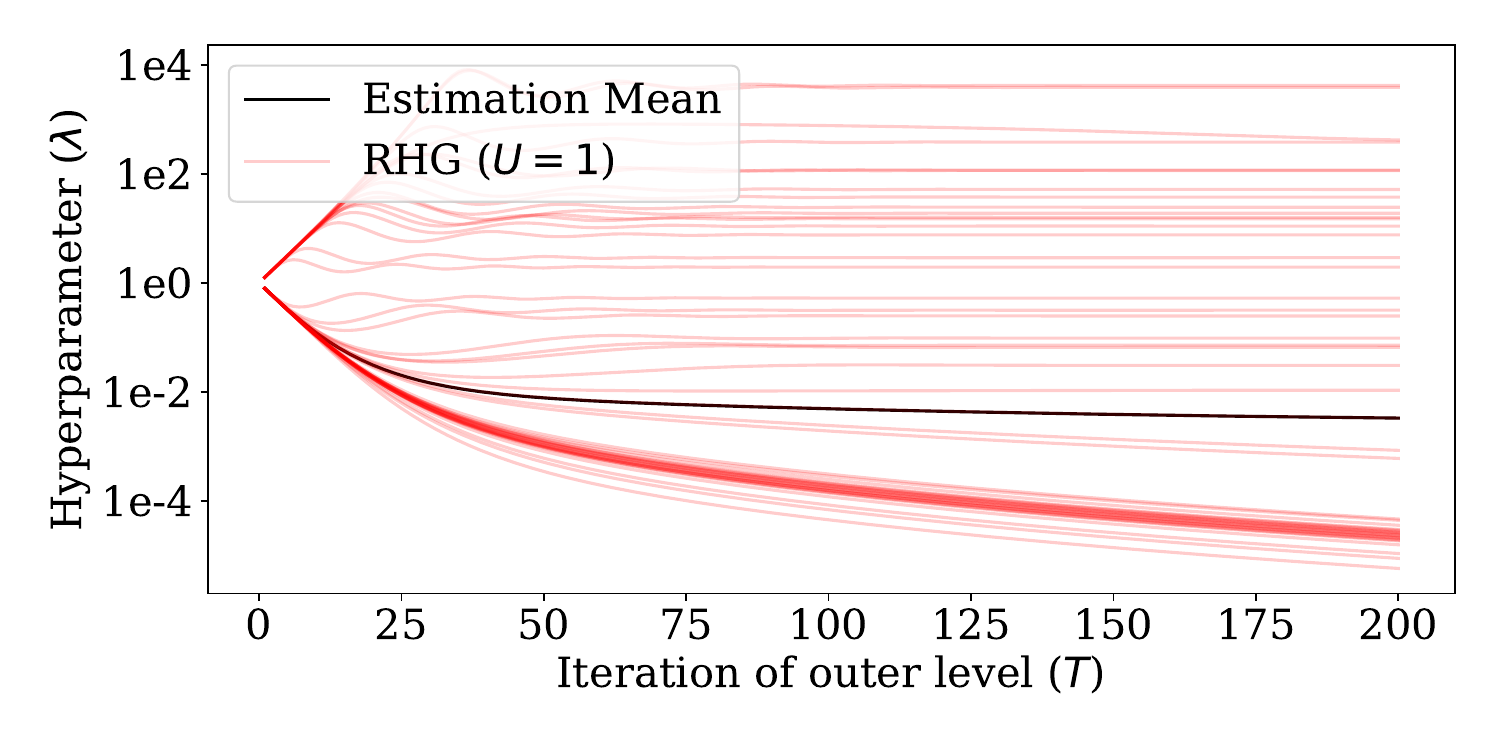}
	}
	\subfigure[]{\includegraphics[width=0.31\textwidth]{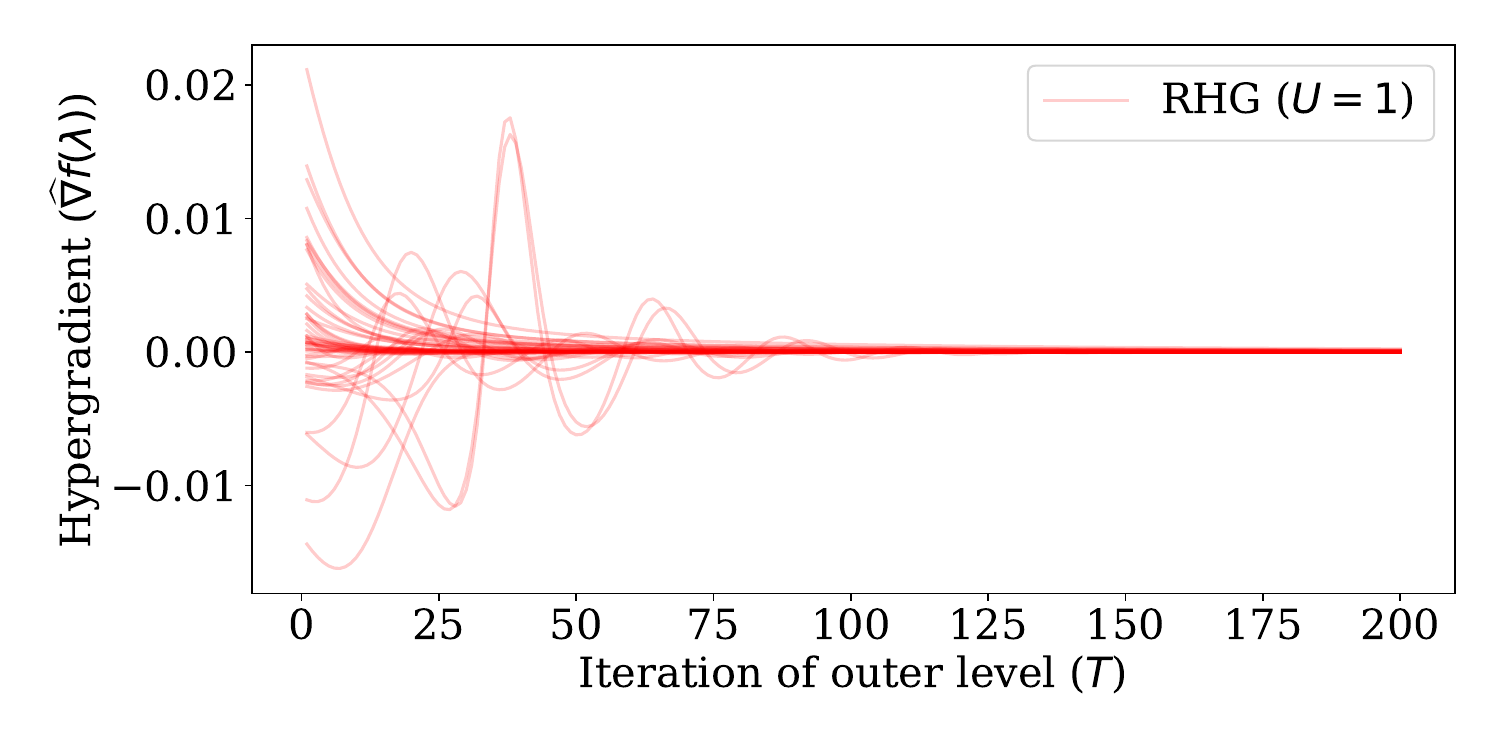}
	}
	\subfigure[]{\includegraphics[width=0.31\textwidth]{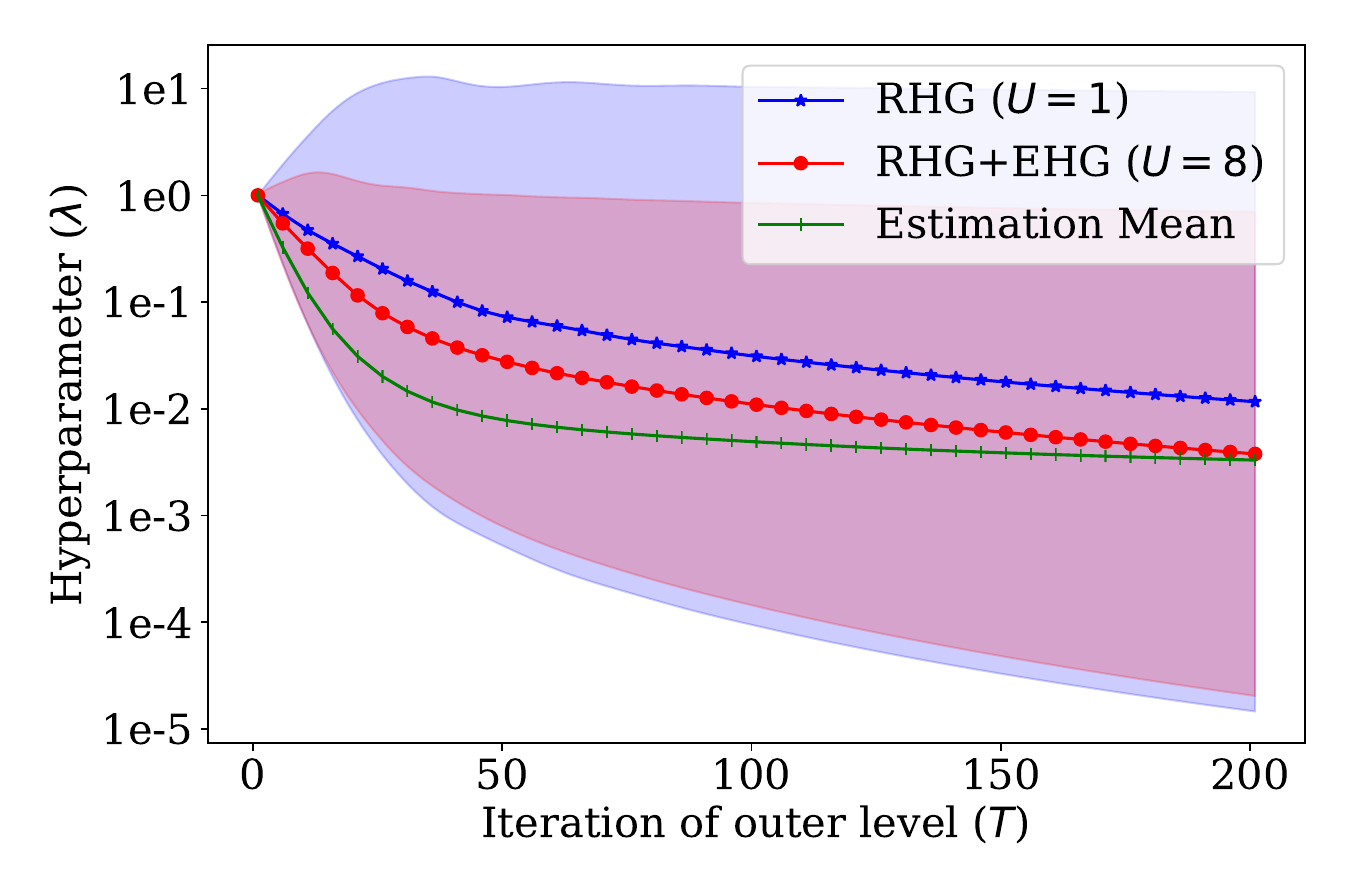}
	}
	\subfigure[]{\includegraphics[width=0.31\textwidth]{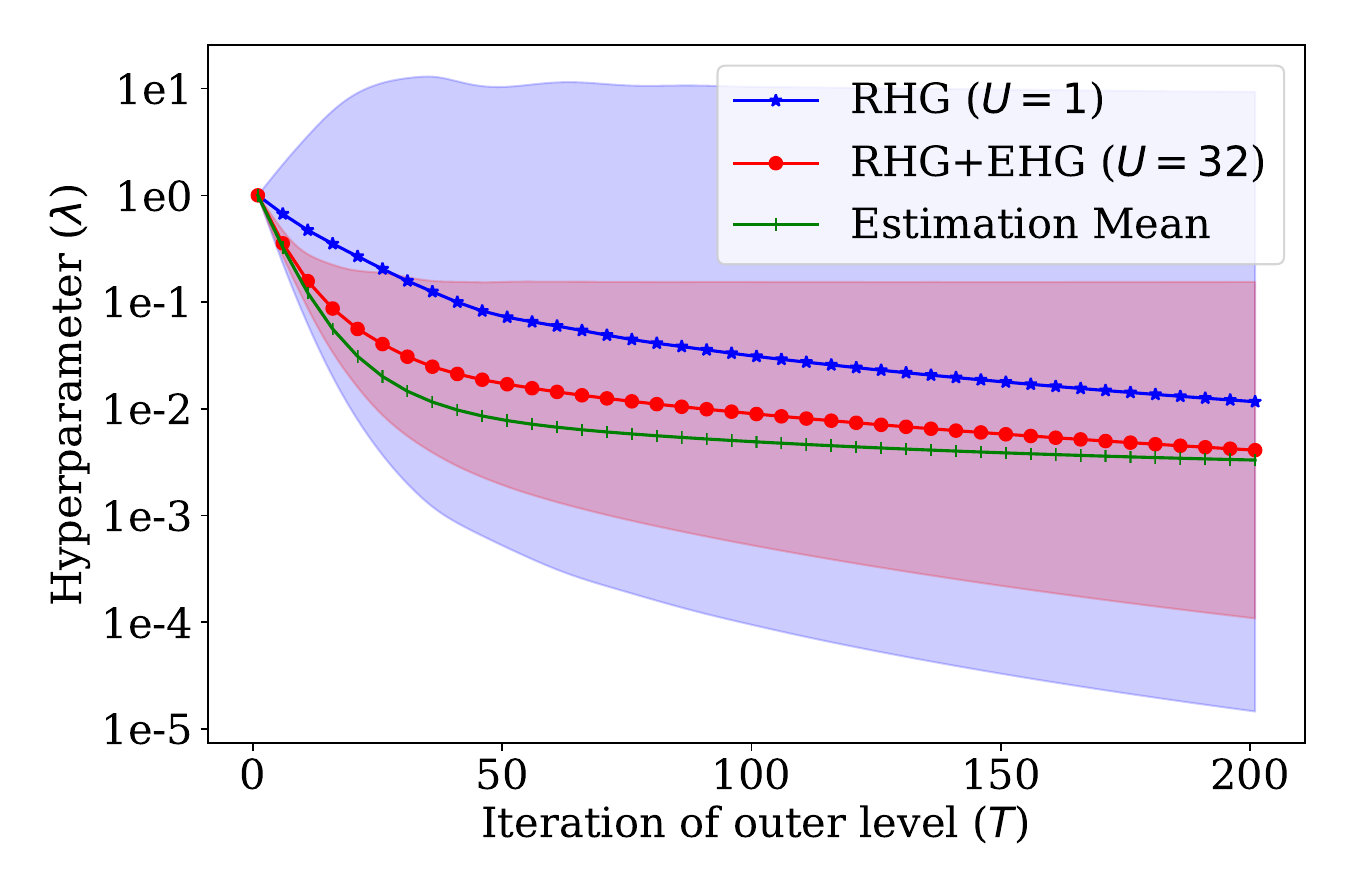}
	}
	\subfigure[]{\includegraphics[width=0.31\textwidth]{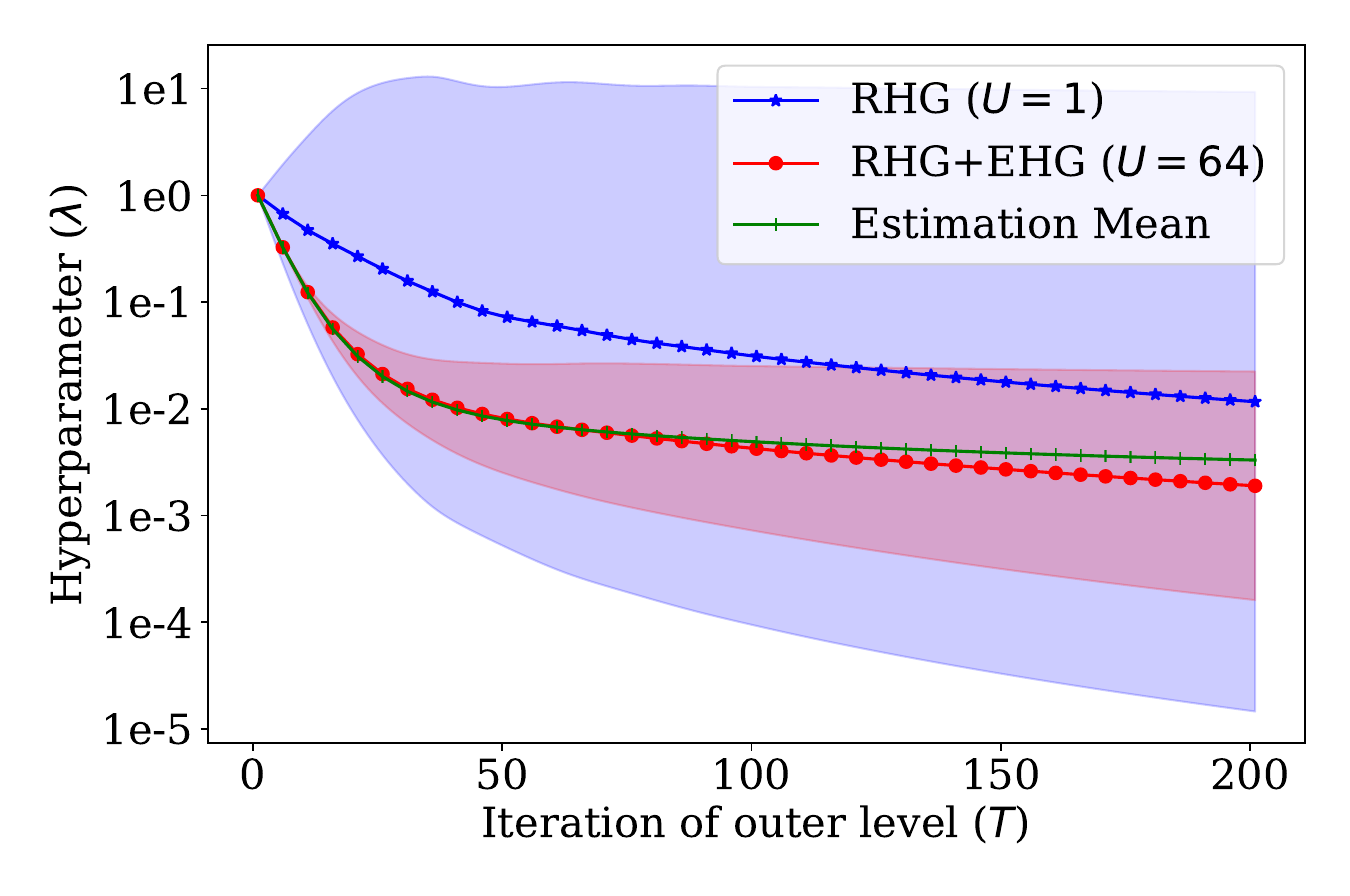}
	}
	\caption{(a): Visualization of hypergradient error, bias, and variance. error\_theory is calculated by the generated data distribution. (b-f): Visualization of hypergradient in HPO. The inner sub-problem is solved via the closed-form solution of ridge regression.}
	\label{fig5-2}
 \vspace{-0.3cm}
\end{figure*}
Then, we can get the specific form of optimal hypergradient $\overline{\nabla}f(\boldsymbol{\lambda})$. Based on the bias-variance decomposition of hypergradient estimation error in Eq. (\ref{equation32-1}), we can determine the specific forms of the components (see Eq. (\ref{equationA15-1}) in the Appendix). 

Therefore, we can generate data to validate the bias-variance decomposition for ridge regression. Fig. \ref{fig5-2}(a) illustrates the theoretical error $\mathbb{E}_{\mathcal{D},u_1}\{\Vert\widehat{\nabla}{f}(\boldsymbol{\lambda};\mathcal{S}_{(\mathcal{D}, u_1)})-\overline{\nabla}f(\boldsymbol{\lambda})\Vert^2\}$, the actual error, the bias $\Vert \widetilde{\nabla}{f}({\boldsymbol{\lambda}})-\overline{\nabla}f({\boldsymbol{\lambda}})  \Vert^2$, and the variance $\mathbb{E}_{\mathcal{D},u_1}\{\Vert\widehat{\nabla}{f}(\boldsymbol{\lambda};\mathcal{S}_{(\mathcal{D}, u_1)})-\widetilde{\nabla}{f}({\boldsymbol{\lambda}})\Vert^2\}$ of the hypergradient under one-dimensional regression setting. We calculate these statistical measures of the hypergradient for different values of $\lambda$.

The conclusions are as follows: (1) The theoretical and empirical values of the hypergradient estimation error are consistent, demonstrating the utility of our bias-variance decomposition framework. (2) The comparison of bias and variance confirms that variance does significantly impact the hypergradient estimation error.
\subsection{The Impact of Variance on Hypergradient Error Estimation}\label{section5-3}
We aim to demonstrate that variance significantly affects hypergradient estimation error and the EHG can significantly reduce hypergradient variance by increasing the number of splittings.

Figs. \ref{fig5-2}(b-c) illustrate the optimization process of hyperparameter and hypergradient across different data splittings, respectively. It can be observed that for different data splittings, the optimization process differs significantly from the true HPO process. This discrepancy is attributed to the large hypergradient variance shown in Fig. \ref{fig5-2}(c), ultimately leading to differences in the hyperparameter values. 

Figs. \ref{fig5-2}(d-f) indicate that EHG can reduce the variance of the hypergradient by increasing data splittings. Consequently, the values of hyperparameter exhibit smaller discrepancies in the optimization process compared to the true optimal HPO process.
\section{Experiments}\label{section7}
We experimentally demonstrate that the proposed variance reduction strategy  (EHG) and OEHG help improve hypergradient estimation across multiple HPO problems, including regularization parameter learning, data hyper-cleaning and few-shot learning.

\begin{table}[!t]
\caption{Summary of the experimental datasets.}\label{Table71-1}%
\begin{tabular}{@{}llllll@{}}
\toprule
Task& Name & Source & Training Instances & Testing Instances & Attributs\\
\midrule
\multirow{6}*{\rotatebox{90}{Regression}}&abalone &  UCI & 835 & 3341 & 8\\
&bodyfat & StatLib & 125 & 126 &14\\
&mg & \cite{flake2002efficient} & 276 & 1108&6\\
&pyrim & UCI & 14 & 59&27\\
&space & StatLib & 621 & 2485&6\\
&triazines & UCI &37 & 148&60\\
\midrule
\multirow{14}*{\rotatebox{90}{Classification}}&a1a & UCI & 1605 & 30956 & 123\\
&a2a & UCI & 2265 & 30296 & 123\\
&a3a & UCI & 3185 & 29376 & 123\\
&a4a & UCI & 4781 & 27780 & 123\\
&a5a & UCI & 6414 & 26147 & 123\\
&a6a & UCI & 11220 & 21341 & 123\\
&a7a & UCI & 16100 & 16461 & 123\\
&a8a & UCI & 22696 & 9865 & 123\\
&a9a & UCI & 32561 & 16281 & 123\\
&diabetes & UCI & 300 & 468 & 8\\
&gisette & \cite{guyon2004result} & 6000 & 1000 & 5000\\
&heart & Statlog & 100 &170 &13 \\
&ionosphere & UCI &200 & 151 & 34 \\
&w1a & \cite{platt1998fast} & 2477 & 47272 & 300 \\
\botrule
\end{tabular}
\end{table}

\begin{figure*}[!t]
	\centering
	\subfigure{\includegraphics[width=0.255\linewidth]{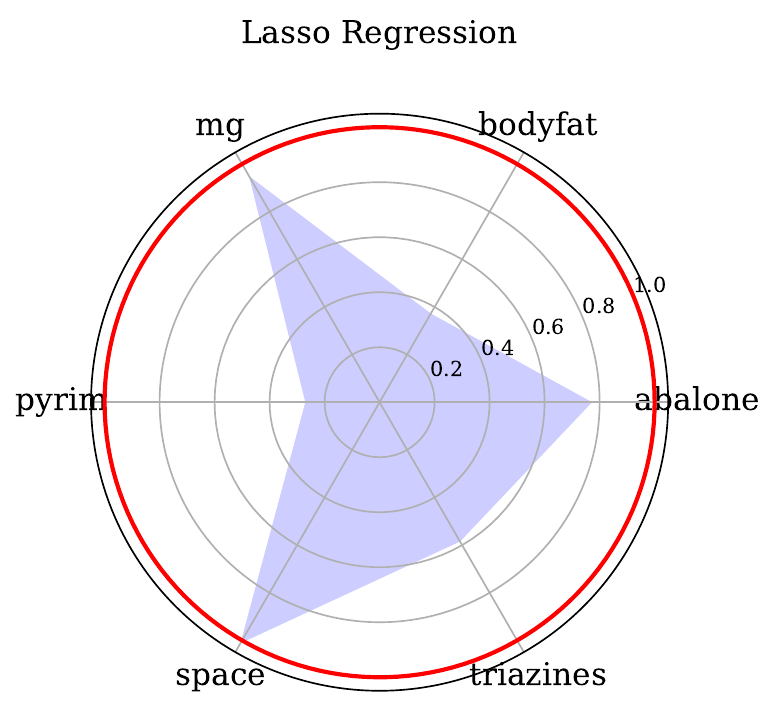}}
	\hspace{-4mm}
	\subfigure{\includegraphics[width=0.255\linewidth]{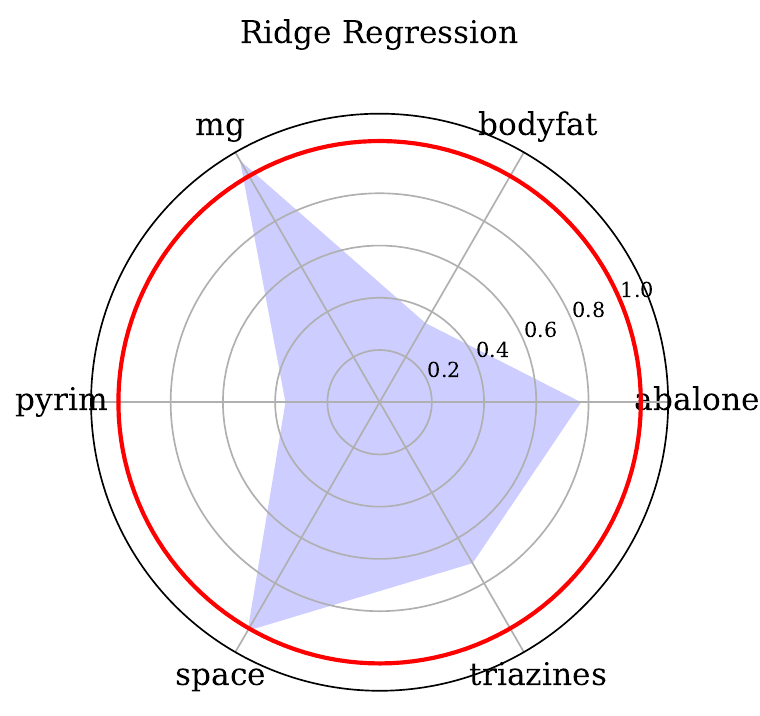}}
	\hspace{-4mm}
	\subfigure{\includegraphics[width=0.255\linewidth]{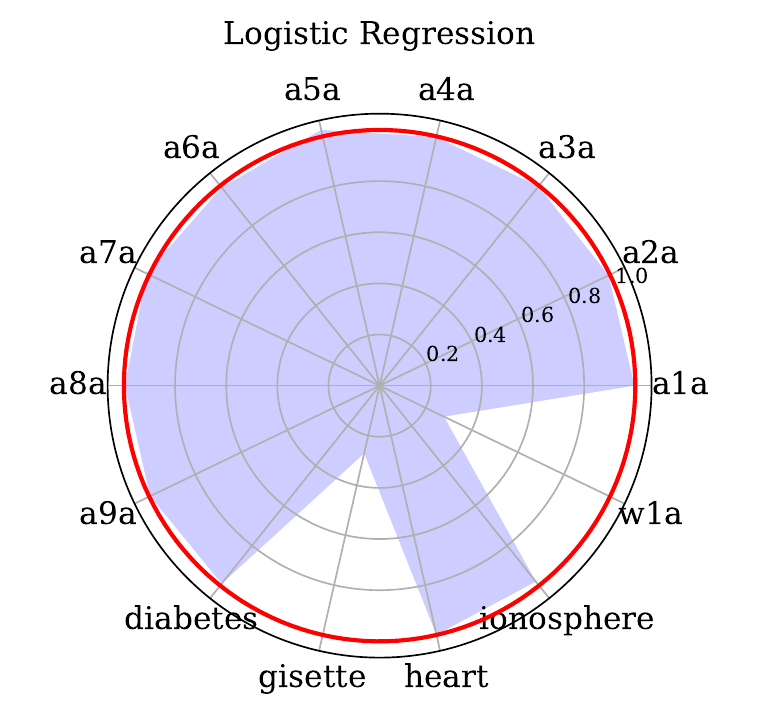}}
	\hspace{-4mm}
	\subfigure{\includegraphics[width=0.255\linewidth]{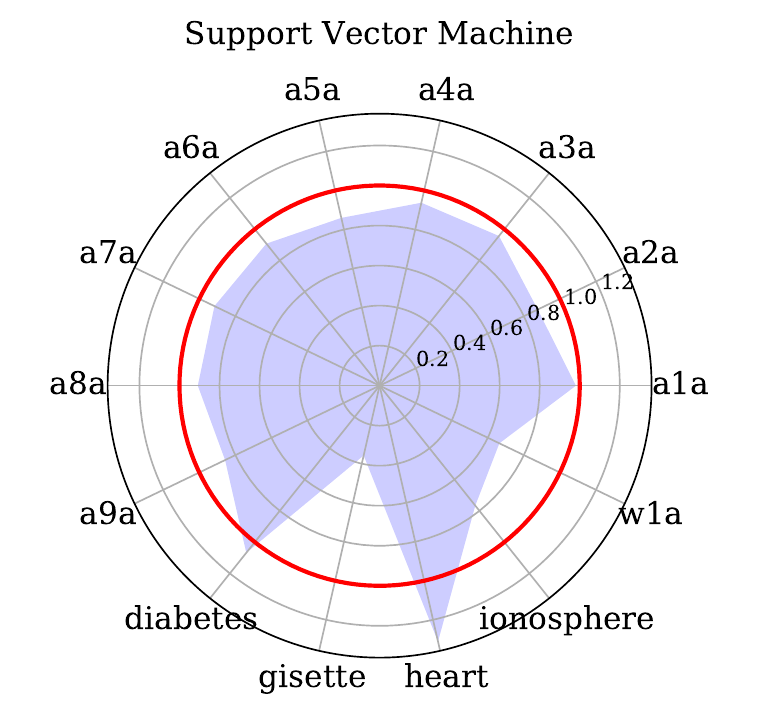}}
	\caption{Error ratio (RHG+EHG test error for $U$=5/RHG test error for $U$=1) results, where $U$ denotes the number of splittings. Red circle represents the position where they are equal. Inside red circle indicates that test error of RHG+EHG is smaller than that of RHG, and vice versa. 4 figures correspond to different models: lasso regression, ridge regression, logistic regression, support vector machine.}\label{Figure71-1}
 \vspace{-0.3cm}
\end{figure*}

\subsection{Validation Experiment for Reducing Hypergradient Variance}\label{section71v2}
In this section, we test the effectiveness of EHG on a linear model.
\textbf{Experimental Setup.} We select the regularization parameter as the hyperparameter to be learned, which is a common practice in machine learning. The total empirical risk function that needs to be minimized is expressed as $\hat{\mathcal{R}}^{tr}(\boldsymbol{\lambda},\boldsymbol{\theta};\mathcal{D}^{tr})=\mathcal{L}(\boldsymbol{\theta}, \mathcal{D}^{tr})+\lambda\cdot \text{Reg}(\boldsymbol{\theta})$, where $\lambda$ is regularization parameter that controls the relative importance of the data-dependent loss $\mathcal{L}(\boldsymbol{\theta}, \mathcal{D}^{tr})$ and the regularization term $\text{Reg}(\boldsymbol{\theta})$. 

For linear regression and its derivatives, the specific form of $\mathcal{L}$ is the mean square error, and when $ \text{Reg}(\cdot)$ is the $\ell_1$ norm or the $\ell_2$ norm, the task corresponds to lasso regression and ridge regression \citep{tibshirani1996regression}, respectively. For binary classification, When $\mathcal{L}$ is binary cross-entropy loss or hinge loss, the task corresponds to logistic regression or support vector machine, respectively. Additionally, $ \text{Reg}(\cdot)$ uses the $\ell_2$ norm.
\textbf{Experimental datasets.} The datasets \footnote{\url{https://www.csie.ntu.edu.tw/~cjlin/libsvmtools/datasets}} are from UCI \citep{asuncion2007uci}, Statlog \citep{king1995statlog}, StatLib \citep{kooperberg1997statlib} and other collections. Further details are demonstrated in Table \ref{Table71-1}.

\subsubsection{Verification of Iterative Differentiation Method}
In this section, we perform experimental validation of reverse hypergradient (RHG) \citep{franceschi2017forward}, a classic technique in ITD, on the four aforementioned machine learning models (lasso regression, ridge regression, logistic regression and support vector machine).
\textbf{Results.} Fig. \ref{Figure71-1} illustrates the generalization performance gains of the EHG over the RHG method. It shows that EHG enhances the HPO process for most models and datasets by reducing hypergradient variance, thereby improving generalization performance.
\subsubsection{Verification of Approximate Implicit Differentiation Method}\label{section73}	
\begin{figure*}[!t]
	\centering
	\subfigure{\includegraphics[width=0.5\linewidth]{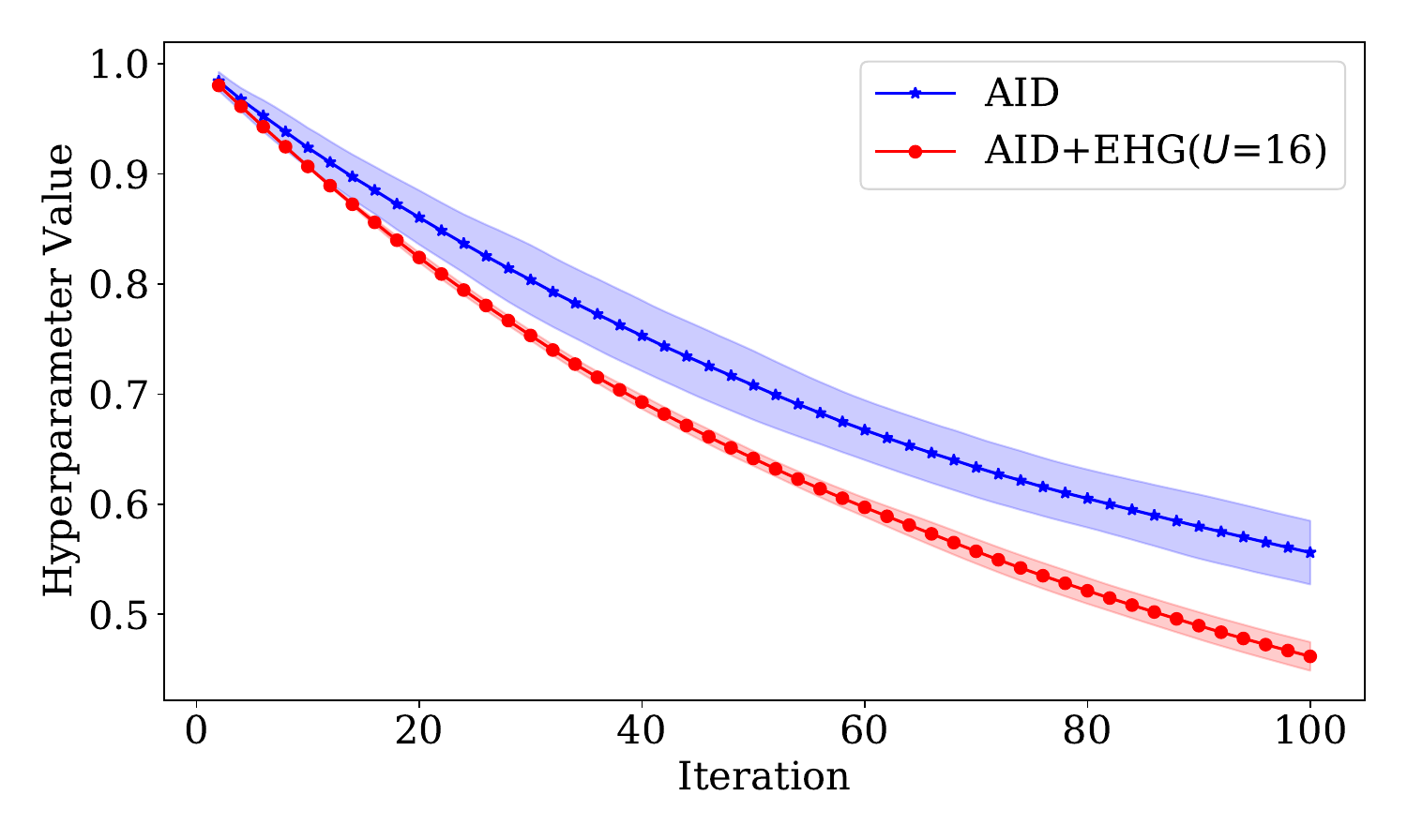}}
	\hspace{-3mm}
	\subfigure{\includegraphics[width=0.5\linewidth]{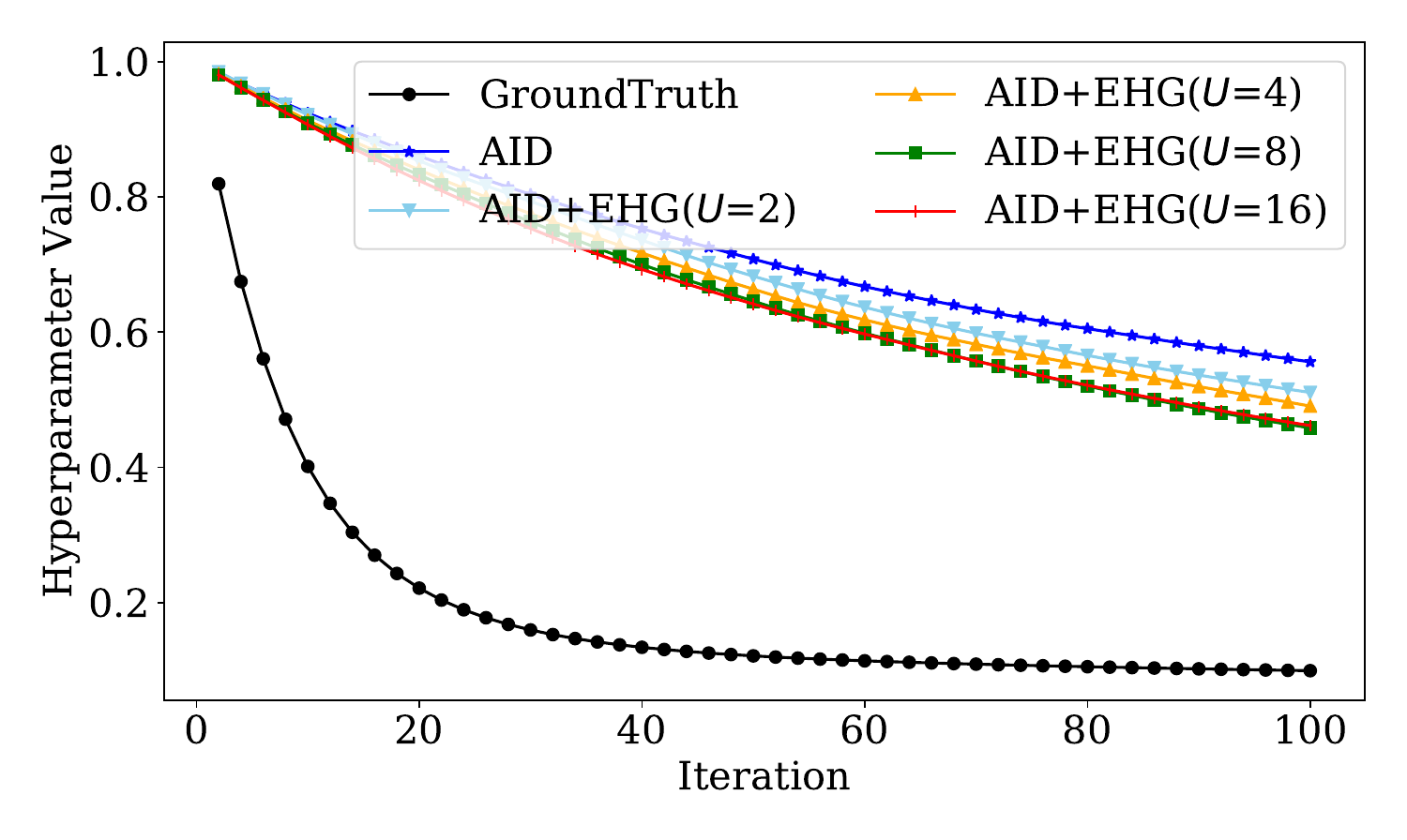}}
	\caption{\textbf{Left:} Illustration of the HPO process of AID and AID+EHG ($U=16$). Specific AID is AID-FP. The curves and shaded regions represent the mean and standard deviation calculated from 10 repeated experiments. \textbf{Right:} Illustration of the HPO process of EHG under AID. GroundTruth curve is calculated by the analytical solution of the lower-level problem. }\label{Figure73-1}
 \vspace{-0.3cm}
\end{figure*}
\begin{figure*}[!t]
	\centering
	\subfigure{\includegraphics[width=0.255\linewidth]{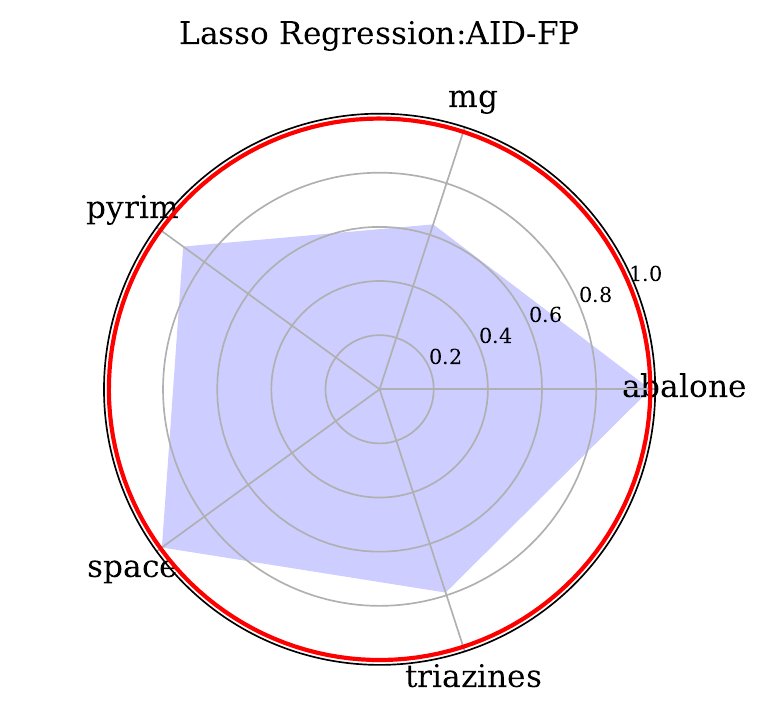}}
	\hspace{-4mm}
	\subfigure{\includegraphics[width=0.255\linewidth]{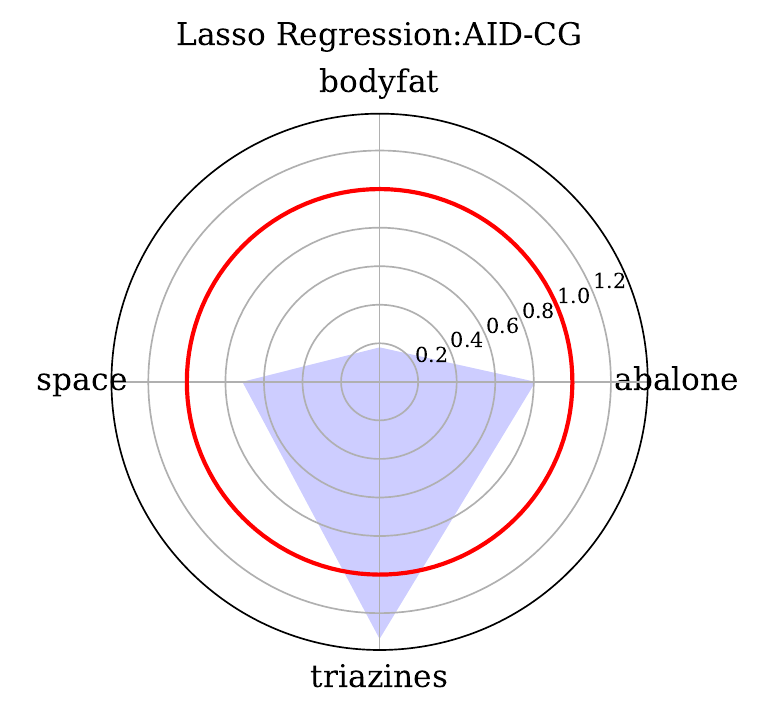}}
	\hspace{-4mm}
	\subfigure{\includegraphics[width=0.255\linewidth]{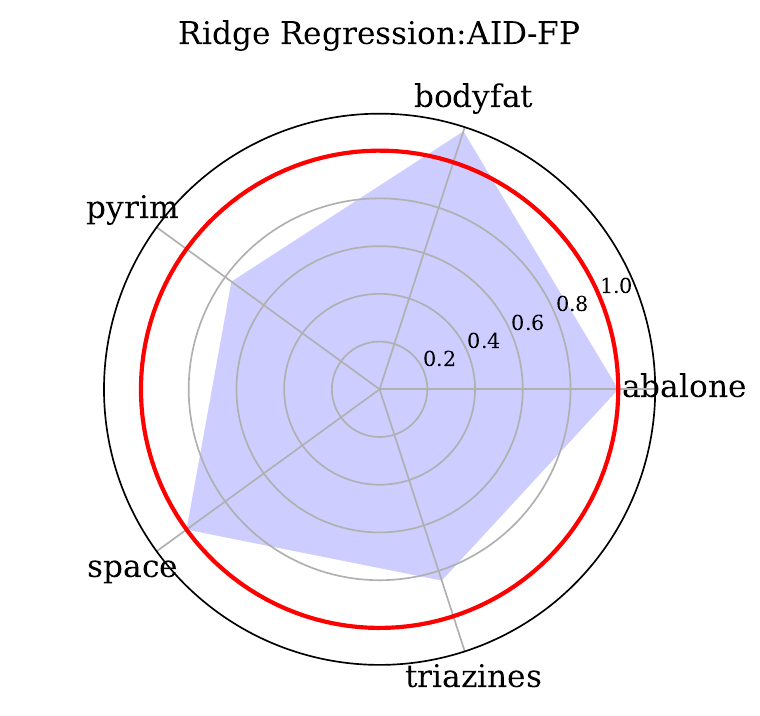}}
	\hspace{-4mm}
	\subfigure{\includegraphics[width=0.255\linewidth]{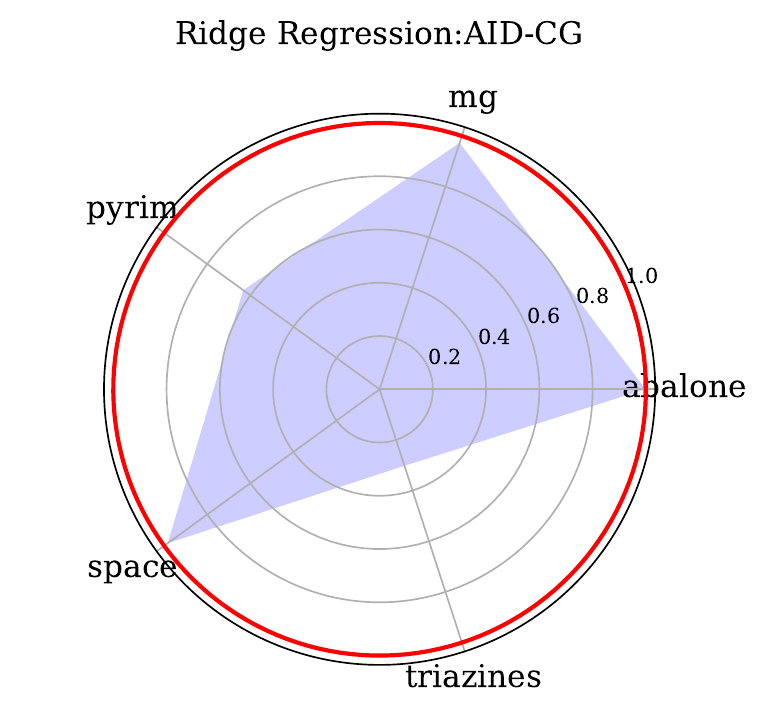}}
	\caption{Error ratio (AID+EHG test error for $U$=5/AID test error for $U$=1) results, where $U$ denotes the number of splittings. Red circle represents the position where they are equal. Inside red circle indicates that the test error of AID+EHG is smaller than that of AID, and vice versa. 4 figures correspond to different cases: lasso regression with AID-FP, lasso regression with AID-CG, ridge regression with AID-FP, ridge regression with AID-CG).}\label{Figure73-2}
 \vspace{-0.3cm}
\end{figure*}
We use the fixed-point (AID-FP) and conjugate gradient (AID-CG) methods \citep{grazzi2020iteration} as our baselines, on the two regression models (lasso regression and ridge regression).

\textbf{Results.} Fig. \ref{Figure73-1} compares the HPO curves of AID under different numbers of data splittings $U$. Fig. \ref{Figure73-1}(left) shows that the EHG effectively reduces hypergradient variance by increasing $U$, whereas AID exhibits higher hypergradient variance. Fig. \ref{Figure73-1}(right) demonstrates that the EHG achieves HPO results closer to the ground-truth by increasing $U$. This improvement can be attributed to the reduction in hypergradient error, leading to a more accurate hypergradient. Additionally, we conduct experimental validation using datasets from real-world scenarios. Fig. \ref{Figure73-2} shows the generalization performance gains of EHG on AID. On most datasets, it is seen that AID+EHG outperforms AID in terms of generalization performance. This experimentally suggests that EHG, by reducing hypergradient variance, achieves a more accurate hypergradient, thereby improving the model's generalization performance.
\begin{table*}[!t]\small
	\caption{Test loss (MSE) of all competing methods under lasso and ridge regression. The best results are in bold. Last row shows the ranking of each approach. ``NaN" indicates an abnormal result.}
	\centering
	\scalebox{1}{\begin{tabular}{cllllll}
			\toprule
			Model & Dataset & RHG & T-RHG & AID-FP & AID-CG & OEHG \\
			\midrule
            \multirow{7}*{\rotatebox{90}{Lasso Regression}}& abalone & 6.46 & 6.98 & 6.42 & 8.21  & \textbf{5.06}\\
			&bodyfat & 8.19e-04 & 1.70e-03 & 8.14e-02 & 9.83e-02 & \textbf{4.72e-05}\\
			&mg & 2.37e-02 & 2.66e-02 & 4.95e-02 & 4.62e-02 & \textbf{2.28e-02}\\
			&pyrim & 3.07e-02 & 5.17e-02 & \textbf{1.50e-02} & NaN & 1.58e-02\\
			&space & \textbf{2.49e-02} & 2.73e-02 & 3.12e-02 & 3.76e-02 & 2.52e-02 \\
			&triazines & 2.24e-02 & 2.77e-02 & 2.90e-02 & 2.47e-02 &  \textbf{1.39e-02} \\
			\cmidrule{2-7}
			&Rank & 2.17	&3.50&	3.50&	4.40&	\textbf{1.33} 
			\\
			\midrule
            \multirow{7}*{\rotatebox{90}{Ridge Regression}}&abalone & 6.47& 6.98&6.42&6.44&\textbf{5.01}\\
			&bodyfat & 1.35e-03 & 6.35e-03&1.44e-02&NaN&\textbf{4.57e-05} \\
			&mg & 4.80e-02 & 5.09e-02 & 5.10e-02 & 5.08e-02 & \textbf{2.25e-02}\\
			&pyrim & 3.08e-02 & 4.86e-02 & 3.38e-02 & 4.37e-02 &\textbf{9.38e-03}\\
			&space & 4.02e-02 & 3.89e-02 & 4.18e-02 & 4.18e-02 &  \textbf{2.55e-02}\\
			&triazines & 2.08e-02 & 2.62e-02 & 2.23e-02 & 4.76e-02 & \textbf{1.29e-02}\\
			\cmidrule{2-7}
			&Rank &2.50&	3.83&3.50&3.90&	\textbf{1.00}
			\\
			\bottomrule
	\end{tabular}}
	\label{Table71-2}
\end{table*}
\subsection{Low-Dimensional HPO}\label{section71}
In this section, we validate the effectiveness of OEHG on the same experimental setup in Section \ref{section71v2}.

\textbf{Comparison methods.} (1) \textbf{Reverse Hypergradient (RHG) \citep{franceschi2017forward}} is based on the reverse-mode differentiation technique, which allows for efficient computation of hypergradient. The method involves computing the gradient of the validation error w.r.t. the output of each iteration of the learning algorithm, and then using the chain rule to compute the hypergradient. (2) \textbf{Truncated-RHG (T-RHG) \citep{shaban2019truncated}}  is a truncated back-propagation method that uses a fixed number of iterations to approximate the gradient of the inner-level optimization problem. (3) \textbf{AID-FP \citep{grazzi2020iteration}} represents a specific instance of the implicit differentiation method. It uses the fixed-point method as the solver for the linear system involved in the computation. (4) \textbf{AID-CG \citep{grazzi2020iteration}} corresponds to a specific instance of the implicit differentiation method. It uses the conjugate gradient method as the solver for the linear system involved in the computation. 
\begin{table*}[!t]\small
\caption{Test loss (binary cross-entropy/hinge) and test accuracy ($\%$) of all competing methods under logistic regression/support vector machine. The best results are in bold. Last row shows the ranking of each approach.}\label{tab4-1-2}
\centering
\resizebox{\textwidth}{!}{
\begin{tabular}{lllllllll}
\toprule
&&\multicolumn{4}{c}{Logistic Regression}&\multicolumn{3}{c}{Support Vector Machine}\\
\midrule
Dataset&& RHG & T-RHG & AID-CG & OEHG (ours)& RHG & T-RHG & OEHG (ours) \\
\midrule
\multirow{2}{*}{a1a}& Loss & 0.5531 & 0.5894 & 0.3818  & \textbf{0.3395}& 0.3935 &0.5284 & \textbf{0.3699}  \\
& Acc. & 75.95 & 75.96 & 82.73 & \textbf{84.20}&83.28&75.95&\textbf{84.23}\\
\midrule
\multirow{2}{*}{a2a}& Loss & 0.5537 & 0.5767 & 0.3864  & \textbf{0.3348}& 0.4228 & 0.4788 & \textbf{0.3655} \\
& Acc. & 76.01 & 76.16 & 82.67 & \textbf{84.63}&83.28&76.01&\textbf{84.23}\\
\midrule
\multirow{2}{*}{a3a}& Loss & 0.5536 & 0.5634 & 0.3841  & \textbf{0.3336} & 0.3917 & 0.5415 & \textbf{0.3653} \\
& Acc. & 75.94 & 76.02 & 82.62 & \textbf{84.48}&83.16 &75.94& \textbf{84.32}\\
\midrule
\multirow{2}{*}{a4a}& Loss & 0.5521 & 0.5890 & 0.3830  & \textbf{0.3320}& 0.3878 & 0.4858 & \textbf{0.3623}\\
& Acc. & 76.05 & 76.07 & 82.58 & \textbf{84.40}& 83.25&76.05& \textbf{84.38}\\
\midrule
\multirow{2}{*}{a5a}& Loss & 0.5368 & 0.5444 & 0.3820  & \textbf{0.3304}& 0.4197 & 0.5156 & \textbf{0.3600}\\
& Acc. & 76.01 & 76.03 & 82.73 & \textbf{84.63}&80.98 &76.01& \textbf{84.51}\\
\midrule
\multirow{2}{*}{a6a}& Loss & 0.5571 & 0.5483 & 0.3828  & \textbf{0.3270}& 0.3917 & 0.5130 & \textbf{0.3564}\\
& Acc. & 75.87 & 75.88 & 82.57 & \textbf{84.83}&83.02 &75.87& \textbf{84.62}\\
\midrule
\multirow{2}{*}{a7a}& Loss & 0.5425 & 0.4759 & 0.3819 & \textbf{0.3244}& 0.3844 & 0.4939 & \textbf{0.3528}\\
& Acc. & 76.17 & 76.17 & 82.70 & \textbf{85.05}&83.36 &76.17& \textbf{84.82}\\
\midrule
\multirow{2}{*}{a8a}& Loss & 0.5486 & 0.5487 & 0.3792  & \textbf{0.3194}& 0.3794 & 0.5302 & \textbf{0.3462}\\
& Acc. & 76.33 & 76.29 & 83.13 & \textbf{85.29}&83.84 &76.33& \textbf{85.01}\\
\midrule
\multirow{2}{*}{a9a}& Loss & 0.5485 & 0.5758 & 0.3789  & \textbf{0.3248}& 0.4120 & 0.5096 & \textbf{0.3528}\\
& Acc. & 76.38 & 76.40 & 82.72 & \textbf{85.19}&81.07 &76.38&\textbf{84.94}\\
\midrule
\multirow{2}{*}{diabetes}& Loss & 0.6409 & 0.6458 & 0.6054  & \textbf{0.4831}& 0.6434 & 0.7472 & \textbf{0.5244}\\
& Acc. & 68.38 & 68.38 & 70.94 & \textbf{77.28}& 69.87&68.38& \textbf{78.13}\\
\midrule
\multirow{2}{*}{gisette}& Loss & 0.3766 & 0.4951 & 0.3821  & \textbf{0.1056}& 0.2421 & 0.3896 & \textbf{0.0879}\\
& Acc. & 89.82 & 85.49 & 89.49 & \textbf{97.30}& 92.66&85.92& \textbf{97.60}\\
\midrule
\multirow{2}{*}{heart}& Loss & 0.6741 & 0.6136 &  0.3898  & \textbf{0.3635}& 0.4521 & 0.4297 & \textbf{0.4121}\\
& Acc. & 61.96 & 80.00 & \textbf{87.65} & 87.06& \textbf{86.86}&85.29& 84.90\\
\midrule
\multirow{2}{*}{ionosphere}& Loss & 0.6650 & 0.6445 & 0.4230  & \textbf{0.3107}& 0.6384 &0.6443 & \textbf{0.2700}\\
& Acc. & 62.91 & 62.91 & 83.66 & \textbf{89.62} & 70.42&71.08& \textbf{89.62}\\
\midrule
\multirow{2}{*}{w1a} &Loss& 0.4792 & 0.5027 & 0.3978  & \textbf{0.1076}& 0.0899 & 0.2965 & \textbf{0.0592}\\
& Acc. & 97.02 & 97.03 & 97.02 & \textbf{97.10}&97.02 &97.02& 97.02\\
\midrule
\multirow{2}{*}{\textbf{Rank}}& Loss & 3.21 & 3.71 & 2.07 & \textbf{1.00}& 1.93 & 2.93 & \textbf{1.00}\\
& Acc. & 3.64 & 3.18 & 2.11 & \textbf{1.07}&2.00 &2.79&\textbf{1.21 }\\
\bottomrule
\end{tabular}}
\end{table*}

\textbf{Implementation details.} We train all methods using stochastic gradient descent (SGD) optimizer for parameter and Adam \citep{kingma2014adam} optimizer for hyperparameter. We set the outer iteration $T$ as 10000 and the inner iteration $K$ as 128 for all the compared methods. For OEHG, the size of data splitting $U$ is 5. To verify the consistent superiority of our method, each reported result is an average of over 5 repeated runs.

\textbf{Results of Text Regression Task.} Table \ref{Table71-2} evaluates the test loss of different gradient-based HPO methods under lasso regression and ridge regression models. For lasso regression, we achieve the lowest test error on most datasets, resulting in the smallest rank (1.33). For ridge regression, we obtain the best generalization performance across all datasets.  This demonstrates that OEHG can achieve better hyperparameter by reducing hypergradient variance, thereby improving generalization performance. Additionally, AID-CG encounters training failures in two scenarios (lasso regression+pyrim and ridge regression+bodyfat), which could be rationally explained by the ill-conditioned nature of Eq. (\ref{equation22-A1}), as summarized by \cite{liu2021investigating}.

\textbf{Results of Text Classification Task.} \label{section6-2-2} Tables \ref{tab4-1-2} evaluates the test metrics under logistic regression (LR) and support vector machine (SVM). We omit results of AID due to the potential ill-conditioning of the linear equations that AID requires solving  \citep{grazzi2020iteration}, as discussed in the previous section. From the table, we can observe that OEHG achieves the highest ranking on 14 datasets. This experimentally demonstrates that OEHG is effective in improving generalization performance by reducing hypergradient variance and the conclusion will be further substantiated in Section \ref{section8}. 
\subsection{High-dimensional HPO}
In this section, we experimentally demonstrate that the proposed OEHG helps improve hypergradient estimation across different HPO problems, including regularization parameter learning, data hyper-cleaning and few-shot learning. We primarily compare ITD that are closely related to OEHG, such as RHG and T-RHG.

\subsubsection{Optimizing regularization parameter for text classification}\label{sec4-2-1}
\textbf{Task Formulation.} Following \cite{snoek2012practical} and \cite{maclaurin2015gradient}, we propose setting a separate regularization hyperparameter for each parameter. Specifically, we solve a problem of the form $\min_{\boldsymbol{\theta}\in\mathbb{R}^r}\{\mathcal{L}(\boldsymbol{\theta}, \mathcal{D}^{tr})+
\sum_{i=1}^r\Vert \lambda_i\cdot\theta_i\Vert_2^2\}$, where $r$ is the number of model parameter. 

\begin{table*}[!t]
\caption{Test loss and test accuracy ($\%$) of high-dimensional regularization parameter under support vector machine. The best results are in bold.}
\centering
\resizebox{\textwidth}{!}{
\begin{tabular}{lllllllllllll}
\toprule
{Dataset} & \multicolumn{2}{c}{a1a}&\multicolumn{2}{c}{diabetes}&\multicolumn{2}{c}{gisette}&\multicolumn{2}{c}{heart}&\multicolumn{2}{c}{ionosphere}&\multicolumn{2}{c}{w1a}\\
\midrule
Metric&Loss&Acc.&Loss&Acc.&Loss&Acc.&Loss&Acc.&Loss&Acc.&Loss&Acc.\\
\midrule
RHG($K$=64)& 0.7866 & 76.75& 0.6220 & 68.55& 0.3225 & 90.55& 0.4481& 82.00& 0.5781 & 76.82 & 0.1944 & 97.03\\
\midrule
RHG($K$=128)& 1.0941 & 79.73& 0.6159 & 72.82& 0.2187 & 93.87& 0.6043& 79.41& 0.4870 & 83.31& 0.1051 & 97.07\\
\midrule
RHG($K$=256)& 0.3875 & 83.32& 0.6831 & 70.51& 0.1800 & 95.52& 0.8222 & 71.18& 0.5349 & 79.87& 0.0616 & 96.98\\
\midrule
OEHG($K$=1)& \textbf{0.3743} & \textbf{83.97}& \textbf{0.5404} & \textbf{78.21}& \textbf{0.0830} & \textbf{97.70}& \textbf{0.4287} & \textbf{84.71}& \textbf{0.2648}& \textbf{89.40}& \textbf{0.0575} & \textbf{97.12} \\
\bottomrule[1pt]
\end{tabular}}
\label{Table721-2}
\end{table*}
\textbf{Results.} We use the datasets presented in Table \ref{Table71-1}. Table \ref{Table721-2} evaluates the test performance on SVM model. It can be observed that:
(1) OEHG achieves the best generalization results, demonstrating that OEHG can obtain better hyperparameter via reducing hypergradient variance, thereby improving generalization performance. 
(2) Compared to low-dimensional HPO experiments, RHG method reduces generalization performance on some datasets (e.g., a1a and heart). This decline should rationally be attributed to the increased number of hyperparameter, which complicates the HPO process. In contrast, OEHG does not show decline, indicating the algorithm's scalability.
(3) Across different datasets, it is evident that the optimal number of inner iterations $K$ for generalization varies considerably. E.g., for heart dataset, the accuracy drops by more than 10 percentage points when $K$=256 compared to $K$=64. Therefore, It's necessary to set $K$ separately for each task in practice if using RHG method.  In contrast, OEHG alleviates this issue by online updates.
    \begin{figure*}[!t]
			\centering
			\subfigure{\includegraphics[width=0.33\linewidth]{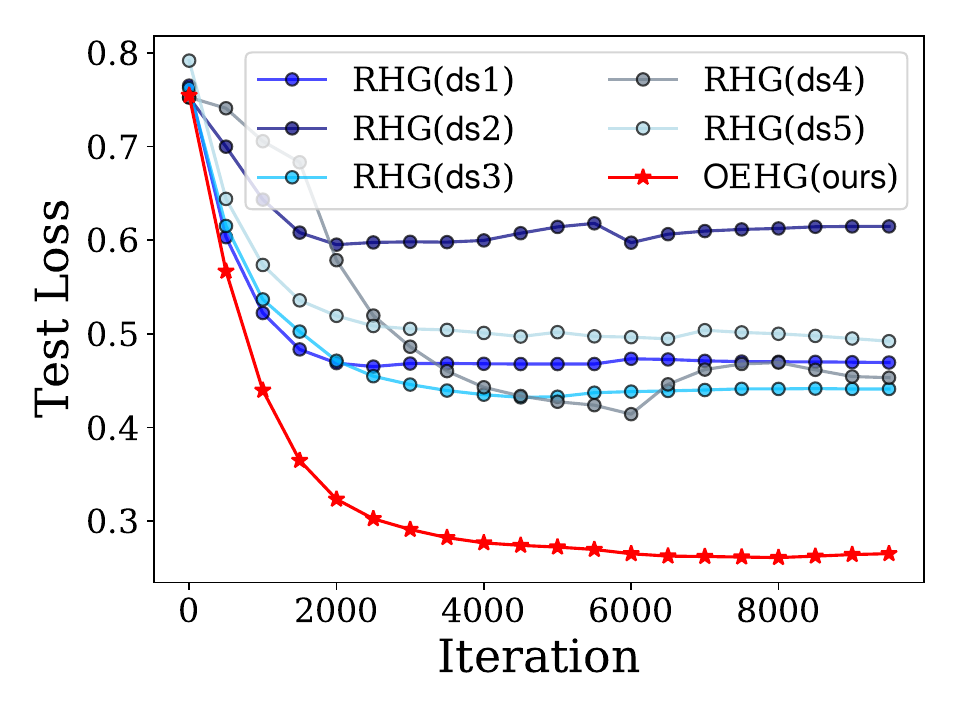}
			}
			\hspace{-0.5cm}
			\subfigure{\includegraphics[width=0.33\linewidth]{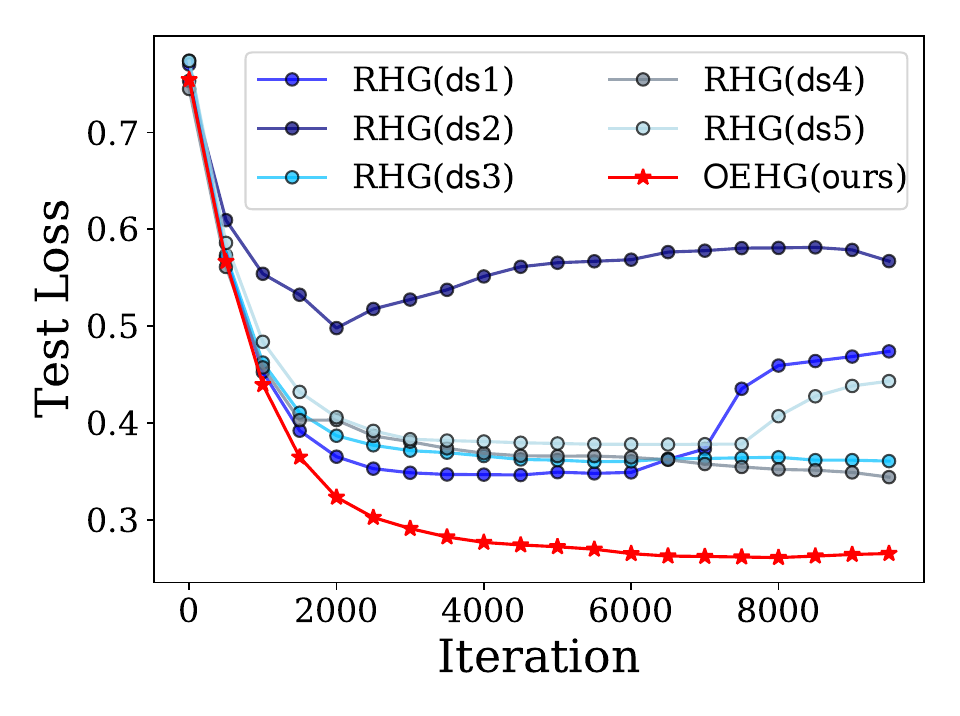}
			}
			\hspace{-0.5cm}
			\subfigure{\includegraphics[width=0.33\linewidth]{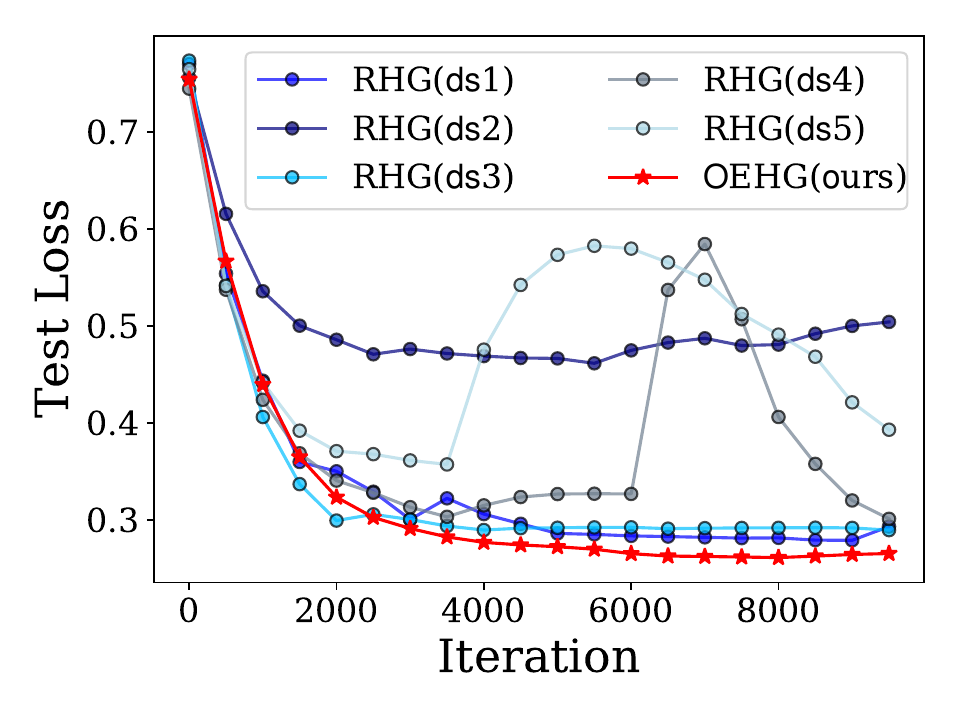}
			}
			\caption{Test loss curves of RHG and OEHGs under SVM model on different data splittings (ds1-ds5), 3 figures correspond to the different numbers of inner iterations: $K$=64 (left), $K$=128 (middle), and $K$=256 (right).}
			\label{Figure721-2}
    \vspace{-0.3cm}
		\end{figure*}
		
		\begin{figure*}[!t]
			\centering
			\subfigure{\includegraphics[width=0.33\linewidth]{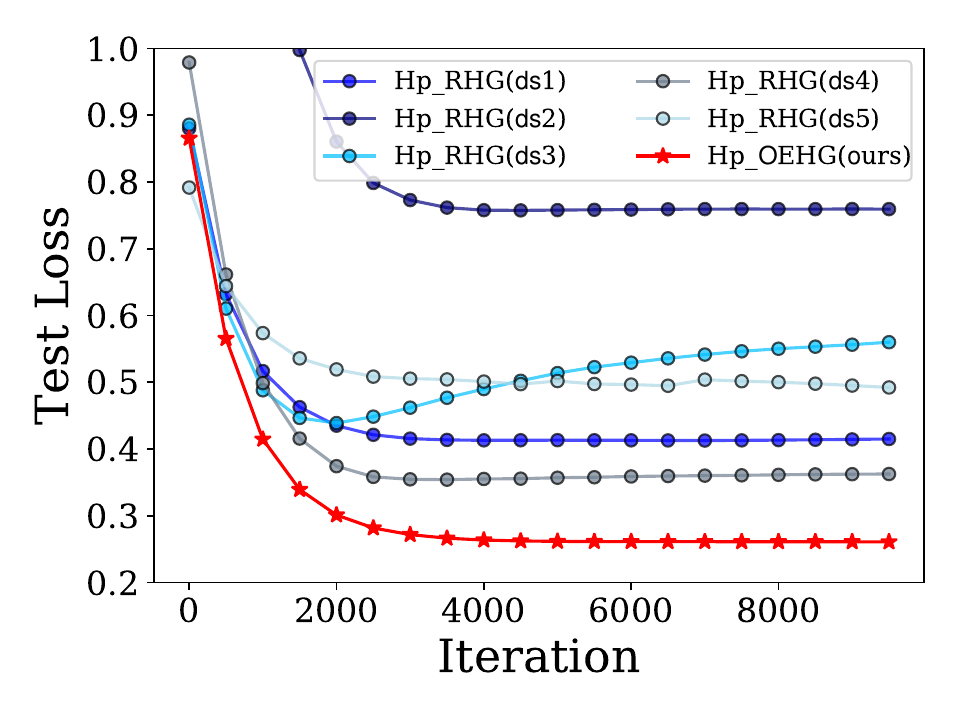}
			}
		\hspace{-0.5cm}
			\subfigure{\includegraphics[width=0.33\linewidth]{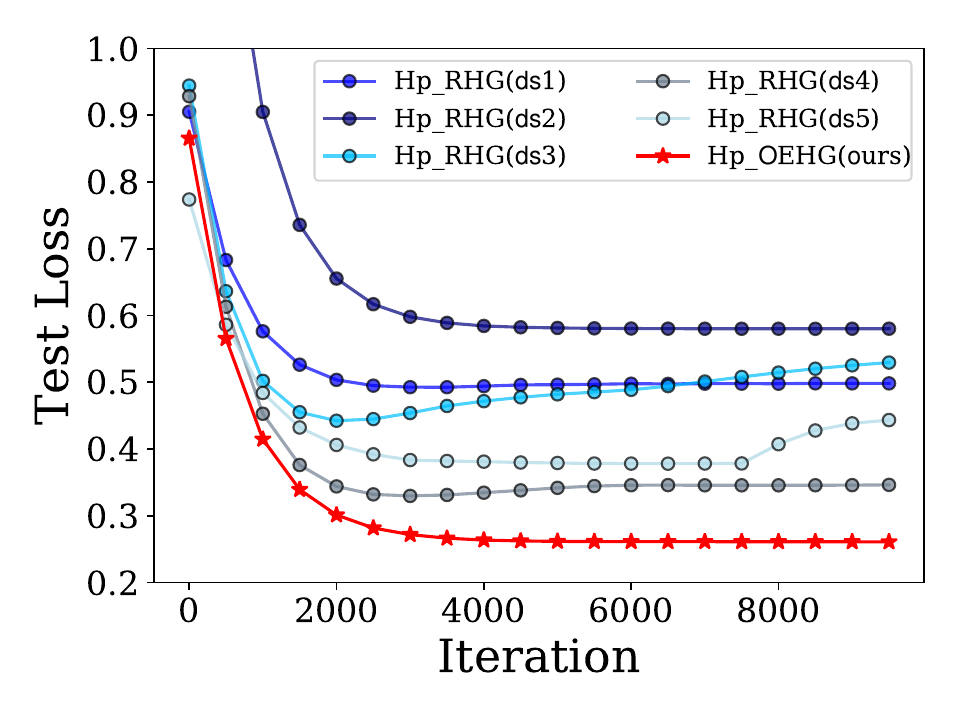}
			}
		\hspace{-0.5cm}
			\subfigure{\includegraphics[width=0.33\linewidth]{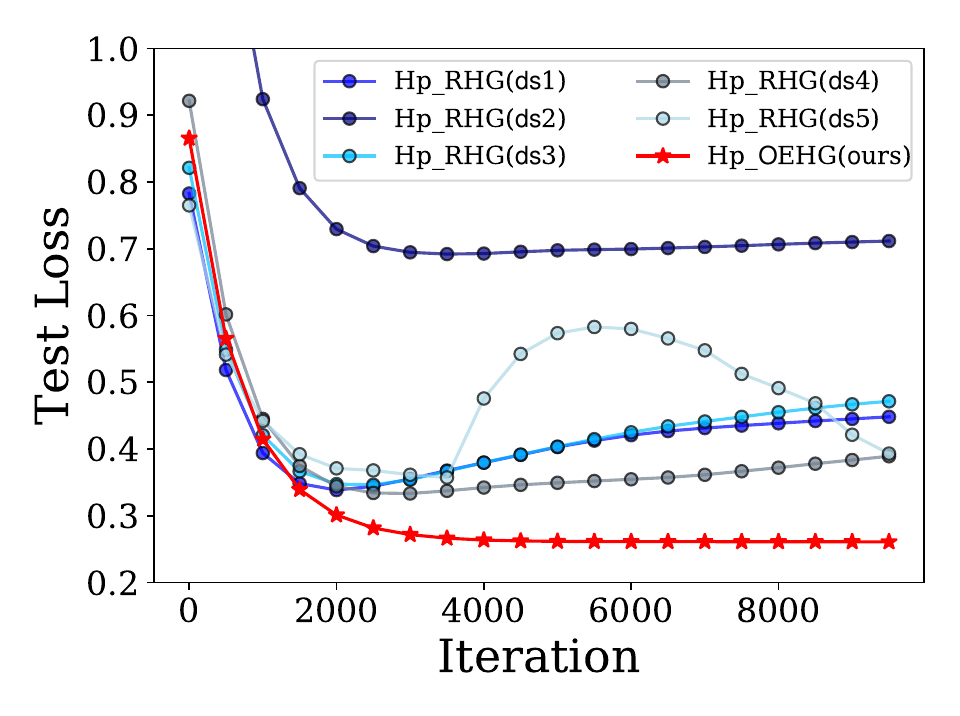}
			}
			\caption{Test loss curves for models retrained with the final hyperparameter via RHG and OEHGs on different data splittings. 3 figures correspond to the different numbers of inner iterations: $K$=64 (left), $K$=128 (middle), and $K$=256 (right).}
			\label{Figure721-3}
    \vspace{-0.3cm}
		\end{figure*}
\noindent\textbf{Discussion:} To gain a deeper understanding of OEHG, we conduct detailed experiments to analyze OEHG. Fig. \ref{Figure721-2} presents the test loss of 5 splittings with SVM on the ionosphere dataset and shows: (1) The generalization performance of RHG varies significantly across different data splittings. It also underscores the necessity of introducing variance in hypergradient error estimation. (2) OEHG reduces the hypergradient variance via multiple splittings and ultimately improves generalization performance. As for the specific relationship between hypergradient variance and generalization performance, we will make an analysis and discussion in Section \ref{section8}.

Furthermore, we retain the values of hyperparameter after running HPO algorithms and retrain the model. Fig. \ref{Figure721-3} presents the curves of test loss during model training with a fixed hyperparameter. Notably: (1) Compared with Fig. \ref{Figure721-2}, the test loss exhibits smoother variations, while overfitting or underfitting persists. This shows the disparity of hyperparameter gained via different data splittings, which can originate from the cascading of hypergradient during the HPO process. (2) The generalization performance of OEHG outperforms RHG. This suggests that hyperparameter by OEHG are better than RHG. This improvement arises from ensemble gradient across splittings, which can reduce the hypergradient error.
		\begin{figure*}[!t]
			\centering
			\subfigure{\includegraphics[width=0.47\linewidth]{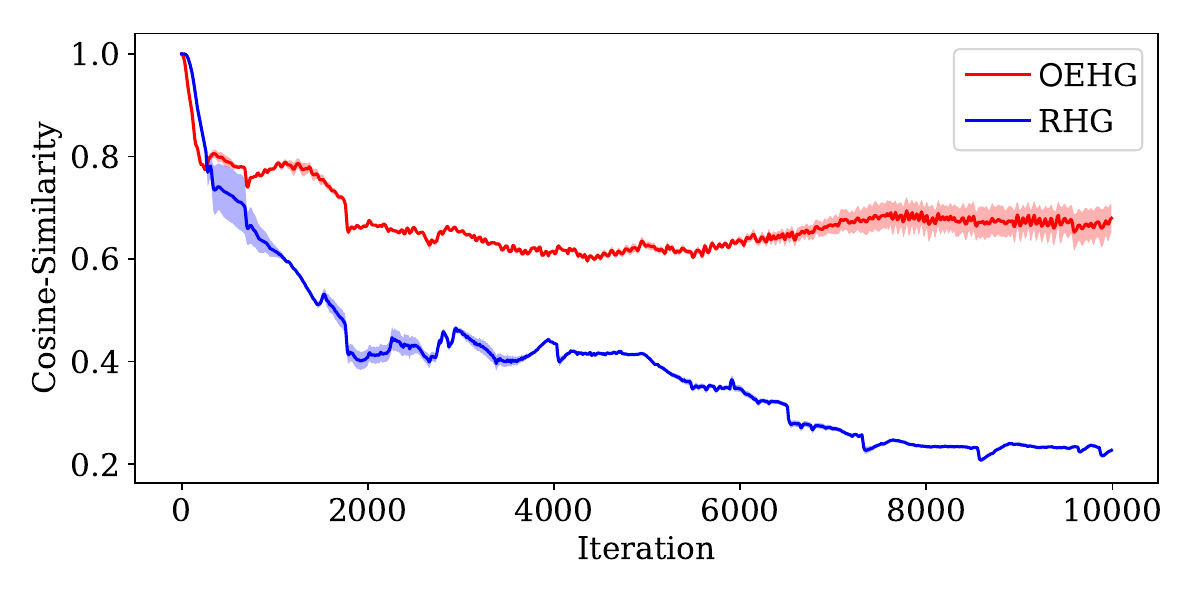}
			}
			\subfigure{\includegraphics[width=0.47\linewidth]{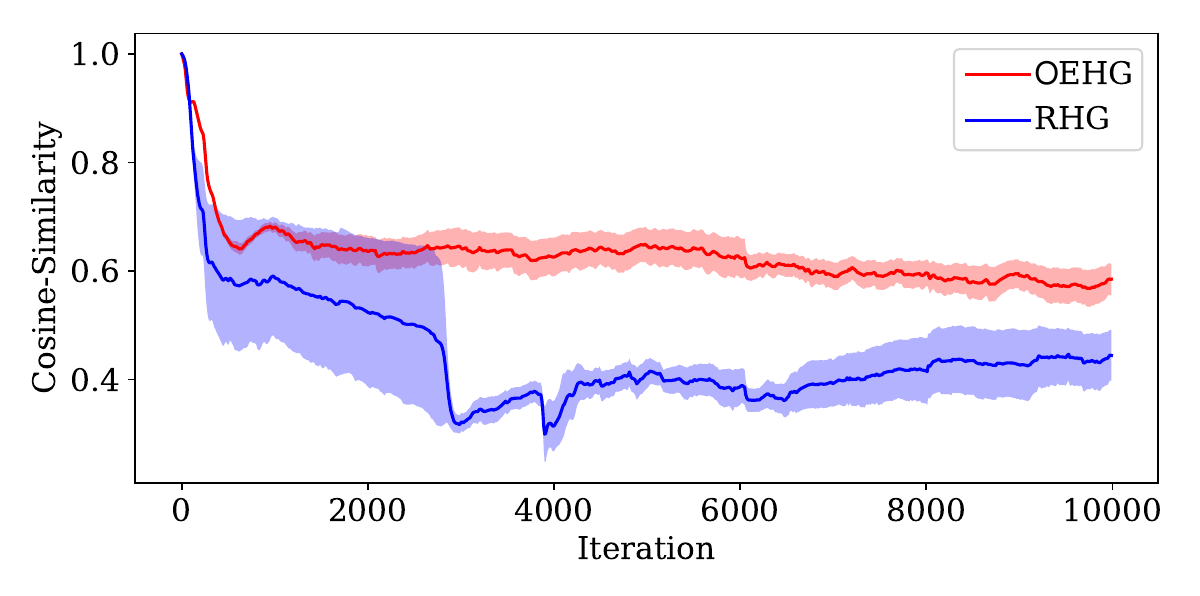}
			}
			\caption{Cosine similarity between the hyperparameter by RHG (or OEHG) algorithm and the ground-truth one, which uses test data for HPO. Left: heart dataset. Right: ionosphere dataset.}
			\label{Figure721-4}
    \vspace{-0.3cm}
		\end{figure*}		
		
To corroborate the assertion, Fig. \ref{Figure721-4} presents cosine similarity between the hyperparameter by OEHG and RHG algorithms relative to the `test-truth' hyperparameter. Here, the `test-truth' hyperparameter refers to those computed via test data, as test data represents the generalization target. Fig. \ref{Figure721-4} shows that the curve of OEHG consistently lies above the curve of RHG, signifying that OEHG's hyparameter exhibits a closer alignment with the `test-truth' hyparameter. This discrepancy arises because OEHG employs an ensemble approach, thereby reducing the hypergradient variance and improving the accuracy of hyperparameter update directions.
\begin{table*}[!t]\small
\caption{Test accuracy ($\%$) and test loss of competing methods on MNIST, Fashion-MNIST and CIFAR-10 under different model architectures. The bset results are in bold. (SR: Softmax Regression, MLP: Multilayer Perceptron.)}\label{Table721-1}
\centering
\resizebox{0.7\textwidth}{!}{
\begin{tabular}{llllll}
\toprule
Dataset& Model&Metric & RHG  & T-RHG  & OEHG (ours)
\\
\midrule
\multirow{4}{*}{MNIST}&\multirow{2}{*}{SR}&Acc. $\uparrow$&87.15&84.37&\textbf{90.73}
\\
&&Loss $\downarrow$&0.5587&0.7911&\textbf{0.3256}
\\
\cmidrule{2-6}
& \multirow{2}{*}{MLP} & Acc. $\uparrow$&87.49&87.84&\textbf{93.40}
\\
&& Loss $\downarrow$ &0.5001&0.5865&\textbf{0.2310}
\\
\midrule
\multirow{4}{*}{F-MNIST}&\multirow{2}{*}{SR}&Acc. $\uparrow$&79.23&75.30&\textbf{84.26}
\\
&&Loss $\downarrow$&0.6767&0.8184&\textbf{0.4599}
\\
\cmidrule{2-6}
& \multirow{2}{*}{MLP} & Acc. $\uparrow$&79.23&79.04&\textbf{85.03}
\\
&& Loss $\downarrow$ &0.6356&0.6535&\textbf{0.4374}
\\
\midrule
\multirow{4}{*}{CIFAR-10}&\multirow{2}{*}{SR}&Acc. $\uparrow$&31.76&32.59&\textbf{38.88}
\\
&&Loss $\downarrow$&1.9497&1.9686&\textbf{1.7693}
\\
\cmidrule{2-6}
& \multirow{2}{*}{MLP} & Acc. $\uparrow$&35.02&35.60&\textbf{39.96}
\\
&& Loss $\downarrow$ &1.9024&1.8475&\textbf{1.7297}
\\
\bottomrule
\end{tabular}}
\end{table*}
		
\subsubsection{Optimizing regularization parameter for image classification}
\textbf{Experimental Setup.} The task formulation is the same as Section \ref{sec4-2-1}. We evaluate all algorithms on the MNIST dataset \citep{lecun1998gradient} following \citep{snoek2012practical}. In addition, we also conduct experiments on two common datasets, Fashion-MNIST \citep{xiao2017fashion} and CIFAR-10 \citep{krizhevsky2009learning}. MNIST and Fashion-MNIST consist of grayscale images of size $28\times28$. CIFAR-10 consists of color images of size $32\times32$ and contains 10 classes. We construct a subset of the above datasets with $10000$ examples as the observed set $\mathcal{D}$ and a test set with $10000$ examples. We set the splitting ratio of training and validation $\gamma$ as 0.2, i.e., $\mathcal{D}^{tr}_{u_i}$ and $\mathcal{D}^{val}_{u_i}$ consist $8000$ and $2000$ examples for each $\mathcal{S}_{(\mathcal{D}, u_i)}$, where $i=1, 2, \dots, U$. For model training, we use an SGD optimizer with learning rates of 0.01 and 0.05 on CIFAR-10 and MNIST (Fashion-MNIST), respectively. For HPO, we use Adam with a learning rate of 0.01.

\textbf{Results.} Table \ref{Table721-1} evaluates the generalization performance of different HPO methods on MNIST, Fashion-MNIST, and CIFAR-10. As can be seen, OEHG significantly outperforms other methods, demonstrating its robustness across different image classification tasks and model architectures. Compared to RHG and its variant T-RHG, OEHG improves the generalization performance via reducing hypergradient variance. The relation of hypergradient variance and generalization will be further analyzed in Section \ref{section8}.

Fig. \ref{Figure721-1} shows the learned parameter of the softmax regression model on MNIST. Because each parameter corresponds to a particular input, this regularization scheme can be seen as a generalization of automatic relevance determination \citep{mackay1994automatic}. From Fig. \ref{Figure721-1}, we can observe that, relative to the baseline without regularization, the parameter obtained via OEHG are closer to zero in the edge regions, which shows that OEHG applies stronger regularization to these areas to alleviate overfitting. Additionally, OEHG results in smaller parameter values at certain positions within the central region, allowing the model to better capture general recognition patterns. Therefore, OEHG can provide clearer outlines of handwritten digits ($0-9$) like Fig. \ref{Figure721-1} shows.
\begin{figure*}[!t]
\centering
\subfigure{\includegraphics[width=1.05\linewidth]{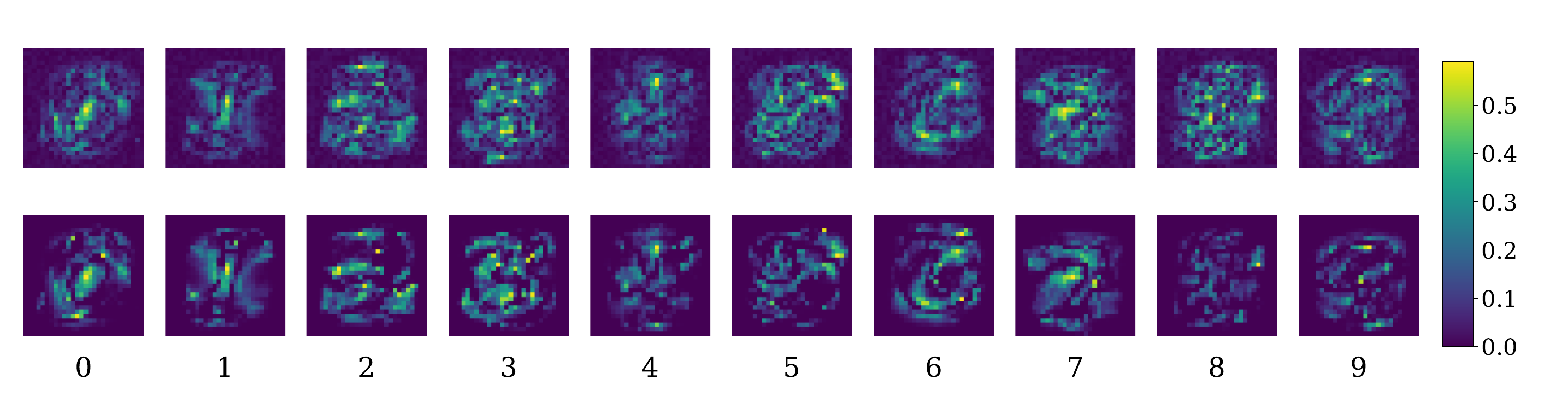}
}
\caption{Visualization of softmax regression parameter under MNIST dataset. \textbf{Top:} the absolute values of training model parameter without regularization. \textbf{Bottom:} the absolute values of training model parameter using OEHG.}
\label{Figure721-1}
\vspace{-0.3cm}
\end{figure*}
\subsubsection{Data Hyper-Cleaning}
\textbf{Task Formulation.} Assuming that some labels in the dataset are contaminated, data hyper-cleaning (DHC) \citep{franceschi2017forward} aims to reduce the impact of incorrect samples by adding hyperparameter to label the corrupted data. Following the classical experimental protocol \citep{franceschi2017forward}, we use cross-entropy as loss function $\mathcal{L}$, and the outer-level and inner-level subproblems are with the forms as
\begin{align} \label{eq422-1}
    \arg\min_{\boldsymbol{\lambda}\in[0,1]^{N^{tr}}}\mathcal{L}^{val}(\boldsymbol{\theta}(\boldsymbol{\lambda}; \mathcal{D}^{tr});\mathcal{D}^{val}),\quad\text{\textit{s.t.}} \quad\boldsymbol{\theta}(\boldsymbol{\lambda}; \mathcal{D}^{tr})=\arg\min_{\boldsymbol{\theta}\in\mathbb{R}^r}\{\boldsymbol{\lambda}\cdot\mathcal{L}^{tr}(\boldsymbol{\theta}; \mathcal{D}^{tr})\},
\end{align}
where $\boldsymbol{\lambda}=[\lambda_1,\lambda_2,\dots,\lambda_{N^{tr}}]$ are sample weights imposed on all training samples in $\mathcal{D}^{tr}$.

\textbf{Experimental Setup.} We use MNIST and Fashion-MNIST following \citep{franceschi2017forward,shaban2019truncated, liu2021towards}. We randomly select 8000, 2000, and 10000 examples for training, validation, and testing, respectively. The label of a training sample is replaced by a uniformly sampled wrong label with a probability of 0.5. The parameter $\boldsymbol{\theta}$ in Eq. (\ref{eq422-1}) represents the parameter in softmax regression and MLP of size 784$\rightarrow$300$\rightarrow$10, respectively.

\textbf{Implementation details.} We use SGD optimizer for model training. For OEHG and EHG, we set the size of splittings as 5, and the approach of constructing data splittings involves incorporating a part of the training data that is considered clean by the model into validation data during the training process. For other comparison methods, we followed the implementation details of the released codes by original literature authors.

\textbf{Results.} Table \ref{Table722-1} presents the performance metrics of different competing methods on the MNIST and Fashion-MNIST datasets. It can be easily observed that:
(1) OEHG consistently outperforms RHG and T-RHG methods in terms of both test accuracy and F1 score, which measures the quality of the data cleaner.  (2) For state-of-the-art methods such as IAPTT-GM \citep{liu2021towards} and VPBGD \citep{shen2023penaltybased}, we integrate EHG into these HPO methods, resulting in consistent improvements. This suggests that EHG can improve generalization by affecting the update directions. 
\begin{table*}[!t]\small
\caption{Test Acc. ($\%$) and F1 score under different model architectures. F1 score measures the quality of the data cleaner \citep{franceschi2017forward}. The best results are in bold. (SR: Softmax Regression, MLP: Multilayer Perceptron.)}
\centering
\resizebox{0.7\textwidth}{!}{
\begin{tabular}{lllll}
\toprule
&\multicolumn{2}{c}{SR} & \multicolumn{2}{c}{MLP}\\
\cmidrule{2-5}
Method& Test Acc. & F1 score & Test Acc. & F1 score \\
\midrule
RHG \citep{franceschi2017forward}& 84.83  & 88.45 & 85.91 & 89.30 \\
T-RHG \citep{shaban2019truncated} & 84.81 & 88.11 & 85.96 & 88.84\\
OEHG (ours) & \textbf{89.20} & 89.94 & 91.24 & 90.89\\
\midrule
Dirty Tr  & 81.52 & -- & 74.21 & --\\
RHG-Weight Tr& 85.14 & -- & 89.52 & -- \\
T-RHG-Weight Tr& 85.05 & -- & 89.62 & -- \\
OEHG-Weight Tr (ours)& 89.05 & -- & 90.52 & -- \\
\midrule
\midrule
VPBGD \citep{shen2023penaltybased} & 85.12 & 88.75 & 91.81 & 90.18\\
 VPBGD(+EHG) & 87.59 & 89.86 & \textbf{93.16} & 90.79 \\
\midrule
Gains & \textcolor{red}{+2.47} & \textcolor{red}{+1.11} & \textcolor{red}{+1.35} & \textcolor{red}{+0.61} \\
\midrule
IPATT-GM \citep{liu2021towards} & 88.47 & 90.73 & 89.47 & 90.83\\
IPATT-GM(+EHG) & 88.68 & \textbf{90.78} & 90.10 & \textbf{91.09}  \\
\midrule
Gains & \textcolor{red}{+0.21} & \textcolor{red}{+0.05} & \textcolor{red}{+0.63} & \textcolor{red}{+0.26} \\
\bottomrule
\end{tabular}}\label{Table722-1}
\end{table*}
		
				
\subsubsection{Few-shot Image Clssification}
The advantage of reducing hypergradient variance can be more pronounced in scenarios with limited data, e.g., few-shot learning \citep{wang2020generalizing}. To validate this, we evaluate the efficacy of the EHG on a few-shot classification benchmark, using ProtoNet \citep{snell2017prototypical} and MetaOptNet \citep{lee2019meta} as the baseline methods.

\textbf{Experimental Setup.} We use a standard 4-layer convolutional network and ResNet-12 in our experiments.
As an optimizer, we use SGD with Nesterov momentum of 0.9 and weight decay of 0.0005. Each mini-batch
consists of 2 episodes. The model was meta-trained for 100
epochs, with each epoch consisting of 1000 episodes. The
learning rate was initially set to 0.1 and then changed to
0.006, 0.0012, and 0.00024 at epochs 20, 40, and 50, respectively, following the practice of \citep{gidaris2018dynamic, lee2019meta}. We use 5-way classification in both meta-training and meta-test stages. Each class contains 15 test (query) samples during meta-training and 15 test samples during meta-testing. For 5-way 1-shot experiments, Our meta-trained model was chosen based on 5-way 1-shot test accuracy on the meta-validation set, and we chose based on 5-way 5-shot test accuracy for 5-way 5-shot experiments. We keep the default hyperparameter setting for the
compared baselines in the original papers. Our implementation is based on the code provided on \href{https://github.com/kjunelee/MetaOptNet.}{https://github.com/kjunelee/MetaOptNet.}

\textbf{Results.} Table \ref{tab433-1} summarizes the results on the 5-way classification tasks with different shots on miniImageNet \citep{vinyals2016matching} and tieredImageNet \citep{ren2018meta} benchmarks. It can be seen that the proposed EHG can also help improve test accuracy in most cases from SOTA baselines method. Moreover, EHG significantly improves the baseline performance when samples are less (tieredImageNet, 1-shot), demonstrating the effectiveness of our strategy. The improvement of generalization can be attributed to the reduction of hypergradient variance, thus obtaining a better meta-knowledge.
\begin{table*}[!t]
\caption{Average few-shot classification accuracies (\%) with 95\% confidence intervals on miniImageNet and tieredImageNet meta-test splits.}
\centering
\resizebox{\textwidth}{!}{
\begin{tabular}{lllll}
\toprule
& \multicolumn{2}{c}{\textbf{miniImageNet 5-way}} & \multicolumn{2}{c}{\textbf{tieredImageNet 5-way}}\\
\cmidrule{2-5}
model & \multicolumn{1}{c}{1-shot} & \multicolumn{1}{c}{5-shot} & \multicolumn{1}{c}{1-shot} & \multicolumn{1}{c}{5-shot}\\
\midrule
\midrule
\textbf{4-layer conv(feature dimension=1600)}\\
ProtoNet \citep{snell2017prototypical} & 50.62 $\pm$ 0.67  & 70.06 $\pm$ 0.54  & 50.36 $\pm$ 0.70  & 69.90 $\pm$ 0.57\\
ProtoNet(+EHG)&52.71 $\pm$ 0.69 \textcolor{red}{2.09$\uparrow$} & 70.71 $\pm$ 0.52 \textcolor{red}{0.65$\uparrow$} & 53.21 $\pm$ 0.72 \textcolor{red}{2.85$\uparrow$} & 71.40 $\pm$ 0.58 \textcolor{red}{1.50$\uparrow$} \\
\midrule
MetaOptNet-RR \citep{lee2019meta} & 51.67 $\pm$ 0.66  & 68.72 $\pm$ 0.54  & 51.74 $\pm$ 0.70  & 69.84 $\pm$ 0.58 \\
MetaOptNet-RR(+EHG) & 52.91 $\pm$ 0.66 \textcolor{red}{1.24$\uparrow$}  & 69.61 $\pm$ 0.54 \textcolor{red}{0.89$\uparrow$} & 54.84 $\pm$ 0.71 \textcolor{red}{3.10$\uparrow$} & 70.95 $\pm$ 0.58 \textcolor{red}{1.11$\uparrow$}\\
\midrule
MetaOptNet-SVM \citep{lee2019meta} & 50.84 $\pm$ 0.65  & 69.67 $\pm$ 0.52  & 50.92 $\pm$ 0.69  & 70.65 $\pm$ 0.58 \\
MetaOptNet-SVM(+EHG) & 52.35 $\pm$ 0.67 \textcolor{red}{1.51$\uparrow$}& 69.69 $\pm$ 0.53 \textcolor{red}{0.02$\uparrow$}& 54.11 $\pm$ 0.74 \textcolor{red}{3.19$\uparrow$} & 71.85 $\pm$ 0.58 \textcolor{red}{1.20$\uparrow$}\\
\midrule
\midrule
\textbf{ResNet-12 (feature dimension=16000)}\\
ProtoNet \citep{snell2017prototypical} & 57.38 $\pm$ 0.70   & 73.80 $\pm$ 0.54  & 57.99 $\pm$ 0.75  & 78.24 $\pm$ 0.59 \\
ProtoNet(+EHG)& 59.27 $\pm$ 0.70 \textcolor{red}{1.89$\uparrow$}& 74.94 $\pm$ 0.53 \textcolor{red}{1.14$\uparrow$}& 63.38 $\pm$ 0.78 \textcolor{red}{5.39$\uparrow$}& 78.93 $\pm$ 0.58 \textcolor{red}{0.69$\uparrow$}\\
\midrule
MetaOptNet-RR \citep{lee2019meta} & 57.60 $\pm$ 0.66  & 74.69 $\pm$ 0.50   & 58.70 $\pm$ 0.75 & 79.21 $\pm$ 0.57 \\
MetaOptNet-RR(+EHG) & 59.27 $\pm$ 0.70 \textcolor{red}{1.67$\uparrow$} & 74.94 $\pm$ 0.53 \textcolor{red}{0.25$\uparrow$} & 63.38 $\pm$ 0.78 \textcolor{red}{4.68$\uparrow$} & 78.93 $\pm$ 0.58 \textcolor{green}{0.28$\downarrow$} \\
\midrule
MetaOptNet-SVM \citep{lee2019meta} & 58.00 $\pm$ 0.67  & 75.35 $\pm$ 0.50    & 58.63 $\pm$ 0.75  & 79.11 $\pm$ 0.56 \\
MetaOptNet-SVM(+EHG) & 58.20 $\pm$ 0.68 \textcolor{red}{0.20$\uparrow$} & 76.50 $\pm$ 0.49 \textcolor{red}{1.15$\uparrow$}& 64.10 $\pm$ 0.77 \textcolor{red}{5.47$\uparrow$} & 80.62 $\pm$ 0.55 \textcolor{red}{1.51$\uparrow$}\\
\bottomrule
\end{tabular}}
\label{tab433-1}
\end{table*}
\subsection{Discussion and Ablation Study}\label{sec4-2-3}
In this section, we conduct experiments and analyze the results to answer the following questions.
 
\noindent\textbf{Question: Is a larger number of data splittings, $U$, more favorable for generalization performance in OEHG?} We answer the question from two perspectives: theoretical analysis and experimental verification. These two views will corroborate each other, ensuring the correctness of our conclusions.

\textbf{Theoretical analysis.} As discussed in Theorem \ref{theorem332-1} and \ref{theorem334-1}, appropriately increasing the number of splittings $U$ can effectively reduce hypergradient error by reducing hypergradient variance. However, continuously increasing $U$ can introduce bias in hypergradient error estimation, which could harm hypergradient estimation and ultimately affect generalization performance. Moreover, the theoretical analysis in Section \ref{section8} also demonstrates that selecting a suitable $U$ can effectively improve the model's generalization performance.

\textbf{Experimental verification.}  Figs. \ref{fig423-1}(a-b) display the test loss curve on OEHG with varying numbers of splittings. It can be observed that: (1) When $U$=1, \textit{i.e.}, using a single data splitting, the test loss curve is above the others, indicating a significant improvement of generalization performance with EHG compared to using a single data splitting. (2) When $U>1$, by comparing different curves, it can be seen that as $U$ starts to increase  (for ionosphere dataset, 1$\rightarrow$10$\rightarrow$20; for heart dataset, 1$\rightarrow$10$\rightarrow$20$\rightarrow$40$\rightarrow$80), test loss gradually decreases. This implies that the EHG effectively reduces hypergradient error, ultimately improving generalization performance. Moreover, the HPO process becomes more stable with the increase of $U$, resulting in smaller variance, as Fig. \ref{Figure721-4} and Fig. \ref{Figure73-1} (Right) show. (3) However, as $U$ continues to increase (for ionosphere dataset, from 20$\rightarrow$40$\rightarrow$80$\rightarrow$160; for heart dataset, from 80$\rightarrow$160), a decline in generalization performance is observed. This decrease can be attributed to the rising hypergradient bias, which leads to inaccuracies in the HPO process and subsequently affects generalization performance. This phenomenon is substantiated by the conclusion in Table \ref{table4-1}.

In conclusion, a larger number of splittings $U$ is not necessarily better, and we need to suitably select it in experiments. In practice, we recommend setting $U$ to 5 or 10, and the experiments could consistently perform well under such simple settings.
\begin{figure*}[!t]
\centering
\subfigure[]{\includegraphics[width=0.37\textwidth]{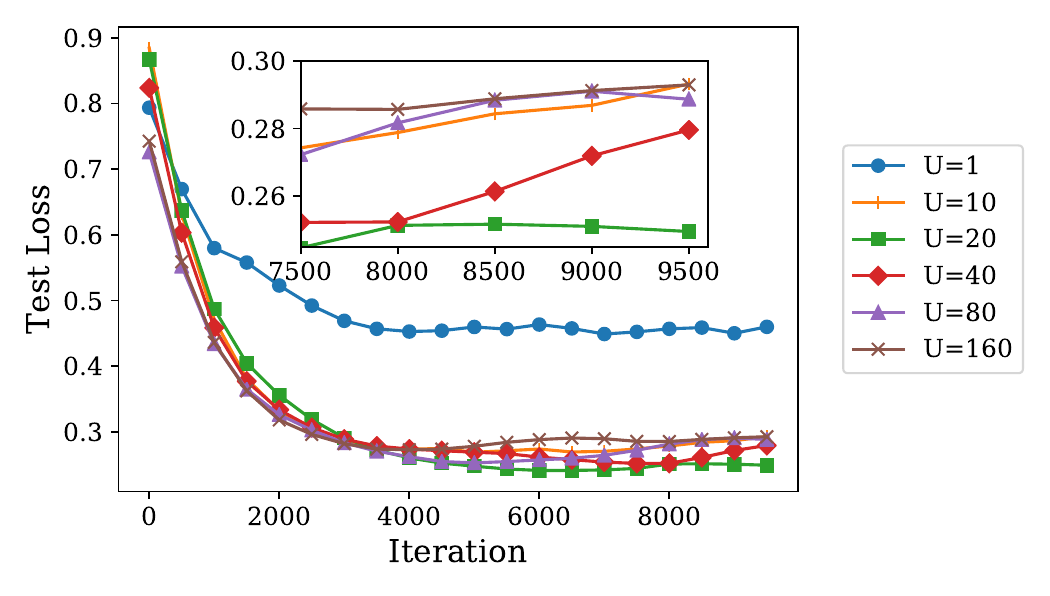}
}\hspace{-13mm}
\subfigure[]{\includegraphics[width=0.37\textwidth]{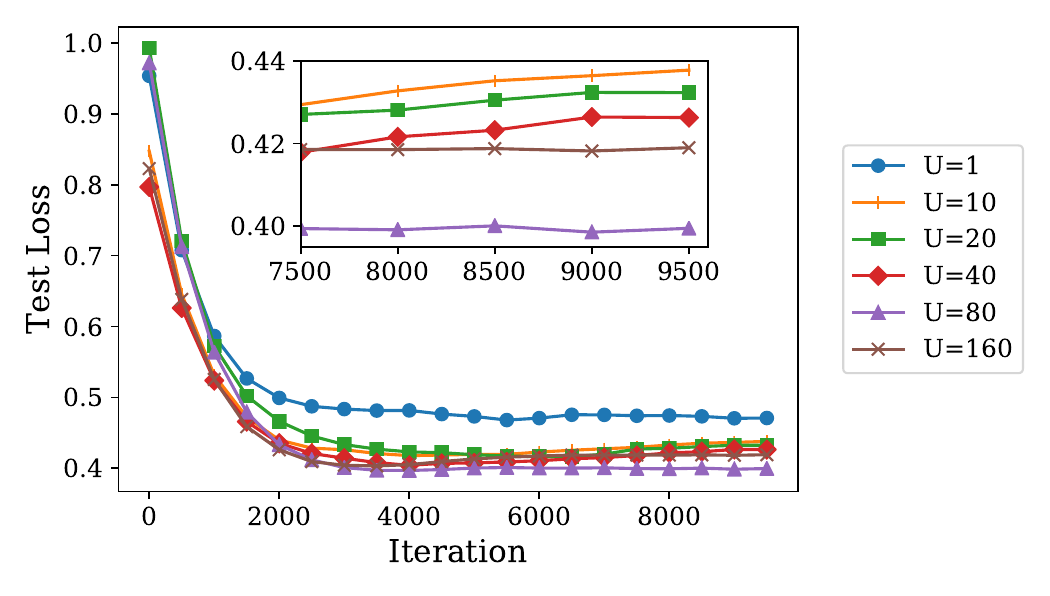}
}
\subfigure[]{\includegraphics[width=0.28\textwidth]{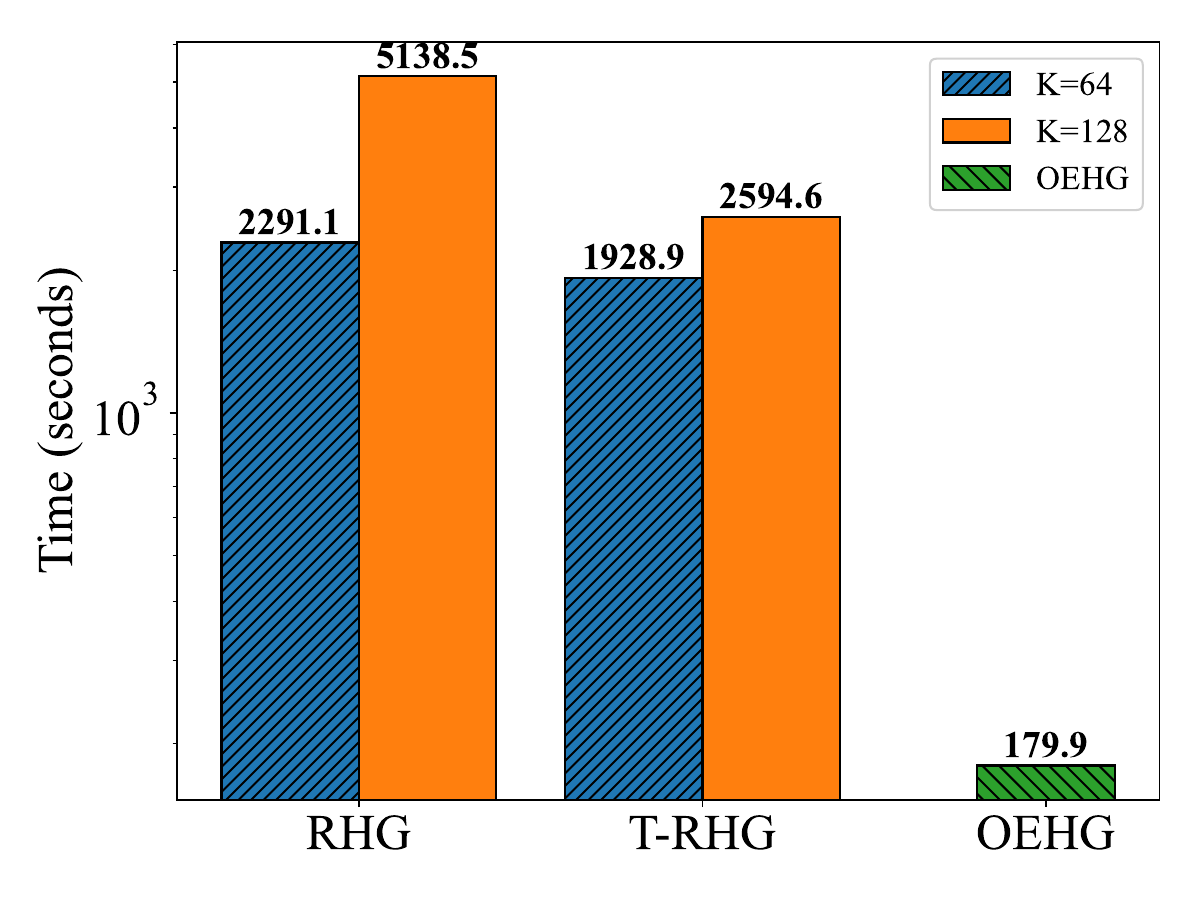}
}
\caption{(a-b): Test loss curve for ionosphere (Left) and heart (Right) dataset under $U$ data splittings of SVM model. (c): Runtime (seconds) of the main competing methods for solving high-dimensional HPO problem.}
\label{fig423-1}
\end{figure*}

\noindent\textbf{Question: How does the cost of OEHG compare to other competing methods?} 
As mentioned earlier, OEHG employs online optimization, which, despite increasing the number of inner-loop, significantly reduces the number of iterations per inner loop compared to methods like RHG. This results in a relatively lower cost for our approach. Fig. \ref{fig423-1}(c) presents a comparison of the time costs of OEHG and the primary competing methods. It can be observed that the OEHG incurs relatively lower time costs compared to RHG and T-RHG, making OEHG relatively more appropriate for complex machine learning tasks.
\section{Generalization of Gradient-based HPO}\label{section8}
In this section, we present a clear conclusion that reducing hypergradient variance contributes to better generalization performance. In other words, we aim to establish a link between hypergradient error analysis and excess error estimation.
\subsection{Excess Error of ITD}\label{section81}
For $\boldsymbol{\lambda}$ and $\boldsymbol{\theta}$ from ITD on data splitting $\mathcal{S}_{(\mathcal{D},u_1)}$, excess error can be decomposed into three terms:
\begin{align}\label{equation81-1}
		&\mathbb{E}\left[{\mathcal{R}}^{val}(\boldsymbol{\lambda}_{\mathcal{D}, u_1},\boldsymbol{\theta}_{\mathcal{D}, u_1}(\boldsymbol{\lambda}_{\mathcal{D}, u_1}))-\mathcal{R}^{val}(\boldsymbol{\lambda}^*,\boldsymbol{\theta}^*(\boldsymbol{\lambda}^*))\right]=\notag\\
		&\underbrace{\mathbb{E}\left[{\mathcal{R}}^{val}(\boldsymbol{\lambda}_{\mathcal{D}, u_1},\boldsymbol{\theta}_{\mathcal{D}, u_1}(\boldsymbol{\lambda}_{\mathcal{D}, u_1}))-\hat{\mathcal{R}}^{val}(\boldsymbol{\lambda}_{\mathcal{D}, u_1},\boldsymbol{\theta}_{\mathcal{D}, u_1}(\boldsymbol{\lambda}_{\mathcal{D}, u_1});  \mathcal{S}_{(\mathcal{D}, u_1)})\right]}_{\text{generalization error}}+\notag\\
		&\underbrace{\mathbb{E}\left[\hat{\mathcal{R}}^{val}(\boldsymbol{\lambda}_{\mathcal{D}, u_1},\boldsymbol{\theta}_{\mathcal{D}, u_1}(\boldsymbol{\lambda}_{\mathcal{D}, u_1});  \mathcal{S}_{(\mathcal{D}, u_1)})-\hat{\mathcal{R}}^{val}(\boldsymbol{\lambda}^*_{\mathcal{D}, u_1},\boldsymbol{\theta}^*_{\mathcal{D}, u_1}(\boldsymbol{\lambda}^*_{\mathcal{D}, u_1});  \mathcal{S}_{(\mathcal{D}, u_1)})\right]}_{\text{training error}}+\notag\\
  &\underbrace{\mathbb{E}\left[\hat{\mathcal{R}}^{val}(\boldsymbol{\lambda}^*_{\mathcal{D}, u_1},\boldsymbol{\theta}^*_{\mathcal{D}, u_1}(\boldsymbol{\lambda}^*_{\mathcal{D}, u_1});  \mathcal{S}_{(\mathcal{D}, u_1)})-\mathcal{R}^{val}(\boldsymbol{\lambda}^*,\boldsymbol{\theta}^*(\boldsymbol{\lambda}^*))\right]}_{\leq 0},
\end{align}where $\mathcal{R}^{val}(\cdot)=\mathbb{E}_{z}[\hat{\mathcal{R}}^{val}(\cdot;z)]$ and $(\boldsymbol{\lambda}^*_{\mathcal{D}, u_1}, \boldsymbol{\theta}^*_{\mathcal{D}, u_1}):=\arg\min_{\boldsymbol{\lambda},\boldsymbol{\theta}}\hat{\mathcal{R}}^{val}(\boldsymbol{\lambda},\boldsymbol{\theta}; \mathcal{S}_{(\mathcal{D}, u_1)})$. It can be verified that the expectation of the third term (over $\mathcal{D}$ and $u_1$) is non-positive since 
\begin{align*}
    &\mathbb{E}[\hat{\mathcal{R}}^{val}(\boldsymbol{\lambda}^*_{\mathcal{D}, u_1},\boldsymbol{\theta}^*_{\mathcal{D}, u_1}(\boldsymbol{\lambda}^*_{\mathcal{D}, u_1});  \mathcal{S}_{(\mathcal{D}, u_1)})]=\\
    &\mathbb{E}[\min_{\boldsymbol{\lambda},\boldsymbol{\theta}}\hat{\mathcal{R}}^{val}(\boldsymbol{\lambda},\boldsymbol{\theta}; \mathcal{S}_{(\mathcal{D}, u_1)})]\leq\min_{\boldsymbol{\lambda},\boldsymbol{\theta}}\mathbb{E}[\hat{\mathcal{R}}^{val}(\boldsymbol{\lambda},\boldsymbol{\theta}; \mathcal{S}_{(\mathcal{D}, u_1)})]=\min_{\boldsymbol{\lambda},\boldsymbol{\theta}}\mathcal{R}^{val}(\boldsymbol{\lambda},\boldsymbol{\theta}),
\end{align*}
and $(\boldsymbol{\lambda}^*, \boldsymbol{\theta}^*):=\arg\min_{\boldsymbol{\lambda},\boldsymbol{\theta}}\mathcal{R}^{val}(\boldsymbol{\lambda},\boldsymbol{\theta})$. Hence, to bound the expected excess error, we should bound the expectation of generalization and training errors.

Considering multiple data splittings $C=(\mathcal{D}_{u_i}^{tr}, \mathcal{D}_{u_i}^{val})_{i=1}^U$, excess error can be decomposed into three terms in the following:

\begin{align*}
&\mathbb{E}\left[{\mathcal{R}}^{val}(\boldsymbol{\lambda}_{\mathcal{D}, \{u_i\}_{i=1}^{U}},\boldsymbol{\theta}_{\mathcal{D}, \{u_i\}_{i=1}^{U}}(\boldsymbol{\lambda}_{\mathcal{D}, \{u_i\}_{i=1}^{U}}))-\mathcal{R}^{val}(\boldsymbol{\lambda}^*,\boldsymbol{\theta}^*(\boldsymbol{\lambda}^*))\right]=\\
&\mathbb{E}\Big{[}\frac{1}{U}\sum_{i=1}^{U}\Big{(}\underbrace{{\mathcal{R}}^{val}(\boldsymbol{\lambda}_{\mathcal{D}, \{u_i\}_{i=1}^{U}},\boldsymbol{\theta}_{\mathcal{D}, u_i}(\boldsymbol{\lambda}_{\mathcal{D}, \{u_i\}_{i=1}^{U}}))-\hat{\mathcal{R}}^{val}(\boldsymbol{\lambda}_{\mathcal{D}, \{u_i\}_{i=1}^{U}},\boldsymbol{\theta}_{\mathcal{D}, u_i}(\boldsymbol{\lambda}_{\mathcal{D}, \{u_i\}_{i=1}^{U}});  \mathcal{S}_{(\mathcal{D}, u_i)})}_{\text{generalization error}}+\\
&\underbrace{\hat{\mathcal{R}}^{val}(\boldsymbol{\lambda}_{\mathcal{D}, \{u_i\}_{i=1}^{U}},\boldsymbol{\theta}_{\mathcal{D}, u_i}(\boldsymbol{\lambda}_{\mathcal{D}, \{u_i\}_{i=1}^{U}});  \mathcal{S}_{(\mathcal{D}, u_i)})-\hat{\mathcal{R}}^{val}(\boldsymbol{\lambda}^*_{\mathcal{D}, \{u_i\}_{i=1}^{U}},\boldsymbol{\theta}^*_{\mathcal{D}, u_i}(\boldsymbol{\lambda}^*_{\mathcal{D}, \{u_i\}_{i=1}^{U}});  \mathcal{S}_{(\mathcal{D}, u_i)})}_{\text{training error}}+\\
&\underbrace{\hat{\mathcal{R}}^{val}(\boldsymbol{\lambda}^*_{\mathcal{D}, \{u_i\}_{i=1}^{U}},\boldsymbol{\theta}^*_{\mathcal{D}, u_i}(\boldsymbol{\lambda}^*_{\mathcal{D}, \{u_i\}_{i=1}^{U}});  \mathcal{S}_{(\mathcal{D}, u_i)})-\mathcal{R}^{val}(\boldsymbol{\lambda}^*,\boldsymbol{\theta}^*(\boldsymbol{\lambda}^*))}_{\leq 0}\Big{)}\Big{]},
\end{align*}
where $\mathcal{R}^{val}(\cdot)=\mathbb{E}[\frac{1}{U}\sum_{i=1}^{U}\hat{\mathcal{R}}^{val}(\cdot;\mathcal{S}_{(\mathcal{D}, u_i)})]$.
\subsection{Excess Error Analysis of ITD}
We construct the upper bound of excess error by addressing both training error and generalization error.
\subsubsection{Analysis of training error}
We employ hypergradient error estimation from ITD (i.e., Theorem \ref{theorem332-1}) to establish the bound of training error.
\begin{proposition}\label{proposition821-1}
	For the given samples $\mathcal{D}\sim\mathscr{P}$ and data splittings $\mathcal{S}_{(\mathcal{D},\{u_i\}_{i=1}^U)}=(\mathcal{D}_{u_i}^{tr}, \mathcal{D}_{u_i}^{val})_{i=1}^U$, suppose that Assumptions \ref{assum33-1} and \ref{assum33-2} hold. Set $\alpha_{in}\leq\frac{2}{L}$ and $\alpha_{out}=\frac{\ln q^{-1}}{L\ln3}$. Then, we have

\begin{align}\label{equation421-1}
&\frac{1}{U}\sum_{i=1}^U\Big{(}\hat{\mathcal{R}}^{val}(\boldsymbol{\lambda}_{\mathcal{D}, \{u_i\}_{i=1}^{U}},\boldsymbol{\theta}_{\mathcal{D}, u_i}(\boldsymbol{\lambda}_{\mathcal{D}, \{u_i\}_{i=1}^{U}});  \mathcal{S}_{(\mathcal{D}, u_i)})-\hat{\mathcal{R}}^{val}(\boldsymbol{\lambda}^*_{\mathcal{D}, \{u_i\}_{i=1}^{U}},\boldsymbol{\theta}^*_{\mathcal{D}, u_i}(\boldsymbol{\lambda}^*_{\mathcal{D}, \{u_i\}_{i=1}^{U}});  \notag\\
&\mathcal{S}_{(\mathcal{D}, u_i)})\Big{)}\leq M\frac{\big{(}\alpha_{out} L(\alpha_{in} L+1)^K+1\big{)}^T-1}{L}\cdot\Big[\sup_i B_{\text{ITD}, u_i}^2(\boldsymbol{\lambda}, K)+\frac{2M^2}{U}\Big{(}\frac{1}{m^{tr}}(1+\frac{L}{\mu})^2+\notag\\
&\frac{1}{m^{val}}\Big{)}\Big]^{1/2}+M(\alpha_{in} L+1)^K\big{\Vert}\boldsymbol{\lambda}^{*(T)}_{\mathcal{D}, \{u_i\}_{i=1}^{U}}-\boldsymbol{\lambda}^*_{\mathcal{D}, \{u_i\}_{i=1}^{U}}\big{\Vert}
+\sup_i M\Big{\Vert}\boldsymbol{\theta}_{\mathcal{D}, u_i}(\boldsymbol{\lambda}^*_{\mathcal{D}, \{u_i\}_{i=1}^{U}})-\notag \\
&\boldsymbol{\theta}^*_{\mathcal{D}, u_i}(\boldsymbol{\lambda}^*_{\mathcal{D}, \{u_i\}_{i=1}^{U}})\Big{\Vert}.
\end{align}where $B_{\text{ITD}, u_i}(\boldsymbol{\lambda}, K):= \Big{(}2LC_{1,\boldsymbol{\lambda},\mathcal{D}^{tr}_{u_i}}(1+\frac{C_{2,\boldsymbol{\lambda},\mathcal{D}^{tr}_{u_i}}}{1-q})(1+\frac{MK}{q})+\frac{MC_{2,\boldsymbol{\lambda},\mathcal{D}^{tr}_{u_i}}}{1-q}\Big{)}q^K$. $C_{1,\boldsymbol{\lambda},\mathcal{D}^{tr}_{u_i}}$ and $C_{2,\boldsymbol{\lambda},\mathcal{D}^{tr}_{u_i}}$ are constants introduced in Lemma \ref{lemma33-1}, and $q\in(0, 1)$ is a constant introduced in Lemma \ref{lemma-a2-5}.
\end{proposition}
\noindent\textbf{Discussion:} We can rewirte Eq. (\ref{equation421-1}) as follows:
\begin{align}\label{equation421-2}
&\frac{1}{U}\sum_{i=1}^U\Big{(}\hat{\mathcal{R}}^{val}(\boldsymbol{\lambda}_{\mathcal{D}, \{u_i\}_{i=1}^{U}},\boldsymbol{\theta}_{\mathcal{D}, u_i}(\boldsymbol{\lambda}_{\mathcal{D}, \{u_i\}_{i=1}^{U}});  \mathcal{S}_{(\mathcal{D}, u_i)})-\hat{\mathcal{R}}^{val}(\boldsymbol{\lambda}^*_{\mathcal{D}, \{u_i\}_{i=1}^{U}},\boldsymbol{\theta}^*_{\mathcal{D}, u_i}(\boldsymbol{\lambda}^*_{\mathcal{D}, \{u_i\}_{i=1}^{U}});\notag\\
  &\mathcal{S}_{(\mathcal{D}, u_i)})\Big{)}\leq\frac{M}{L}\sqrt{\text{TE1}+\text{TE2}}+\text{TE3}+\text{TE4},
\end{align}where
\begin{align*}
&\text{TE1}={
\Big{(}\big{(}\alpha_{out} L(\alpha_{in} L+1)^K+1\big{)}^T-1\Big{)}^2\sup_i B_{\text{ITD},u_i}^2(\boldsymbol{\lambda}, K)},\\
&\text{TE2}=\frac{2M^2}{U}\big{(}\big{(}\alpha_{out} L(\alpha_{in} L+1)^K+1\big{)}^T-1\big{)}^2\Big{(}\frac{1}{m^{tr}}(1+\frac{L}{\mu})^2+\frac{1}{m^{val}}\Big{)},\\
&\text{TE3}=M(\alpha_{in} L+1)^K\big{\Vert}\boldsymbol{\lambda}^{*(T)}_{\mathcal{D}, \{u_i\}_{i=1}^{U}}-\boldsymbol{\lambda}^*_{\mathcal{D}, \{u_i\}_{i=1}^{U}}\big{\Vert},\\
&\text{TE4}=\sup_i M\Big{\Vert}\boldsymbol{\theta}_{\mathcal{D}, u_i}(\boldsymbol{\lambda}^*_{\mathcal{D}, \{u_i\}_{i=1}^{U}})-\boldsymbol{\theta}^*_{\mathcal{D}, u_i}(\boldsymbol{\lambda}^*_{\mathcal{D}, \{u_i\}_{i=1}^{U}})\Big{\Vert}.
\end{align*}
We decompose the bound of training error into 4 components: TE1 and TE2 represent the part caused by hypergradient estimation error. Specifically, TE1 and TE2 correspond to the training errors caused by hypergradient bias and variance, respectively. TE3 and TE4 represent the training errors arising from the outer and inner optimization processes, respectively. The trends of training error w.r.t. various variables are listed in Table \ref{table4-1}.

\subsubsection{Analysis of generalization error}
To establish a generalization error bound, we define the following notion of uniform stability on observed data.
\begin{definition}\label{definition722-1}
	Given data splittings $\mathcal{S}_{(\mathcal{D},\{u_i\}_{i=1}^U)}=(\mathcal{D}_{u_i}^{tr}, \mathcal{D}_{u_i}^{tr})_{i=1}^U$, an HPO algorithm $\mathcal{A}_{\text{hpo}}$ is $\beta$-uniformly stable on observed samples $\mathcal{D}$ if for all samples $\mathcal{D},{\mathcal{D}}'\in\mathscr{P}$ such that $\mathcal{D},{\mathcal{D}}'$ differ in at most one sample. Then, $\forall z\in\mathscr{P}$, we have
	\begin{align*}
			 &\frac{1}{U}\Big{(}\sum_{i=1}^{U_1}[\hat{\mathcal{R}}^{val}(\mathcal{A}_{\text{hpo}}(\mathcal{D}_{u_i}^{tr},\mathcal{S}_{(\mathcal{D},\{u_i\}_{i=1}^U)});z)-\hat{\mathcal{R}}^{val}(\mathcal{A}_{\text{hpo}}({\mathcal{D}_{u_i}^{tr}}',\mathcal{S}_{\mathcal{D}',\{u_i\}_{i=1}^U});z)]+\\
			 &\sum_{i=1}^{U_2}[\hat{\mathcal{R}}^{val}(\mathcal{A}_{\text{hpo}}(\mathcal{D}_{u_i}^{tr},\mathcal{S}_{(\mathcal{D},\{u_i\}_{i=1}^U)});z)-\hat{\mathcal{R}}^{val}(\mathcal{A}_{\text{hpo}}({\mathcal{D}_{u_i}^{tr}},\mathcal{S}_{\mathcal{D}',\{u_i\}_{i=1}^U});z)]\Big{)}
			 \leq\beta,
	\end{align*}
where $U_1$ and $U_2$ represent the number of occurrences where the unique sample $z'$ falls into the training splitting and validation splitting, respectively.
\end{definition}
\noindent\textbf{Remark.} Compared to existing work \citep{bao2021stability}, Definition \ref{definition722-1} considers the influence of changing one observed sample in an HPO algorithm. The reason is that the changed sample can be included in either the training set or validation set, depending on the splitting, making it necessary to consider its impact across different data splittings.

In this work, we assume that $U_1=\frac{Um^{tr}}{m^{tr}+m^{val}}$ and $U_2=\frac{Um^{val}}{m^{tr}+m^{val}}$. \footnote{This is natural since the event that $z'$ falls into the training or validation splitting is due to random sampling. For other cases of $U_1$ and $U_2$, we give detailed conclusions in the Appendix.} Then we have the conclusion that ITD is $\beta$-uniformly stable on observed samples $\mathcal{D}$ as the following. 
\begin{lemma}\label{lemma822-1}
	For the given multiple observed samples $\mathcal{D}\sim\mathscr{P}$ and splittings $\mathcal{S}_{(\mathcal{D},\{u_i\}_{i=1}^U)}=(\mathcal{D}_{u_i}^{tr}, \mathcal{D}_{u_i}^{tr})_{i=1}^U$, suppose that Assumptions \ref{assum33-1}, \ref{assum33-2} and \ref{assum33-3} hold. Then, ITD algorithm with $T$-step gradient descent and learning rate $\alpha_{out}\leq\frac{1}{M}$ in the outer-level is $\beta$-uniformly stable with
	\begin{align*}
		\beta=\frac{M^2}{L(m^{tr}+m^{val})}\Big{(}{\big{(}{\big{(}1+\alpha_{out} L(\alpha_{in} L+1)^K\big{)}^{T}+1}\big{)}\big{(}\alpha_{in} L+1\big{)}^K}\Big{)}.
	\end{align*}
\end{lemma}
Then we have the following high probability bound.
\begin{theorem}\label{theorem822-1}
	For the given samples $\mathcal{D}\sim\mathscr{P}$, random seeds $\{u_i\}_{i=1}^{U}$. Under the conditions in Lemma \ref{lemma822-1}, suppose that the loss function $\hat{\mathcal{R}}^{val}$ is bounded by $S\geq0$, then for all $\delta\in(0,1)$, with probability at least $1-\delta$,
	\begin{align*}
			&\frac{1}{U}\sum_{i=1}^{U}\left[\hat{\mathcal{R}}^{val}(\boldsymbol{\lambda}_{\mathcal{D}, \{u_i\}_{i=1}^{U}},\boldsymbol{\theta}_{\mathcal{D}, u_i}(\boldsymbol{\lambda}_{\mathcal{D}, \{u_i\}_{i=1}^{U}});  \mathcal{S}_{(\mathcal{D}, u_i)})-{\mathcal{R}}^{val}(\boldsymbol{\lambda}_{\mathcal{D}, \{u_i\}_{i=1}^{U}},\boldsymbol{\theta}_{\mathcal{D}, u_i}(\boldsymbol{\lambda}_{\mathcal{D}, \{u_i\}_{i=1}^{U}}))\right]\\
			\leq &\beta+\big{(}\frac{S}{m^{tr}+m^{val}}+2\beta\big{)}\sqrt{\frac{\ln\delta^{-1}(m^{tr}+m^{val})}{2}},
	\end{align*}
where $(\boldsymbol{\lambda}_{\mathcal{D}, \{u_i\}_{i=1}^{U}},\boldsymbol{\theta}_{\mathcal{D}, u_i}(\boldsymbol{\lambda}_{\mathcal{D}, \{u_i\}_{i=1}^{U}})=\mathcal{A}_{\text{ITD}}(\mathcal{D}_{u_i}^{tr},\mathcal{S}_{(\mathcal{D}, \{u_i\}_{i=1}^U)})$, i.e., $(\boldsymbol{\lambda}, \boldsymbol{\theta})$ is obtained by ITD algorithm and $\beta$ is given by Lemma \ref{lemma822-1}.
\end{theorem}
\noindent\textbf{Remark.} Combining the results in Lemma \ref{lemma822-1} and Theorem \ref{theorem822-1}, we obtain the error bound of ITD that depends on the number of steps in the outer-level $T$, the number of steps in the inner
level $K$, the training sample size $m^{tr}$ and the validation sample size $m^{val}$. The trends of generalization error w.r.t. various variables are listed in Table \ref{table4-1}.
\subsection{Conclusion of Excess Error Analysis}\label{section7-3}
\begin{table*}[!t]
\caption{The analysis of excess risk for ITD Algorithm. GE: generalization error. TE1, TE2, TE3 and TE4 are the 4 parts of training error and are defined in Eq. (\ref{equation421-2}). ``$-$" indicates no correlation, and ``$\sim$" indicates an uncertain trend.}
\centering
\resizebox{0.85\textwidth}{!}{
\begin{tabular}{lllllll}
\toprule
Description &Error Type&GE&TE1&TE2&TE3&TE4\\
\midrule
Inner-level iteration&$K(\uparrow)$&$\nearrow$&$\sim$&$\nearrow$&$\nearrow$&$\searrow$\\
\midrule
Outer-level iteration&$T(\uparrow)$&$\nearrow$&$\nearrow$&$\nearrow$&$\searrow$&$-$\\
\midrule
Number of data splitting &$U(\uparrow)$&$-$&$\nearrow$&$\searrow$&$-$&$\nearrow$\\
\midrule
Size of single training set&$m^{tr}$($\uparrow$)&$\searrow$&$-$&$\searrow$&$-$&$-$\\
\midrule
Size of single validation set&$m^{val}$($\uparrow$)&$\searrow$&$-$&$\searrow$&$-$&$-$\\
\bottomrule
\end{tabular}}
\label{table4-1}
\end{table*}
Combining the properties between training error (Proposition \ref{proposition821-1}) and generalization error (Theorem \ref{theorem822-1}), we can get the main result of excess error, as outlined in Table \ref{table4-1}. Specifically, the findings can be summarized as follows:
\begin{itemize}
	\item[(1)] For generalization error, the increase of the number of iterations ($K$, $T$) leads to a rise of generalization error, while an increase in size of data ($m^{tr}$, $m^{val}$) results in a decrease of generalization error.
	\item[(2)] Considering TE1 (the training error part by hypergradient bias), which increases with the growth of the iterations of outer-level $T$ and the splittings $U$. For $K$, the trend of TE1 is uncertain and depends on the settings of $T$ and the learning rates for the inner and outer-level. For a detailed analysis, please refer to Appendix \ref{sectionA14}.	
	\item[(2)] Considering TE2 (the training error component caused by hypergradient variance), it increases with the growth of optimization factors (number of inner iterations $K$ and number of outer iterations $T$) and decreases with the growth of data factors (data size $(m_{tr} + m^{val})$ and the number of data splits $U$).
	\item[(3)] With an increase in the number of outer-level iteration $T$, only TE3 decreases, while errors in other parts tend to increase. Furthermore, there exists a  $t(\epsilon)$ such that when $T > t(\epsilon)$, it ensures that $\mathbb{E}_{\mathcal{D}}{[}{\Vert}\boldsymbol{\lambda}^{*(T)}-\boldsymbol{\lambda}^*{\Vert}{]}\leq\epsilon$. This explains why the majority of existing gradient-based HPO algorithms require multiple steps of outer-level gradient descent. However, with the continued increase of $T$, other errors (GE, TE1, TE2) except TE3 experience exponential growth, which will exceed $\epsilon$, leading to an increase in excess error. Therefore, $T$ needs to be carefully chosen, and early stopping in machine learning seems to be an easy and intuitive strategy.
	\item[(4)] With the number of inner iterations $K$ increases, only TE4 decreases, while other error terms increase. Intuitively, a larger $K$ can heighten the risk of overfitting, which might be one source of the observed overfitting (see \cite{franceschi2018bilevel, bao2021stability}). In practice, we demonstrate that the selection of $K$ is a critical factor that significantly impacts generalization performance, as shown in Table \ref{Table721-2}, Figs. \ref{Figure721-2} and \ref{Figure721-3}.
	\item[(5)] The effect of the number of data splittings, $U$, has seldom been analyzed in previous gradient-based HPO studies, where $U=1$ was typically considered. We provide a more general conclusion by examining the variation of each error w.r.t. $U$. Specifically, an increase in $U$ can lead to a decrease in TE2, due to the decreased variance in hypergradient estimation, as shown in Figs. \ref{Figure71-1},  \ref{Figure73-1}, \ref{Figure73-2}, \ref{Figure721-4} and \ref{fig423-1}. Furthermore, the generalization experiment results in Section \ref{section7} indicate that, in most cases, appropriately increasing $U$ can reduce the risk of overfitting and thus improve generalization performance. However, if $U$ continues to increase, it will rise hypergradient bias and then increase the excess risk as in Section \ref{sec4-2-3} shows.
\end{itemize}
\section{Related Work} \label{section6}
\textbf{Conventional HPO algorithms.} They aim at searching optimal hyperparameter configurations to enhance the generalization performance of machine learning models. The early attempts contain grid search or random search techniques \citep{randomsearch}. To develop more efficient methods, researchers have explored the utilization of Bayesian methods for modeling HPO \citep{snoek2012practical}, which aim to expedite the identification of effective hyperparameter configurations compared to conventional approaches such as random search. Nevertheless, these methods encounter challenges in dealing with high-dimensional hyperparameter and expensive computational costs.

\noindent\textbf{Gradient-based HPO algorithms.} They could optimize hyperparameter with millions of dimensions by making use of nested automatic differentiation. According to different strategies of computing hypergradient, gradient-based HPO methods can be classified into two main categories: iterative differentiation (ITD) and approximate implicit differentiation (AID). The key idea underlying ITD is to hierarchically calculate gradients of inner-level and outer-level objectives. Specifically,  the works in \cite{franceschi2017forward} first calculate gradient representations of model parameter and then perform either reverse or forward gradient computations (termed as reverse hypergradient (RHG) and forward hypergradient (FHG)) for calculating of hypergradient sub-problem. However, since ITD requires to calculation automatic differentiation for the entire trajectory of the dynamic iteration of the inner-level objective, the computation load is relatively heavy to calculate the hypergradient with reasonable preciseness. To reduce the amount of computation, \cite{shaban2019truncated} propose truncated reverse hypergradient (T-RHG) to truncate the gradient trajectory. However, the efficiency of T-RHG is certainly sensitive to the truncated path length. A short truncated path may deteriorate the accuracy of the calculated hypergradient, while a long truncated path always cannot satisfactorily reduce the computation cost.

Another method is to decouple the calculation process of hypergradient from the dynamic system. For this purpose, approximate implicit differentiation (AID) \citep{pedregosa2016hyperparameter, rajeswaran2019meta, lorraine2020optimizing} replaces the inner-level sub-problem with an implicit differential equation. Specifically, taking advantage of the celebrated implicit function theorem, hypergradient can be calculated by solving an implicit differential equation. However, this scheme needs to repeatedly compute the inverse of the Hessian matrix. In practice, the Conjugate Gradient (CG) \citep{pedregosa2016hyperparameter} method or Neumann method \citep{lorraine2020optimizing} are designed for fast inverse computation, however, repeated products of vectors and matrices are still required. Therefore, it is still expensive to compute, and causes numerical instabilities, especially when the implicit differential equation is ill-conditioned.

While gradient-based methods have made significant advances in various applications, most research has been proposed from the perspective of reducing hypergradient estimation error in the context of bilevel programming \citep{franceschi2017forward, ji2021bilevel, liu2020generic, grazzi2020iteration}. Specifically, when it comes to the task of HPO, particularly from the perspective of data usage, most studies have adopted a simplified assumption, i.e., the fixed data splitting of training and validation could accurately approximate the training-validation data distribution. However, there has been limited analysis of the validity of this assumption. In this study, we have conducted a bias-variance decomposition of hypergradient estimation error, revealing the deficiencies of existing gradient-based HPO methods. The proposed error bounds of hypergradient estimation emphasize the influence of the variance term and provide a proper explanation for certain phenomena observed in practice, like overfitting to the validation set. This phenomenon was only explained from an optimization perspective \citep{bao2021stability} previously, certainly ignoring the impact of potential data distribution on the generalization error for bilevel optimization algorithms. In comparison to works of stochastic bilevel optimization method \citep{ji2021bilevel}, we focus more on deterministic bilevel optimization, particularly on the analysis of variance term related to the data distribution in terms of hypergradient error estimation.

\noindent\textbf{Bilevel optimization and meta learning.} Gradient-based HPO methods mainly use bilevel optimization techniques \citep{liu2021investigating} to find proper hyperparameter configurations. Current theoretical works of bilevel optimization and HPO problem mainly focus on the error reduction between estimated and ground-truth hypergradient, and thereby ensure the convergence of algorithms.
However, they ignore the influence of data distribution on hypergradient estimation error. In this study, we make a supplemental analysis of the variance term, and provide a comprehensive error bound. Gradient-based HPO problems could be categorized into the remit of meta learning \citep{hospedales2021meta,shu2021learning}. Especially, meta learning has helped machine learning improve the algorithm automation and
generalization, like neural architecture search \citep{elsken2019neural}, sample weighting \citep{shu2019meta,shu2022cmw}, label noise learning \citep{shu2020meta,wu2021learning,ding2023improve}, semi-supervised learning \citep{pham2021meta}, loss/regularization learning \citep{balaji2018metareg,shu2020learning}, gradient/learning rate \citep{andrychowicz2016learning,ravi2016optimization,shu2022mlr}, etc. These meta learning methods can be considered as learning proper hyperparameter setting rules from multiple learning tasks \citep{shu2021learning}, which is expected to be readily used for new learning tasks. Previous theoretical works of meta learning mainly pay attention to the generalization error analysis, and we attempt to build a connection between generalization error analysis and hypergradient estimation error for gradient-based HPO algorithms in this paper.

\noindent\textbf{Ensembles methods in machine learning.} Combining the outputs of multiple models to enhance individual performance has a long history in machine learning, as proposed by works such as \citep{levin1990statistical, hansen1990neural, geman1992neural, krogh1994neural, opitz1999popular, dietterich2000ensemble}. Different runs of a model can result in varying parameter due to differences in initial weights, data partitioning, and other factors. These variations represent different methods of generalizing patterns on the training set. Since each network tends to make estimation errors on different parts of the input space, \citep{hansen1990neural} argue that the collective decision of an ensemble is less likely to be erroneous than decisions made by individual networks. Furthermore, \citep{krogh1994neural} shows that one way to form an ensemble is to train models on different training sets, while \citep{wenzel2020hyperparameter} focuses on leveraging the diversity created by combining neural networks defined by different hyperparameter. In the context of HPO, \citep{micaelli2021gradient} address gradient degradation issues by sharing temporally contiguous hyperparameter. In this study, we use ensemble averaging of hypergradient across different data splits to reduce the variance term of the estimation error, thereby achieving better generalization in bilevel optimization algorithms.

\section{Conclusion}
This study has revisited hypergradient estimation error and demonstrated the intrinsic issue that existing theoretical works mainly focus on the error reduction between estimated and true hypergradient, while certainly ignoring the influence of data distribution. Against this problem, we conduct a bias-variance decomposition of hypergradient estimation error and further provide a supplemental analysis of the variance term ignored by previous works. Besides, a comprehensive error bound of hypergradient estimation for existing gradient-based HPO algorithms is presented, which could soundly explain certain phenomena observed in practice, like overfitting to the validation set. The theoretical and empirical analysis for the one-dimensional ridge regression problem is well-aligned with such insightful understanding.
To improve hypergradient estimation of existing HPO algorithms, we present a variance reduction strategy inspired by derived error bounds. We have further substantiated the beneficial effects brought by the proposed strategy in typical HPO applications, including regularization parameter learning, data hyper-cleaning, and few-shot learning. To further interpret why the proposed strategy improves performance, we establish a connection between excess error analysis and hypergradient estimation error for HPO algorithms. Our results suggest that better hypergradient estimation inclines to bring better generalization performance. Experimental results also validate our theoretical findings.

We believe our excess error analysis will be potentially beneficial to the HPO and bilevel optimization fields, which have limited research before, to provide theoretical insights to further help improve the generalization capability of gradient-based HPO algorithms among various tasks.

\textbf{Limitations and future work.} Although the theory-inspired approach is both natural and straightforward, it requires computing the mean over multiple splittings, which incurs a slightly higher computational cost than using a single splitting. Therefore, in future work we will investigate efficient methods for constructing validation sets that more closely approximate the true data distribution—for example, by employing curriculum learning \citep{bengio2009curriculum}, coreset selection \citep{borsos2024data}, and other related techniques.

Moreover, we will refine the theoretical bounds for gradient‐based HPO—deriving tighter upper bounds for practical problems and unveiling intrinsic generalization insights in meta‐learning \citep{shu2021learning}. We will also extend our analysis to non‐convex bilevel optimization under looser conditions and explore variance‐reduction schemes for large‐scale HPO. 

\bibliography{refer}

\appendix
\newpage
\tableofcontents
\addtocontents{toc}{\protect\setcounter{tocdepth}{3}}
\section{Proof of Theorems in Section \ref{sec3-3}}
In the following we give the proof of analysis of hypergradient estimation error in the main paper.
\subsection{Some auxiliary lemmas}\label{sectionA2}
First note that the Lipschitz properties in Assumption \ref{assum33-2} imply the following lemma.
\begin{lemma}\label{lemmaA2-1}
	Suppose Assumption \ref{assum33-2} holds. Then, for the stochastic sampled data $\xi$, $\nabla\hat{\mathcal{R}}^{val}(w;\xi)$, $\nabla\hat{\mathcal{R}}^{tr}(w;\xi)$ and $\nabla_{\boldsymbol{\lambda}}\nabla_{\boldsymbol{\theta}}\hat{\mathcal{R}}^{tr}(w;\xi)$ and $\nabla_{\boldsymbol{\theta}}^2\hat{\mathcal{R}}^{tr}(w;\xi)$ have bounded variances, i.e., for any $w$ and $\xi$, $\mathbb{E}_{\xi}\Vert\nabla\hat{\mathcal{R}}^{val}(w;\xi)-\nabla{\mathcal{R}}^{val}(w)\Vert^2\leq M^2$, $\mathbb{E}_{\xi}\Vert\nabla\hat{\mathcal{R}}^{tr}(w;\xi)-\nabla{\mathcal{R}}^{tr}(w)\Vert^2\leq M^2$, $\mathbb{E}_{\xi}\Vert\nabla^2_{\boldsymbol{\lambda}, \boldsymbol{\theta}}\hat{\mathcal{R}}^{tr}(w;\xi)-\nabla^2_{\boldsymbol{\lambda}, \boldsymbol{\theta}}{\mathcal{R}}^{tr}(w)\Vert^2\leq L^2$ and $\mathbb{E}_{\xi}\Vert\nabla_{\boldsymbol{\theta}}^2\hat{\mathcal{R}}^{tr}(w;\xi)-\nabla_{\boldsymbol{\theta}}^2{\mathcal{R}}^{tr}(w)\Vert^2\leq L^2$.
\end{lemma}
According to Lemma \ref{lemmaA2-1}, we can obtain the variance properties of multiple training and validation samples.
\begin{lemma}\label{lemmaA2-2}
	Suppose Assumption \ref{assum33-2} holds. Then, for the stochastic samples $\mathcal{D}^{tr}$ or $\mathcal{D}^{val}$, $\nabla\hat{\mathcal{R}}^{val}(w;\mathcal{D}^{val})$, $\nabla\hat{\mathcal{R}}^{tr}(w;\mathcal{D}^{tr})$ and $\nabla^2_{\boldsymbol{\lambda}, \boldsymbol{\theta}}\hat{\mathcal{R}}^{tr}(w;\mathcal{D}^{tr})$ and $\nabla_{\boldsymbol{\theta}}^2\hat{\mathcal{R}}^{tr}(w;\mathcal{D}^{tr})$ have bounded variances, i.e., for any $w$ and $\xi$, $\mathbb{E}_{\mathcal{D}^{val}}\Vert\nabla\hat{\mathcal{R}}^{val}(w;\mathcal{D}^{val})-\nabla{\mathcal{R}}^{val}(w)\Vert^2\leq{M^2}/{m^{val}}$, $\mathbb{E}_{\mathcal{D}^{tr}}\Vert\nabla\hat{\mathcal{R}}^{tr}(w;\mathcal{D}^{tr})-\nabla{\mathcal{R}}^{tr}(w)\Vert^2\leq M^2/m^{tr}$, $\mathbb{E}_{\mathcal{D}^{tr}}\Vert\nabla^2_{\boldsymbol{\lambda}, \boldsymbol{\theta}}\hat{\mathcal{R}}^{tr}(w;\mathcal{D}^{tr})-\nabla^2_{\boldsymbol{\lambda}, \boldsymbol{\theta}}{\mathcal{R}}^{tr}(w)\Vert^2\leq L^2/m^{tr}$ and $\mathbb{E}_{\mathcal{D}^{tr}}\Vert\nabla_{\boldsymbol{\theta}}^2\hat{\mathcal{R}}^{tr}(w;\mathcal{D}^{tr})-\nabla_{\boldsymbol{\theta}}^2{\mathcal{R}}^{tr}(w)\Vert^2\leq L^2/m^{tr}$.
\end{lemma}
\noindent\textbf{Proof.} According to the following computation rules, 
\begin{align}\label{equationA2-1}
	&\hat{\mathcal{R}}^{tr}(w;\mathcal{D}^{tr})=\frac{1}{m^{tr}}\sum_{\xi}\hat{\mathcal{R}}^{tr}(w;\xi),\quad\hat{\mathcal{R}}^{val}(w;\mathcal{D}^{val})=\frac{1}{m^{val}}\sum_{\zeta}\hat{\mathcal{R}}^{val}(w;\zeta),\notag \\
    & \mathbb{E}_{\mathcal{D}^{val}}\Vert\nabla\hat{\mathcal{R}}^{val}(w;\mathcal{D}^{val})-\nabla{\mathcal{R}}^{val}(w)\Vert^2=\frac{1}{m^{val}}\mathbb{E}_{\zeta}\Vert\nabla\hat{\mathcal{R}}^{val}(w;\zeta)-\nabla{\mathcal{R}}^{val}(w)\Vert^2\leq \frac{M^2}{m^{val}}.
\end{align}Then, the proof is completed.\hfill $\square$

\begin{lemma}\label{assumA2-3}
	Suppose Assumption \ref{assum33-2} holds. Then for any $(\mathcal{D}^{tr},\mathcal{D}^{val})\sim \mathscr{P}$, the loss functions $\hat{\mathcal{R}}^{val}(w;\mathcal{D}^{val})$ and $\hat{\mathcal{R}}^{tr}(w;\mathcal{D}^{tr})$ satisfy (i) $\hat{\mathcal{R}}^{val}(w;\mathcal{D}^{val})$ is M-Lipschitz; (ii) $\nabla\hat{\mathcal{R}}^{val}(w;\mathcal{D}^{val})$ and $\nabla\hat{\mathcal{R}}^{tr}(w;\mathcal{D}^{tr})$ are L-Lipschitz.
\end{lemma}
\noindent\textbf{Proof.} According to Eq. (\ref{equationA2-1}), we can obtain the above conclusion. \hfill $\square$

Let $\Phi$ satisfy that:
$
\Phi(\boldsymbol{\lambda},\boldsymbol{\theta};\mathcal{D}^{tr})=\boldsymbol{\theta}-\alpha_{in}\nabla_{\boldsymbol{\theta}}\hat{\mathcal{R}}^{tr}(\boldsymbol{\lambda},\boldsymbol{\theta};\mathcal{D}^{tr})
$
is differentiable. Drawing inspiration from \citep{grazzi2020iteration}, we can present an analogous lemma which introduces some additional constants that will occur in the complexity bound.
\begin{lemma}\textbf{\textup{\citep{grazzi2020iteration}}}\label{lemma33-1}
	For any training set $\mathcal{D}^{tr}$, let $\boldsymbol{\lambda}\in\mathbb{R}^p$ and $C_{1,\boldsymbol{\lambda},\mathcal{D}^{tr}}>0$ satisfy $\Vert\boldsymbol{\theta}^*(\boldsymbol{\lambda};\mathcal{D}^{tr})\Vert\leq C_{1, \boldsymbol{\lambda},\mathcal{D}^{tr}}$. Then there exists $ C_{2, \boldsymbol{\lambda},\mathcal{D}^{tr}}\in\mathbb{R}^{+}$ to make it hold that
	$
	\sup_{\Vert\boldsymbol{\theta}\Vert\leq2C_{1,\boldsymbol{\lambda},\mathcal{D}^{tr}}}\Vert\nabla_{\boldsymbol{\lambda}}\Phi(\boldsymbol{\lambda},\boldsymbol{\theta};\mathcal{D}^{tr})\Vert\leq C_{2,\boldsymbol{\lambda},\mathcal{D}^{tr}}.
	$
\end{lemma}
The proof exploits the fact that the image of a continuous function applied to a compact set remains compact. Then, we can give the following lemma to ensure the iteration of inner-level is a contraction. 
\begin{lemma}\textbf{\textup{\citep{grazzi2020iteration}}}\label{lemma-a2-5}
	Suppose Assumptions \ref{assum33-1} and \ref{assum33-2} hold, and then for every $\boldsymbol{\lambda}\in\mathbb{R}^p$, setting $\alpha_{in}=2/(L+\mu)$, $\Phi(\boldsymbol{\lambda},\boldsymbol{\theta})$ is a contraction with constant $q=(L-\mu)/(L+\mu)$. More generally, for $\alpha_{in}\leq2/L$, the constant $q=\max\{1-\alpha_{in}\mu, \alpha_{in} L-1\}$.
\end{lemma}

\subsection{Proof of the hypergradient estimation in ITD}\label{sectionA3}
\textbf{Proof of Lemma \ref{lemma33-3}.} Using Proposition \ref{sec22-propos1}, we have, for $(\xi,\mathcal{D}^{val})$,
\begin{align}\label{equationA3-1}
&\Big{\Vert}\nabla {f}(\boldsymbol{\lambda};(\xi,\mathcal{D}^{val}))\Big{\Vert}\leq \Big{\Vert}\nabla_{\boldsymbol{\lambda}}\hat{\mathcal{R}}^{val}(\boldsymbol{\lambda},\boldsymbol{\theta}_{K}(\xi);\mathcal{D}^{val})\Big{\Vert}+\Big{\Vert}\alpha_{in}\sum_{k=0}^{K-1}\nabla^2_{\boldsymbol{\lambda}, \boldsymbol{\theta}}\hat{\mathcal{R}}^{tr}(\boldsymbol{\lambda},\boldsymbol{\theta};\xi)\Big{|}_{\boldsymbol{\theta}=\boldsymbol{\theta}_{k}(\xi)}\times \notag\\
&\prod_{j=k+1}^{K-1}(I-\alpha_{in}\nabla_{\boldsymbol{\theta}}^2\hat{\mathcal{R}}^{tr}(\boldsymbol{\lambda},\boldsymbol{\theta};\xi))\Big{|}_{\boldsymbol{\theta}=\boldsymbol{\theta}_{j}(\xi)}\nabla_{\boldsymbol{\theta}}\hat{\mathcal{R}}^{val}(\boldsymbol{\lambda},\boldsymbol{\theta};\mathcal{D}^{val}))\Big{|}_{\boldsymbol{\theta}=\boldsymbol{\theta}_{K}(\xi)}\Big{\Vert}\stackrel{\text{(i)}}{\leq}M+\alpha_{in} LM\sum_{k=0}^{K-1} \notag\\
&(1-\alpha_{in}\mu)^{K-k-1}\leq M+\frac{LM}{\mu}.
\end{align}For the given $\mathcal{D}\sim\mathscr{P}$ and random seed $u_1$, we can get the variance as the following:
\begin{align}\label{eq-A-B-40}
		&\mathbb{E}_{\mathcal{D}, u_1}\Big{\Vert}\widehat{\nabla}{f}(\boldsymbol{\lambda};\mathcal{S}_{(\mathcal{D}, u_1)})-\nabla\overline{f}(\boldsymbol{\lambda})\Big{\Vert}^2=\mathbb{E}_{\mathcal{D}^{tr}, \mathcal{D}^{val}}\Big{\Vert}\widehat{\nabla}{f}(\boldsymbol{\lambda};(\mathcal{D}^{tr}, \mathcal{D}^{val}))-\nabla\overline{f}(\boldsymbol{\lambda})\Big{\Vert}^2 \leq \notag\\
	&2\mathbb{E}_{\mathcal{D}^{tr},\mathcal{D}^{val}}\Big{\Vert}\widehat{\nabla}{f}(\boldsymbol{\lambda};(\mathcal{D}^{tr},\mathcal{D}^{val}))-\nabla_{\boldsymbol{\lambda}}\hat{\mathcal{R}}^{val}(\boldsymbol{\lambda}, \boldsymbol{\theta}^*(\boldsymbol{\lambda}); \mathcal{D}^{val})\Big{\Vert}^2+2\mathbb{E}_{\mathcal{D}^{val}}\Big{\Vert}\nabla_{\boldsymbol{\lambda}}\hat{\mathcal{R}}^{val}(\boldsymbol{\lambda}, \notag \\
    &\boldsymbol{\theta}^*(\boldsymbol{\lambda}); \mathcal{D}^{val})-\nabla\overline{f}(\boldsymbol{\lambda})\Big{\Vert}^2 \stackrel{\text{(i)}}{\leq}2\mathbb{E}_{\mathcal{D}^{val}}\Big{[}\mathbb{E}_{\mathcal{D}^{tr}}\Big{\Vert}\widehat{\nabla}{f}(\boldsymbol{\lambda};(\mathcal{D}^{tr}, \mathcal{D}^{val}))-\nabla_{\boldsymbol{\lambda}}\hat{\mathcal{R}}^{val}(\boldsymbol{\lambda}, \boldsymbol{\theta}^*(\boldsymbol{\lambda});  \notag\\
		&\mathcal{D}^{val})\Big{\Vert}^2\Big{]}+\frac{2M^2}{m^{val}}\leq2\mathbb{E}_{\mathcal{D}^{val}}\Big{[}\frac{\mathbb{E}_{\xi}{\Vert}\widehat{\nabla}{f}(\boldsymbol{\lambda};(\xi,\mathcal{D}^{val})){\Vert}^2}{m^{tr}}\Big{]}+\frac{2M^2}{m^{val}}\stackrel{\text{(ii)}}{\leq}\frac{2M^2}{m^{tr}}(1+\frac{L}{\mu})^2+\frac{2M^2}{m^{val}},
\end{align}
where $\nabla_{\boldsymbol{\lambda}}\hat{\mathcal{R}}^{val}(\boldsymbol{\lambda}, \boldsymbol{\theta}^*(\boldsymbol{\lambda}); \mathcal{D}^{val})=\mathbb{E}_{\mathcal{D}^{tr}}[\widehat{\nabla}{f}(\boldsymbol{\lambda};(\mathcal{D}^{tr},\mathcal{D}^{val}))]$, (i) follows from Lemma \ref{lemmaA2-2} and (ii) follows from Eq. (\ref{equationA3-1}). Then, the proof is completed.\hfill $\square$
\newline
\textbf{Proof of Theorem \ref{theorem33-1}.} Using Eq. (\ref{equation32-3}) and Combining the conclusion of Lemma \ref{lemma33-3} and Theorem \ref{theorem32-1}, the proof is completed.\hfill $\square$
\newline
\textbf{Proof of Lemma \ref{lemma332-1}.} For the given samples $\mathcal{D}\sim\mathscr{P}$ and random seeds $\{u_i\}_{i=1}^{U}$, we have 
\begin{align}\label{eq-A-B-40-1}
\mathbb{E}_{\mathcal{D}, \{u_i\}_{i=1}^{U}}\Big{\Vert}\frac{\sum_{i=1}^{U}\widehat{\nabla}{f}(\boldsymbol{\lambda};\mathcal{S}_{(\mathcal{D}, u_i)})}{U}-\widetilde{\nabla}{f}(\boldsymbol{\lambda})\Big{\Vert}^2=\frac{1}{U}\mathbb{E}_{\mathcal{D}, u_i}\Big{\Vert}{\widehat{\nabla}{f}(\boldsymbol{\lambda};\mathcal{S}_{(\mathcal{D}, u_i)})}-\widetilde{\nabla}f(\boldsymbol{\lambda})\Big{\Vert}^2,
\end{align}
where $i=1, 2, \dots, U$. Combining Eq. (\ref{eq-A-B-40}), we get the conclusion of Lemma \ref{lemma332-1}.\hfill $\square$ 
\newline
\textbf{Proof of Theorem \ref{theorem332-1}.} Using Eq. (\ref{equation32-3}) and Combining the conclusion of Lemma \ref{lemma332-1} and Theorem \ref{theorem32-1}, the proof is completed.\hfill $\square$
\subsection{Proof of the hypergradient estimation in AID}\label{sectionA4}
\textbf{Proof of Lemma \ref{lemma333-1}.}
\begin{align}\label{equationA4-1}
&\mathbb{E}_{\mathcal{D}, u_1}\Big{\Vert}\widehat{\nabla}{f}(\boldsymbol{\lambda};\mathcal{S}_{(\mathcal{D}, u_1)})-\widehat{\nabla}{f}(\boldsymbol{\lambda})\Big{\Vert}^2=\mathbb{E}_{\mathcal{D}^{tr}, \mathcal{D}^{val}}\Big{\Vert}\widehat{\nabla}{f}(\boldsymbol{\lambda};(\mathcal{D}^{tr}, \mathcal{D}^{val}))-\widehat{\nabla}{f}(\boldsymbol{\lambda})\Big{\Vert}^2= \notag\\
&\mathbb{E}_{\mathcal{D}^{tr},\mathcal{D}^{val}}[\Vert\nabla_{\boldsymbol{\lambda}}\hat{\mathcal{R}}^{val}(\boldsymbol{\lambda},\boldsymbol{\theta};\mathcal{D}^{val})-\nabla^2_{\boldsymbol{\lambda}, \boldsymbol{\theta}}\hat{\mathcal{R}}^{tr}(\boldsymbol{\lambda},\boldsymbol{\theta};\mathcal{D}^{tr})[\nabla_{\boldsymbol{\theta}}^2\hat{\mathcal{R}}^{tr}(\boldsymbol{\lambda},\boldsymbol{\theta};\mathcal{D}^{tr})]^{-1}\nabla_{\boldsymbol{\theta}}\hat{\mathcal{R}}^{val}(
\notag\\
&\boldsymbol{\lambda},\boldsymbol{\theta};\mathcal{D}^{val})-\nabla_{\boldsymbol{\lambda}}\hat{\mathcal{R}}^{val}(\boldsymbol{\lambda},\boldsymbol{\theta})+\nabla^2_{\boldsymbol{\lambda}, \boldsymbol{\theta}}\hat{\mathcal{R}}^{tr}(\boldsymbol{\lambda},\boldsymbol{\theta})[\nabla_{\boldsymbol{\theta}}^2\hat{\mathcal{R}}^{tr}(\boldsymbol{\lambda},\boldsymbol{\theta})]^{-1}\nabla_{\boldsymbol{\theta}}\hat{\mathcal{R}}^{val}(\boldsymbol{\lambda},\boldsymbol{\theta})\Vert^2]\stackrel{\text{(i)}}{\leq} \notag\\
&\frac{2M^2}{m^{val}}+2\mathbb{E}_{\mathcal{D}^{tr},\mathcal{D}^{val}}[\Vert\nabla^2_{\boldsymbol{\lambda}, \boldsymbol{\theta}}\hat{\mathcal{R}}^{tr}(\boldsymbol{\lambda},\boldsymbol{\theta};\mathcal{D}^{tr})[\nabla_{\boldsymbol{\theta}}^2\hat{\mathcal{R}}^{tr}(\boldsymbol{\lambda},\boldsymbol{\theta};\mathcal{D}^{tr})]^{-1}\nabla_{\boldsymbol{\theta}}\hat{\mathcal{R}}^{val}(\boldsymbol{\lambda},\boldsymbol{\theta};\mathcal{D}^{val})- \notag\\
&\nabla^2_{\boldsymbol{\lambda}, \boldsymbol{\theta}}\hat{\mathcal{R}}^{tr}(\boldsymbol{\lambda},\boldsymbol{\theta})[\nabla_{\boldsymbol{\theta}}^2\hat{\mathcal{R}}^{tr}(\boldsymbol{\lambda},\boldsymbol{\theta})]^{-1}\nabla_{\boldsymbol{\theta}}\hat{\mathcal{R}}^{val}(\boldsymbol{\lambda},\boldsymbol{\theta})\Vert^2]\stackrel{\text{(ii)}}{\leq}\frac{2M^2}{m^{val}}+\frac{4M^2}{\mu^2}\mathbb{E}_{\mathcal{D}^{tr}}[\Vert\nabla^2_{\boldsymbol{\lambda}, \boldsymbol{\theta}}\hat{\mathcal{R}}^{tr}(\boldsymbol{\lambda},\notag \\
&\boldsymbol{\theta};\mathcal{D}^{tr})-\nabla^2_{\boldsymbol{\lambda}, \boldsymbol{\theta}}\hat{\mathcal{R}}^{tr}(\boldsymbol{\lambda},\boldsymbol{\theta})\Vert^2]+4L^2\mathbb{E}_{\mathcal{D}^{tr},\mathcal{D}^{val}}[\Vert[\nabla_{\boldsymbol{\theta}}^2\hat{\mathcal{R}}^{tr}(\boldsymbol{\lambda},\boldsymbol{\theta};\mathcal{D}^{tr})]^{-1}\nabla_{\boldsymbol{\theta}}\hat{\mathcal{R}}^{val}(\boldsymbol{\lambda},\boldsymbol{\theta};\mathcal{D}^{val})\notag \\
&-[\nabla_{\boldsymbol{\theta}}^2\hat{\mathcal{R}}^{tr}(\boldsymbol{\lambda},\boldsymbol{\theta})]^{-1}\nabla_{\boldsymbol{\theta}}\hat{\mathcal{R}}^{val}(\boldsymbol{\lambda},\boldsymbol{\theta})\Vert^2]\stackrel{\text{(iii)}}{\leq}\frac{2M^2}{m^{val}}+\frac{4L^2M^2}{\mu^2m^{tr}}+\frac{8L^2M^2}{\mu^2m^{tr}}+\frac{8L^2M^2}{\mu^2m^{val}}=\notag
\\
&\frac{2M^2}{m^{val}}\Big{(}1+\frac{4L^2}{\mu^2}\Big{)}+\frac{12L^2M^2}{\mu^2m^{tr}},
\end{align}
where (i) holds based on Lemma \ref{lemmaA2-2}, (ii) holds based on Assumption \ref{assum33-2} and Young's inequality, and (iii) holds based on Assumption \ref{assum33-2}, Lemma \ref{lemmaA2-2} and the strong convexity of $\hat{\mathcal{R}}^{tr}$. Then, the proof is completed. \hfill $\square$

The proofs of Theorems \ref{theorem333-1}, \ref{theorem334-1} and Lemma \ref{lemma334-1} are similar with the above section.
\section{Analysis of Excess Error}
We present the bound of excess error for ITD on a single data splitting by Eq. (\ref{equation81-1}), and also provide the proof across multiple data splittings in the main paper.
\subsection{Excess Error Anlysis of ITD on splitting $\mathcal{S}_{(\mathcal{D}, u_1)}$}
\subsubsection{training error of ITD on splitting $\mathcal{S}_{(\mathcal{D}, u_1)}$}
We firstly give the bound of training error which is required to provide characterization of the excess error of ITD.
\begin{proposition}\label{proposition811-1}
For the given samples $\mathcal{D}\sim\mathscr{P}$ and data splittings $\mathcal{S}_{(\mathcal{D},u_1)}=(\mathcal{D}^{tr}, \mathcal{D}^{val})$, suppose that Assumptions \ref{assum33-1} and \ref{assum33-2} hold. Set $\alpha_{in}\leq\frac{2}{L}$ and $\alpha_{out}=\frac{\ln q^{-1}}{L\ln3}$. Then, we have
\begin{align}\label{equation411-2}
&\hat{\mathcal{R}}^{val}(\boldsymbol{\lambda}_{\mathcal{D}, u_1},\boldsymbol{\theta}_{\mathcal{D}, u_1}(\boldsymbol{\lambda}_{\mathcal{D}, u_1});  \mathcal{S}_{(\mathcal{D}, u_1)})-\hat{\mathcal{R}}^{val}(\boldsymbol{\lambda}^*_{\mathcal{D}, u_1},\boldsymbol{\theta}^*_{\mathcal{D}, u_1}(\boldsymbol{\lambda}^*_{\mathcal{D}, u_1});  \mathcal{S}_{(\mathcal{D}, u_1)})\notag\\
\leq&\frac{M}{L}\sqrt{\text{TE1}+\text{TE2}}+\text{TE3}+\text{TE4},
\end{align}where $B_{\text{ITD}, u_1}(\boldsymbol{\lambda}, K):= \Big{(}2LC_{1,\boldsymbol{\lambda},\mathcal{D}^{tr}}(1+\frac{C_{2,\boldsymbol{\lambda},\mathcal{D}^{tr}}}{1-q})(1+\frac{MK}{q})+\frac{MC_{2,\boldsymbol{\lambda},\mathcal{D}^{tr}}}{1-q}\Big{)}q^K$. $C_{1,\boldsymbol{\lambda},\mathcal{D}^{tr}}$ and $C_{2,\boldsymbol{\lambda},\mathcal{D}^{tr}}$ are constants introduced in Lemma \ref{lemma33-1}, $q\in(0, 1)$ is a constant introduced in Lemma \ref{lemma-a2-5}, and $\text{TE1}={
\big{(}\big{(}\alpha_{out} L(\alpha_{in} L+1)^K+1\big{)}^T-1\big{)}^2\cdot B_{\text{ITD},u_1}^2(\boldsymbol{\lambda}, K)}$, $\text{TE2}={2M^2}\big{(}\big{(}\alpha_{out} L(\alpha_{in} L+1)^K+1\big{)}^T-1\big{)}^2\Big{(}\frac{1}{m^{tr}}(1+\frac{L}{\mu})^2+\frac{1}{m^{val}}\Big{)}$, $\text{TE3}=M(\alpha_{in} L+1)^K\Big{[}\big{\Vert}\boldsymbol{\lambda}^{*(T)}_{\mathcal{D}, u_1}-\boldsymbol{\lambda}^*_{\mathcal{D}, u_1}\big{\Vert}\Big{]}$ and $\text{TE4}=M\Big{\Vert}\boldsymbol{\theta}_{\mathcal{D}, u_1}(\boldsymbol{\lambda}^*_{\mathcal{D}, u_1})-\boldsymbol{\theta}^*_{\mathcal{D}, u_1}(\boldsymbol{\lambda}^*_{\mathcal{D}, u_1})\Big{\Vert}$.
\end{proposition}\noindent\textbf{Discussion:} Noting that if we set $\alpha_{in}\leq{2}/{L}$ and $\alpha_{out}={\ln q^{-1}}/{(L\ln3)}$, as the number of inner iteration $K$ increases, TE1 and TE2 will both increase. \footnote{The trend of TE1 w.r.t. $K$ is associated with the learning rate. Further analysis is provided in Appendix \ref{sectionA14}.} As the number of outer iteration $T$ increases, the overall training error will also increase. For TE3, we have $\boldsymbol{\lambda}^{*(T)}=\boldsymbol{\lambda}^{(0)}-\alpha_{out}\sum_{t-0}^{T-1}\nabla_{\boldsymbol{\lambda}}R^{val}(\boldsymbol{\lambda},\boldsymbol{\theta}){|}_{\boldsymbol{\lambda}=\boldsymbol{\lambda}^{*(t)}}$,
and using gradient descent can yield a (local) minimum for $R^{val}(\boldsymbol{\lambda},\boldsymbol{\theta}^*(\boldsymbol{\lambda}))$. {However, since $\boldsymbol{\lambda}^*:=\arg\min_{\boldsymbol{\lambda}\in\boldsymbol{\Lambda}}R^{val}(\boldsymbol{\lambda},\boldsymbol{\theta}^*(\boldsymbol{\lambda}))$ is the global minimum, we need certain conditions to ensure that $\boldsymbol{\lambda}^*$ can be obtained from $\boldsymbol{\lambda}^{(0)}$ by gradient descent. Firstly, because $R^{val}(\boldsymbol{\lambda},\boldsymbol{\theta}^*(\boldsymbol{\lambda}))$ satisfies Assumption \ref{assum33-2}, there exists a $\boldsymbol{\lambda}^{(0)}$ and an $\alpha_{out}$ such that the trajectory obtained by gradient descent leads to $\boldsymbol{\lambda}^*$. Additionally, we can restrict $\boldsymbol{\Lambda}$ to $\boldsymbol{\Lambda}_1$, which is a range with a single local optimum to ensure that $\boldsymbol{\lambda}^*$ can be obtained through gradient descent. In this case, $\lim_{T\to\infty}\Vert\boldsymbol{\lambda}^{*(T)}-\boldsymbol{\lambda}^*\Vert=0$, where $\boldsymbol{\lambda}^*=\arg\min_{\boldsymbol{\lambda}\in\boldsymbol{\Lambda}_1}R^{val}(\boldsymbol{\lambda},\boldsymbol{\theta}^*(\boldsymbol{\lambda}))$ and $\boldsymbol{\Lambda}_1\subset\boldsymbol{\Lambda}$.} As for TE4, \cite{liu2021investigating} introduce two
elementary properties on ITD algorithm as follows: (1) \textbf{Uniform approximation quality to the inner-level solution:} $\{\boldsymbol{\theta}_K(\boldsymbol{\lambda})\}$ is uniformly bounded on $\boldsymbol{\Lambda}$, and for any $\epsilon>0$, there exists $k(\epsilon)>0$ such that whenever
$K>k(\epsilon)$, we have $\sup_{\boldsymbol{\lambda}\in\boldsymbol{\Lambda}}\{{\mathcal{R}}^{val}(\boldsymbol{\lambda},\boldsymbol{\theta}_K(\boldsymbol{\lambda}))-\min_{\boldsymbol{\theta}}{\mathcal{R}}^{val}(\boldsymbol{\lambda},\boldsymbol{\theta}(\boldsymbol{\lambda}))\}\leq\epsilon$ or $\sup_{\boldsymbol{\lambda}\in\boldsymbol{\Lambda}}\Vert\nabla_{\boldsymbol{\theta}}{\mathcal{R}}^{val}(\boldsymbol{\lambda},\boldsymbol{\theta})|_{\boldsymbol{\theta}=\boldsymbol{\theta}_K(\boldsymbol{\lambda})}\Vert\leq\epsilon$. (2) \textbf{Point-wise approximation quality to the inner-level solution:} For each $\boldsymbol{\lambda}\in\boldsymbol{\Lambda}$, we have $\lim_{K\to\infty}\text{dist}(\boldsymbol{\theta}_K(\boldsymbol{\lambda}),\mathcal{S}(\boldsymbol{\lambda}))=0$, where $\mathcal{S}(\boldsymbol{\lambda})$ represents the solution set of the inner-level subproblem of ITD algorithm and $\text{dist}(\cdot, \cdot)$ denotes the point-to-set distance. Equipped with the above two properties on $\boldsymbol{\theta}_K(\boldsymbol{\lambda})$ and Assumption \ref{assum33-1}, the solution only has one element $\boldsymbol{\theta}^*(\boldsymbol{\lambda})$, and we use $\Vert\cdot\Vert$ as $\text{dist}(\cdot)$, so we have $\lim_{K\to\infty}\Vert\boldsymbol{\theta}_K(\boldsymbol{\lambda})-\boldsymbol{\theta}^*(\boldsymbol{\lambda})\Vert=0$.

\noindent\textbf{Proof.} We have
\begin{align}\label{equationA5-2}
		&\Big{|}\hat{\mathcal{R}}^{val}(\boldsymbol{\lambda}_{\mathcal{D}, u_1},\boldsymbol{\theta}_{\mathcal{D}, u_1}(\boldsymbol{\lambda}_{\mathcal{D}, u_1});  \mathcal{S}_{(\mathcal{D}, u_1)})-\hat{\mathcal{R}}^{val}(\boldsymbol{\lambda}^*_{\mathcal{D}, u_1},\boldsymbol{\theta}^*_{\mathcal{D}, u_1}(\boldsymbol{\lambda}^*_{\mathcal{D}, u_1}); \mathcal{S}_{(\mathcal{D}, u_1)})\Big{|}\leq\Big{|}\hat{\mathcal{R}}^{val}(\notag\\
		&\boldsymbol{\lambda}_{\mathcal{D}, u_1},\boldsymbol{\theta}_{\mathcal{D}, u_1}(\boldsymbol{\lambda}_{\mathcal{D}, u_1});  \mathcal{S}_{(\mathcal{D}, u_1)})-\hat{\mathcal{R}}^{val}(\boldsymbol{\lambda}_{\mathcal{D}, u_1},\boldsymbol{\theta}_{\mathcal{D}, u_1}(\boldsymbol{\lambda}^*_{\mathcal{D}, u_1});  \mathcal{S}_{(\mathcal{D}, u_1)})\Big{|}+\Big{|}\hat{\mathcal{R}}^{val}(\boldsymbol{\lambda}_{\mathcal{D}, u_1},\notag\\
		&\boldsymbol{\theta}_{\mathcal{D}, u_1}(\boldsymbol{\lambda}^*_{\mathcal{D}, u_1});  \mathcal{S}_{(\mathcal{D}, u_1)})-\hat{\mathcal{R}}^{val}(\boldsymbol{\lambda}^*_{\mathcal{D}, u_1},\boldsymbol{\theta}_{\mathcal{D}, u_1}(\boldsymbol{\lambda}^*_{\mathcal{D}, u_1});  \mathcal{S}_{(\mathcal{D}, u_1)})\Big{|}+\Big{|}\hat{\mathcal{R}}^{val}(\boldsymbol{\lambda}^*_{\mathcal{D}, u_1},\boldsymbol{\theta}_{\mathcal{D}, u_1}(\notag\\
		&\boldsymbol{\lambda}^*_{\mathcal{D}, u_1});  \mathcal{S}_{(\mathcal{D}, u_1)})-\hat{\mathcal{R}}^{val}(\boldsymbol{\lambda}^*_{\mathcal{D}, u_1},\boldsymbol{\theta}^*_{\mathcal{D}, u_1}(\boldsymbol{\lambda}^*_{\mathcal{D}, u_1});  \mathcal{S}_{(\mathcal{D}, u_1)})\Big{|}\leq M\Big{\Vert}\boldsymbol{\theta}_{\mathcal{D}, u_1}(\boldsymbol{\lambda}_{\mathcal{D}, u_1})-\boldsymbol{\theta}_{\mathcal{D}, u_1}(\notag\\
		&\boldsymbol{\lambda}^*_{\mathcal{D}, u_1})\Big{\Vert}+M\Big{\Vert}\boldsymbol{\lambda}_{\mathcal{D}, u_1}-\boldsymbol{\lambda}^*_{\mathcal{D}, u_1}\Big{\Vert}+M\Big{\Vert}\boldsymbol{\theta}_{\mathcal{D}, u_1}(\boldsymbol{\lambda}^*_{\mathcal{D}, u_1})-\boldsymbol{\theta}^*_{\mathcal{D}, u_1}(\boldsymbol{\lambda}^*_{\mathcal{D}, u_1})\Big{\Vert}.
\end{align}
The first term measures the parameter variation resulting from changes in hyperparameter. Therefore, when the update method for $\boldsymbol{\theta}$ is $K$-step gradient descent, the update formula of $\boldsymbol{\theta}$ is as follows.
It is worth noted that the data used for performing gradient descent may be obtained by some ways from $\mathcal{D}$, and the specific cases depend on the actual available training data. Generally, we use $\mathcal{D}^{tr}$. \footnote{A detailed discussion of the specific form of training data is unnecessary regarding the factors influencing parameter variation since it does not depend on the specific form.}\label{discussionAB-1}
\begin{align}\label{equationA5-3}
&\boldsymbol{\theta}_{\mathcal{D}, u_1}(\boldsymbol{\lambda}_{\mathcal{D}, u_1})=\boldsymbol{\theta}_0-\alpha_{in}\sum_{k=0}^{K-1}\nabla_{\boldsymbol{\theta}}\hat{\mathcal{R}}^{tr}(\boldsymbol{\lambda}_{\mathcal{D}, u_1},\boldsymbol{\theta};\mathcal{S}_{(\mathcal{D}, u_1)})\Big{|}_{\boldsymbol{\theta}=\boldsymbol{\theta}_k(\boldsymbol{\lambda}_{\mathcal{D}, u_1};\mathcal{S}_{(\mathcal{D}, u_1)})},\notag\\
&\boldsymbol{\theta}_{\mathcal{D}, u_1}(\boldsymbol{\lambda}^*_{\mathcal{D}, u_1})=\boldsymbol{\theta}_0-\alpha_{in}\sum_{k=0}^{K-1}\nabla_{\boldsymbol{\theta}}\hat{\mathcal{R}}^{tr}(\boldsymbol{\lambda}^*_{\mathcal{D}, u_1},\boldsymbol{\theta};\mathcal{S}_{(\mathcal{D}, u_1)})\Big{|}_{\boldsymbol{\theta}=\boldsymbol{\theta}_k(\boldsymbol{\lambda}^*_{\mathcal{D}, u_1};\mathcal{S}_{(\mathcal{D}, u_1)})}.
\end{align}
Taking Eq. (\ref{equationA5-3}) into the first term in Eq. (\ref{equationA5-2}), we have
\begin{align}\label{equationA5-4}
&\Big{\Vert}\boldsymbol{\theta}_{\mathcal{D}, u_1}(\boldsymbol{\lambda}_{\mathcal{D}, u_1})-\boldsymbol{\theta}_{\mathcal{D}, u_1}(\boldsymbol{\lambda}^*_{\mathcal{D}, u_1})\Big{\Vert}\leq\alpha_{in}\sum_{k=0}^{K-1}\Big{\Vert}\nabla_{\boldsymbol{\theta}}\hat{\mathcal{R}}^{tr}(\boldsymbol{\lambda}_{\mathcal{D}, u_1},\boldsymbol{\theta};\mathcal{S}_{(\mathcal{D}, u_1)})\Big{|}_{\boldsymbol{\theta}=\boldsymbol{\theta}_k(\boldsymbol{\lambda}_{\mathcal{D}, u_1})}-\notag\\
&\nabla_{\boldsymbol{\theta}}\hat{\mathcal{R}}^{tr}(\boldsymbol{\lambda}^*_{\mathcal{D}, u_1},\boldsymbol{\theta};\mathcal{S}_{(\mathcal{D}, u_1)})\Big{|}_{\boldsymbol{\theta}=\boldsymbol{\theta}_k(\boldsymbol{\lambda}^*_{\mathcal{D}, u_1};)}\Big{\Vert}\leq\alpha_{in}\sum_{k=0}^{K-1}\Big{(}L\Vert\boldsymbol{\theta}_k(\boldsymbol{\lambda}_{\mathcal{D}, u_1};\mathcal{S}_{(\mathcal{D}, u_1)})-\boldsymbol{\theta}_k(\boldsymbol{\lambda}^*_{\mathcal{D}, u_1};\notag\\
&\mathcal{S}_{(\mathcal{D}, u_1)})\Vert+L\Vert\boldsymbol{\lambda}_{\mathcal{D}, u_1}-\boldsymbol{\lambda}^*_{\mathcal{D}, u_1}\Vert\Big{)}\leq\big{(}(\alpha_{in} L+1)^K-1\big{)}\big{\Vert}\boldsymbol{\lambda}_{\mathcal{D}, u_1}-\boldsymbol{\lambda}^*_{\mathcal{D}, u_1}\big{\Vert}.
\end{align}
Taking Eq. (\ref{equationA5-4}) into Eq. (\ref{equationA5-2}), we have
\begin{align}\label{equationA5-5}
&\hat{\mathcal{R}}^{val}(\boldsymbol{\lambda}_{\mathcal{D}, u_1},\boldsymbol{\theta}_{\mathcal{D}, u_1}(\boldsymbol{\lambda}_{\mathcal{D}, u_1});  \mathcal{S}_{(\mathcal{D}, u_1)})-\hat{\mathcal{R}}^{val}(\boldsymbol{\lambda}^*_{\mathcal{D}, u_1},\boldsymbol{\theta}^*_{\mathcal{D}, u_1}(\boldsymbol{\lambda}^*_{\mathcal{D}, u_1});  \mathcal{S}_{(\mathcal{D}, u_1)})\leq M{\Vert}\boldsymbol{\theta}_{\mathcal{D}, u_1}(\notag\\
&\boldsymbol{\lambda}^*_{\mathcal{D}, u_1})-\boldsymbol{\theta}^*_{\mathcal{D}, u_1}(\boldsymbol{\lambda}^*_{\mathcal{D}, u_1}){\Vert}+M(\alpha_{in} L+1)^K{(}\big{\Vert}\boldsymbol{\lambda}_{\mathcal{D}, u_1}-\boldsymbol{\lambda}^{*(T)}_{\mathcal{D}, u_1}\big{\Vert}+\big{\Vert}\boldsymbol{\lambda}^{*(T)}_{\mathcal{D}, u_1}-\boldsymbol{\lambda}^*_{\mathcal{D}, u_1}\big{\Vert}{)}.
\end{align}
Next, according to the update rule of hyperparameter $\boldsymbol{\lambda}$ as the following:
\begin{align}\label{equationA5-6}
		&\boldsymbol{\lambda}_{\mathcal{D}, u_1}=\boldsymbol{\lambda}^{(T)}_{\mathcal{D}, u_1}=\boldsymbol{\lambda}^{(0)}-\alpha_{out}\sum_{t=0}^{T-1}\nabla_{\boldsymbol{\lambda}}\hat{\mathcal{R}}^{val}(\boldsymbol{\lambda},\boldsymbol{\theta}_{\mathcal{D}, u_1}(\boldsymbol{\lambda});\mathcal{S}_{(\mathcal{D}, u_1)})\Big{|}_{\boldsymbol{\lambda}=\boldsymbol{\lambda}^{(t)}_{\mathcal{D}, u_1}},\notag\\
		&\boldsymbol{\lambda}^{*(T)}_{\mathcal{D}, u_1}=\boldsymbol{\lambda}^{(0)}-\alpha_{out}\sum_{t=0}^{T-1}\nabla_{\boldsymbol{\lambda}}{\mathcal{R}}^{val}(\boldsymbol{\lambda},\boldsymbol{\theta}^*_{\mathcal{D}, u_1}(\boldsymbol{\lambda}))\Big{|}_{\boldsymbol{\lambda}=\boldsymbol{\lambda}^{*(t)}_{\mathcal{D}, u_1}}.
\end{align}
Taking Eq. (\ref{equationA5-6}) into the following part of Eq. (\ref{equationA5-5}), we have:
\begin{align}\label{equationA5-7}
&M(\alpha_{in} L+1)^K\big{\Vert}\boldsymbol{\lambda}_{\mathcal{D}, u_1}-\boldsymbol{\lambda}^{*(T)}_{\mathcal{D}, u_1}\big{\Vert}\leq M(\alpha_{in} L+1)^K\Big{(}\alpha_{out}\sum_{t=0}^{T-1}\big{\Vert}\nabla_{\boldsymbol{\lambda}}\hat{\mathcal{R}}^{val}(\boldsymbol{\lambda},\boldsymbol{\theta}_{\mathcal{D}, u_1}(\boldsymbol{\lambda});\notag\\
&\mathcal{S}_{(\mathcal{D}, u_1)})\Big{|}_{\boldsymbol{\lambda}=\boldsymbol{\lambda}^{(t)}_{\mathcal{D}, u_1}}-\nabla_{\boldsymbol{\lambda}}{\mathcal{R}}^{val}(\boldsymbol{\lambda},\boldsymbol{\theta}^*_{\mathcal{D}, u_1}(\boldsymbol{\lambda}))\Big{|}_{\boldsymbol{\lambda}=\boldsymbol{\lambda}^{*(t)}_{\mathcal{D}, u_1}}\big{\Vert}\Big{)}.
\end{align}
Next, for the $\Vert\cdot\Vert$ term in the last line of Eq. (\ref{equationA5-7}), we have:
\begin{align}\label{equationA5-8}
&\big{\Vert}\nabla_{\boldsymbol{\lambda}}\hat{\mathcal{R}}^{val}(\boldsymbol{\lambda},\boldsymbol{\theta}_{\mathcal{D}, u_1}(\boldsymbol{\lambda});\mathcal{S}_{(\mathcal{D}, u_1)})\Big{|}_{\boldsymbol{\lambda}=\boldsymbol{\lambda}^{(t)}_{\mathcal{D}, u_1}}-\nabla_{\boldsymbol{\lambda}}{\mathcal{R}}^{val}(\boldsymbol{\lambda},\boldsymbol{\theta}^{*}_{\mathcal{D}, u_1}(\boldsymbol{\lambda}))\Big{|}_{\boldsymbol{\lambda}=\boldsymbol{\lambda}^{*(t)}_{\mathcal{D}, u_1}}\big{\Vert}\leq\notag\\
&\big{\Vert}\nabla_{\boldsymbol{\lambda}}\hat{\mathcal{R}}^{val}(\boldsymbol{\lambda},\boldsymbol{\theta}_{\mathcal{D}, u_1}(\boldsymbol{\lambda});\mathcal{S}_{(\mathcal{D}, u_1)})\Big{|}_{\boldsymbol{\lambda}=\boldsymbol{\lambda}^{(t)}_{\mathcal{D}, u_1}}-\nabla_{\boldsymbol{\lambda}}\hat{\mathcal{R}}^{val}(\boldsymbol{\lambda},\boldsymbol{\theta}_{\mathcal{D}, u_1}(\boldsymbol{\lambda});\mathcal{S}_{(\mathcal{D}, u_1)})\Big{|}_{\boldsymbol{\lambda}=\boldsymbol{\lambda}^{*(t)}_{\mathcal{D}, u_1}}\big{\Vert}+\notag\\
&\big{\Vert}\nabla_{\boldsymbol{\lambda}}\hat{\mathcal{R}}^{val}(\boldsymbol{\lambda},\boldsymbol{\theta}_{\mathcal{D}, u_1}(\boldsymbol{\lambda});\mathcal{S}_{(\mathcal{D}, u_1)})\Big{|}_{\boldsymbol{\lambda}=\boldsymbol{\lambda}^{*(t)}_{\mathcal{D}, u_1}}-\nabla_{\boldsymbol{\lambda}}{\mathcal{R}}^{val}(\boldsymbol{\lambda},\boldsymbol{\theta}^*_{\mathcal{D}, u_1}(\boldsymbol{\lambda}))\Big{|}_{\boldsymbol{\lambda}=\boldsymbol{\lambda}^{*(t)}_{\mathcal{D}, u_1}}\big{\Vert}.
\end{align}
We write $\nabla_{\boldsymbol{\lambda}}\hat{\mathcal{R}}^{val}(\boldsymbol{\lambda},\boldsymbol{\theta}_{\mathcal{D}, u_1}(\boldsymbol{\lambda});\mathcal{S}_{(\mathcal{D}, u_1)})$ as $\nabla\hat{\mathcal{R}}^{val}(\boldsymbol{\lambda}^{(t)}_{\mathcal{D}, u_1},\boldsymbol{\theta}_{\mathcal{D}, u_1}(\boldsymbol{\lambda}^{(t)}_{\mathcal{D}, u_1});\mathcal{S}_{(\mathcal{D}, u_1)})$, and then we have:
\begin{align}\label{equationA5-9}
&\Vert\nabla\hat{\mathcal{R}}^{val}(\boldsymbol{\lambda}^{(t)}_{\mathcal{D}, u_1},\boldsymbol{\theta}_{\mathcal{D}, u_1}(\boldsymbol{\lambda}^{(t)}_{\mathcal{D}, u_1});\mathcal{S}_{(\mathcal{D}, u_1)})-\nabla\hat{\mathcal{R}}^{val}(\boldsymbol{\lambda}^{*(t)}_{\mathcal{D}, u_1},\boldsymbol{\theta}_{\mathcal{D}, u_1}(\boldsymbol{\lambda}^{*(t)}_{\mathcal{D}, u_1});\mathcal{S}_{(\mathcal{D}, u_1)})\Vert\leq L\Vert\notag\\
&\boldsymbol{\theta}_{\mathcal{D}, u_1}(\boldsymbol{\lambda}^{(t)}_{\mathcal{D}, u_1})-\boldsymbol{\theta}_{\mathcal{D}, u_1}(\boldsymbol{\lambda}^{*(t)}_{\mathcal{D}, u_1})\Vert+L\Vert\boldsymbol{\lambda}^{(t)}_{\mathcal{D}, u_1}-\boldsymbol{\lambda}^{*(t)}_{\mathcal{D}, u_1}\Vert\leq L(\alpha_{in}L+1)^K\Vert\boldsymbol{\lambda}^{(t)}_{\mathcal{D}, u_1}-\boldsymbol{\lambda}^{*(t)}_{\mathcal{D}, u_1}\Vert.
\end{align}
For the second term in the last line of Eq. (\ref{equationA5-8}), this is the expectation form of hypergradient error, and we write this term as $\text{Err}_{\text{hg}}(\mathcal{D},u_1)$. Then taking Eq. (\ref{equationA5-9}) into Eq. (\ref{equationA5-7}) and the first term of (\ref{equationA5-8}), we have:
\begin{align}\label{equationA5-10}
		&\Vert\boldsymbol{\lambda}^{(T)}_{\mathcal{D},u_1}-\boldsymbol{\lambda}^{*(T)}_{\mathcal{D},u_1}\Vert\leq\alpha_{out}\sum_{t=0}^{T-1}\Big{(}L(\alpha_{in} L+1)^K\big{(}\Vert\boldsymbol{\lambda}^{(t)}_{\mathcal{D},u_1}-\boldsymbol{\lambda}^{*(t)}_{\mathcal{D},u_1}\Vert+\text{Err}_{\text{hg}}(\mathcal{D},u_1)\big{)}\leq\notag\\
		&\frac{\big{(}\alpha_{out} L(\alpha_{in} L+1)^K+1\big{)}^T-1}{L(\alpha_{in} L+1)^K}\cdot\text{Err}_{\text{hg}}(\mathcal{D},u_1).
\end{align}
Then, taking Eq. (\ref{equationA5-10}) into Eq. (\ref{equationA5-2}), we have:
\begin{align}\label{equationAB-12}
&\hat{\mathcal{R}}^{val}(\boldsymbol{\lambda}_{\mathcal{D}, u_1},\boldsymbol{\theta}_{\mathcal{D}, u_1}(\boldsymbol{\lambda}_{\mathcal{D}, u_1});  \mathcal{S}_{(\mathcal{D}, u_1)})-\hat{\mathcal{R}}^{val}(\boldsymbol{\lambda}^*_{\mathcal{D}, u_1},\boldsymbol{\theta}^*_{\mathcal{D}, u_1}(\boldsymbol{\lambda}^*_{\mathcal{D}, u_1});  \mathcal{S}_{(\mathcal{D}, u_1)})\leq\notag\\
&M\frac{\big{(}\alpha_{out} L(\alpha_{in} L+1)^K+1\big{)}^T-1}{L}\cdot\text{Err}_{\text{hg}}(\mathcal{D},u_1)+M(\alpha_{in} L+1)^K\Big{[}\big{\Vert}\boldsymbol{\lambda}^{*(T)}_{\mathcal{D}, u_1}-\boldsymbol{\lambda}^*_{\mathcal{D}, u_1}\big{\Vert}\Big{]}+\notag\\
&M\Big{\Vert}\boldsymbol{\theta}_{\mathcal{D}, u_1}(\boldsymbol{\lambda}^*_{\mathcal{D}, u_1})-\boldsymbol{\theta}^*_{\mathcal{D}, u_1}(\boldsymbol{\lambda}^*_{\mathcal{D}, u_1})\Big{\Vert}.
\end{align}
Using Theorem \ref{theorem33-1} and Cauchy–Schwarz inequality, the proof is completed.\hfill $\square$

\subsubsection{generalization error of ITD on splitting $\mathcal{S}_{(\mathcal{D}, u_1)}$}
To establish the bound of  generalization error, we firstly define the following notion of \textit{uniform stability} for an HPO algorithm. For the sake of uniform notation, we have $\mathcal{A}_{\text{hpo}}(\mathcal{D}^{tr},(\mathcal{D}^{tr}, \mathcal{D}^{val}))=\big{(}\boldsymbol{\lambda}_{\mathcal{D}^{tr},\mathcal{D}^{val} },\boldsymbol{\theta}_{\mathcal{D}^{tr}}(\boldsymbol{\lambda}_{\mathcal{D}^{tr},\mathcal{D}^{val} })\big{)}$, where $(\mathcal{D}^{tr}, \mathcal{D}^{val})=\mathcal{S}_{(\mathcal{D}, u_1)}$, the inputs include the data of $\boldsymbol{\theta}$ and the data splitting of $\boldsymbol{\lambda}$, the outputs include hyperparameter $\boldsymbol{\lambda}$ and the final parameter $\boldsymbol{\theta}$. For the HPO algorithm $\mathcal{A}_{\text{hpo}}$, since the output space is given by $\boldsymbol{\Lambda}\times\boldsymbol{\Theta}$, both training and validation data may be crucial.

The following definition of $\beta$-uniformly stable on training is similar to \citep{mohri2018foundations, bao2021stability}.
\begin{definition}\label{definition412-1}
An HPO algorithm $\mathcal{A}_{\text{hpo}}$ is $\beta$-uniformly stable on training if for all training samples $\mathcal{D}^{tr},{\mathcal{D}^{tr}}'\in\mathscr{P}$ such that $\mathcal{D}^{tr},{\mathcal{D}^{tr}}'$ differ in at most one sample, we have
\begin{align*}
\forall z\in\mathscr{P}, \hat{\mathcal{R}}^{val}(\mathcal{A}_{\text{hpo}}(\mathcal{D}^{tr},(\mathcal{D}^{tr}, \mathcal{D}^{val}));z)-\hat{\mathcal{R}}^{val}(\mathcal{A}_{\text{hpo}}({\mathcal{D}^{tr}}',({\mathcal{D}^{tr}}', \mathcal{D}^{val}));z)\leq\beta.
\end{align*}
\end{definition}If an HPO algorithm is $\beta$-uniformly stable on training, then we have the following high
probability bound.
\begin{theorem}\label{theorem812-1}
For the given multiple observed samples $\mathcal{D}\sim\mathscr{P}$, random seed $u_1$ and $\mathcal{S}_{(\mathcal{D}, u_1)}=(\mathcal{D}^{tr}, \mathcal{D}^{val})$. Suppose an HPO algorithm $\mathcal{A}_{\text{hpo}}$ is $\beta$-uniformly stable on training and the loss function $\hat{\mathcal{R}}^{val}$ is bounded by $S\geq0$, then for all $\delta\in(0,1)$, with probability at least $1-\delta$,
\begin{align}\label{equation412-3}
\mathcal{R}^{val}(\mathcal{A}_{\text{hpo}}(\mathcal{D}^{tr},(\mathcal{D}^{tr}, \mathcal{D}^{val})))-\hat{\mathcal{R}}^{val}(\mathcal{A}_{\text{hpo}}(\mathcal{D}^{tr},(\mathcal{D}^{tr}, \mathcal{D}^{val}));\mathcal{D}^{val})\leq\beta+\sqrt{{2\beta^2 m^{tr}\ln\delta^{-1}}}.
\end{align}
\end{theorem}
We then give that ITD algorithm satisfies the uniform stability on training.
\begin{theorem}\label{theorem812-2}
For the given samples $\mathcal{D}\sim\mathscr{P}$, random seed $u_1$ and $\mathcal{S}_{(\mathcal{D}, u_1)}$, suppose that Assumptions \ref{assum33-1}, \ref{assum33-2} and \ref{assum33-3} hold. Then, ITD algorithm with $T$-step gradient descent, learning rate $\alpha_{out}\leq{1}/{M}$ in the outer-level is $\beta$-uniformly stable of training with
\begin{align*}
\beta=\frac{2M^2((\alpha_{in} L+1)^K-1)}{m^{tr}L}\big{(}1+\alpha_{out}L(\alpha_{in}L+1)^K\big{)}^T.
\end{align*}
\end{theorem}
\noindent\textbf{Remark.} The expectation bound of ITD depends on the number of outer-level iterations $T$, the number of inner-level iterations $K$ and the number of training splitting samples $m^{tr}$. Generally, the generalization error bound tends to increase with an increase of $T$ and $K$, but decreases with an increase in $m^{tr}$.

Then we give the following definition of $\beta$-uniformly stable on validation in \cite{bao2021stability}.
\begin{definition}\label{definition412-2}
\textbf{\textup{\citep{bao2021stability}}}
An HPO algorithm $\mathcal{A}_{\text{hpo}}$ is $\beta$-uniformly stable on validation if for all validation samples $\mathcal{D}^{val},{\mathcal{D}^{val}}'\in\mathscr{P}$ satisfying that $\mathcal{D}^{val},{\mathcal{D}^{val}}'$ differ in at most one sample, it holds that
\begin{align*}
\forall z\in\mathscr{P}, \hat{\mathcal{R}}^{val}(\mathcal{A}_{\text{hpo}}(\mathcal{D}^{tr},(\mathcal{D}^{tr}, \mathcal{D}^{val}));z)-\hat{\mathcal{R}}^{val}(\mathcal{A}_{\text{hpo}}(\mathcal{D}^{tr},(\mathcal{D}^{tr}, {\mathcal{D}^{val}}'));z)\leq\beta.
\end{align*}
\end{definition}
If an HPO algorithm is $\beta$-uniformly stable on validation, then we have the following high
probability bound.
\begin{theorem}\label{theorem812-3}
For the given multiple observed samples $\mathcal{D}\sim\mathscr{P}$, random seed $u_1$ and $\mathcal{S}_{(\mathcal{D}, u_1)}=(\mathcal{D}^{tr}, \mathcal{D}^{val})$. Suppose an HPO algorithm $\mathcal{A}_{\text{hpo}}$ is $\beta$-uniformly stable on validation and the loss function $\hat{\mathcal{R}}^{val}$ is bounded by $S\geq0$, then for all $\delta\in(0,1)$, with probability at least $1-\delta$,
\begin{footnotesize}
\begin{align}\label{equation812-5}
\mathcal{R}^{val}(\mathcal{A}_{\text{hpo}}(\mathcal{D}^{tr},(\mathcal{D}^{tr}, \mathcal{D}^{val})))-\hat{\mathcal{R}}^{val}(\mathcal{A}_{\text{hpo}}(\mathcal{D}^{tr},(\mathcal{D}^{tr}, \mathcal{D}^{val}));\mathcal{D}^{val})\leq\beta+\sqrt{\frac{(2\beta m^{val}+S)^2\ln\delta^{-1}}{2m^{val}}}.
\end{align}
\end{footnotesize}
\end{theorem}
We then give that ITD algorithm satisfies the uniform stability on validation.
\begin{theorem}\label{theorem812-4}
For the given samples $\mathcal{D}\sim\mathscr{P}$, random seed $u_1$ and $(\mathcal{D}^{tr}, \mathcal{D}^{val})=\mathcal{S}_{(\mathcal{D}, u_1)}$, suppose that Assumptions \ref{assum33-1}, \ref{assum33-2} and \ref{assum33-3} hold. Then, ITD algorithm with $T$-step gradient descent, learning rate $\alpha_{out}\leq{1}/{M}$ in the outer-level is $\beta$-uniformly stable on validation with
\begin{align}
\beta=\frac{2M^2}{m^{val}L}\Big{(}\big{(}1+\alpha_{out} L(\alpha_{in} L+1)^K\big{)}^T-1\Big{)}.
\end{align}
\end{theorem}
\noindent\textbf{Remark.} The expectation bound of ITD depends on the number of outer-level iterations $T$, the number of inner-level iterations $K$ and the number of validation splitting samples $m^{val}$. Generally, the generalization error bound tends to increase with an increase of $T$ and $K$, but decreases with an increase in $m^{val}$.

\noindent\textbf{Discussion:} In this section, we establish the generalization error bound of ITD algorithm based on the concept of algorithm stability and demonstrate that ITD algorithm exhibits algorithm stability on both training and validation. Generally, the generalization error bound of ITD algorithm increases with the increase in the number of inner and outer iterations, denoted as $K$ and $T$ respectively, while decrease with the increase in $m^{tr}$ and $m^{val}$. Furthermore, the generalization error bound depends on the minmum obtained from Eq. (\ref{equation412-3}) and  (\ref{equation812-5}), \textit{i.e.}, for all $\delta\in(0,1)$, with probability at least $1-\delta$, we have
\begin{align}
		&{\mathcal{R}}^{val}(\boldsymbol{\lambda}_{\mathcal{D}, u_1},\boldsymbol{\theta}_{\mathcal{D}, u_1}(\boldsymbol{\lambda}_{\mathcal{D}, u_1}))-\hat{\mathcal{R}}^{val}(\boldsymbol{\lambda}_{\mathcal{D}, u_1},\boldsymbol{\theta}_{\mathcal{D}, u_1}(\boldsymbol{\lambda}_{\mathcal{D}, u_1});  \mathcal{S}_{(\mathcal{D}, u_1)})\leq\notag\\
		&\min\Big{\{}\beta_1+\sqrt{{2\beta_1^2 m^{tr}\ln\delta^{-1}}},\quad\beta_2+\sqrt{\frac{(2\beta_2 m^{val}+S)^2\ln\delta^{-1}}{2m^{val}}}\Big{\}},
\end{align}
where $\big{(}\boldsymbol{\lambda}_{\mathcal{S}_{(\mathcal{D}, u_1)} },\boldsymbol{\theta}_{\mathcal{D}^{tr}}(\boldsymbol{\lambda}_{\mathcal{S}_{(\mathcal{D}, u_1)} })\big{)}=\mathcal{A}_{\text{ITD}}(\mathcal{D}^{tr},(\mathcal{S}_{(\mathcal{D}, u_1)}))$ is obtained by ITD algorithm, and $\beta_1$ and $\beta_2$ is obtained $\beta$ in Theorem \ref{theorem812-2} and \ref{theorem812-4}, respectively.

We will give the proof of the above analysis as follows.

\noindent\textbf{Proof of Theorem \ref{theorem812-1}.} Suppose $\mathcal{D}^{val},{\mathcal{D}^{val}}'\in\mathscr{P}$ differ in at most one point, and let $\Psi(\mathcal{D}^{tr},{\mathcal{D}^{val}})=\mathcal{R}^{val}(\mathcal{A}_{\text{hpo}}(\mathcal{D}^{tr},\mathcal{D}^{val}))-\hat{\mathcal{R}}^{val}(\mathcal{A}_{\text{hpo}}(\mathcal{D}^{tr},\mathcal{D}^{val});\mathcal{D}^{val})$, then
\begin{align*}
&|\Psi(\mathcal{D}^{tr},{\mathcal{D}^{val}})-\Psi(\mathcal{D}^{tr},{\mathcal{D}^{val}}')|\leq|\mathcal{R}^{val}(\mathcal{A}_{\text{hpo}}(\mathcal{D}^{tr},{\mathcal{D}^{val}}))-\mathcal{R}^{val}(\mathcal{A}_{\text{hpo}}(\mathcal{D}^{tr},{\mathcal{D}^{val}}'))|+\\
&|\hat{\mathcal{R}}^{val}(\mathcal{A}_{\text{hpo}}(\mathcal{D}^{tr},{\mathcal{D}^{val}});\mathcal{D}^{val})-\hat{\mathcal{R}}^{val}(\mathcal{A}_{\text{hpo}}(\mathcal{D}^{tr},{\mathcal{D}^{val}}');{\mathcal{D}^{val}}')|.
\end{align*}
For the first term,
\begin{align*}
&|\mathcal{R}^{val}(\mathcal{A}_{\text{hpo}}(\mathcal{D}^{tr},{\mathcal{D}^{val}}))-\mathcal{R}^{val}(\mathcal{A}_{\text{hpo}}(\mathcal{D}^{tr},{\mathcal{D}^{val}}'))|=|\mathbb{E}_{z}[\hat{\mathcal{R}}^{val}(\mathcal{A}_{\text{hpo}}(\mathcal{D}^{tr},{\mathcal{D}^{val}}));\\
&z)-\hat{\mathcal{R}}^{val}(\mathcal{A}_{\text{hpo}}(\mathcal{D}^{tr},{\mathcal{D}^{val}}'));z)]|\leq\beta.
\end{align*}
For the second term,
\begin{align*}
|\hat{\mathcal{R}}^{val}(\mathcal{A}_{\text{hpo}}(\mathcal{D}^{tr},{\mathcal{D}^{val}});\mathcal{D}^{val})-\hat{\mathcal{R}}^{val}(\mathcal{A}_{\text{hpo}}(\mathcal{D}^{tr},{\mathcal{D}^{val}}');{\mathcal{D}^{val}}')|\leq\frac{m^{val}-1}{m^{val}}\beta+\frac{S}{m^{val}}.
\end{align*}
As a result, $|\Psi(\mathcal{D}^{tr},{\mathcal{D}^{val}})-\Psi(\mathcal{D}^{tr},{\mathcal{D}^{val}}')|\leq\frac{S}{m^{val}}+2\beta$. According to McDiarmid’s inequality, we have that for all $\epsilon\in\mathbb{R}^{+}$,
\begin{align*}
	P_{\mathcal{D}^{tr},\mathcal{D}^{val}}(\Psi(\mathcal{D}^{tr},{\mathcal{D}^{val}})-\mathbb{E}_{\mathcal{D}^{tr},\mathcal{D}^{val}}[\Psi(\mathcal{D}^{tr},{\mathcal{D}^{val}})]\geq\epsilon)\leq \text{exp}(-2\frac{m^{val}\epsilon^2}{(S+2m^{val}\beta)^2}).
\end{align*}
Besides, we have
\begin{align*}
&\mathbb{E}_{\mathcal{D}^{tr},\mathcal{D}^{val}}[\Psi(\mathcal{D}^{tr},{\mathcal{D}^{val}})]=\mathbb{E}_{\mathcal{D}^{tr},\mathcal{D}^{val}}[\mathcal{R}^{val}(\mathcal{A}_{\text{hpo}}(\mathcal{D}^{tr},{\mathcal{D}^{val}}))-\hat{\mathcal{R}}^{val}(\mathcal{A}_{\text{hpo}}(\mathcal{D}^{tr},{\mathcal{D}^{val}});\\
&\mathcal{D}^{val})]=\mathbb{E}_{z,z_1}[\hat{\mathcal{R}}^{val}(\mathcal{A}_{\text{hpo}}(\mathcal{D}^{tr},{\mathcal{D}^{val}});z)-\hat{\mathcal{R}}^{val}(\mathcal{A}_{\text{hpo}}(\mathcal{D}^{tr},{\mathcal{D}^{val}});z_1)]=\\
&\mathbb{E}_{z,z_1}[\hat{\mathcal{R}}^{val}(\mathcal{A}_{\text{hpo}}(\mathcal{D}^{tr},{\mathcal{D}^{val}});z)-\hat{\mathcal{R}}^{val}(\mathcal{A}_{\text{hpo}}(\mathcal{D}^{tr},{\mathcal{D}^{val}});z_1)]=\\	
&\mathbb{E}_{z,z_1}[\hat{\mathcal{R}}^{val}(\mathcal{A}_{\text{hpo}}({\mathcal{D}^{tr},{\mathcal{D}^{val}}\backslash z_1}\cup z);z_1)-\hat{\mathcal{R}}^{val}(\mathcal{A}_{\text{hpo}}(\mathcal{D}^{tr},{\mathcal{D}^{val}});z_1)]\leq\beta.
\end{align*}
Thereby, we have $P_{\mathcal{D}^{tr},\mathcal{D}^{val}}(\Psi(\mathcal{D}^{tr},{\mathcal{D}^{val}})-\beta\geq\epsilon)\leq \text{exp}(-2\frac{m^{val}\epsilon^2}{(S+2m^{val}\beta)^2})$ for $\forall \epsilon\in\mathbb{R}^{+}$. Equivalently, we have $\forall \delta\in(0,1)$,
\begin{align*}
	P_{\mathcal{D}^{tr},\mathcal{D}^{val}}\Big{(}\Psi(\mathcal{D}^{tr},{\mathcal{D}^{val}})\leq\beta+\sqrt{\frac{(2\beta m^{val}+S)^2\ln\delta^{-1}}{2m^{val}}}\Big{)}\geq1-\delta.
\end{align*}
Then, the proof is completed.\hfill $\square$

\noindent\textbf{Proof of Theorem \ref{theorem812-2}.} We use the following equation to denote the updating rule in the outer-level,
\begin{align*}
		&\Upsilon(\boldsymbol{\lambda},\mathcal{D}^{tr},\mathcal{D}^{val})=\boldsymbol{\lambda}-\alpha_{out}\nabla_{\boldsymbol{\lambda}}\hat{\mathcal{R}}^{val}(\boldsymbol{\lambda},\boldsymbol{\theta}(\boldsymbol{\lambda};\mathcal{D}^{tr});\mathcal{D}^{val}).
\end{align*}
Suppose $\mathcal{D}^{val},{\mathcal{D}^{val}}'\in\mathscr{P}$ differ in at most one point, and let $\{\boldsymbol{\lambda}_t\}_{t\geq 0}$ and $\{\boldsymbol{\lambda}'_t\}_{t\geq 0}$ be the trace of gradient descent
with $\mathcal{D}^{val}$ and ${\mathcal{D}^{val}}'$ respectively. Let $\delta_t=\Vert\boldsymbol{\lambda}_t-\boldsymbol{\lambda}'_t\Vert$, and then
\begin{align*}
&\delta_{t+1}=\Vert\Upsilon(\boldsymbol{\lambda}_t,\mathcal{D}^{tr},\mathcal{D}^{val})-\Upsilon(\boldsymbol{\lambda}'_t,\mathcal{D}^{tr},{\mathcal{D}^{val}}')\Vert
\leq\delta_t+\alpha_{out}\Vert\nabla_{\boldsymbol{\lambda}}\hat{\mathcal{R}}^{val}(\boldsymbol{\lambda},\boldsymbol{\theta}(\boldsymbol{\lambda};\mathcal{D}^{tr});\\
&\mathcal{D}^{val})\Big{|}_{\boldsymbol{\lambda}=\boldsymbol{\lambda}_t}-\nabla_{\boldsymbol{\lambda}}\hat{\mathcal{R}}^{val}(\boldsymbol{\lambda},\boldsymbol{\theta}(\boldsymbol{\lambda};{\mathcal{D}^{tr}});{\mathcal{D}^{val}}')\Big{|}_{\boldsymbol{\lambda}=\boldsymbol{\lambda}'_t}\Vert.
\end{align*}
We rewrite $\nabla_{\boldsymbol{\lambda}}\hat{\mathcal{R}}^{val}(\boldsymbol{\lambda},\boldsymbol{\theta}(\boldsymbol{\lambda};\mathcal{D}^{tr});\mathcal{D}^{val})\Big{|}_{\boldsymbol{\lambda}=\boldsymbol{\lambda}_t}$ as $\nabla\hat{\mathcal{R}}^{val}(\boldsymbol{\lambda}_t,\boldsymbol{\theta}(\boldsymbol{\lambda}_t;\mathcal{D}^{tr});\mathcal{D}^{val})$. According to the defintion of $\mathcal{D}^{val}$ and ${\mathcal{D}^{val}}'$, we have:
\begin{align*}
\Vert\nabla\hat{\mathcal{R}}^{val}(\boldsymbol{\lambda}_t,\boldsymbol{\theta}(\boldsymbol{\lambda}_t;\mathcal{D}^{tr});\mathcal{D}^{val})-\nabla\hat{\mathcal{R}}^{val}(\boldsymbol{\lambda}'_t,\boldsymbol{\theta}(\boldsymbol{\lambda}'_t;{\mathcal{D}^{tr}});{\mathcal{D}^{val}}')\Vert\leq L(\alpha_{in} L+1)^K\delta_t+\frac{2M}{m^{val}}.
\end{align*}
Thereby, for all $t\geq 0$, we have:
\begin{align*}
&\delta_{t+1}\leq\delta_{t}+\alpha_{out}\Big{(}L(\alpha_{in} L+1)^K\delta_t+\frac{2M}{m^{val}}\Big{)}=\big{(}1+\alpha_{out} L(\alpha_{in} L+1)^K\big{)}\delta_t+\frac{2\alpha_{out} M}{m^{val}}\notag\\
&\leq\frac{2M}{m^{val}L(\alpha_{in} L+1)^K}\Big{(}\big{(}1+\alpha_{out} L(\alpha_{in} L+1)^K\big{)}^t-1\Big{)}.
\end{align*}
Finally, $\forall z\in\mathscr{P}$, we have:
\begin{align*}
&|\hat{\mathcal{R}}^{val}(\mathcal{A}_{\text{hpo}}(\mathcal{D}^{tr},{\mathcal{D}^{val}});z)-\hat{\mathcal{R}}^{val}(\mathcal{A}_{\text{hpo}}(\mathcal{D}^{tr},{\mathcal{D}^{val}}');z)|\leq|\hat{\mathcal{R}}^{val}(\boldsymbol{\lambda}_{\mathcal{D}^{tr},\mathcal{D}^{val}},\boldsymbol{\theta}_{\mathcal{D}^{tr}}(\\
&\boldsymbol{\lambda}_{\mathcal{D}^{tr},\mathcal{D}^{val}});z)-\hat{\mathcal{R}}^{val}(\boldsymbol{\lambda}_{\mathcal{D}^{tr},{\mathcal{D}^{val}}'},\boldsymbol{\theta}_{\mathcal{D}^{tr}}(\boldsymbol{\lambda}_{\mathcal{D}^{tr},\mathcal{D}^{val}});z)|+|\hat{\mathcal{R}}^{val}(\boldsymbol{\lambda}_{\mathcal{D}^{tr},{\mathcal{D}^{val}}'},\boldsymbol{\theta}_{\mathcal{D}^{tr}}(\\
&\boldsymbol{\lambda}_{\mathcal{D}^{tr},\mathcal{D}^{val}});z)-\hat{\mathcal{R}}^{val}(\boldsymbol{\lambda}_{\mathcal{D}^{tr},{\mathcal{D}^{val}}'},\boldsymbol{\theta}_{\mathcal{D}^{tr}}(\boldsymbol{\lambda}_{\mathcal{D}^{tr},{\mathcal{D}^{val}}'});\mathcal{D}^{tr});z)|\leq M\cdot\delta_T+M\Vert\boldsymbol{\theta}_{\mathcal{D}^{tr}}(\\
&\boldsymbol{\lambda}_{\mathcal{D}^{tr},\mathcal{D}^{val}})-\boldsymbol{\theta}_{\mathcal{D}^{tr}}(\boldsymbol{\lambda}_{\mathcal{D}^{tr},{\mathcal{D}^{val}}'})\Vert\leq M\cdot\delta_T+M((\alpha_{in} L+1)^K-1)\delta_T\leq\notag\\
&M(\alpha_{in} L+1)^K\cdot\delta_T=\frac{2M^2}{m^{val}L}\Big{(}\big{(}1+\alpha_{out} L(\alpha_{in} L+1)^K\big{)}^T-1\Big{)}.
\end{align*}
Then, the proof is completed.\hfill $\square$

\noindent\textbf{Proof of Theorem \ref{theorem812-3}.} Suppose $\mathcal{D}^{tr},{\mathcal{D}^{tr}}'\in\mathscr{P}$ differ in at most one point, and let $\Psi(\mathcal{D}^{tr},{\mathcal{D}^{val}})=\mathcal{R}^{val}(\mathcal{A}_{\text{hpo}}(\mathcal{D}^{tr},\mathcal{D}^{val}))-\hat{\mathcal{R}}^{val}(\mathcal{A}_{\text{hpo}}(\mathcal{D}^{tr},\mathcal{D}^{val});\mathcal{D}^{val})$. Then
\begin{align*}
&|\Psi({\mathcal{D}^{tr}},{\mathcal{D}^{val}})-\Psi({\mathcal{D}^{tr}}',{\mathcal{D}^{val}})|\leq|\mathcal{R}^{val}(\mathcal{A}_{\text{hpo}}(\mathcal{D}^{tr},{\mathcal{D}^{val}}))-\mathcal{R}^{val}(\mathcal{A}_{\text{hpo}}({\mathcal{D}^{tr}}',{\mathcal{D}^{val}}))|+\\
&|\hat{\mathcal{R}}^{val}(\mathcal{A}_{\text{hpo}}(\mathcal{D}^{tr},{\mathcal{D}^{val}});\mathcal{D}^{val})-\hat{\mathcal{R}}^{val}(\mathcal{A}_{\text{hpo}}({\mathcal{D}^{tr}}',{\mathcal{D}^{val}});{\mathcal{D}^{val}})|\leq2\beta.
\end{align*}According to McDiarmid’s inequality, we have for all $\epsilon\in\mathbb{R}^{+}$, $P_{\mathcal{D}^{tr}, \mathcal{D}^{val}}(\Psi(\mathcal{D}^{tr},{\mathcal{D}^{val}})-\mathbb{E}_{\mathcal{D}^{tr}, \mathcal{D}^{val}}[\Psi(\mathcal{D}^{tr},{\mathcal{D}^{val}})]\geq\epsilon)\leq \text{exp}(-\frac{\epsilon^2}{2m^{tr}\beta^2})$. Besides, we have
\begin{align*}
&\mathbb{E}_{\mathcal{D}^{tr}, \mathcal{D}^{val}}[\Psi(\mathcal{D}^{tr},{\mathcal{D}^{val}})]=\mathbb{E}_{\mathcal{D}^{tr}, \mathcal{D}^{val}}[\mathcal{R}^{val}(\mathcal{A}_{\text{hpo}}(\mathcal{D}^{tr},{\mathcal{D}^{val}}))-\hat{\mathcal{R}}^{val}(\mathcal{A}_{\text{hpo}}(\mathcal{D}^{tr},{\mathcal{D}^{val}});\\
&\mathcal{D}^{val})]=\mathbb{E}_{z,z_1}[\hat{\mathcal{R}}^{val}(\mathcal{A}_{\text{hpo}}(\mathcal{D}^{tr},{\mathcal{D}^{val}});z)-\hat{\mathcal{R}}^{val}(\mathcal{A}_{\text{hpo}}(\mathcal{D}^{tr},{\mathcal{D}^{val}});z_1)]=\mathbb{E}_{z,z_1}[\hat{\mathcal{R}}^{val}(\\
&\mathcal{A}_{\text{hpo}}(\mathcal{D}^{tr},{\mathcal{D}^{val}});z)-\hat{\mathcal{R}}^{val}(\mathcal{A}_{\text{hpo}}(\mathcal{D}^{tr},{\mathcal{D}^{val}});z_1)]=\mathbb{E}_{z,z_1}[\hat{\mathcal{R}}^{val}(\mathcal{A}_{\text{hpo}}({\mathcal{D}^{tr}\backslash z_1}\cup z,{\mathcal{D}^{val}});\\	
&z_1)-\hat{\mathcal{R}}^{val}(\mathcal{A}_{\text{hpo}}(\mathcal{D}^{tr},{\mathcal{D}^{val}});z_1)]\leq\beta.
\end{align*}
Thereby, we have for all $\epsilon\in\mathbb{R}^{+}$, $P_{\mathcal{D}^{tr}, \mathcal{D}^{val}}(\Psi(\mathcal{D}^{tr},{\mathcal{D}^{val}})-\beta\geq\epsilon)\leq \text{exp}(-\frac{\epsilon^2}{2m^{tr}\beta^2})$. Equivalently, we have $\forall \delta\in(0,1)$,
\begin{align*}
	P_{\mathcal{D}^{tr}, \mathcal{D}^{val}}\Big{(}\Psi(\mathcal{D}^{tr},{\mathcal{D}^{val}})\leq\beta+\sqrt{{2\beta^2 m^{tr}\ln\delta^{-1}}}\Big{)}\geq1-\delta.
\end{align*}Then, the proof is completed.\hfill $\square$

\noindent\textbf{Proof of Theorem \ref{theorem812-4}.} We use the following equation to denote the updating rule in the outer-level, $\Upsilon(\boldsymbol{\lambda},\mathcal{D}^{tr}, \mathcal{D}^{val})=\boldsymbol{\lambda}-\alpha_{out}\nabla_{\boldsymbol{\lambda}}\hat{\mathcal{R}}^{val}(\boldsymbol{\lambda},\boldsymbol{\theta}(\boldsymbol{\lambda};\mathcal{D}^{tr});\mathcal{D}^{val})$. Suppose ${\mathcal{D}^{tr}},{\mathcal{D}^{tr}}'\in\mathscr{P}$ differ in at most one point, and let $\{\boldsymbol{\lambda}_t\}_{t\geq 0}$ and $\{\boldsymbol{\lambda}'_t\}_{t\geq 0}$ be the trace of gradient descent
with ${\mathcal{D}^{tr}}$ and ${\mathcal{D}^{tr}}'$ respectively. Let $\delta_t=\Vert\boldsymbol{\lambda}_t-\boldsymbol{\lambda}'_t\Vert$, and then
\begin{align*}
&\delta_{t+1}=\Vert\Upsilon(\boldsymbol{\lambda}_t,\mathcal{D}^{tr}, \mathcal{D}^{val})-\Upsilon(\boldsymbol{\lambda}'_t,{\mathcal{D}^{tr}}', \mathcal{D}^{val})\Vert\leq\delta_t+\alpha_{out}\Vert\nabla_{\boldsymbol{\lambda}}\hat{\mathcal{R}}^{val}(\boldsymbol{\lambda},\boldsymbol{\theta}(\boldsymbol{\lambda};\mathcal{D}^{tr});\\
&\mathcal{D}^{val})\Big{|}_{\boldsymbol{\lambda}=\boldsymbol{\lambda}_t}-\nabla_{\boldsymbol{\lambda}}\hat{\mathcal{R}}^{val}(\boldsymbol{\lambda},\boldsymbol{\theta}(\boldsymbol{\lambda};{\mathcal{D}^{tr}}');{\mathcal{D}^{val}})\Big{|}_{\boldsymbol{\lambda}=\boldsymbol{\lambda}'_t}\Vert.
\end{align*}
We write $\nabla_{\boldsymbol{\lambda}}\hat{\mathcal{R}}^{val}(\boldsymbol{\lambda},\boldsymbol{\theta}(\boldsymbol{\lambda};\mathcal{D}^{tr});\mathcal{D}^{val})\Big{|}_{\boldsymbol{\lambda}=\boldsymbol{\lambda}_t}$ as $\nabla\hat{\mathcal{R}}^{val}(\boldsymbol{\lambda}_t,\boldsymbol{\theta}(\boldsymbol{\lambda}_t;\mathcal{D}^{tr});\mathcal{D}^{val})$. According to the defintion of ${\mathcal{D}^{tr}}$ and ${\mathcal{D}^{tr}}'$, we have:
\begin{align*}
&\Vert\nabla\hat{\mathcal{R}}^{val}(\boldsymbol{\lambda}_t,\boldsymbol{\theta}(\boldsymbol{\lambda}_t;\mathcal{D}^{tr});\mathcal{D}^{val})-\nabla\hat{\mathcal{R}}^{val}(\boldsymbol{\lambda}'_t,\boldsymbol{\theta}(\boldsymbol{\lambda}'_t;{\mathcal{D}^{tr}}');{\mathcal{D}^{val}})\Vert\leq L(\alpha_{in} L+1)^K\delta_t+\\
&L\Vert\boldsymbol{\theta}(\boldsymbol{\lambda}'_t;{\mathcal{D}^{tr}})-\boldsymbol{\theta}(\boldsymbol{\lambda}'_t;{\mathcal{D}^{tr}}')\Vert\leq L(\alpha_{in} L+1)^K\delta_t+\frac{2M\big{(}(\alpha_{in} L+1)^K-1\big{)}}{m_{tr}}.
\end{align*}
Thereby, for all $t\geq 0$, we have:
\begin{align*}
\delta_{t+1}\leq\frac{2M((\alpha_{in} L+1)^K-1)}{m_{tr}L(\alpha_{in} L+1)^K}\Big{(}\big{(}1+\alpha_{out} L(\alpha_{in} L+1)^K\big{)}^t-1\Big{)}.
\end{align*}
Finally, $\forall z\in\mathscr{P}$, we have:
\begin{align*}
&|\hat{\mathcal{R}}^{val}(\mathcal{A}_{\text{hpo}}({\mathcal{D}^{tr}},\mathcal{D}^{val});z)-\hat{\mathcal{R}}^{val}(\mathcal{A}_{\text{hpo}}({\mathcal{D}^{tr}}',\mathcal{D}^{val});z)|\leq M\cdot\delta_T+M\Vert\boldsymbol{\theta}_{\mathcal{D}^{tr}}(\\
&\boldsymbol{\lambda}_{\mathcal{D}^{tr},\mathcal{D}^{val}})-\boldsymbol{\theta}_{{\mathcal{D}^{tr}}'}(\boldsymbol{\lambda}_{{\mathcal{D}^{tr}}',\mathcal{D}^{val}})\Vert\leq M(\alpha_{in} L+1)^K\delta_T+\frac{2M^2\big{(}(\alpha_{in} L+1)^K-1\big{)}}{m_{tr}L}\leq\\
&\frac{2M^2((\alpha_{in} L+1)^K-1)}{m_{tr}L}\big{(}1+\alpha_{out} L(\alpha_{in} L+1)^K\big{)}^T.
\end{align*}
Then, the proof is completed.\hfill $\square$
\subsection{Proof of Theorems in Section \ref{section8}}\label{sectionA5-2}
\noindent\textbf{Proof of Proposition \ref{proposition821-1}.}
Firstly, we have
\begin{align}\label{equationABv2-2}
&\frac{1}{U}\sum_{i=1}^{U}\Big{[}\hat{\mathcal{R}}^{val}(\boldsymbol{\lambda}_{\mathcal{D}, \{u_i\}_{i=1}^{U}},\boldsymbol{\theta}_{\mathcal{D}, u_i}(\boldsymbol{\lambda}_{\mathcal{D}, \{u_i\}_{i=1}^{U}});  \mathcal{S}_{(\mathcal{D}, u_i)})-\hat{\mathcal{R}}^{val}(\boldsymbol{\lambda}^*_{\mathcal{D}, \{u_i\}_{i=1}^{U}},\boldsymbol{\theta}^*_{\mathcal{D}, u_i}( \notag\\
&\boldsymbol{\lambda}^*_{\mathcal{D}, \{u_i\}_{i=1}^{U}}); \mathcal{S}_{(\mathcal{D}, u_i)})]\Big{]}\leq\frac{1}{U}\sum_{i=1}^{U}\Big{[}\Big{|}\hat{\mathcal{R}}^{val}(\boldsymbol{\lambda}_{\mathcal{D}, \{u_i\}_{i=1}^{U}},\boldsymbol{\theta}_{\mathcal{D}, u_i}(\boldsymbol{\lambda}_{\mathcal{D}, \{u_i\}_{i=1}^{U}});  \mathcal{S}_{(\mathcal{D}, u_i)})-\notag \\
&\hat{\mathcal{R}}^{val}(\boldsymbol{\lambda}^*_{\mathcal{D}, \{u_i\}_{i=1}^{U}},\boldsymbol{\theta}^*_{\mathcal{D}, u_i}(\boldsymbol{\lambda}^*_{\mathcal{D}, \{u_i\}_{i=1}^{U}}); \mathcal{S}_{(\mathcal{D}, u_i)})]\Big{|}\Big{]}.
\end{align}
Then, we have
\begin{align}\label{equationABv2-3}
&\Big{|}\hat{\mathcal{R}}^{val}(\boldsymbol{\lambda}_{\mathcal{D}, \{u_i\}_{i=1}^{U}},\boldsymbol{\theta}_{\mathcal{D}, u_i}(\boldsymbol{\lambda}_{\mathcal{D}, \{u_i\}_{i=1}^{U}});  \mathcal{S}_{(\mathcal{D}, u_1)})-\hat{\mathcal{R}}^{val}(\boldsymbol{\lambda}^*_{\mathcal{D}, \{u_i\}_{i=1}^{U}},\boldsymbol{\theta}^*_{\mathcal{D}, u_i}(\boldsymbol{\lambda}^*_{\mathcal{D}, \{u_i\}_{i=1}^{U}});\notag \\
&\mathcal{S}_{(\mathcal{D}, u_1)})\Big{|}
\leq\Big{|}\hat{\mathcal{R}}^{val}(\boldsymbol{\lambda}_{\mathcal{D}, \{u_i\}_{i=1}^{U}},\boldsymbol{\theta}_{\mathcal{D}, u_i}(\boldsymbol{\lambda}_{\mathcal{D}, \{u_i\}_{i=1}^{U}});  \mathcal{S}_{(\mathcal{D}, u_1)})-\hat{\mathcal{R}}^{val}(\boldsymbol{\lambda}_{\mathcal{D}, \{u_i\}_{i=1}^{U}},\boldsymbol{\theta}_{\mathcal{D}, u_i}(\notag\\
&\boldsymbol{\lambda}^*_{\mathcal{D}, \{u_i\}_{i=1}^{U}});  \mathcal{S}_{(\mathcal{D}, u_1)})\Big{|}
+\Big{|}\hat{\mathcal{R}}^{val}(\boldsymbol{\lambda}_{\mathcal{D}, \{u_i\}_{i=1}^{U}},\boldsymbol{\theta}_{\mathcal{D}, u_i}(\boldsymbol{\lambda}^*_{\mathcal{D}, \{u_i\}_{i=1}^{U}});  \mathcal{S}_{(\mathcal{D}, u_1)})-\hat{\mathcal{R}}^{val}(\notag\\
&\boldsymbol{\lambda}^*_{\mathcal{D}, \{u_i\}_{i=1}^{U}},\boldsymbol{\theta}_{\mathcal{D}, u_i}(\boldsymbol{\lambda}^*_{\mathcal{D}, \{u_i\}_{i=1}^{U}});  \mathcal{S}_{(\mathcal{D}, u_1)})\Big{|}
+\Big{|}\hat{\mathcal{R}}^{val}(\boldsymbol{\lambda}^*_{\mathcal{D}, \{u_i\}_{i=1}^{U}},\boldsymbol{\theta}_{\mathcal{D}, u_i}(\boldsymbol{\lambda}^*_{\mathcal{D}, \{u_i\}_{i=1}^{U}});\notag\\
&\mathcal{S}_{(\mathcal{D}, u_1)})-\hat{\mathcal{R}}^{val}(\boldsymbol{\lambda}^*_{\mathcal{D}, \{u_i\}_{i=1}^{U}},\boldsymbol{\theta}^*_{\mathcal{D}, u_i}(\boldsymbol{\lambda}^*_{\mathcal{D}, \{u_i\}_{i=1}^{U}});  \mathcal{S}_{(\mathcal{D}, u_1)})\Big{|}\leq\notag\\
&M\Big{\Vert}\boldsymbol{\theta}_{\mathcal{D}, u_i}(\boldsymbol{\lambda}_{\mathcal{D}, \{u_i\}_{i=1}^{U}})-\boldsymbol{\theta}_{\mathcal{D}, u_i}(\boldsymbol{\lambda}^*_{\mathcal{D}, \{u_i\}_{i=1}^{U}})\Big{\Vert}+M\Big{\Vert}\boldsymbol{\lambda}_{\mathcal{D}, \{u_i\}_{i=1}^{U}}-\boldsymbol{\lambda}^*_{\mathcal{D}, \{u_i\}_{i=1}^{U}}\Big{\Vert}+\notag\\
&M\Big{\Vert}\boldsymbol{\theta}_{\mathcal{D}, u_i}(\boldsymbol{\lambda}^*_{\mathcal{D}, \{u_i\}_{i=1}^{U}})-\boldsymbol{\theta}^*_{\mathcal{D}, u_i}(\boldsymbol{\lambda}^*_{\mathcal{D}, \{u_i\}_{i=1}^{U}})\Big{\Vert}.
\end{align}
The first term measures the parameter variation resulting from changes in hyperparameter. Therefore, when the update method for $\boldsymbol{\theta}$ is $K$-step gradient descent, the update formula of $\boldsymbol{\theta}$ is as follows. For random seed $u_i$, where $i=1,2,\dots,U$, we have:
\begin{align}\label{equationABv2-4}
&\boldsymbol{\theta}_{\mathcal{D}, u_i}(\boldsymbol{\lambda}_{\mathcal{D}, \{u_i\}_{i=1}^U})
=\boldsymbol{\theta}_0-\alpha_{in}\sum_{k=0}^{K-1}\nabla_{\boldsymbol{\theta}}\hat{\mathcal{R}}^{tr}(\boldsymbol{\lambda}_{\mathcal{D}, \{u_i\}_{i=1}^U},\boldsymbol{\theta};\mathcal{S}_{(\mathcal{D}, u_i)})\Big{|}_{\boldsymbol{\theta}=\boldsymbol{\theta}_k(\boldsymbol{\lambda}_{\mathcal{D}, \{u_i\}_{i=1}^U};\mathcal{S}_{(\mathcal{D}, u_i)})},\notag\\
&\boldsymbol{\theta}_{\mathcal{D}, u_i}(\boldsymbol{\lambda}^*_{\mathcal{D}, \{u_i\}_{i=1}^U})
=\boldsymbol{\theta}_0-\alpha_{in}\sum_{k=0}^{K-1}\nabla_{\boldsymbol{\theta}}\hat{\mathcal{R}}^{tr}(\boldsymbol{\lambda}^*_{\mathcal{D}, \{u_i\}_{i=1}^U},\boldsymbol{\theta};\mathcal{S}_{(\mathcal{D}, u_i)})\Big{|}_{\boldsymbol{\theta}=\boldsymbol{\theta}_k(\boldsymbol{\lambda}^*_{\mathcal{D}, \{u_i\}_{i=1}^U};\mathcal{S}_{(\mathcal{D}, u_i)})}.
\end{align}
Taking Eq. (\ref{equationABv2-4}) into the first term in Eq. (\ref{equationABv2-3}), we have
\begin{align}\label{equationABv2-5}
&\Big{\Vert}\boldsymbol{\theta}_{\mathcal{D}, u_i}(\boldsymbol{\lambda}_{\mathcal{D}, \{u_i\}_{i=1}^U})-\boldsymbol{\theta}_{\mathcal{D}, u_i}(\boldsymbol{\lambda}^*_{\mathcal{D}, \{u_i\}_{i=1}^U})\Big{\Vert}\leq\alpha_{in}\sum_{k=0}^{K-1}\Big{(}L\Vert\boldsymbol{\theta}_k(\boldsymbol{\lambda}_{\mathcal{D}, \{u_i\}_{i=1}^U};\mathcal{S}_{(\mathcal{D}, u_i)})-\boldsymbol{\theta}_k(\notag\\
&\boldsymbol{\lambda}^*_{\mathcal{D}, \{u_i\}_{i=1}^U};\mathcal{S}_{(\mathcal{D}, u_i)})\Vert+L\Vert\boldsymbol{\lambda}_{\mathcal{D}, \{u_i\}_{i=1}^U}-\boldsymbol{\lambda}^*_{\mathcal{D}, \{u_i\}_{i=1}^U}\Vert\Big{)}\leq\big{(}(\alpha_{in} L+1)^K-1\big{)}\cdot\notag \\
&\big{\Vert}\boldsymbol{\lambda}_{\mathcal{D}, \{u_i\}_{i=1}^U}-\boldsymbol{\lambda}^*_{\mathcal{D}, \{u_i\}_{i=1}^U}\big{\Vert}.
\end{align}
Taking Eq. (\ref{equationABv2-5}) into (\ref{equationABv2-3}), we have
\begin{align}\label{equationABv2-6}
		&\Big{|}\hat{\mathcal{R}}^{val}(\boldsymbol{\lambda}_{\mathcal{D}, \{u_i\}_{i=1}^{U}},\boldsymbol{\theta}_{\mathcal{D}, u_i}(\boldsymbol{\lambda}_{\mathcal{D}, \{u_i\}_{i=1}^{U}});  \mathcal{S}_{(\mathcal{D}, u_1)})-\hat{\mathcal{R}}^{val}(\boldsymbol{\lambda}^*_{\mathcal{D}, \{u_i\}_{i=1}^{U}},\boldsymbol{\theta}^*_{\mathcal{D}, u_i}(\boldsymbol{\lambda}^*_{\mathcal{D}, \{u_i\}_{i=1}^{U}}); \notag\\
        &\mathcal{S}_{(\mathcal{D}, u_1)})\Big{|}\leq M\Big{\Vert}\boldsymbol{\theta}_{\mathcal{D}, u_i}(\boldsymbol{\lambda}^*_{\mathcal{D}, \{u_i\}_{i=1}^{U}})-\boldsymbol{\theta}^*_{\mathcal{D}, u_i}(\boldsymbol{\lambda}^*_{\mathcal{D}, \{u_i\}_{i=1}^{U}})\Big{\Vert}+M(\alpha_{in} L+1)^K\Big{(}\notag \\
        &\big{\Vert}\boldsymbol{\lambda}_{\mathcal{D}, \{u_i\}_{i=1}^{U}}-\boldsymbol{\lambda}^{*(T)}_{\mathcal{D}, \{u_i\}_{i=1}^{U}}\big{\Vert}+\big{\Vert}\boldsymbol{\lambda}^{*(T)}_{\mathcal{D}, \{u_i\}_{i=1}^{U}}-\boldsymbol{\lambda}^*_{\mathcal{D}, \{u_i\}_{i=1}^{U}}\big{\Vert}\Big{)}.
\end{align}
Next, according to the update rule of hyperparameter $\boldsymbol{\lambda}$, we know that:
\begin{align}\label{equationABv2-7}
		&\boldsymbol{\lambda}_{\mathcal{D}, \{u_i\}_{i=1}^{U}}=\boldsymbol{\lambda}^{(0)}-\alpha_{out}\sum_{t=0}^{T-1}\frac{1}{U}\sum_{u=1}^{U}\Big{[}\nabla_{\boldsymbol{\lambda}}\hat{\mathcal{R}}^{val}(\boldsymbol{\lambda},\boldsymbol{\theta}_{\mathcal{D}, u_i}(\boldsymbol{\lambda});\mathcal{S}_{(\mathcal{D}, u_i)})\Big{|}_{\boldsymbol{\lambda}=\boldsymbol{\lambda}^{(t)}_{\mathcal{D}, \{u_i\}_{i=1}^{U}}}\Big{]},\notag\\
		&\boldsymbol{\lambda}^{*(T)}_{\mathcal{D}, \{u_i\}_{i=1}^{U}}=\boldsymbol{\lambda}^{(0)}-\alpha_{out}\sum_{t=0}^{T-1}\mathbb{E}_{\mathcal{D},u_1}\Big{[}\nabla_{\boldsymbol{\lambda}}\hat{\mathcal{R}}^{val}(\boldsymbol{\lambda},\boldsymbol{\theta}^*_{\mathcal{D}, u_1}(\boldsymbol{\lambda});\mathcal{S}_{(\mathcal{D}, u_1)})\Big{|}_{\boldsymbol{\lambda}=\boldsymbol{\lambda}^{*cc(t)}_{\mathcal{D}, \{u_i\}_{i=1}^{U}}}\Big{]}.
\end{align}
Taking Eq. (\ref{equationABv2-7}) into the following part of Eq. (\ref{equationABv2-6}), we have:
\begin{align}\label{equationABv2-8}
&M(\alpha_{in} L+1)^K\big{\Vert}\boldsymbol{\lambda}_{\mathcal{D}, \{u_i\}_{i=1}^U}-\boldsymbol{\lambda}^{*(T)}_{\mathcal{D}, \{u_i\}_{i=1}^U}\big{\Vert}
\leq M(\alpha_{in} L+1)^K\Big{(}\alpha_{out}\sum_{t=0}^{T-1}\Big{[}\big{\Vert}\frac{1}{U}\sum_{u=1}^{U}\notag\\
&\nabla_{\boldsymbol{\lambda}}\hat{\mathcal{R}}^{val}(\boldsymbol{\lambda},\boldsymbol{\theta}_{\mathcal{D}, u_i}(\boldsymbol{\lambda});\mathcal{S}_{(\mathcal{D}, u_i)})\Big{|}_{\boldsymbol{\lambda}=\boldsymbol{\lambda}^{(t)}_{\mathcal{D}, \{u_i\}_{i=1}^{U}}}
-\mathbb{E}_{\mathcal{D},u_1}{[}\nabla_{\boldsymbol{\lambda}}\hat{\mathcal{R}}^{val}(\boldsymbol{\lambda},\boldsymbol{\theta}^*_{\mathcal{D}, u_1}(\boldsymbol{\lambda});
\notag\\
&\mathcal{S}_{(\mathcal{D}, u_1)})\Big{|}_{\boldsymbol{\lambda}=\boldsymbol{\lambda}^{*(t)}_{\mathcal{D}, \{u_i\}_{i=1}^{U}}}{]}\big{\Vert}\Big{]}\Big{)}.
\end{align}
Next, for the $\Vert\cdot\Vert$ term in Eq. (\ref{equationABv2-8}), we have:
\begin{align}\label{equationABv2-9}
&\big{\Vert}\frac{1}{U}\sum_{u=1}^{U}\nabla_{\boldsymbol{\lambda}}\hat{\mathcal{R}}^{val}(\boldsymbol{\lambda},\boldsymbol{\theta}_{\mathcal{D}, u_i}(\boldsymbol{\lambda});\mathcal{S}_{(\mathcal{D}, u_i)})\Big{|}_{\boldsymbol{\lambda}=\boldsymbol{\lambda}^{(t)}_{
    \mathcal{D}, \{u_i\}_{i=1}^{U}}}-\mathbb{E}_{\mathcal{D},u_1}{[}\nabla_{\boldsymbol{\lambda}}\hat{\mathcal{R}}^{val}(\boldsymbol{\lambda},\boldsymbol{\theta}^*_{\mathcal{D}, u_1}(\boldsymbol{\lambda});\notag\\
    &\mathcal{S}_{(\mathcal{D}, u_1)})\Big{|}_{\boldsymbol{\lambda}=\boldsymbol{\lambda}^{*(t)}_{\mathcal{D}, \{u_i\}_{i=1}^{U}}}{]}\big{\Vert}
\leq \big{\Vert}\frac{1}{U}\sum_{u=1}^{U}\nabla_{\boldsymbol{\lambda}}\hat{\mathcal{R}}^{val}(\boldsymbol{\lambda},\boldsymbol{\theta}_{\mathcal{D}, u_i}(\boldsymbol{\lambda});\mathcal{S}_{(\mathcal{D}, u_i)})\Big{|}_{\boldsymbol{\lambda}=\boldsymbol{\lambda}^{(t)}_{
    \mathcal{D}, \{u_i\}_{i=1}^{U}}}-\notag\\
    &\frac{1}{U}\sum_{u=1}^{U}\nabla_{\boldsymbol{\lambda}}\hat{\mathcal{R}}^{val}(\boldsymbol{\lambda},\boldsymbol{\theta}_{\mathcal{D}, u_i}(\boldsymbol{\lambda});\mathcal{S}_{(\mathcal{D}, u_i)})\Big{|}_{\boldsymbol{\lambda}=\boldsymbol{\lambda}^{*(t)}_{
    \mathcal{D}, \{u_i\}_{i=1}^{U}}}\big{\Vert}+\big{\Vert}\frac{1}{U}\sum_{u=1}^{U}\nabla_{\boldsymbol{\lambda}}\hat{\mathcal{R}}^{val}(\boldsymbol{\lambda},\boldsymbol{\theta}_{\mathcal{D}, u_i}(\boldsymbol{\lambda});\notag\\
    &\mathcal{S}_{(\mathcal{D}, u_i)})\Big{|}_{\boldsymbol{\lambda}=\boldsymbol{\lambda}^{*(t)}_{
    \mathcal{D}, \{u_i\}_{i=1}^{U}}}-\mathbb{E}_{\mathcal{D},u_1}{[}\nabla_{\boldsymbol{\lambda}}\hat{\mathcal{R}}^{val}(\boldsymbol{\lambda},\boldsymbol{\theta}^*_{\mathcal{D}, u_1}(\boldsymbol{\lambda});\mathcal{S}_{(\mathcal{D}, u_1)})\Big{|}_{\boldsymbol{\lambda}=\boldsymbol{\lambda}^{*(t)}_{\mathcal{D}, \{u_i\}_{i=1}^{U}}}{]}\big{\Vert}.
\end{align}
For the first term in Eq. (\ref{equationABv2-9}), we write $\nabla_{\boldsymbol{\lambda}}\hat{\mathcal{R}}^{val}(\boldsymbol{\lambda},\boldsymbol{\theta}_{\mathcal{D}, u_i}(\boldsymbol{\lambda});\mathcal{S}_{(\mathcal{D}, u_i)})\Big{|}_{\boldsymbol{\lambda}=\boldsymbol{\lambda}^{(t)}_{\mathcal{D}, \{u_i\}_{i=1}^U}}$ as $\nabla\hat{\mathcal{R}}^{val}(\boldsymbol{\lambda}^{(t)}_{\mathcal{D}, \{u_i\}_{i=1}^U},\boldsymbol{\theta}_{\mathcal{D}, u_i}(\boldsymbol{\lambda}^{(t)}_{\mathcal{D}, \{u_i\}_{i=1}^U});\mathcal{S}_{(\mathcal{D}, u_i)})$, and then we have:
\begin{align}\label{equationABv2-10}
&\big{\Vert}\frac{1}{U}\sum_{u=1}^{U}\Big{[}\nabla\hat{\mathcal{R}}^{val}(\boldsymbol{\lambda}^{(t)}_{\mathcal{D}, \{u_i\}_{i=1}^{U}},\boldsymbol{\theta}_{\mathcal{D}, u_i}(\boldsymbol{\lambda}^{(t)}_{\mathcal{D}, \{u_i\}_{i=1}^{U}});\mathcal{S}_{(\mathcal{D}, u_i)})-\nabla\hat{\mathcal{R}}^{val}(\boldsymbol{\lambda}^{*(t)}_{\mathcal{D}, \{u_i\}_{i=1}^{U}},\boldsymbol{\theta}_{\mathcal{D}, u_i}(\notag\\
&\boldsymbol{\lambda}^{*(t)}_{\mathcal{D}, \{u_i\}_{i=1}^{U}});\mathcal{S}_{(\mathcal{D}, u_i)})\Big{]}\big{\Vert}\leq \frac{1}{U}\sum_{u=1}^{U}\Big{(}L\Vert\boldsymbol{\theta}_{\mathcal{D}, u_i}(\boldsymbol{\lambda}^{(t)}_{\mathcal{D}, \{u_i\}_{i=1}^{U}})-\boldsymbol{\theta}_{\mathcal{D}, u_i}(\boldsymbol{\lambda}^{*(t)}_{\mathcal{D}, \{u_i\}_{i=1}^{U}})\Vert+L\cdot\notag\\
&\Vert\boldsymbol{\lambda}^{(t)}_{\mathcal{D}, \{u_i\}_{i=1}^{U}}-\boldsymbol{\lambda}^{*(t)}_{\mathcal{D}, \{u_i\}_{i=1}^{U}}\Vert\Big{)}
\leq  L(\alpha_{in}L+1)^K\Vert\boldsymbol{\lambda}^{(t)}_{\mathcal{D}, \{u_i\}_{i=1}^{U}}-\boldsymbol{\lambda}^{*(t)}_{\mathcal{D}, \{u_i\}_{i=1}^{U}}\Vert.
\end{align}
For the second term in Eq. (\ref{equationABv2-9}), this is the expectation form of hypergradient error, and we write this term as $\text{Err}_{\text{hg}}(\mathcal{D},\{u_i\}_{i=1}^{U})$. Then taking Eq. (\ref{equationABv2-10}) into Eq. (\ref{equationABv2-8}) and the first term of (\ref{equationABv2-9}), we have:
\begin{align}\label{equationABv2-11}
&\Vert\boldsymbol{\lambda}_{\mathcal{D},\{u_i\}_{i=1}^{U}}-\boldsymbol{\lambda}^{*(T)}_{\mathcal{D},\{u_i\}_{i=1}^{U}}\Vert
\leq\alpha_{out}\sum_{t=0}^{T-1}\Big{(}L(\alpha_{in} L+1)^K\big{(}\Vert\boldsymbol{\lambda}^{(t)}_{\mathcal{D},\{u_i\}_{i=1}^{U}}-\boldsymbol{\lambda}^{*(t)}_{\mathcal{D},\{u_i\}_{i=1}^{U}}\Vert+\notag\\
&\text{Err}_{\text{hg}}(\mathcal{D},\{u_i\}_{i=1}^{U})\big{)}
\leq\frac{\big{(}\alpha_{out} L(\alpha_{in} L+1)^K+1\big{)}^T-1}{L(\alpha_{in} L+1)^K}\cdot\text{Err}_{\text{hg}}(\mathcal{D},\{u_i\}_{i=1}^{U}).
\end{align}
Then, taking Eq. (\ref{equationABv2-11}) into Eq. (\ref{equationABv2-2}) and Eq. (\ref{equationABv2-3}), we have:
\begin{align}\label{equationABv2-12}
&\frac{1}{U}\sum_{i=1}^{U}\Big{[}\hat{\mathcal{R}}^{val}(\boldsymbol{\lambda}_{\mathcal{D}, \{u_i\}_{i=1}^{U}},\boldsymbol{\theta}_{\mathcal{D}, u_i}(\boldsymbol{\lambda}_{\mathcal{D}, \{u_i\}_{i=1}^{U}});  \mathcal{S}_{(\mathcal{D}, u_i)})-\hat{\mathcal{R}}^{val}(\boldsymbol{\lambda}^*_{\mathcal{D}, \{u_i\}_{i=1}^{U}},\boldsymbol{\theta}^*_{\mathcal{D}, u_i}(\notag\\
&\boldsymbol{\lambda}^*_{\mathcal{D}, \{u_i\}_{i=1}^{U}}); \mathcal{S}_{(\mathcal{D}, u_i)})]\Big{]}
\leq M\frac{\big{(}\alpha_{out} L(\alpha_{in} L+1)^K+1\big{)}^T-1}{L}\cdot\text{Err}_{\text{hg}}(\mathcal{D},\{u_i\}_{i=1}^{U})+M(\notag\\
&\alpha_{in} L+1)^K\big{\Vert}\boldsymbol{\lambda}^{*(T)}_{\mathcal{D}, \{u_i\}_{i=1}^{U}}-\boldsymbol{\lambda}^*_{\mathcal{D}, \{u_i\}_{i=1}^{U}}\big{\Vert}
+\sup_i M\Big{\Vert}\boldsymbol{\theta}_{\mathcal{D}, u_i}(\boldsymbol{\lambda}^*_{\mathcal{D}, \{u_i\}_{i=1}^{U}})-\boldsymbol{\theta}^*_{\mathcal{D}, u_i}(\boldsymbol{\lambda}^*_{\mathcal{D}, \{u_i\}_{i=1}^{U}})\Big{\Vert}.
\end{align}
Using Theorem \ref{theorem332-1} and Cauchy–Schwarz inequality, the proof is completed.\hfill $\square$

\noindent\textbf{Proof of Lemma \ref{lemma822-1}.} We use the following equation to denote the updating rule in the outer-level, $\Upsilon(\boldsymbol{\lambda},\mathcal{S}_{(\mathcal{D},\{u_i\}_{i=1}^U)})=\boldsymbol{\lambda}-\alpha_{out}\frac{1}{U}\sum_{i=1}^{U}\nabla_{\boldsymbol{\lambda}}\hat{\mathcal{R}}^{val}(\boldsymbol{\lambda},\boldsymbol{\theta}(\boldsymbol{\lambda};\mathcal{D}_{u_i}^{tr});\mathcal{D}_{u_i}^{val})$.
Suppose $\mathcal{D},{\mathcal{D}}'\in\mathscr{P}$ differ in at most one point, and let $\{\boldsymbol{\lambda}_t\}_{t\geq 0}$ and $\{\boldsymbol{\lambda}'_t\}_{t\geq 0}$ be the trace of gradient descent
with $\mathcal{D}$ and ${\mathcal{D}}'$ respectively. Let $\delta_t=\Vert\boldsymbol{\lambda}_t-\boldsymbol{\lambda}'_t\Vert$, and then
\begin{align}\label{equation-a12-2}
&\delta_{t+1}\leq\delta_t+\frac{\alpha_{out}}{U}\Big{(}\sum_{i=1}^{U_1}\Vert\nabla_{\boldsymbol{\lambda}}\hat{\mathcal{R}}^{val}(\boldsymbol{\lambda},\boldsymbol{\theta}(\boldsymbol{\lambda};{\mathcal{D}_{u_i}^{tr}});\mathcal{D}_{u_i}^{val})\Big{|}_{\boldsymbol{\lambda}=\boldsymbol{\lambda}_t}-\nabla_{\boldsymbol{\lambda}}\hat{\mathcal{R}}^{val}(\boldsymbol{\lambda},\boldsymbol{\theta}(\boldsymbol{\lambda};{\mathcal{D}_{u_i}^{tr}}');\notag\\
&{\mathcal{D}_{u_i}^{val}})\Big{|}_{\boldsymbol{\lambda}=\boldsymbol{\lambda}'_t}\Vert
+\sum_{i=1}^{U_2}\Vert\nabla_{\boldsymbol{\lambda}}\hat{\mathcal{R}}^{val}(\boldsymbol{\lambda},\boldsymbol{\theta}(\boldsymbol{\lambda};\mathcal{D}_{u_i}^{tr});\mathcal{D}_{u_i}^{val})\Big{|}_{\boldsymbol{\lambda}=\boldsymbol{\lambda}_t}-\nabla_{\boldsymbol{\lambda}}\hat{\mathcal{R}}^{val}(\boldsymbol{\lambda},\boldsymbol{\theta}(\boldsymbol{\lambda};{\mathcal{D}_{u_i}^{tr}});\notag\\
&{\mathcal{D}_{u_i}^{val}}')\Big{|}_{\boldsymbol{\lambda}=\boldsymbol{\lambda}'_t}\Vert\Big{)}.
\end{align}
We rewrite $\nabla_{\boldsymbol{\lambda}}\hat{\mathcal{R}}^{val}(\boldsymbol{\lambda},\boldsymbol{\theta}(\boldsymbol{\lambda};\mathcal{D}_{u_i}^{tr});\mathcal{D}_{u_i}^{val})\Big{|}_{\boldsymbol{\lambda}=\boldsymbol{\lambda}_t}$ as $\nabla\hat{\mathcal{R}}^{val}(\boldsymbol{\lambda}_t,\boldsymbol{\theta}(\boldsymbol{\lambda}_t;\mathcal{D}_{u_i}^{tr});\mathcal{D}_{u_i}^{val})$. According to the defintion of ${\mathcal{D}_{u_i}^{tr}}$ and ${\mathcal{D}_{u_i}^{tr}}'$, we have:
\begin{align}\label{equationAE-2-4}
		&\Vert\nabla\hat{\mathcal{R}}^{val}(\boldsymbol{\lambda}_t,\boldsymbol{\theta}(\boldsymbol{\lambda}_t;\mathcal{D}_{u_i}^{tr});\mathcal{D}_{u_i}^{val})-\nabla\hat{\mathcal{R}}^{val}(\boldsymbol{\lambda}'_t,\boldsymbol{\theta}(\boldsymbol{\lambda}'_t;{\mathcal{D}_{u_i}^{tr}}');{\mathcal{D}_{u_i}^{val}})\Vert
		\leq L(\alpha_{in} L+1)^K\delta_t+\notag\\
        &L\Vert\boldsymbol{\theta}(\boldsymbol{\lambda}'_t;{\mathcal{D}_{u_i}^{tr}})-\boldsymbol{\theta}(\boldsymbol{\lambda}'_t;{\mathcal{D}_{u_i}^{tr}}')\Vert\leq L(\alpha_{in} L+1)^K\delta_t+\frac{2M\big{(}(\alpha_{in} L+1)^K-1\big{)}}{m_{tr}}.
\end{align}
Thereby, for all $t\geq 0$, we have:
\begin{align*}
&\delta_{t+1}\leq\big{(}1+\alpha_{out} L(\alpha_{in} L+1)^K\big{)}\delta_t+\frac{2\alpha_{out} M\big{(}(\alpha_{in} L+1)^K-1\big{)}}{m_{tr}}
\leq\notag\\
&\frac{2M((\alpha_{in} L+1)^K-1)}{m_{tr}L(\alpha_{in} L+1)^K}\Big{(}\big{(}1+\alpha_{out} L(\alpha_{in} L+1)^K\big{)}^t-1\Big{)}.
\end{align*}
Then we have: 
\begin{align}\label{equationAE-2-5}
&\Vert\nabla_{\boldsymbol{\lambda}}\hat{\mathcal{R}}^{val}(\boldsymbol{\lambda},\boldsymbol{\theta}(\boldsymbol{\lambda};\mathcal{D}_{u_i}^{tr});\mathcal{D}_{u_i}^{val})\Big{|}_{\boldsymbol{\lambda}=\boldsymbol{\lambda}_t}-\nabla_{\boldsymbol{\lambda}}\hat{\mathcal{R}}^{val}(\boldsymbol{\lambda},\boldsymbol{\theta}(\boldsymbol{\lambda};{\mathcal{D}_{u_i}^{tr}});{\mathcal{D}_{u_i}^{val}}')\Big{|}_{\boldsymbol{\lambda}=\boldsymbol{\lambda}'_t}\Vert\leq L(\alpha_{in} L+\notag\\
&1)^K\delta_t+\frac{2M}{m^{val}}.
\end{align}
Thereby, for all $t\geq 0$, we have:
\begin{align*}
&\delta_{t+1}\leq\frac{2M}{m^{val}L(\alpha_{in} L+1)^K}\Big{(}\big{(}1+\alpha_{out} L(\alpha_{in} L+1)^K\big{)}^t-1\Big{)}.
\end{align*}
Combining Eq. (\ref{equationAE-2-4}) and Eq. (\ref{equationAE-2-5}) into Eq. (\ref{equation-a12-2}), we have: 
\begin{align*}
		&\delta_{t+1}=\Vert\Upsilon(\boldsymbol{\lambda}_t,\mathcal{S}_{(\mathcal{D},\{u_i\}_{i=1}^U)})-\Upsilon(\boldsymbol{\lambda}'_t,\mathcal{S}_{({\mathcal{D}}',\{u_i\}_{i=1}^U)})\Vert
		\leq\delta_t+\frac{\alpha_{out}}{U}\Big{(}U_1(L(\alpha_{in} L+1)^K\delta_t\\
        &+\frac{2M\big{(}(\alpha_{in} L+1)^K-1\big{)}}{m_{tr}})+{U_2}(L(\alpha_{in} L+1)^K\delta_t+\frac{2M}{m^{val}})\Big{)}\leq\\
        &M\frac{\big{(}1+\alpha_{out} L(\alpha_{in} L+1)^K\big{)}^{t+1}-1}{UL(\alpha_{in} L+1)^K}\Big{(}\frac{\big{(}(\alpha_{in} L+1)^K-1\big{)}U_1}{m^{tr}}+\frac{U_2}{m^{val}}\Big{)}.
\end{align*}
Finally, $\forall z\in\mathscr{P}$, we have:
\begin{align}\label{eq:append-u1u2}
		&\frac{1}{U}\Big{(}\sum_{i=1}^{U_1}|\hat{\mathcal{R}}^{val}(\mathcal{A}_{\text{hpo}}({\mathcal{D}_{u_i}^{tr}},\mathcal{S}_{(\mathcal{D},\{u_i\}_{i=1}^U)});z)-\hat{\mathcal{R}}^{val}(\mathcal{A}_{\text{hpo}}({\mathcal{D}_{u_i}^{tr}}',\mathcal{S}_{(\mathcal{D}',\{u_i\}_{i=1}^U)});z)|+\sum_{i=1}^{U_2}|\notag\\
        &\hat{\mathcal{R}}^{val}(\mathcal{A}_{\text{hpo}}({\mathcal{D}_{u_i}^{tr}},\mathcal{S}_{(\mathcal{D},\{u_i\}_{i=1}^U)});z)-\hat{\mathcal{R}}^{val}(\mathcal{A}_{\text{hpo}}({\mathcal{D}_{u_i}^{tr}},\mathcal{S}_{(\mathcal{D}',\{u_i\}_{i=1}^U)});z)|\Big{)}
		=\frac{1}{U}\Big{(}\sum_{i=1}^{U_1}|\notag\\
        &\hat{\mathcal{R}}^{val}(\boldsymbol{\lambda}_{\mathcal{S}_{(\mathcal{D},\{u_i\}_{i=1}^U)} },\boldsymbol{\theta}_{\mathcal{D}_{u_i}^{tr}}(\boldsymbol{\lambda}_{\mathcal{S}_{(\mathcal{D},\{u_i\}_{i=1}^U)}});z)-\hat{\mathcal{R}}^{val}(\boldsymbol{\lambda}_{\mathcal{S}_{(\mathcal{D}',\{u_i\}_{i=1}^U)} },\boldsymbol{\theta}_{{\mathcal{D}_{u_i}^{tr}}'}(\boldsymbol{\lambda}_{\mathcal{S}_{(\mathcal{D}',\{u_i\}_{i=1}^U)}});\notag\\
		&z)|+\sum_{i=1}^{U_2}|\hat{\mathcal{R}}^{val}(\boldsymbol{\lambda}_{\mathcal{S}_{(\mathcal{D},\{u_i\}_{i=1}^U)} },\boldsymbol{\theta}_{\mathcal{D}_{u_i}^{tr}}(\boldsymbol{\lambda}_{\mathcal{S}_{(\mathcal{D},\{u_i\}_{i=1}^U)}});z)-\hat{\mathcal{R}}^{val}(\boldsymbol{\lambda}_{\mathcal{S}_{(\mathcal{D}',\{u_i\}_{i=1}^U)} },\boldsymbol{\theta}_{{\mathcal{D}_{u_i}^{tr}}}(\notag\\
        &\boldsymbol{\lambda}_{\mathcal{S}_{(\mathcal{D}',\{u_i\}_{i=1}^U)}});z)|\Big{)}.
\end{align}
For the first $|\cdot|$ term in Eq. \eqref{eq:append-u1u2}, we have the bound is
\begin{align}\label{equation-a11-p2-2}&M^2\frac{\big{(}1+\alpha_{out} L(\alpha_{in} L+1)^K\big{)}^{T}-1}{UL}\Big{(}\frac{\big{(}(\alpha_{in} L+1)^K-1\big{)}U_1}{m^{tr}}+\frac{U_2}{m^{val}}\Big{)}+\notag\\
&\frac{2M^2\big{(}(\alpha_{in} L+1)^K-1\big{)}}{m_{tr}L}.
\end{align}
For the second $|\cdot|$ term in Eq. \eqref{eq:append-u1u2}, we have the bound is
\begin{align}\label{equation-a11-p2-3}M^2\frac{\big{(}1+\alpha_{out} L(\alpha_{in} L+1)^K\big{)}^{T}-1}{UL}\Big{(}\frac{\big{(}(\alpha_{in} L+1)^K-1\big{)}U_1}{m^{tr}}+\frac{U_2}{m^{val}}\Big{)}.
\end{align}
Finally, $\forall z\in\mathscr{P}$, we have:
\begin{small}
\begin{align*}
		&\frac{1}{U}\Big{(}\sum_{i=1}^{U_1}|\hat{\mathcal{R}}^{val}(\mathcal{A}_{\text{hpo}}({\mathcal{D}_{u_i}^{tr}},\mathcal{S}_{(\mathcal{D},\{u_i\}_{i=1}^U)});z)-\hat{\mathcal{R}}^{val}(\mathcal{A}_{\text{hpo}}({\mathcal{D}_{u_i}^{tr}}',\mathcal{S}_{(\mathcal{D}',\{u_i\}_{i=1}^U)});z)|+\sum_{i=1}^{U_2}|\hat{\mathcal{R}}^{val}(\\
&\mathcal{A}_{\text{hpo}}({\mathcal{D}_{u_i}^{tr}},\mathcal{S}_{(\mathcal{D},\{u_i\}_{i=1}^U)});z)-\hat{\mathcal{R}}^{val}(\mathcal{A}_{\text{hpo}}({\mathcal{D}_{u_i}^{tr}},\mathcal{S}_{(\mathcal{D}',\{u_i\}_{i=1}^U)});z)|\Big{)}\leq\frac{M^2}{UL}\Big{(}U_1\\
&\frac{\big{(}{\big{(}1+\alpha_{out} L(\alpha_{in} L+1)^K\big{)}^{T}+1}\big{)}\big{(}(\alpha_{in} L+1)^K-1\big{)}}{m^{tr}}+U_2\frac{{\big{(}1+\alpha_{out} L(\alpha_{in} L+1)^K\big{)}^{T}-1}}{m^{val}}\Big{)}.
\end{align*}
\end{small}
Moreover, we assume that $U_1=\frac{Um^{tr}}{m^{tr}+m^{val}}$ and $U_2=\frac{Um^{val}}{m^{tr}+m^{val}}$, and then the proof is completed.\hfill $\square$

\noindent\textbf{Proof of Theorem \ref{theorem822-1}.} For multiple random seeds $\{u_i\}_{i=1}^U$, we can generate $U$ splittings $\{\mathcal{D}_{u_i}^{tr}, \mathcal{D}_{u_i}^{val}\}_{i=1}^U$. Suppose $\mathcal{D}, \mathcal{D}'\sim\mathscr{P}$ differ in at most one point. Then, there are two cases for these $U$ splittings: (1) the different data points $z$ and $z'$ exist in the training splitting part $\{\mathcal{D}_{u_i}^{tr}\}_{i=1}^{U_1}$; (2) the different data points $z$ and $z'$ exist in the validation splitting part $\{\mathcal{D}_{u_i}^{val}\}_{i=1}^{U_2}$, where $U_1+U_2=U$.

Let $\Psi(\mathcal{D}_{u_i}^{tr},{\mathcal{D}_{u_i}^{val}},\mathcal{S}_{(\mathcal{D},\{u_i\}_{i=1}^U)})=\mathcal{R}^{val}(\mathcal{A}_{\text{hpo}}(\mathcal{D}_{u_i}^{tr},\mathcal{S}_{(\mathcal{D},\{u_i\}_{i=1}^U)}))-\hat{\mathcal{R}}^{val}(\mathcal{A}_{\text{hpo}}(\mathcal{D}_{u_i}^{tr},\mathcal{S}_{(\mathcal{D},\{u_i\}_{i=1}^U)});\mathcal{D}_{u_i}^{val})$ for the $U_1$ splittings of case (1). Suppose $\mathcal{D}_{u_i}^{tr},{\mathcal{D}_{u_i}^{tr}}'\in\mathscr{P}$ differ in at most one point, and then
\begin{align*}
		&|\Psi(\mathcal{D}_{u_i}^{tr},{\mathcal{D}_{u_i}^{val}},\mathcal{S}_{(\mathcal{D},\{u_i\}_{i=1}^U)})-\Psi({\mathcal{D}_{u_i}^{tr}}',{\mathcal{D}_{u_i}^{val}},\mathcal{S}_{(\mathcal{D}',\{u_i\}_{i=1}^U)})|
		\leq|\mathcal{R}^{val}(\mathcal{A}_{\text{hpo}}(\mathcal{D}_{u_i}^{tr},
        \\
        &\mathcal{S}_{(\mathcal{D},\{u_i\}_{i=1}^U)}))-\mathcal{R}^{val}(\mathcal{A}_{\text{hpo}}({\mathcal{D}_{u_i}^{tr}}',\mathcal{S}_{(\mathcal{D}',\{u_i\}_{i=1}^U)}))|
		+|\hat{\mathcal{R}}^{val}(\mathcal{A}_{\text{hpo}}(\mathcal{D}_{u_i}^{tr},\mathcal{S}_{(\mathcal{D},\{u_i\}_{i=1}^U)});\mathcal{D}_{u_i}^{val})\\
        &-\hat{\mathcal{R}}^{val}(\mathcal{A}_{\text{hpo}}({\mathcal{D}_{u_i}^{tr}}',\mathcal{S}_{(\mathcal{D}',\{u_i\}_{i=1}^U)});{\mathcal{D}_{u_i}^{val}})|\leq 2\beta_1,
\end{align*}
where $\beta_1$ is the value of Eq. (\ref{equation-a11-p2-2}).

For the $U_2$ splittings of case (2), suppose $\mathcal{D}_{u_i}^{val},{\mathcal{D}_{u_i}^{val}}'\in\mathscr{P}$ differ in at most one point, and then
\begin{align*}
		&|\Psi(\mathcal{D}_{u_i}^{tr},{\mathcal{D}_{u_i}^{val}},\mathcal{S}_{(\mathcal{D},\{u_i\}_{i=1}^U)})-\Psi(\mathcal{D}_{u_i}^{tr},{\mathcal{D}_{u_i}^{val}}',\mathcal{S}_{(\mathcal{D}',\{u_i\}_{i=1}^U)})|
		\leq\beta_2
		+|\hat{\mathcal{R}}^{val}(\mathcal{A}_{\text{hpo}}(\mathcal{D}_{u_i}^{tr},\\
        &\mathcal{S}_{(\mathcal{D},\{u_i\}_{i=1}^U)});\mathcal{D}_{u_i}^{val})-\hat{\mathcal{R}}^{val}(\mathcal{A}_{\text{hpo}}(\mathcal{D}_{u_i}^{tr},\mathcal{S}_{(\mathcal{D}',\{u_i\}_{i=1}^U)});{\mathcal{D}_{u_i}^{val}}')|,
\end{align*}where $\beta_2$ is the value of Eq. (\ref{equation-a11-p2-3}). For the second term,
\begin{align*}
		&|\hat{\mathcal{R}}^{val}(\mathcal{A}_{\text{hpo}}(\mathcal{D}_{u_i}^{tr},\mathcal{S}_{(\mathcal{D},\{u_i\}_{i=1}^U)}));\mathcal{D}_{u_i}^{val})-\hat{\mathcal{R}}^{val}(\mathcal{A}_{\text{hpo}}(\mathcal{D}_{u_i}^{tr},\mathcal{S}_{(\mathcal{D}',\{u_i\}_{i=1}^U)});{\mathcal{D}_{u_i}^{val}}')|\leq\\
		&\frac{1}{m^{val}}\sum_{j=1}^{m^{val}}|\hat{\mathcal{R}}^{val}(\mathcal{A}_{\text{hpo}}(\mathcal{D}_{u_i}^{tr},\mathcal{S}_{(\mathcal{D},\{u_i\}_{i=1}^U)}));{z_{i,j}})-\hat{\mathcal{R}}^{val}(\mathcal{A}_{\text{hpo}}(\mathcal{D}_{u_i}^{tr},\mathcal{S}_{(\mathcal{D}',\{u_i\}_{i=1}^U)});{z_{i,j}}')|\\
		\leq&\frac{m^{val}-1}{m^{val}}\beta_2+\frac{S}{m^{val}}.
\end{align*}
As a result,
\begin{align*}
	|\Psi(\mathcal{D}_{u_i}^{tr},{\mathcal{D}_{u_i}^{val}},\mathcal{S}_{(\mathcal{D},\{u_i\}_{i=1}^U)})-\Psi(\mathcal{D}_{u_i}^{tr},{\mathcal{D}_{u_i}^{val}}',\mathcal{S}_{(\mathcal{D}',\{u_i\}_{i=1}^U)})|\leq\frac{S}{m^{val}}+2\beta_2.
\end{align*}
Let {\footnotesize $\Gamma(\mathcal{D},\{u_i\}_{i=1}^U)=\Big{(}\sum_{i=1}^{U_1}\big{(}\Psi(\mathcal{D}_{u_i}^{tr},{\mathcal{D}_{u_i}^{val}},\mathcal{S}_{(\mathcal{D},\{u_i\}_{i=1}^U)})+\sum_{i=1}^{U_2}\big{(}\Psi(\mathcal{D}_{u_i}^{tr},{\mathcal{D}_{u_i}^{val}},\mathcal{S}_{(\mathcal{D},\{u_i\}_{i=1}^U)})\Big{)}/U$}, Then we have
\begin{align*}
&\Gamma(\mathcal{D},\{u_i\}_{i=1}^U)-\Gamma(\mathcal{D}',\{u_i\}_{i=1}^U)=\frac{1}{U}\Big{(}\sum_{i=1}^{U_1}\big{(}\Psi(\mathcal{D}_{u_i}^{tr},{\mathcal{D}_{u_i}^{val}},\mathcal{S}_{(\mathcal{D},\{u_i\}_{i=1}^U)})-\Psi({\mathcal{D}_{u_i}^{tr}}',{\mathcal{D}_{u_i}^{val}},\\
&\mathcal{S}_{(\mathcal{D}',\{u_i\}_{i=1}^U)})\big{)}+\sum_{i=1}^{U_2}\big{(}\Psi(\mathcal{D}_{u_i}^{tr},{\mathcal{D}_{u_i}^{val}},\mathcal{S}_{(\mathcal{D},\{u_i\}_{i=1}^U)})-\Psi(\mathcal{D}_{u_i}^{tr},{\mathcal{D}_{u_i}^{val}}',\mathcal{S}_{(\mathcal{D}',\{u_i\}_{i=1}^U)})\big{)}\Big{)}\leq\\
&2\beta+\frac{SU_2}{Um^{val}},
\end{align*}
where $\beta=\frac{U_1\beta_1+U_2\beta_2}{U}$. According to McDiarmid’s inequality, we have that for all $\epsilon\in\mathbb{R}^{+}$, $P_{\mathcal{D},\{u_i\}_{i=1}^U}(\Gamma(\mathcal{D},\{u_i\}_{i=1}^U)-\mathbb{E}_{\mathcal{D},\{u_i\}_{i=1}^U}[\Gamma(\mathcal{D},\{u_i\}_{i=1}^U)]\geq\epsilon)
\leq\text{exp}(-2\frac{(Um^{val}\epsilon)^2}{(SU_2+2Um^{val}\beta)^2(m^{tr}+m^{val})})$. Besides, we have
\begin{align*}
		&\mathbb{E}_{\mathcal{D},\{u_i\}_{i=1}^U}[\Gamma(\mathcal{D},\{u_i\}_{i=1}^U)]=\mathbb{E}_{\mathcal{D},\{u_i\}_{i=1}^U}\Big{[}\frac{1}{U}\Big{(}U_1\big{(}\Psi(\mathcal{D}_{u_i}^{tr},{\mathcal{D}_{u_i}^{val}},\mathcal{S}_{(\mathcal{D},\{u_i\}_{i=1}^U)})+U_2\big{(}\Psi(\mathcal{D}_{u_i}^{tr},\\
        &{\mathcal{D}_{u_i}^{val}},\mathcal{S}_{(\mathcal{D},\{u_i\}_{i=1}^U)})\Big{)}\Big{]}=\frac{U_1}{U}\mathbb{E}_{\mathcal{D},\{u_i\}_{i=1}^U}[\big{(}\Psi(\mathcal{D}_{u_i}^{tr},{\mathcal{D}_{u_i}^{val}},\mathcal{S}_{(\mathcal{D},\{u_i\}_{i=1}^U)})]+\frac{U_2}{U}\mathbb{E}_{\mathcal{D},\{u_i\}_{i=1}^U}[\big{(}\Psi(\\
        &\mathcal{D}_{u_i}^{tr},{\mathcal{D}_{u_i}^{val}},\mathcal{S}_{(\mathcal{D},\{u_i\}_{i=1}^U)})]=\frac{U_1}{U}\mathbb{E}_{z,z_1}[\hat{\mathcal{R}}^{val}(\mathcal{A}_{\text{hpo}}(\mathcal{D}_{u_i}^{tr}\backslash z_1\cup z,(\mathcal{D}_{u_i}^{tr}\backslash z_1\cup z,\mathcal{D}_{u_i}^{val})_{i=1}^U);z_1)-\\
        &\hat{\mathcal{R}}^{val}(\mathcal{A}_{\text{hpo}}(\mathcal{D}_{u_i}^{tr},(\mathcal{D}_{u_i}^{tr},\mathcal{D}_{u_i}^{val})_{i=1}^U);z_1)]
	+\frac{U_2}{U}\mathbb{E}_{z,z_1}[\hat{\mathcal{R}}^{val}(\mathcal{A}_{\text{hpo}}(\mathcal{D}_{u_i}^{tr},(\mathcal{D}_{u_i}^{tr},\mathcal{D}_{u_i}^{val}\backslash z_1\cup z)_{i=1}^U);\\
    &z_1)-\hat{\mathcal{R}}^{val}(\mathcal{A}_{\text{hpo}}(\mathcal{D}_{u_i}^{tr},(\mathcal{D}_{u_i}^{tr},\mathcal{D}_{u_i}^{val})_{i=1}^U);z_1)]\leq\beta.
\end{align*}
Thereby, we have that for all $\epsilon\in\mathbb{R}^{+}$, $P_{\mathcal{D},\{u_i\}_{i=1}^U}(\Gamma(\mathcal{D},\{u_i\}_{i=1}^U)-\beta\geq\epsilon)
		\leq\text{exp}(-2\frac{(Um^{val}\epsilon)^2}{(SU_2+2Um^{val}\beta)^2(m^{tr}+m^{val})})$. Equivalently, we have that $\forall \delta\in(0,1)$,
\begin{align*}
P_{\mathcal{D},\{u_i\}_{i=1}^U}\Big{(}\Gamma(\mathcal{D},\{u_i\}_{i=1}^U)\leq\beta+\big{(}\frac{SU_2}{Um^{val}}+2\beta\big{)}\sqrt{\frac{\ln\delta^{-1}(m^{tr}+m^{val})}{2}}\Big{)}\geq1-\delta.
\end{align*}
Moreover, we assume that $U_1=\frac{Um^{tr}}{m^{tr}+m^{val}}$ and $U_2=\frac{Um^{val}}{m^{tr}+m^{val}}$, and then we have:
\begin{align*}
			P_{\mathcal{D},\{u_i\}_{i=1}^U}\Big{(}\Gamma(\mathcal{D},\{u_i\}_{i=1}^U)\leq\beta+\big{(}\frac{S}{m^{tr}+m^{val}}+2\beta\big{)}\sqrt{\frac{\ln\delta^{-1}(m^{tr}+m^{val})}{2}}\Big{)}\geq1-\delta,
\end{align*}
where $\beta=\frac{M^2}{L(m^{tr}+m^{val})}\Big{(}{\big{(}{\big{(}1+\alpha_{out} L(\alpha_{in} L+1)^K\big{)}^{T}+1}\big{)}\big{(}\alpha_{in} L+1\big{)}^K}\Big{)}$.
Then, the proof is completed.\hfill $\square$
\subsection{Analysis of TE1 in Eq. (\ref{equation421-2})}\label{sectionA14}
In this section, we attempt to elucidate two results regarding TE1:
\begin{itemize}
	\item[(1)] When $T$ is small, TE1 may decrease with the increase of $K$.
	\item[(2)] Under the learning rate setting in Proposition \ref{proposition821-1}, i.e., Set $\alpha_{in}\leq\frac{2}{L}$ and $\alpha_{out}=\frac{\ln q^{-1}}{L\ln3}$, and then we can ensure that TE1 increases with the increase of the number of inner-level iteration $K$.
\end{itemize}
Firstly, we rewrite TE1 in Eq. (\ref{equation421-2}) as a function of $K$. For simplicity, we have omitted the data items $h(K):=\sqrt{\text{TE1}(K)}=\big{(}(AB^K)^T-1\big{)}\big{(}Cq^K+DKq^{K-1}\big{)}$, where $A=\alpha_{out}L$, $B=\alpha_{in}L+1$, $C=(1+\frac{C_{2,\boldsymbol{\lambda},\mathcal{D}^{tr}_{u_i}}}{1-q})2LC_{1,\boldsymbol{\lambda},\mathcal{D}^{tr}_{u_i}}+\frac{MC_{2,\boldsymbol{\lambda},\mathcal{D}^{tr}_{u_i}}}{1-q}$ and $D=M(1+\frac{C_{2,\boldsymbol{\lambda},\mathcal{D}^{tr}_{u_i}}}{1-q})2LC_{1,\boldsymbol{\lambda},\mathcal{D}^{tr}_{u_i}}$. $C_{1,\boldsymbol{\lambda},\mathcal{D}^{tr}_{u_i}}$ and $C_{2,\boldsymbol{\lambda},\mathcal{D}^{tr}_{u_i}}$ are constants introduced in Lemma \ref{lemma33-1}, and $q\in(0, 1)$ is a constant introduced in Lemma \ref{lemma-a2-5}.

Then, the derivative of $h(K)$ takes the following form $h'(K)=q^{K-1}{(}T(AB^K+1)^{T-1}AB^K\ln B(Cq+DK)+((AB^K+1)^T-1)(Cq\ln q+D+DK\ln q){)}$. For the first result, we want to prove $h'(K)\leq0$ holds, and then we have:
\begin{align}\label{equationA14-2}
T\ln B-T\ln B\frac{(AB^K+1)^{T-1}-1}{(AB^K+1)^T-1}\leq-(\ln q+\frac{D}{Cq+DK}).
\end{align}
When $K\to\infty$ in Eq. (\ref{equationA14-2}), we have $T\ln B\leq-\ln q$.  
In other words, when $T\leq\frac{-\ln q}{\ln B}$, we observe a decrease in TE1 with the increase of $K$, where $q$ is given by Lemma \ref{lemma-a2-5}. If we set $\alpha_{in}=2/(L+\mu)$, the above results hold when $T\leq\frac{\ln(L+\mu)-\ln(L-\mu)}{\ln(3L+\mu)-\ln(L+\mu)}$. It is worth noting that the specific value of $T$ depends on $L$ and $\mu$, so when $T$ is relatively large, the reduction of TE1 cannot be guaranteed.

For the first result, we want to prove $h'(K)\geq0$ holds, and then we have:
\begin{align}\label{equationA14-3}
		&-\frac{T(AB^K+1)^{T-1}AB^K\ln B}{(AB^K+1)^T-1}\leq\ln q+\frac{D}{Cq+DK}.
\end{align}
When $K\to\infty$, it is easy to get that the second term on the right-hand side of the inequality tends to be 0, which is a necessary condition for the validity of the above inequality. If we want to let Eq. (\ref{equationA14-3}) hold, we can use the following equivalence relation:
\begin{align*}
		-\frac{T(AB^K+1)^{T-1}AB^K\ln B}{(AB^K+1)^T}\leq\ln q\iff-\frac{TAB^K\ln B}{(AB^K+1)}\leq\ln q.
\end{align*}
To this aim, we can make use of the following equivalence relation:
\begin{align}\label{equationA14-4v10}
		-\frac{TAB^K\ln B}{(AB^K+1)}\leq\ln q\iff-\frac{T\ln B}{(1+\frac{1}{AB^K})}\leq\ln q.
\end{align}
Next, we can use the result of $-\frac{\ln B}{(1+\frac{1}{AB})}\leq\ln q$ based on Eq. (\ref{equationA14-4v10}). Then, by setting $\alpha_{out}=\frac{\ln q^{-1}}{L\ln3}$, the above inequality can then be satisfied, where we use $q\geq\alpha_{in}L-1$ and $\alpha_{in}\leq\frac{2}{L}$.
\section{Additional Discussions}\label{sectionB-1}
\subsection{Discussions of Section \ref{section21}}\label{sectionA11}
We presented the objectives of HPO and its practical implications in Section \ref{section21}. In this section, we will illustrate that within this framework, not only can we analyze gradient-based HPO methods as discussed in the main paper, but also traditional methods such as grid search combined with cross-validation. For $U$-fold cross-validation, which involves generating $U$ sets of splitting using without-replacement sampling and approximating the expectation in Eq. (\ref{equation21-4}) using the average of these $U$ sets. For leave-one-out cross-validation, where $\gamma=1/N$ and $U=N$, a similar approach to $U$-fold cross-validation is adopted. The analysis of grid search assumes that the hypothesis space for $\boldsymbol{\lambda}$ consists of a finite number, rather than a continuous form resembling gradient-based methods. From this perspective, the EHG and OEHG proposed in this paper are natural extension versions of $U$-fold cross-validation under the gradient-based HPO method.
\subsection{Discussions of EHG}\label{sectionA10-v3}
\subsubsection{The Overview of EHG Strategy}
As shown in {Section \ref{sec4-1}}, $\widehat{\nabla}{f}(\boldsymbol{\lambda};\mathcal{S}_{(\mathcal{D},\{u_i\}_{i=1}^U)})$ is computed over $U$ data splittings because of computational constraints, where $U<V$ and $V=\binom{N}{N\cdot\gamma}$ denotes the size of all the different splittings. Moreover, $\mathcal{S}_{(\mathcal{D},\{u_i\}_{i=1}^U)}$ is a multiset, as different random seeds may correspond to the same data splitting, meaning that $\mathcal{S}_{(\cdot)}$ may not be an injective function.

The issue with computing hypergradient using $\mathcal{S}_{(\mathcal{D},\{u_i\}_{i=1}^U)}$ is that it requires calculating hypergradient on $U$ different data splittings. However, for the same data splitting generated by different random seeds $\{u_i\}_{i=1}^U$, this approach futilely increases computational cost. Therefore, we aim to use different $U'$ data splittings $\mathcal{S}_{(\mathcal{D},\{{u}'_i\}_{i=1}^{U'})}$ for computing the average hypergradient, and the set $\mathcal{S}_{(\mathcal{D},\{{u}'_i\}_{i=1}^{U'})}$ satisfies: 1. $\mathcal{S}_{(\mathcal{D},\{{u}'_i\}_{i=1}^{U'})}\subseteq \mathcal{S}_{(\mathcal{D},\{u_i\}_{i=1}^U)}$. 2. $U'\leq V$. 3. $\forall i,j\in {U}'$, if $i\neq j$, then $\mathcal{S}_{(\mathcal{D},u_i)}\neq \mathcal{S}_{(\mathcal{D},u_j)}$. In practice, the generation processes of $\mathcal{S}_{(\mathcal{D},\{u_i\}_{i=1}^U)}$ and 
$\mathcal{S}_{(\mathcal{D},\{{u}'_i\}_{i=1}^{U'})}$ can be considered as sampling with/without replacement, respectively.

As depicted in Table \ref{table51-1}, most current gradient-based HPO methods employ a single data splitting, denoted as $\mathcal{S}_{(\mathcal{D},u_1)}$, to estimate the expectation $\overline{\nabla}f(\boldsymbol{\lambda})$. In contrast, our approach involves the use of multiple data splittings ($U>1$ or $U'>1$) to estimate the expectation.
\begin{table*}[!t]
\caption{Comparison of hypergradient values and optimization objectives.}
\centering
\resizebox{1.\textwidth}{!}{
\begin{tabular}{llll}
\toprule
Hypergradient & gradient-based HPO & w/ replacement & w/o replacement\\
\midrule
HPO Algorithm Estimation  & $\widehat{\nabla}{f}(\boldsymbol{\lambda};\mathcal{S}_{(\mathcal{D}, u_1)})$ & $\widehat{\nabla}{f}(\boldsymbol{\lambda};\mathcal{S}_{(\mathcal{D}, \{u_i\}_{i=1}^U)})$& $\widehat{\nabla}{f}(\boldsymbol{\lambda};\mathcal{S}_{(\mathcal{D}, \{u'_i\}_{i=1}^{U'})})$\\
\midrule
Expectation Estimation & ${\nabla}{f}(\boldsymbol{\lambda};\mathcal{S}_{(\mathcal{D}, u_1)})$ & ${\nabla}{f}(\boldsymbol{\lambda};\mathcal{S}_{(\mathcal{D}, \{u_i\}_{i=1}^U)})$& ${\nabla}{f}(\boldsymbol{\lambda};\mathcal{S}_{(\mathcal{D}, \{u'_i\}_{i=1}^{U'})})$\\
\midrule
Expectation (Ground truth)& \multicolumn{3}{c}{$\overline{\nabla}f(\boldsymbol{\lambda}):=\mathbb{E}_{\mathcal{D}, u_i}\nabla f(\boldsymbol{\lambda};\mathcal{S}_{(\mathcal{D}, u_i)})$}\\
\bottomrule
\end{tabular}}
\label{table51-1}
\end{table*}

\subsubsection{The Analysis of EHG Strategy}\label{section52}
To provide a clearer explanation of this instantiation method, some analyses will be presented in this section. It is worth noting that these analyses are obtained after observing the data $\mathcal{D}$. In other words, we are considering the analysis of hypergradient error for the population represented by these $N$ data points. The process of generating data splittings can be viewed as a sampling process.

The splitting operation on the observed data $\mathcal{D}$ can be regarded as a sampling process. Specifically, for $N$ observed data points in $\mathcal{D}$ (the population), we perform sampling of $m^{val}$ validation data points, using the remaining $m^{tr}$ data points as training data. Due to the variability in specific samples included in each sampling, the calculated hypergradient metrics are also different. Therefore, these sampling errors are random variables. We consider using the square of the sampling mean error to measure the sampling error. We introduce the following symbolic notation as follows: $\bar{x}=\frac{1}{U}\sum_{i=1}^U\widehat{\nabla}{f}(\boldsymbol{\lambda};(\mathcal{D}^{tr}_{u_i},\mathcal{D}^{val}_{u_i})),\mathbb{E}(x)=\bar{X}=\frac{1}{V}\sum_{i=1}^V\widehat{\nabla}{f}(\boldsymbol{\lambda};(\mathcal{D}^{tr}_{u_i},\mathcal{D}^{val}_{u_i})),\sigma^2=\frac{\sum(X-\bar{X})^2}{V}.$

We analyze the calculation of the sampling mean error, denoted as $\mu_{\bar{x}}$, as follows:
\begin{align}\label{equation52-2}
		&\mu_{\bar{x}}^2=\mathbb{E}[\bar{x}-\mathbb{E}(x)]^2=\mathbb{E}[\bar{x}-\bar{X}]^2=\mathbb{E}\Big{[}\frac{x_1+x_2+\cdots+x_U}{U}-\frac{\bar{X}+\bar{X}+\cdots+\bar{X}}{U}\Big{]}^2\notag\\
		=&\frac{1}{U^2}\mathbb{E}\Big{[}(x_1-\bar{X})+(x_2-\bar{X})+\cdots+(x_U-\bar{X})\Big{]}^2.
\end{align}
Then, we should consider the cases of sampling with/without replacement as follows:
\begin{align}\label{equation52-3}
	\mu_{\bar{x}}^2= 
	\left\{
	\begin{array}{cc}
		\frac{\sigma^2}{U}& \text{sampling with replacement,}\\
		\\
		\frac{(V-U)\sigma^2}{U(V-1)}& \text{sampling without replacement.}\\
	\end{array}
	\right.
\end{align}
\noindent\textbf{Discussion:} Firstly, it is important to note that the discussion in this section is applicable to both AID and ITD. This is because, for different gradient-based HPO methods, the only difference lies in the way that the hypergradient is computed. The key focus of this section is to analyze the variance impact of data on hypergradient calculation, and these two directions are orthogonal. Furthermore, from Eq. (\ref{equation52-3}), it can be observed that without replacement sampling can yield a smaller sampling mean error when $U$ is the same. Therefore, in subsequent experiments, we adopted this approach. Intuitively, this method is similar to the data splitting used in k-fold cross-validation. Finally, in comparison with the original ITD or AID ($U=1$), we can observe that increasing the number of samplings $U$ can result in a smaller sampling mean error, thereby obtaining a more accurate hypergradient.
\section{Implementation Details}
\begin{figure*}[!t]
\centering
\subfigure{\includegraphics[width=0.9\linewidth]{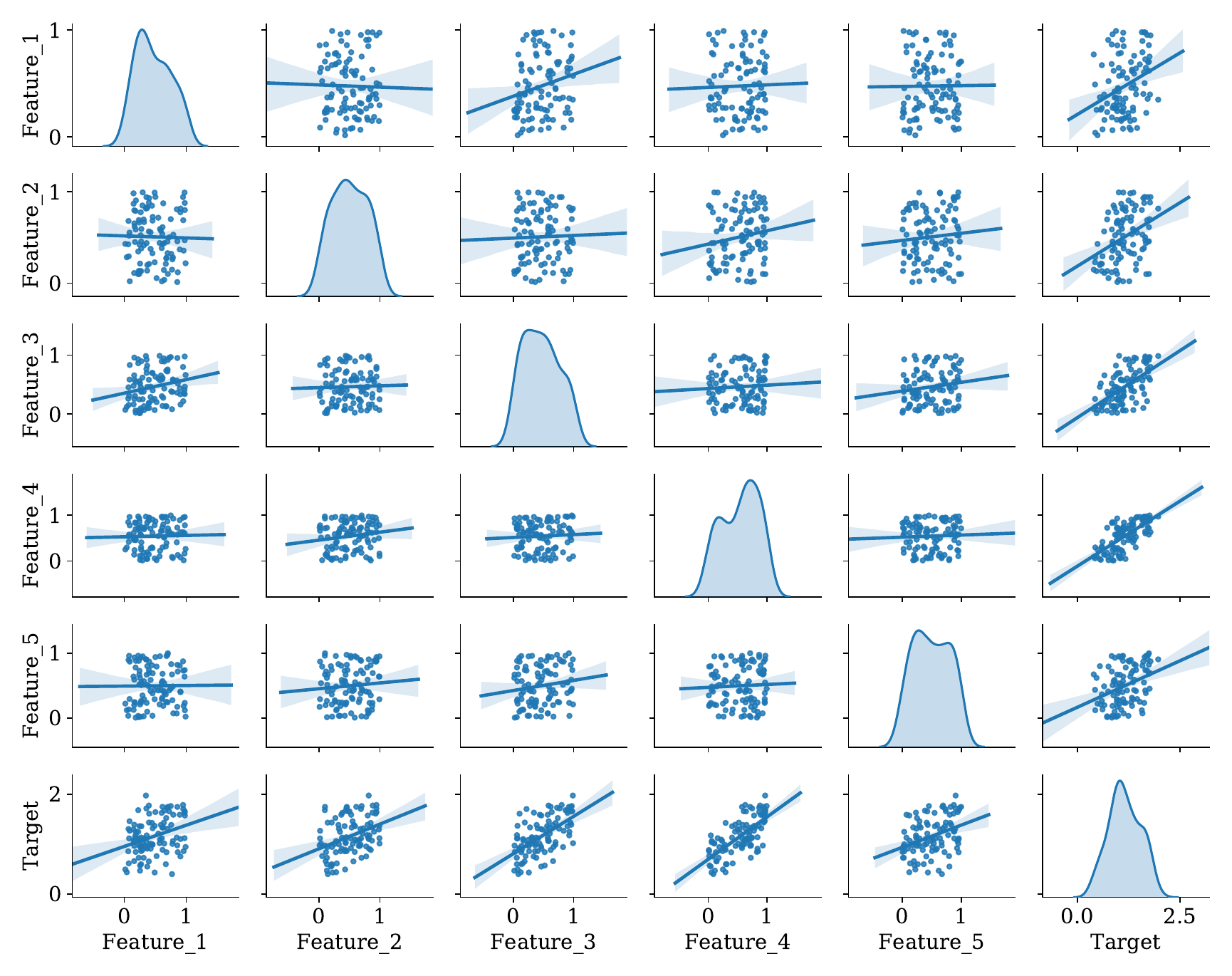}
}
\caption{Visualization of the features and targets of generated data and their correlations.}
\label{figA19-1}
\end{figure*}
\noindent\textbf{Experimental setup of Fig. \ref{figure31-1}.}\label{sectionA16}
Firstly, we generated 100 5-dimensional data points, $Y_i = \beta X_i + \epsilon_i$, where $i = 1, 2, \dots, 100$, with noise $\epsilon_i \sim\mathscr{N}(0, 0.1)$ and data generation parameter $\beta \sim\mathscr{N}(0, 1)$. Fig. \ref{figA19-1} visualizes the features and targets of the generated data and their correlations.

For the training phase of the elastic net, we used RHG as HPO method, which is an ITD. The outer loop consists of 5 iterations, utilizing Adam as the optimizer with a learning rate of 0.01. The inner loop comprises 100 iterations, employing SGD as the optimizer with a learning rate of 0.1. For the RHG+EHG method, we set the number of sampled random seeds to 10.

\noindent\textbf{Experimental setup of Section \ref{section5-3}.}\label{sectionA15-2}
We conduct numerical experiments to illustrate the above conclusion. We sample $n=100$ data points $(\mathcal{X},\mathcal{Y})$ and generate data splittings in proportion to $\gamma=0.01$, where $\mathcal{X}\in\mathbb{R}^5$ and $\mathcal{Y}\in\mathbb{R}^1$.

\noindent\textbf{Supplementary content of Section \ref{section5-2}.}\label{sectionA15-1}
For ridge regression, the model parameter have a closed-form solution. Therefore, we can explicitly provide the specific form of bias-variance decomposition (Eq. (\ref{equation32-1})):
\begin{align}\label{equationA15-1}
	\left\{
	\begin{aligned}
		\begin{split}
			&\nabla f({\lambda};(\mathcal{D}^{tr}, \mathcal{D}^{val})) = \nabla_{\lambda}[\big{(}(\mathcal{X})^{\top}\mathcal{X}/n+\lambda I\big{)}^{-1}(\mathcal{X})^{\top}\mathcal{Y}/n]
			\\
			&\widehat{\nabla}{f}(\lambda;(\mathcal{D}^{tr},\mathcal{D}^{val})) =\nabla_{{\lambda}}\Big{(}\Vert\mathcal{X}^{val}\boldsymbol{\theta}_K(\lambda)-\mathcal{Y}^{val}\Vert^2\Big{)}-\alpha_{in}\sum_{k=0}^{K-1}\nabla^2_{\lambda,\boldsymbol{\theta}}\Big{(}\Vert\mathcal{X}^{tr}\boldsymbol{\theta}_K-\mathcal{Y}^{tr}\Vert^2+
			\\
			&\lambda\Vert\boldsymbol{\theta}_K\Vert_2^2\Big{)}\prod_{j=k+1}^{K-1}\Big{(}I-\alpha_{in}\nabla_{\boldsymbol{\theta}}^2\big{(}\Vert\mathcal{X}^{tr}\boldsymbol{\theta}_K-\mathcal{Y}^{tr}\Vert^2+\lambda\Vert\boldsymbol{\theta}_K\Vert^2_2\Big{)}\nabla_{{\boldsymbol{\theta}}}\Big{(}\Vert\mathcal{X}^{val}\boldsymbol{\theta}_K-\mathcal{Y}^{val}\Vert^2\Big{)}\\
			&\overline{\nabla}f({\lambda})=\mathbb{E}_{\mathcal{D}, u_1}[\nabla f({\lambda};\mathcal{S}_{(\mathcal{D}, u_1)})]\\
			&\widetilde{\nabla}{f}(\lambda)=\mathbb{E}_{\mathcal{D}, u_1}[\widehat{\nabla}f(\lambda;\mathcal{S}_{(\mathcal{D}, u_1)})]
		\end{split}
	\end{aligned}
	\right.
\end{align}
where $\boldsymbol{\theta}_K=\boldsymbol{\theta}_0-\alpha_{in}\Sigma_{k=0}^{K-1}\nabla_{{\boldsymbol{\theta}}}(\Vert \mathcal{X}^{tr}\boldsymbol{\theta}-\mathcal{Y}^{tr}\Vert^2+\lambda\Vert\boldsymbol{\theta}\Vert_2^2)\Big{|}_{\boldsymbol{\theta}=\boldsymbol{\theta}_k}$.

\noindent\textbf{Experimental setup of Section \ref{section5-2}.} We sample the regression coefficient $\theta$ from a standard normal distribution $\mathscr{N}(0,1)$. Subsequently, we generate 10 data points, where the noise term $\sigma$ also follows a standard normal distribution $\mathscr{N}(0,1)$. For $\widehat{\nabla}{f}(\lambda;(\mathcal{D}^{tr},\mathcal{D}^{val}))$, we employ the RHG algorithm, which belongs to ITD. We selected 50 $\lambda$ values ranging from 0.3 to 3 to compute the hypergradient statistics according to Eq. (\ref{equationA15-1}).

\subsection{Supplementary Experiments}\label{sectionA15-3}
\begin{figure*}[!t]
\centering
\subfigure{\includegraphics[width=0.33\linewidth]{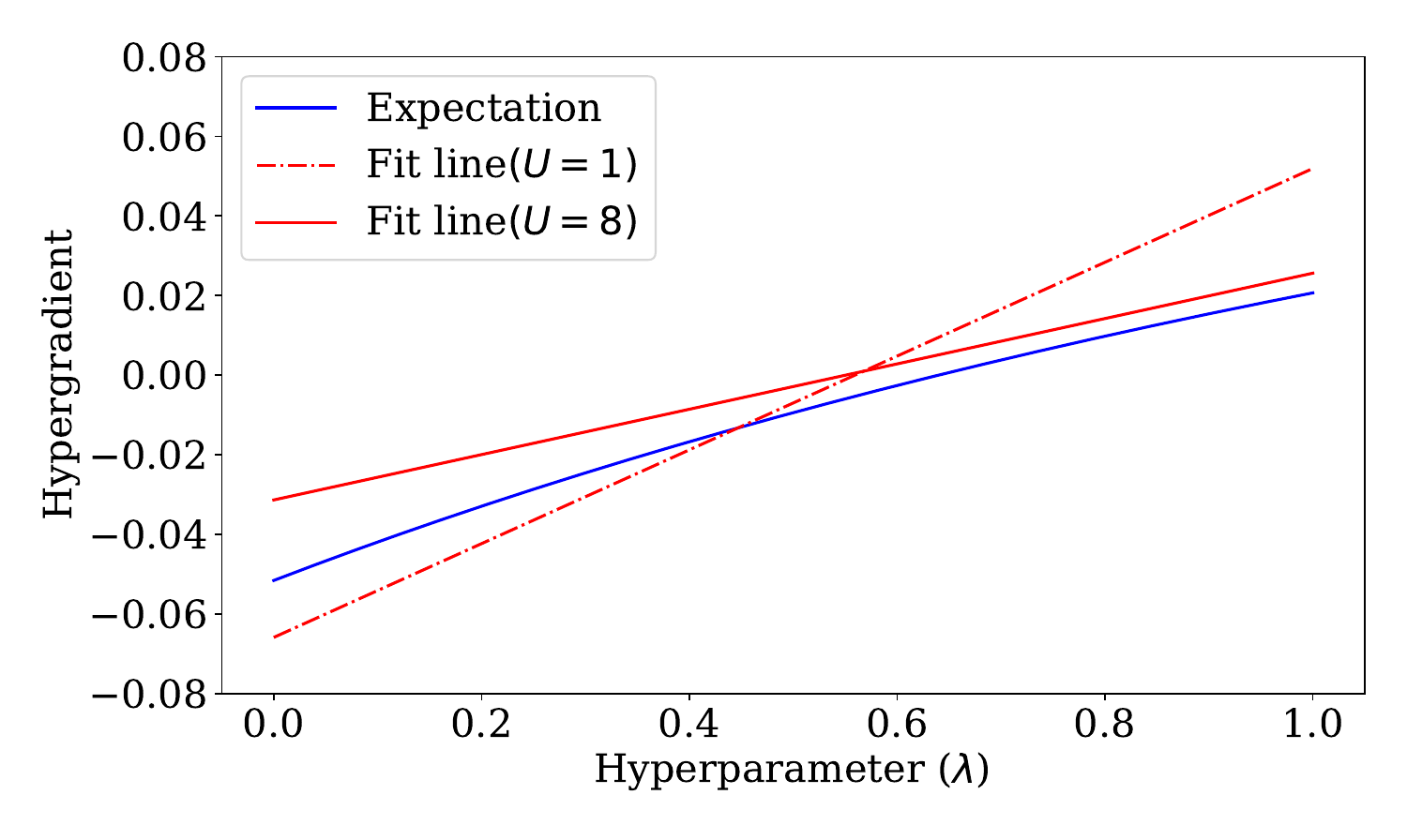}
}
\hspace{-0.4cm}
\subfigure{\includegraphics[width=0.33\linewidth]{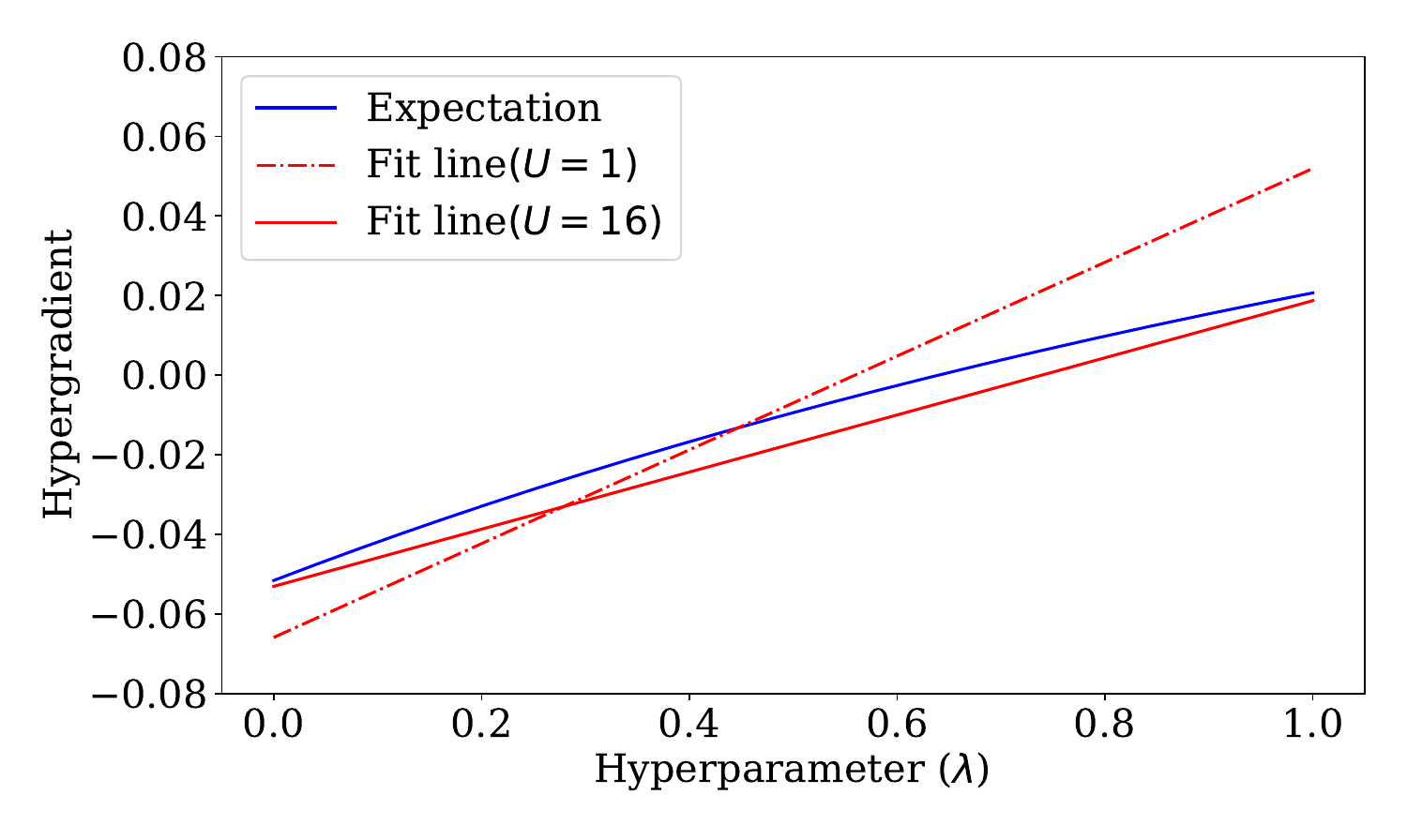}
}
\hspace{-0.4cm}
\subfigure{\includegraphics[width=0.33\linewidth]{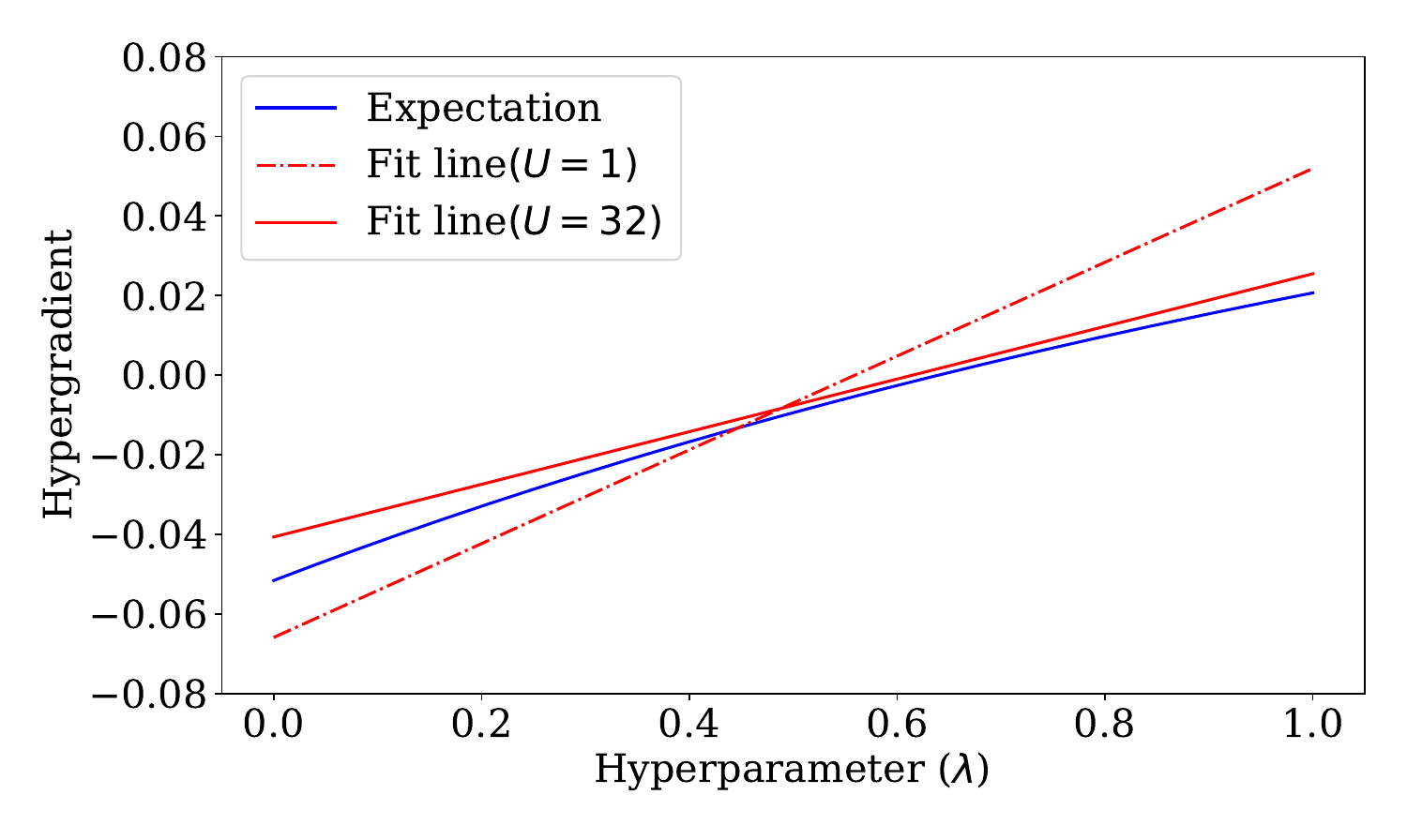}
}
\caption{Visualization of hypergradient values for different methods.  The line is computed by fitting the generated scatter points with a linear model. The inner sub-problem is solved using the closed-form solution of ridge regression.}
\label{figure5-1}
\end{figure*}
\textbf{Experimental setup.} In this case, since the bias and variance of the hypergradient can be explicitly expressed, under the instantiation method outlined in Section \ref{section4}, we can calculate both the bias and variance. Firstly, we sample some exmaples from data distribution as $\mathcal{D}$. For that, We conducted numerical experiments to illustrate the above conclusion. We sample $n=22$ data points $(\mathcal{X},\mathcal{Y})$ and generate data splittings in proportion to $\gamma=1/11$, where $\mathcal{X}\in\mathbb{R}^{5}$ and $\mathcal{Y}\in\mathbb{R}^1$. Then, we have $\overline{\nabla}f({\lambda}) = \mathbb{E}_{u_1}[\nabla f({\lambda};\mathcal{S}_{(\mathcal{D}, u_1)})]$,  $\widetilde{\nabla}{f}(\lambda)=\mathbb{E}_{u_1}[\widehat{\nabla}f(\lambda;\mathcal{S}_{(\mathcal{D}, u_1)})]$.

\noindent\textbf{Results.} In Fig. \ref{figure5-1}, we present the hypergradient by multiple splittings ($U=1, U=8, U=16, U=32$) and the hypergradient expectation. Fig. \ref{figure5-1} shows that as $U$ increases, the computed hypergradient becomes gradually more tightly clustered around the expectation, indicating a reduction in variance. This observation is consistent with the analysis in Theorems \ref{theorem332-1} and \ref{theorem334-1}. It is evident that larger values of $U$ yield more accurate hypergradient compared to $U=1$, thereby enhancing the HPO process.

\end{document}